\pdfoutput=1
\documentclass[twoside,11pt]{article}

\usepackage{amsmath}
\usepackage{amssymb}
\usepackage{amsthm}
\usepackage{mathtools}
\usepackage{algorithm}
\usepackage{algorithmic}
\usepackage{natbib}
\usepackage{tikz}
\usepackage{pgfplots}
\usepackage{booktabs}
\usepackage{tabularx}
\usepackage{longtable}

\usepackage[preprint]{jmlr2e}

\usepackage{cleveref}
\hypersetup{hypertexnames=false}

\usepackage{lastpage}
\jmlrheading{0}{2026}{1-\pageref{LastPage}}{--}{--}{--}{Huang}
\ShortHeadings{Self-Certification of Representation Adequacy}{Huang}
\firstpageno{1}

\newtheorem{assumption}[theorem]{Assumption}
\newtheorem*{definition*}{Definition}
\newtheorem*{assumption*}{Assumption}
\newtheorem*{corollary*}{Corollary}
\newtheorem*{remark*}{Remark}
\newtheorem*{example*}{Example}

\renewenvironment{proof}[1][Proof]{%
  \par\pushQED{\qed}\normalfont\medskip
  \noindent{\bfseries #1}\quad\ignorespaces
}{%
  \popQED\par\medskip
}

\makeatletter
\newcommand{\setmodelanchor}[2]{\def\@currentlabel{#2}\phantomsection\label{#1}}
\makeatother

\allowdisplaybreaks[1]
\begin{document}

\title{Self-Certification of Representation Adequacy:\\
       Sequential Certification at Minimum Task Loss}

\author{\name Zijie Huang \email 750691178@qq.com \\
       \addr Independent Researcher}

\editor{--}

\maketitle

\begin{abstract}%   <- trailing '%' for backward compatibility of .sty file
We study the task loss required to certify whether a fixed representation
preserves the Bayes risk of the full task history. In a finite specified
environment class, actions generate observations and incur unobserved,
environment-dependent task losses. Adapting controlled sequential testing,
we obtain a covering-LP lower bound and a Certification Track-and-Stop
policy with matching first-order loss $C_i\log(1/\delta)$ under
identifiability and full support. We explicitly separate this general
finite-label testing analysis from its representation interpretation.
We then ask when the stronger requirement of identifying the environment
can be weakened to identifying its adequacy label. For strictly positive
losses, simultaneous attainment of all environment-wise constants is
possible exactly when observationally equivalent environments have a
common optimal covering design. Otherwise a joint lower bound quantifies
the obstruction; a task-derived example gives an exact factor-two loss
penalty despite label identifiability. For zero-cost evidence, the
pointwise infimum is exactly zero at each fixed confidence level;
we characterize when it is attained. Static Bayes-risk and
value-of-information calculations provide the interpretation and
one-shot boundaries. Policy switching and representation repair remain
outside the fixed-kernel guarantees.
\end{abstract}


% Introduction: contribution scope revised after the September 2026 audit.
\providecommand{\cE}{\mathcal{E}}
\providecommand{\cA}{\mathcal{A}}
\providecommand{\cY}{\mathcal{Y}}
\providecommand{\cF}{\mathcal{F}}
\providecommand{\E}{\mathbb{E}}
\providecommand{\Prob}{\mathbb{P}}
\providecommand{\ind}{\mathbf{1}}
\providecommand{\kl}{\mathrm{kl}}
\providecommand{\Alt}{\mathrm{Alt}}
\providecommand{\LLR}{\mathrm{LLR}}
\providecommand{\TV}{\mathrm{TV}}

\section{Introduction}\label{sec:intro}

\subsection{The question and its scope}\label{sec:intro-motivation}

A compressed history can discard distinctions needed for optimal action.
Perceptual aliasing, utility-based state refinement, and policy-preserving
abstraction have studied this problem for decades
\citep{chrisman1992aliasing,mccallum1994state,jong2005abstraction}.
Decision-theoretic sufficiency likewise predates representation learning:
risk preservation belongs to the statistical decision tradition
\citep{blackwell1951,lecam1964,takeuchi1975adequacy}, and loss-dependent
Bayes-sufficient representations have been formulated in supervised
learning \citep{sevetlidis2026bayes}. We use this established perspective to
define the adequacy of a fixed representation $f$ in environment $E$:
\[
  \Theta(E,f)=\ind\{R_E^*(f(H))=R_E^*(H)\}.
\]
The question is the cost of certifying this property when the environment is
unknown and the actions used to gather evidence also incur task loss.

Statistical assessment of an abstraction or internal model is also established
\citep{jiang2015abstraction,albrecht2015criticising,shi2020markov,lin2023transfer}.
Nor is charging experimentation in regret units new: fixed-confidence
identification with low cumulative regret is studied by
\citet{degenne2019bridging} and \citet{yang2026minimal}.
Our purpose is therefore not to introduce representation inadequacy,
risk-based adequacy, statistical self-assessment, or regret-priced
identification as separate ideas.

We study a finite, specified environment class. The representation and all
candidate observation channels, action losses, and adequacy labels are known;
the true environment index is not. At each round the certifier chooses an
action $A_t$, observes $Y_t\sim K_i^{A_t}$, and incurs loss $g_i(A_t)$ in
the true environment $E_i$. It must stop almost surely and report $\Theta_i$
with error at most $\delta$ under every candidate environment, while
minimizing
\[
  R(\pi,E_i)=\E_i^\pi\sum_{t\le\tau}g_i(A_t).
\]
The specified observation $Y_t$ is the feedback available to the certifier;
the loss is not an extra observation. Task loss can be an expected raw loss
or a specified excess loss, but the convention must be fixed as part of the
model. Subtracting a baseline changes an objective with a variable stopping
time and is not an innocuous normalization.

The label concerns the representation used for the underlying task. The
certifier can retain the action--observation transcript needed for likelihood
calculations; it is not assumed to compress that testing transcript by the
same $f$. Actions can include task execution, clarification, verification,
and deferral. Changing the representation itself is outside the fixed-label
model. These restrictions make M1 a model-based certification problem, not
a distribution-free test of a learned representation.

\subsection{Results and their status}\label{sec:intro-layers}
\label{sec:intro-preview}

\paragraph{A. Certification complexity.}
The finite controlled-testing characterization uses task-loss units. For $\Alt(i)=\{j:\Theta_j\ne\Theta_i\}$, define
\begin{equation}\label{eq:intro-lp}
 C_i=\min_{n\ge0}\left\{\sum_a g_i(a)n_a:
       \sum_a n_a d(K_i^a\|K_j^a)\ge1\quad\forall j\in\Alt(i)\right\}.
\end{equation}
For positive losses this is equivalently
\begin{equation}\label{eq:intro-mm}
 C_i=\left[\sup_{q\in\Delta(\cA)}
       \frac{\min_{j\in\Alt(i)}\sum_a q_a d(K_i^a\|K_j^a)}
            {\sum_a q_a g_i(a)}\right]^{-1}.
\end{equation}
Here $q$ denotes sampling frequencies, not fractions of loss expenditure.
The LP remains the primary definition at zero loss; its infeasibility marks
an opposite-label pair that no action can distinguish in the finite-KL case.

Theorem~\ref{thm:LB} gives the transportation lower bound
$R(\pi,E_i)\ge C_i\kl(\delta,1-\delta)$.
Under finite alphabets, full support, and pairwise environment
distinguishability, Theorem~\ref{thm:UB} gives a Certification Track-and-Stop
(CTS) policy attaining
\[
 R_\delta^*(E_i)=C_i\log(1/\delta)+o(\log(1/\delta)).
\]
The arguments adapt controlled experimental design, the transportation
inequality, and Track-and-Stop
\citep{chernoff1959sequential,kaufmann2016complexity,garivier2016optimal}.
\paragraph{B. Finite-table reduction.}
The sequential arguments depend on $f$ entirely through the supplied label map.
Proposition~\ref{prop:table-reduction} states this reduction explicitly.
We do not present the LP's form or these proof techniques as new mathematics
specific to representation adequacy.

\paragraph{C. Compatibility of environment-wise optima.}
The environment-indexed cost table creates a separate attainability question. An adequacy label can be identifiable even when some
same-label environments have identical channels. If those environments have
different task losses, a certifier may be unable to attain their individual
optimal costs simultaneously. Theorem~\ref{thm:observable-compatibility}
gives an exact criterion for positive losses: the information-covering LPs
within each observational equivalence class must have a common minimizer.
When they do, CTS on the observable quotient attains all the constants;
when they do not, a joint lower bound quantifies the obstruction.
Section~\ref{sec:observable:example} realizes an exact factor-two obstruction
using one task loss function from which both the adequacy labels and the
charged losses are derived. This explains why label identifiability alone
does not replace the upper bound's stronger assumption. The criterion
applies to general hidden-cost label testing as well as to adequacy.
The common-optimum principle has an explicit antecedent in monitored
MDPs \citep[Appendix~B.2]{parisi2024monitored}. Here it is instantiated
as a covering-LP criterion with a fixed-confidence joint loss bound and
quotient-CTS attainability proof; the principle itself is not claimed as new.

\paragraph{The zero-loss boundary.}
Proposition~\ref{prop:zero-infimum} proves that $C_i=0$ implies the
stronger fixed-$\delta$ identity $R_\delta^*(E_i)=0$ in the full-support
model, even without pairwise environment identifiability. The infimum
can be approached by finite free probing and a fresh fallback without
being attained. Proposition~\ref{prop:zero-attainment} characterizes
actual zero-cost attainment, including its common-policy version.
These results distinguish the pointwise optimum from the positive
residual loss of a particular CTS family.

The static results provide the interpretation and boundaries of certification.
The Bayes-risk grouping identity is standard, and the internal-certification
bound uses the Le Cam two-point method. Under the explicit one-shot fallback
convention, verification has value
$c^*=(Tr/2)(1-\TV)$ for an equal label prior. This is an application of
classical value of information. Switching and repair remain outside the
sequential theorem.

\subsection{Relation to the closest theory}\label{sec:intro-audit}

Controlled testing already optimizes information per unit cost
\citep{nitinawarat2015controlled}. Pure exploration permits general and
multiple-correct-answer formulations \citep{degenne2019pure}, while
cost-aware best-arm identification treats heterogeneous sampling costs
\citep{kanarios2024cost}. Best-arm identification with minimal regret
already has fixed-confidence lower bounds and matching asymptotics
\citep{yang2026minimal}. Thus neither a general answer map nor
environment-dependent experimentation loss establishes novelty on its own.

M1 uses a different answer---a prescribed representation's Bayes-risk
adequacy---and allows task losses whose realized values are not supplied as
feedback. These are modeling distinctions; they do not establish that a
general controlled-testing theorem could not cover M1 after substitution.
The reduction in Section~\ref{sec:observable} expressly acknowledges this.
Our contribution is the formulation, its complete finite-model analysis,
and the compatibility boundary for simultaneously attaining its
environment-wise costs. The comparison in Section~\ref{sec:related}
records overlap and theorem assumptions, rather than treating failure to
find an identical formulation as proof of priority.

\subsection{Organization}\label{sec:intro-org}

Section~\ref{sec:static} defines adequacy and one-shot certification.
Section~\ref{sec:model} states M1, the LP, the sequential bounds,
CTS, and boundaries B1--B3. Section~\ref{sec:observable} gives the
finite-table reduction, the observable-class compatibility theorem,
and the zero-cost infimum and attainment results.
Section~\ref{sec:switching} delimits switching and repair;
Section~\ref{sec:numerical} reports reproducible finite-instance checks;
Section~\ref{sec:related} compares the nearest literature.
Appendices~\ref{sec:prooflb} and~\ref{sec:proofub} contain the lower- and
upper-bound proofs. Numerical checks support the stated finite examples;
they are not evidence for unproved dynamic extensions.

% sec_static.tex -- JMLR full version, Chapter: The Static Theory
% "Self-Certification of Representation Adequacy"
% Layer 1 (static representation layer) + Layer 2 (one-shot certification layer),
% providing the foundation for the sequential certification model of
% Section~\ref{sec:model} (Theorems~\ref{thm:LB} and \ref{thm:UB}).
%
% Required from the shared preamble (do NOT redefine here):
%   amsthm with a shared counter: theorem, proposition, corollary, lemma
%   (numbered); definition, assumption, example, remark (numbered, same
%   counter); and the unnumbered variants definition*, assumption*,
%   corollary*, remark*, example*.  This file uses the numbered environments
%   ONLY for the seven fixed results, in the fixed order:
%     Theorem 1 (thm:aliasing), Proposition 2 (prop:binary),
%     Theorem 3 (thm:tvbound), Proposition 4 (prop:js),
%     Proposition 5 (prop:boundary), Theorem 6 (thm:verify),
%     Proposition 7 (prop:mug).
%   natbib (\citep/\citet) with the project's 24-key bibliography.
%
% Notation macros (guarded; identical to the mathematical specification):
\providecommand{\Adeq}{\mathrm{Adeq}}
\providecommand{\TV}{\mathrm{TV}}
\providecommand{\JS}{\mathrm{JS}}
\providecommand{\KL}{\mathrm{KL}}
\providecommand{\cA}{\mathcal{A}}
\providecommand{\cE}{\mathcal{E}}
\providecommand{\cF}{\mathcal{F}}
\providecommand{\cT}{\mathcal{T}}
\providecommand{\cV}{\mathcal{V}}
\providecommand{\cX}{\mathcal{X}}
\providecommand{\cZ}{\mathcal{Z}}
\providecommand{\cY}{\mathcal{Y}}
\providecommand{\E}{\mathbb{E}}
\providecommand{\ind}{\mathbf{1}}

\section{The Static Theory: Representation Adequacy and One-Shot Certification}
\label{sec:static}

This chapter develops the first two layers of the theory: the \emph{static representation layer}, which defines decision-theoretic adequacy and decomposes its cost exactly (Section~\ref{sec:static:decomp}), and the \emph{one-shot certification layer}, which asks whether adequacy can be certified from internal transcripts alone and prices an external verification channel when it cannot (Sections~\ref{sec:static:impossibility}--\ref{sec:static:verify}). Everything here is static in two senses: transcript laws are fixed independently of the agent's decisions, and certification is a single purchase decision. The sequential problem---buying accuracy $\delta$ at minimum task cost over time---is taken up in Section~\ref{sec:model}, which builds on the threshold characterization of Theorem~\ref{thm:verify}.

\subsection{The Certification Framework}
\label{sec:static:framework}

\paragraph{Single-shot decision problem.}
Let the history space $H$ be finite, with elements $h \in H$ called histories (or states), and let the action set $\cA$ be finite. An \emph{environment} $E$ specifies a distribution $p_E$ on $H$ and a loss $\ell_E : H \times \cA \to [0,1]$ (and, in the sequential model of Section~\ref{sec:model}, a transition law). A (possibly randomized) decision rule is a map $\delta : H \to \Delta(\cA)$ into the simplex over $\cA$, with risk $R(\delta; p_E, \ell_E) := \sum_{h} p_E(h) \sum_a \delta(a \mid h)\, \ell_E(h,a)$; the \emph{Bayes risk} of observing the full history is
\begin{equation}
R^*_E(H) := \min_{\delta}\, R(\delta; p_E, \ell_E)
= \sum_{h \in H} p_E(h) \min_{a \in \cA} \ell_E(h,a),
\label{eq:static:riskH}
\end{equation}
where the minimum over all randomized rules is attained deterministically: for each $h$, the inner sum $\sum_a \delta(a \mid h)\,\ell_E(h,a)$ is a convex combination of the numbers $\{\ell_E(h,a)\}_{a \in \cA}$, hence bounded below by $\min_a \ell_E(h,a)$, with equality whenever $\delta(\cdot \mid h)$ is supported on $\arg\min_a \ell_E(h,a)$, which is nonempty since $\cA$ is finite. Bayes risk as a difficulty measure is standard \citep{degroot1970,blackwellgirshick1954}.

\paragraph{Representation and aliasing cells.}
A \emph{representation map} is a deterministic $f : H \to U$ into a representation space $U$; the \emph{aliasing cell} of $u \in U$ is
\begin{equation}
H_u := \{ h \in H : f(h) = u \}.
\label{eq:static:cell}
\end{equation}
The cells partition $H$; two histories in one cell are indistinguishable under the agent's representation (\emph{aliasing}). An \emph{$f$-based rule} $\tilde\delta : U \to \Delta(\cA)$ acts as $\delta(h) = \tilde\delta(f(h))$, constant within cells, and its Bayes risk is
\begin{equation}
R^*_E(U) := \min_{\tilde\delta : U \to \Delta(\cA)}\,
\sum_{h \in H} p_E(h)\, \E_{a \sim \tilde\delta(f(h))}\bigl[\ell_E(h,a)\bigr].
\label{eq:static:riskU}
\end{equation}
Since $\{\tilde\delta \circ f\} \subseteq \{\delta : H \to \Delta(\cA)\}$, one has $R^*_E(U) \ge R^*_E(H)$; the gap is the \emph{aliasing regret}, decomposed cell by cell by Theorem~\ref{thm:aliasing}.

\paragraph{The certification quintuple.}
The meta-decision problem ``can the agent certify the adequacy of its own representation'' is specified by the quintuple
\begin{equation}
\langle\ \cE,\ \cF,\ \cT,\ \cV,\ c\ \rangle .
\label{eq:static:quintuple}
\end{equation}
The \emph{environment class} $\cE$ collects the candidate environments; $|\cE| \ge 2$ with label disagreement is the premise of every unidentifiability result below (a singleton class makes adequacy directly computable, Proposition~\ref{prop:boundary}(b)(i)). The \emph{representation class} $\cF$ contains deterministic maps; $f \in \cF$ is fixed and known throughout---we certify a given representation, we do not select one. The \emph{internal transcript channel} $\cT$ is the cost-free information source: $t$ periods of interaction produce the internal transcript
\begin{equation}
T_t := (X_1, A_1, \dots, X_t, A_t),
\label{eq:static:transcript}
\end{equation}
with $X_s$ the period-$s$ observation and $A_s \in \cA$ the period-$s$ action, and $P_E^t$ the induced transcript law under environment $E$.
\begin{quote}
\emph{Key assumption (no loss feedback).} Except where explicitly stated, $T_t$ contains no feedback on the realized loss $\ell_E$. Every unidentifiability result in this chapter sources from this assumption.
\end{quote}
The \emph{external \textsc{verify} channel} $\cV$: the action \textsc{verify} costs $c \ge 0$ and returns a signal $Y \sim Q(\cdot \mid E)$ on a signal space $\cY$, with $Q$ known; \emph{perfect} \textsc{verify} ($Y$ reveals $E$, hence the adequacy label) and \emph{non-augmenting} \textsc{verify} ($Q(\cdot \mid E)$ does not depend on $E$) are the two recurring special cases. Theorem~\ref{thm:verify} additionally uses the horizon $T$ and the per-period aliasing regret $r > 0$.

\paragraph{The adequacy label.}
Certification tests a label, and the choice of label is a framework-level decision. This paper tests \emph{decision-theoretic adequacy},
\begin{equation}
\Adeq_E(f) := \ind\bigl[\, R^*_E(U) = R^*_E(H) \,\bigr]
\quad\Longleftrightarrow\quad \text{aliasing regret} = 0,
\label{eq:static:adeqdef}
\end{equation}
rather than the environment identity itself. Adequacy is risk-level, not information-level: $f$ may discard much about $h$ and remain adequate whenever the optimal attainable risk is unchanged; and it is task- and environment-dependent through both $p_E$ and $\ell_E$---there is no ``absolutely adequate'' representation. Certification is non-trivial only when the label $\Theta(E) := \Adeq_E(f)$ takes both values on $\cE$.

\paragraph{Label priors and aggregated mixtures.}
Certification is a binary testing problem over the label $\Theta(E) = \Adeq_E(f) \in \{0,1\}$, whatever the size of the environment class. Let $\cE$ be an arbitrary finite class with prior $q$; the prior induces the \emph{label prior} $\pi_i := q(\Theta = i)$ and the \emph{label-aggregated transcript mixtures}
\begin{equation}
\bar{P}_i^t \;:=\; P(T_t \in \cdot \mid \Theta = i)
\;=\; \sum_{E :\, \Theta(E) = i} q(E \mid \Theta = i)\, P_E^t,
\qquad i \in \{0,1\},
\label{eq:static:mixtures}
\end{equation}
through which every bound in this chapter is stated. The two-point class $\cE = \{E_0, E_1\}$ with disagreeing labels is the special case $|\cE| = 2$, where each mixture reduces to a single transcript law, $\bar{P}_i^t = P_{E_i}^t$. \emph{Equal label priors} $\pi_0 = \pi_1 = 1/2$ are a standing assumption wherever the closed form $(1-\TV)/2$ appears (Sections~\ref{sec:static:impossibility} and~\ref{sec:static:verify}); Theorem~\ref{thm:verify}(e)(iii) shows this assumption is indispensable to the \emph{shape} of the closed form, and gives the general-prior formula.

\paragraph{Relation to Blackwell sufficiency.}
The experiment induced by $f$ is a deterministic garbling of the identity experiment, so $f$ is less informative in Blackwell's order \citep{blackwell1951,blackwell1953,lecam1964}, consistent with $R^*_E(U) \ge R^*_E(H)$; but Blackwell sufficiency requires risk preservation for \emph{all} losses and priors, whereas $\Adeq_E(f)$ requires it only for the fixed pair $(p_E, \ell_E)$---the specialization that makes adequacy environment-dependent and certification non-trivial.

\subsection{Aliasing Regret Decomposition}
\label{sec:static:decomp}

This subsection states the definitional tool of the static layer---the cell-wise decomposition of aliasing regret, a standard Bayes-risk grouping identity under a partition (attribution below; no novelty claimed)---and fixes the cell-level source of the per-period regret $r$ that Theorem~\ref{thm:verify} prices.

\begin{theorem}[Aliasing regret decomposition]
\label{thm:aliasing}
For any environment $E$ and representation $f : H \to U$,
\begin{equation}
R^*_E(U) - R^*_E(H)
= \sum_{u \in U}\Biggl[ \min_{a \in \cA} \sum_{h \in H_u} p_E(h)\, \ell_E(h,a)
- \sum_{h \in H_u} p_E(h) \min_{a \in \cA} \ell_E(h,a) \Biggr].
\label{eq:static:decomp}
\end{equation}
\end{theorem}

\emph{Attribution.} A representation map is a deterministic garbling of the identity experiment, so the inequality $R^*_E(U) \ge R^*_E(H)$ is a direct corollary of the risk-domination characterization of Blackwell's comparison of experiments \citep{blackwell1951,lecam1964}, and \eqref{eq:static:decomp} is its finite, cell-explicit form; the cellwise equality condition below is the single-period, risk-level analogue of value-preserving aggregation criteria in the state-abstraction literature \citep{givan2003,li2006}. We claim no novelty for the identity itself.

\begin{proof}[Proof of Theorem~\ref{thm:aliasing}]
We proceed in three steps.

\emph{Step 1 (cell expansion of $R^*_E(U)$).} An $f$-based rule acts within each cell $H_u$ through a single (possibly mixed) action $q(\cdot \mid u) := \tilde\delta(u) \in \Delta(\cA)$; grouping its expected loss by cell,
\begin{equation*}
\sum_{h \in H} p_E(h) \sum_{a} q(a \mid f(h))\, \ell_E(h,a)
= \sum_{u \in U} \sum_{a \in \cA} q(a \mid u)
\Biggl[ \sum_{h \in H_u} p_E(h)\, \ell_E(h,a) \Biggr].
\end{equation*}
The bracket for cell $u$ depends only on $q(\cdot \mid u)$, so the minimization over all $f$-based rules separates into independent cellwise minimizations. For each cell, the mixture risk $\sum_a q(a \mid u) \bigl[ \sum_{h \in H_u} p_E(h) \ell_E(h,a) \bigr]$ is a convex combination, with weights $q(\cdot \mid u)$, of the pure-action cell risks $\sum_{h \in H_u} p_E(h) \ell_E(h,a)$; a convex combination of real numbers is bounded below by their minimum, so the cell contribution is at least $\min_a \sum_{h \in H_u} p_E(h) \ell_E(h,a)$, with equality exactly when $q(\cdot \mid u)$ is supported on the cell's Bayes action set $\arg\min_a \sum_{h \in H_u} p_E(h) \ell_E(h,a)$ (nonempty, $\cA$ finite). Hence the minimum over the class of all---including randomized---$f$-based rules is attained by a deterministic rule, and
\begin{equation}
R^*_E(U) = \sum_{u \in U} \min_{a \in \cA} \sum_{h \in H_u} p_E(h)\, \ell_E(h,a).
\label{eq:static:RU}
\end{equation}

\emph{Step 2 (cell expansion of $R^*_E(H)$).} By the deterministic attainment of \eqref{eq:static:riskH}, $R^*_E(H) = \sum_{h \in H} p_E(h) \min_a \ell_E(h,a)$; regrouping this finite sum by cell,
\begin{equation}
R^*_E(H) = \sum_{u \in U} \sum_{h \in H_u} p_E(h) \min_{a \in \cA} \ell_E(h,a).
\label{eq:static:RH}
\end{equation}

\emph{Step 3 (subtraction).} Subtracting \eqref{eq:static:RH} from \eqref{eq:static:RU}, both sides being sums over $u \in U$, yields \eqref{eq:static:decomp} term by term.
\end{proof}

The proof uses only finiteness of $H$ and $\cA$ (all minima attained, all sums finite); the identity holds verbatim for any real-valued loss on finite spaces, the restriction $\ell_E \in [0,1]$ being used only later, when per-period regrets are compared across environments in Theorem~\ref{thm:verify}. Two properties are read off \eqref{eq:static:decomp} by comparing the two terms within each cell.

\paragraph{Property (i): nonnegativity and the equality condition.}
Fix a cell $u$ and a cell Bayes action $a^*_u \in \arg\min_a \sum_{h \in H_u} p_E(h) \ell_E(h,a)$. For every $h \in H_u$, $\min_a \ell_E(h,a) \le \ell_E(h,a^*_u)$ by the definition of the pointwise minimum; multiplying by $p_E(h)$ and summing over $h \in H_u$ shows each cell term is nonnegative, hence $R^*_E(U) \ge R^*_E(H)$. Equality $R^*_E(U) = R^*_E(H)$ holds if and only if every cell term vanishes, and a single cell term vanishes if and only if there exists an action $a^*$ with $a^* \in \arg\min_a \ell_E(h,a)$ for all $h \in H_u$ with $p_E(h) > 0$: if such an $a^*$ exists, the cell Bayes action achieves the pointwise minimum on all positive-mass histories and the two terms coincide; conversely, if no action is simultaneously Bayes on all positive-mass histories, then for every candidate $a'$ some positive-mass $h$ has $\ell_E(h,a') > \min_a \ell_E(h,a)$ and the cell term is strictly positive. Zero-mass histories neither contribute nor constrain, which is why adequacy depends on $p_E$ as well as $\ell_E$.

\paragraph{Property (ii): positive contribution $\Longleftrightarrow$ action conflict.}
Say cell $u$ exhibits an \emph{action conflict} if its positive-mass histories share no common Bayes action. By (i), the cell's contribution is strictly positive if and only if a conflict occurs. Together, (i) and (ii) give the cell-level characterization used throughout:
\begin{equation}
\Adeq_E(f) = 1 \;\Longleftrightarrow\; \text{within every cell, histories share a Bayes action } p_E\text{-a.s.}
\label{eq:static:cellchar}
\end{equation}
This reduces adequacy to a risk-level event---``no cell contains an action conflict''---and thereby reduces certification to testing (Section~\ref{sec:static:impossibility}).

\begin{proposition}[Binary closed form]
\label{prop:binary}
Suppose cell $H_u$ contains exactly two histories, written $h = 0, 1$; $\cA = \{0,1\}$; and $\ell_E(h,a) = \ind[a \neq h]$. With cell mass $m := \sum_{h \in H_u} p_E(h)$ and within-cell posterior $p := P(h = 1 \mid u)$, the cell's contribution to \eqref{eq:static:decomp} is
\begin{equation}
m \cdot \min(p,\, 1 - p),
\label{eq:static:binary}
\end{equation}
maximized at $p = 1/2$ with value $m/2$.
\end{proposition}

\begin{proof}[Proof of Proposition~\ref{prop:binary}]
For each history, $\min_a \ell_E(h,a) = 0$ (take $a = h$), so the second cell term in \eqref{eq:static:decomp} vanishes. For the first term, action $a = 1$ errs exactly on $h = 0$, giving cell risk $p_E(0) = m(1-p)$; action $a = 0$ errs exactly on $h = 1$, giving cell risk $p_E(1) = mp$. Hence the cell contribution is $\min\{mp, m(1-p)\} = m \min(p, 1-p)$. Finally, $\min(p, 1-p)$ is the pointwise minimum of two affine functions of $p \in [0,1]$, hence concave, strictly below $1/2$ on both sides of $p = 1/2$, and equal to $1/2$ there.
\end{proof}

\emph{Attribution.} Equation \eqref{eq:static:binary} is the standard binary Bayes-error form $L^* = \E[\min(\eta(X), 1-\eta(X))]$ \citep[Theorem~2.1]{dgl1996} restricted to one cell. The reading that matters later: aliasing is most destructive at $p = 1/2$, where any constant action errs on half the cell mass, and harmless at $p \in \{0,1\}$; the balanced cell is the canonical worst case in the constructions of Theorem~\ref{thm:verify}.

\begin{remark*}[Relation to state aggregation and to weighted forgetting]
The equality condition of property (i) overlaps with value-preserving state aggregation \citep{givan2003,li2006}; \eqref{eq:static:cellchar} is its single-period, risk-level statement, and we claim no novelty. The difference is usage: those works check the condition forward under a known model, whereas we fix it as the label $\Adeq_E(f)$ and ask whether it is verifiable from loss-free internal transcripts. The concurrent preprint of \citet{bennett2026forgetting} proves a cell-level aliasing decomposition of the same type as Theorem~\ref{thm:aliasing}: its $K\rho$ is the minimal margin-weighted deletion cost making every (representation, test) cell's optimal bet single-valued, and its own Remark~14 identifies $K\rho$ as a weighted Bayes error on the partition---consistent with our attribution of Theorem~\ref{thm:aliasing} as a standard grouping identity. The strict differences are fourfold: (i) \emph{loss convention}---normalized betting regret with margin weights there, unnormalized Bayes risk difference with $\ell \in [0,1]$ here; (ii) \emph{cells}---joint (representation $\times$ diagnostic test) cells there, representation cells $H_u$ here; (iii) \emph{object}---a policy-wise decomposition (realized regret $= K\rho +$ within-cell misprediction) there, a Bayes-risk difference (optimal $U$-rule minus optimal $H$-rule) here; (iv) \emph{goal}---that work connects weighted weakness to generalization probability and does not study certification; this paper defines the adequacy label from the decomposition and studies its certifiability and the value of external verification.
\end{remark*}

\subsection{One-Shot Impossibility of Internal Certification}
\label{sec:static:impossibility}

Can the agent determine the label $\Adeq_E(f)$ from $T_t$ alone? For an arbitrary finite environment class the answer is no, with the degree of impossibility controlled exactly by the total variation distance between the two label-aggregated transcript laws. We declare in advance: the theorem of this subsection is a direct application of Le Cam's two-point method and the Bayes error--TV identity \citep{tsybakov2009,lecam1986,dgl1996}; the only new element is the test label, $\Adeq_E(f)$ rather than the environment itself.

\paragraph{The certification problem.}
Let $\cE$ be an arbitrary finite environment class with prior $q$, $f$ fixed, and $\Theta(E) = \Adeq_E(f)$ taking both values on $\cE$; equal label priors $\pi_0 = \pi_1 = 1/2$ are assumed wherever the closed form $(1-\TV)/2$ appears. The data are governed by the label-aggregated mixtures $\bar{P}_0^t, \bar{P}_1^t$ of \eqref{eq:static:mixtures}---laws of observation--action sequences only, by the key no-loss-feedback assumption.

\begin{definition*}[Internal certification rule]
An \emph{internal certification rule} is any measurable $C : \cT^t \to \{0,1\}$ depending only on $T_t$; $C = 1$ certifies ``$f$ adequate'', $C = 0$ rejects. Randomized rules (extra internal randomization independent of environment and data) are allowed. The error events are $\{C = 1\}$ under $\Theta = 0$ (false certification) and $\{C = 0\}$ under $\Theta = 1$ (false rejection).
\end{definition*}

Structurally this is two-point testing between the mixtures $\bar{P}_0^t$ and $\bar{P}_1^t$, whose error bounds are given completely by Le Cam's method \citep{lecam1986,tsybakov2009}; the sole difference from the classical setup is the label: not the environment identity $E$ but $\Adeq_E(f)$, environments sharing a label being aggregated into one mixture. Conditioning on the binary label collapses any finite class to a two-hypothesis problem, so the classical closed form survives for every finite $\cE$---in contrast to multi-class environment identification, which no single pairwise TV governs (see the remark after Theorem~\ref{thm:verify}).

\begin{theorem}[TV lower bound for internal certification]
\label{thm:tvbound}
In the setup above, write $\TV(P,Q) := \sup_B |P(B) - Q(B)|$ over measurable events. For any internal certification rule $C$ (including randomized rules),
\begin{equation}
\max\bigl\{ \bar{P}_0^t(C = 1),\, \bar{P}_1^t(C = 0) \bigr\}
\;\ge\; \frac{1 - \TV(\bar{P}_0^t, \bar{P}_1^t)}{2}.
\label{eq:static:tvminimax}
\end{equation}
Under equal label priors $\pi_0 = \pi_1 = 1/2$, the Bayes error satisfies
\begin{equation}
\mathrm{err}(C) := \tfrac12\, \bar{P}_0^t(C{=}1) + \tfrac12\, \bar{P}_1^t(C{=}0)
\;\ge\; \frac{1 - \TV(\bar{P}_0^t, \bar{P}_1^t)}{2},
\label{eq:static:tvbayes}
\end{equation}
with equality attained by the likelihood-ratio test. In particular, if $\bar{P}_0^t = \bar{P}_1^t$, every internal certification rule has Bayes error exactly $\tfrac12$---no better than random guessing.
\end{theorem}

\emph{Attribution.} The minimax bound \eqref{eq:static:tvminimax} is Theorem~2.2(i) of \citet{tsybakov2009}, resting on the Scheff\'e identity (their Lemma~2.1), applied here to the label-aggregated mixtures; the equal-prior identity also appears in \S3.2 of \citet{dgl1996}. The only contribution of this paper here is the label.

\begin{proof}[Proof of Theorem~\ref{thm:tvbound}]
Write $P_0 = \bar{P}_0^t$ and $P_1 = \bar{P}_1^t$. For a randomized rule, condition on each realization of its internal randomization: every inequality below holds pointwise on each realization, and taking expectation over the randomization preserves the inequalities; hence it suffices to treat deterministic $C$. Let $A = \{C = 1\}$.

\emph{Step 0 (conditioning on the label).} By the definition \eqref{eq:static:mixtures} and the tower property, $P(C{=}1 \mid \Theta{=}0) = \sum_{E : \Theta(E) = 0} q(E \mid \Theta{=}0)\, P_E^t(C{=}1) = \bar{P}_0^t(C{=}1)$, and likewise $P(C{=}0 \mid \Theta{=}1) = \bar{P}_1^t(C{=}0)$. Hence $\mathrm{err}(C)$ is exactly the equal-prior Bayes error of $C$ for the two-hypothesis problem $(P_0, P_1)$, whatever the size of $\cE$; no step below uses the cardinality of $\cE$.

\emph{Step 1 (sum of the two error probabilities).} Substituting $P_1(A^c) = 1 - P_1(A)$ and then the definition of $\TV$,
\begin{equation}
P_0(A) + P_1(A^c)
= 1 - \bigl(P_1(A) - P_0(A)\bigr)
\ge 1 - \sup_B \bigl|P_1(B) - P_0(B)\bigr|
= 1 - \TV(P_0, P_1).
\label{eq:static:sum}
\end{equation}

\emph{Step 2 (minimax bound).} The maximum of two numbers is at least their average, so \eqref{eq:static:sum} gives $\max\{P_0(A), P_1(A^c)\} \ge \bigl[P_0(A) + P_1(A^c)\bigr]/2 \ge (1-\TV)/2$, which is \eqref{eq:static:tvminimax}.

\emph{Step 3 (Bayes bound).} Under equal label priors, $\mathrm{err}(C) = \tfrac12 \bigl[ P_0(A) + P_1(A^c) \bigr] \ge (1-\TV)/2$ by \eqref{eq:static:sum}.

\emph{Step 4 (achievability).} Let $p_0, p_1$ be densities of $P_0, P_1$ with respect to a common dominating measure $\mu$ (one may take $\mu = (P_0 + P_1)/2$), and consider the likelihood-ratio test $A^* = \{p_1 > p_0\}$. Writing out its Bayes error pointwise, the integrand equals $p_0$ on $A^*$ and $p_1$ on $A^{*c}$, and in both cases equals $\min(p_0, p_1)$:
\begin{equation*}
\mathrm{err}(A^*)
= \tfrac12 \int \Bigl[ p_0\, \ind\{p_1 > p_0\} + p_1\, \ind\{p_1 \le p_0\} \Bigr] d\mu
= \tfrac12 \int \min(p_0, p_1)\, d\mu .
\end{equation*}
Now substitute the Scheff\'e identity $\TV(P_0,P_1) = \tfrac12 \int |p_0 - p_1|\, d\mu = 1 - \int \min(p_0, p_1)\, d\mu$ \citep[Lemma~2.1]{tsybakov2009}, whose second equality follows from the pointwise identity $|p_0 - p_1| = p_0 + p_1 - 2\min(p_0,p_1)$ and $\int (p_0 + p_1)\, d\mu = 2$: hence $\mathrm{err}(A^*) = (1-\TV)/2$, attaining the lower bound of Step~3, which applies to every rule. The minimal equal-prior Bayes error is therefore exactly $(1-\TV)/2$.

\emph{Step 5 (indistinguishable case).} If $P_0 = P_1$ then $\TV = 0$, so every rule has Bayes error at least $\tfrac12$ by Step~3; the constant rules $C \equiv 0$ and $C \equiv 1$ each attain exactly $\tfrac12$ (one error probability is $1$, the other $0$). Hence every rule has Bayes error exactly $\tfrac12$: no rule beats random guessing.
\end{proof}

\begin{corollary*}[Internal unidentifiability]
If $\bar{P}_0^t = \bar{P}_1^t$, then for every $t$ and every internal certification rule $C$, the equal-label-prior Bayes error is exactly $\tfrac12$: internal transcripts cannot statistically separate ``$f$ adequate'' from ``$f$ inadequate''.
\end{corollary*}

This case is non-vacuous, because transcript laws depend only on observation--action sequences while $\Adeq_E(f)$ also depends on losses. \emph{Explicit construction} (the two-point special case $|\cE| = 2$): let $H = \{h_0, h_1\}$ be aliased into a single cell by a constant $f$, with the cell index as the only observation and identical policies in both environments, so the transcript laws coincide, $P_{E_0}^t = P_{E_1}^t$ for every $t$; let $E_0$ carry the loss $\ell(h,a) = \ind[a \neq h]$ (the cell contains an action conflict, so $\Adeq_{E_0}(f) = 0$ by property (ii) of Theorem~\ref{thm:aliasing}) and $E_1$ carry $\ell(h,a) = \ind[a \neq 0]$ (action $a = 0$ is Bayes on both histories, so $\Adeq_{E_1}(f) = 1$). Then $\Adeq_{E_0}(f) \neq \Adeq_{E_1}(f)$ while $\bar{P}_0^t = \bar{P}_1^t$: adequacy is unidentifiable from internal data in the strong sense of the corollary.

\paragraph{Information-theoretic form.}
Under equal priors, the same distinguishability has an exact information-theoretic statement. With base-2 logarithms (units of bits), the Jensen--Shannon divergence is \citep{lin1991}
\begin{equation}
\JS(P,Q) := \tfrac12\, \KL\bigl(P \,\big\|\, \tfrac{P+Q}{2}\bigr)
+ \tfrac12\, \KL\bigl(Q \,\big\|\, \tfrac{P+Q}{2}\bigr),
\label{eq:static:jsdef}
\end{equation}
where $\KL(P \| Q) := \int p \log_2(p/q)\, d\mu$ for densities $p, q$ against a common dominating $\mu$.

\begin{proposition}[Mutual-information form]
\label{prop:js}
Under equal label priors $\pi_0 = \pi_1 = \tfrac12$,
\begin{equation}
I(\Theta;\, T_t) = \JS(\bar{P}_0^t,\, \bar{P}_1^t).
\label{eq:static:mijs}
\end{equation}
Consequently $\bar{P}_0^t = \bar{P}_1^t$ implies $I(\Theta; T_t) = 0$: the internal transcript carries no information whatsoever about adequacy.
\end{proposition}

\emph{Attribution.} Proposition~\ref{prop:js} is a direct rewriting of the JS-divergence definition \citep{lin1991} for an equal-prior two-component mixture, not a new information-theoretic principle; it makes explicit that certification at $\bar{P}_0^t = \bar{P}_1^t$ is information-theoretically impossible ($I(\Theta; T_t) = 0$), not merely worst-case hard.

\begin{proof}[Proof of Proposition~\ref{prop:js}]
Write $M = (\bar{P}_0^t + \bar{P}_1^t)/2$, let $p_0, p_1$ be the densities of $\bar{P}_0^t, \bar{P}_1^t$ against a dominating $\mu$, and set $\eta(t) := p_1(t) / (p_0(t) + p_1(t))$. Under the equal label prior, $\eta(t)$ is exactly the posterior $P(\Theta = 1 \mid T_t = t)$, and the Radon--Nikodym derivatives against $M$ are $d\bar{P}_0^t / dM = 2(1-\eta)$ and $d\bar{P}_1^t / dM = 2\eta$. Substituting into \eqref{eq:static:jsdef} and expanding $\log_2(2u) = 1 + \log_2 u$ term by term,
\begin{equation}
\JS(\bar{P}_0^t, \bar{P}_1^t)
= \int \Bigl[ (1-\eta) \log_2 \bigl(2(1-\eta)\bigr) + \eta \log_2 (2\eta) \Bigr] dM
= \int \bigl[ 1 - h_2(\eta) \bigr] dM,
\label{eq:static:jsexpand}
\end{equation}
where $h_2(u) = -u \log_2 u - (1-u) \log_2 (1-u)$ is the binary entropy function and the last step uses $(1-\eta) + \eta = 1$. On the other hand, by the entropy representation of mutual information, the equal label prior $H(\Theta) = h_2(\tfrac12) = 1$, and the marginal law of the transcript being $M$,
\begin{equation}
I(\Theta; T_t)
= H(\Theta) - \E_{t \sim M}\, h_2\bigl( P(\Theta{=}1 \mid T_t{=}t) \bigr)
= 1 - \int h_2(\eta)\, dM .
\label{eq:static:mi}
\end{equation}
Comparing \eqref{eq:static:mi} with \eqref{eq:static:jsexpand} yields \eqref{eq:static:mijs}. Finally, $\JS(P,Q) = 0$ if and only if both KL terms vanish; each KL divergence is nonnegative and vanishes only at equality of its arguments (Jensen's inequality applied to the strictly convex $-\log_2$), so $\JS(\bar{P}_0^t, \bar{P}_1^t) = 0$ if and only if $\bar{P}_0^t = \bar{P}_1^t$, in which case $I(\Theta; T_t) = 0$.
\end{proof}

\subsection{The Honest Boundary: A Counterexample and Possibility Conditions}
\label{sec:static:boundary}

Theorem~\ref{thm:tvbound} is a binary testing statement about two observed, label-conditional transcript laws. Its theorem-level requirements are that both adequacy labels occur with positive prior mass and that the certifier is restricted to the stated transcript; the bound then holds regardless of which variables the transcript contains. The loss-free, aliased construction above is an important \emph{mechanism} for making the two laws coincide: adequacy depends on losses, while the transcript need not reveal them. It is not an additional hypothesis of the total-variation inequality. The counterexample of part (a) is load-bearing: an early stage of this research program conjectured a closure claim for all finite-state systems, and the counterexample below refutes it in a precise form. We report the refutation in the body, not an appendix, because it pins down the actual scope of the claim.

\begin{proposition}[Boundary and counterexample package]
\label{prop:boundary}
\textbf{(a) (Counterexample A.)} There exists a finite Markov decision process whose optimal actions change the world state with probability one every period, yet whose state representation is decision-theoretically adequate and internally certifiable; hence the closure claim ``for all finite-state systems, if the agent's actions can change the world state, then the system cannot internally certify the decision-theoretic adequacy of its state representation'' is false.
\textbf{(b) (Sufficient conditions.)} Adequacy is certifiable or directly computable whenever any one of the following holds: (i) the environment class is a singleton (the model is known); (ii) the transcript carries loss feedback, with structure permitting counterfactual comparison; (iii) the representation is aliasing-free ($f$ bijective, or the state fully observable).
\textbf{(c) (Belief-state note.)} In partially observable environments the belief state is a sufficient information state for optimal control \citep{astrom1965,kaelbling1998}; representation inadequacy can therefore arise only when the agent's representation is coarser than the belief state.
\end{proposition}

\paragraph{Proof of (a): Counterexample A, full analysis.}
\emph{Construction.} Consider the finite MDP with world state $W_t \in \{0,1\}$ fully observed ($Z_t = W_t$), action set $\cA = \{0,1\}$, loss
\begin{equation}
\ell(W, a) = \ind[a \neq W],
\label{eq:static:celoss}
\end{equation}
and transition law
\begin{equation}
W_{t+1} = 1 - A_t,
\label{eq:static:cetrans}
\end{equation}
so the next world state is fully determined by the current action, independently of the current state. The representation $f$ is the identity map, and the environment class $\cE$ is a singleton: the model \eqref{eq:static:celoss}--\eqref{eq:static:cetrans} is fully known to the agent. The value-iteration argument below follows the standard treatment of finite MDPs in \citet{suttonbarto2018}.

\emph{Analysis.} The Bellman optimality equation for this system is
\begin{equation}
V^*(w) = \min_{a \in \{0,1\}} \Bigl\{ \ind[a \neq w] + V^*(1-a) \Bigr\},
\qquad w \in \{0,1\}.
\label{eq:static:bellman}
\end{equation}
Substituting $V^* \equiv 0$ into the right-hand side gives $\min_a \ind[a \neq w] = 0$, matching the left-hand side, so $V^* \equiv 0$ solves \eqref{eq:static:bellman}. To confirm optimality without appeal to uniqueness of the Bellman solution: the per-period loss is nonnegative under every policy, so the expected total loss of any policy is bounded below by $0$; the policy $a^*_t = W_t$ attains loss $0$ in every period (its action always equals the current state), hence it is optimal and the optimal value function is identically zero.

Along the optimal trajectory $a^*_t = W_t$, the transition law \eqref{eq:static:cetrans} gives $W_{t+1} = 1 - A_t = 1 - W_t$, and therefore
\begin{equation}
P\bigl(W_{t+1} \neq W_t\bigr) = 1
\label{eq:static:flip}
\end{equation}
in every period: the agent's optimal action changes the world state with probability one. Meanwhile, every aliasing cell of the identity $f$ is a singleton, so no cell contains an action conflict; by Theorem~\ref{thm:aliasing} every cell term vanishes, $R^*_E(U) = R^*_E(H)$, i.e.\ $\Adeq_E(f) = 1$. Moreover, since the environment class is a singleton and the model is known, the agent computes $\Adeq_E(f)$ directly---both Bayes risks are explicit optimizations over known objects---without consulting any transcript statistics. Internal certification is therefore trivially feasible.

\emph{The falsified claim, stated precisely.}
\begin{quote}
\textbf{Claim (falsified).} For all finite-state systems: if the agent's actions can change the world state, then the system cannot internally certify the decision-theoretic adequacy of its state representation.
\end{quote}
Counterexample A satisfies the antecedent---by \eqref{eq:static:flip}, the optimal policy changes the world state with probability one every period---and violates the consequent---$\Adeq_E(f) = 1$, computable internally. A closure claim for all finite-state systems whose antecedent holds and whose consequent fails on a single instance is false; the historical heuristic conjecture referred to as the ``state-closure paradox'' is refuted in precisely this form. The refutation does not touch Theorem~\ref{thm:tvbound}: a singleton, known environment class induces no label-discrimination problem at all. It also lies outside the loss-free aliasing construction used above, because $f$ is the identity. Any meaningful replacement for the rejected universal claim must state an explicit observational indistinguishability premise, not merely that actions can alter the world.

\paragraph{Proof of (b): three sufficient conditions for certifiability.}
Each condition describes a regime in which adequacy is computable or the observed transcript can become informative. They are useful exits from the motivating loss-free aliasing construction, but they are not logical negations of the hypotheses of Theorem~\ref{thm:tvbound}.
\emph{(i) Singleton class (known model).} If $|\cE| = 1$, then $p_E$, $\ell_E$ (and any transition law) are known, and $\Adeq_E(f) = \ind[R^*_E(U) = R^*_E(H)]$ is directly computable: both Bayes risks are explicit optimizations over known objects. No hypotheses need to be distinguished, the two-point premise of Theorem~\ref{thm:tvbound} fails, and the certification problem does not arise; Counterexample A falls into the possibility region exactly along this route.
\emph{(ii) Loss feedback.} If the internal transcript contains realized losses, the realized loss of the $f$-based rule is an observable sequence, and under additional structure permitting counterfactual comparison---the agent can estimate, within each cell, the losses of the actions it did not take---adequacy testing can be converted into online estimation of the aliasing regret. Such feedback may make the two label-conditional laws distinguishable, but it does not do so automatically; Theorem~\ref{thm:tvbound} still applies to the resulting observed laws. By way of contrast, windowed adaptive methods for concept drift such as the FLORA family of \citet{widmer1996} presuppose observable prediction-error signals to detect environmental change; their feasibility rests precisely on loss feedback and is orthogonal to the no-loss-feedback assumption of this paper.
\emph{(iii) Aliasing-free representation.} If $f$ is a bijection---in particular, if the state is fully observable and $f$ is the identity---every cell is a singleton, no cell contains an action conflict, and by Theorem~\ref{thm:aliasing} the aliasing regret vanishes identically: $\Adeq_E(f) \equiv 1$ holds a priori as a definitional fact and requires no certification.

These examples are not a converse characterization. For Theorem~\ref{thm:tvbound}, a non-vacuous certification problem needs both labels to occur and needs the two label-conditional transcript laws; the numerical lower bound is determined entirely by their total variation. A multi-environment class with label disagreement, a loss-free internal record, and an aliased representation describe a central ESD route to indistinguishability, but loss feedback may still be insufficient without counterfactual structure, and aliasing alone need not make the two transcript laws close. Conversely, any feedback mechanism that changes those laws simply changes the total variation in the theorem. \hfill $\square$

\paragraph{Proof of (c): belief-state note.}
\citet{astrom1965} proved that in a partially observable Markov process the conditional distribution of the state given the action--observation history---the \emph{belief state}---is a sufficient information state, so the partially observable optimal control problem can be reformulated on the space of these conditional distributions; \citet{kaelbling1998} formalized the resulting belief-MDP reduction for planning under partial observability, and \citet{krishnamurthy2016} gives a comprehensive modern treatment including controlled sensing. It follows that representation inadequacy can occur only when the representation the agent actually maintains is strictly \emph{coarser than the belief state}---for instance, when a memory system applies lossy semantic compression on top of the history. The certification problem of this paper is therefore not a consequence of partial observability itself, but of an additional lossy representation layer below the belief state; this delimits the problem domain precisely. \hfill $\square$

\subsection{The Value of One-Shot External Verification}
\label{sec:static:verify}

We now state the main theorem of the static theory: on an arbitrary finite environment class, under equal label priors and symmetric alignment, the value of an external \textsc{verify} action is characterized by a TV-weighted closed-form threshold. The claim is scoped in advance to the combination ``adequacy label $+$ channel augmentation $+$ threshold characterization''; every component tool is standard and is credited where used.

\paragraph{Setup.}
$\cE$ is an arbitrary finite environment class with prior $q$, $f$ fixed, and the label $\Theta(E) = \Adeq_E(f)$ taking both values; the induced label prior is equal, $\pi_0 = \pi_1 = 1/2$ (standing assumption; item (e)(iii) of the theorem shows it is indispensable to the closed-form shape). Transcript laws enter through the label-aggregated mixtures $\bar{P}_0^t, \bar{P}_1^t$ of \eqref{eq:static:mixtures}; the two-point class is the special case $|\cE| = 2$. By property (ii) of Theorem~\ref{thm:aliasing}, in every inadequate environment ($\Theta = 0$) the $f$-rule incurs a per-period aliasing regret $r > 0$, whose cell-level source is the decomposition \eqref{eq:static:decomp}; in every adequate environment ($\Theta = 1$) it incurs per-period regret $0$. The transcript laws $P_E^t$ are fixed and independent of the switching decision (\emph{transcript exogeneity}); certification is performed once, after which the policy is fixed for the $T$-period horizon (\emph{one-shot certification}).

\begin{assumption*}[A1 (symmetric alignment)]
Certification drives a policy switch: $C = $ ``adequate'' $\Rightarrow$ use the $f$-rule; $C = $ ``inadequate'' $\Rightarrow$ use a fallback rule. Matched rule--label pairs have per-period regret $0$ (the $f$-rule under any environment with $\Theta = 1$, by $\Adeq_E(f) = 1$; the fallback under any environment with $\Theta = 0$, by construction); mismatched pairs incur per-period regret exactly $r > 0$, equal in both directions: misuse of the $f$-rule in an inadequate environment costs the same $r$ per period as misuse of the fallback in an adequate environment. ``Optimal'' in item (a) below means optimal within the A1 certification--switching scheme class.
\end{assumption*}

The fallback rule of A1 is \emph{exogenously given}, known to be optimal in inadequate environments; its information source and acquisition cost lie outside the model (for instance, supplied by a designer or obtained through an independent fully observable channel). Theorem~\ref{thm:verify} prices the \emph{certification} of the label, not the acquisition of the fallback. Asymmetry is why A1 must be symmetric: with mismatch costs $r_0 \neq r_1$ in the two directions, the expected total regret of any certifier is
\begin{equation}
R(C) = \frac{T}{2}\Bigl[\, r_0\, \bar{P}_1^t(C{=}0) + r_1\, \bar{P}_0^t(C{=}1) \Bigr],
\label{eq:static:asym}
\end{equation}
a weighted testing problem whose optimum $\frac{T}{2} \int \min\bigl(r_0\, d\bar{P}_1^t,\, r_1\, d\bar{P}_0^t\bigr)$, attained by the weighted likelihood-ratio test $A^* = \{r_1 \bar{p}_0 < r_0 \bar{p}_1\}$ (with $A=\{C=1\}$; equality may be assigned either way), depends on the full likelihood-ratio geometry and the ratio $r_0/r_1$, not on the scalar $\TV$ alone; the closed form $(1-\TV)/2$ and the threshold of item (b) then fail.

\paragraph{The \textsc{verify} channel.}
\textsc{verify} costs $c \ge 0$ and returns $Y \sim Q(\cdot \mid E)$: \emph{perfect} ($Y$ reveals the label) or \emph{noisy}, in which case we assume $Y$ is conditionally independent of the internal transcript given $E$, so that under label $i$ the agent faces the label-aggregated product law
\begin{equation}
\bar{J}_i := \sum_{E :\, \Theta(E) = i} q(E \mid \Theta = i)\, \bigl[ P_E^t \otimes Q(\cdot \mid E) \bigr],
\qquad i \in \{0,1\}.
\label{eq:static:jointmix}
\end{equation}

\begin{theorem}[Decision-theoretic value of external verification]
\label{thm:verify}
In the setup above (arbitrary finite $\cE$, equal label priors $\pi_0 = \pi_1 = 1/2$, A1), write $\TV := \TV(\bar{P}_0^t, \bar{P}_1^t)$ over the label-aggregated mixtures \eqref{eq:static:mixtures} and $B^*(\cdot)$ for the minimal equal-label-prior Bayes error of the binary label test on the experiment indicated. Then:
\begin{enumerate}
\item[(a)] The internal-only optimal expected total regret is
\begin{equation}
R^*_{\mathrm{int}} = \frac{T\,r}{2}\,\bigl(1 - \TV(\bar{P}_0^t, \bar{P}_1^t)\bigr),
\label{eq:static:rint}
\end{equation}
attained within the A1 scheme class by the certifier induced by the likelihood-ratio test \citep{tsybakov2009,dgl1996}.
\item[(b)] The net value of a perfect \textsc{verify} is
\begin{equation}
V_{\mathrm{verify}} = \frac{T\,r}{2}\,(1 - \TV) - c,
\label{eq:static:vverify}
\end{equation}
and \textsc{verify} strictly dominates internal-only certification if and only if
\begin{equation}
c < c^* := \frac{T\,r}{2}\,\bigl(1 - \TV(\bar{P}_0^t, \bar{P}_1^t)\bigr).
\label{eq:static:cstar}
\end{equation}
\item[(c)] \emph{Two extremes:} $\TV = 0$ ($\bar{P}_0^t = \bar{P}_1^t$) gives $c^* = Tr/2$, verification maximally valuable; $\TV \to 1$ gives $c^* \to 0$, internal certification asymptotically sufficient and \textsc{verify} valueless.
\item[(d)] \emph{Noisy \textsc{verify}:} with $Y$ conditionally independent of $T_t$ given $E$ (product laws \eqref{eq:static:jointmix}), the $(1-\TV)$ factor in the threshold is replaced by the marginal distinguishability gain
\begin{equation}
B^*(\bar{P}^t) - B^*(\bar{J}) \;\ge\; 0,
\label{eq:static:gain}
\end{equation}
i.e.\ noisy \textsc{verify} strictly dominates if and only if $c < T\,r\,\bigl[B^*(\bar{P}^t) - B^*(\bar{J})\bigr]$. If $Q(\cdot \mid E)$ does not depend on $E$ (the signal is uninformative about the label and \textsc{verify} does not augment the experiment), the gain is $0$ and \textsc{verify} should not be used at any $c > 0$ (minimal counterexample) \citep{blackwell1951,lecam1964}.
\item[(e)] \emph{Indispensability:} removing any one of \textup{(i)} the A1 symmetry, \textup{(ii)} the experiment-augmentation property of \textsc{verify}, \textup{(iii)} the equal label prior invalidates the formula of (a) or the closed-form threshold of (b); the counterexamples are constructed in Appendix~\ref{app:static}. For \textup{(iii)}, a closed form survives under a general label prior: with $\pi_0 \neq \pi_1$ the minimal Bayes error of the label test is $\int \min\bigl(\pi_0\, d\bar{P}_0^t,\, \pi_1\, d\bar{P}_1^t\bigr)$ and the threshold is $c^* = T\,r \int \min\bigl(\pi_0\, d\bar{P}_0^t,\, \pi_1\, d\bar{P}_1^t\bigr)$; what is lost is the $(1-\TV)/2$ \emph{shape}, in which the prior cancels.
\end{enumerate}
\end{theorem}

\begin{proof}[Proof of Theorem~\ref{thm:verify}]
\textbf{(a).} Take any certifier $C$ in the A1 scheme class. The certification output selects a rule through the A1 switch, and the adequacy label determines whether the selection matches; Table~\ref{tab:static:a1cases} spells out the four joint cases.

\begin{table}[t]
\centering
\caption{Case bookkeeping for A1 (symmetric alignment) in the proof of Theorem~\ref{thm:verify}(a). Matched rows incur zero per-period regret; mismatched rows incur per-period regret $r$, hence total regret $T\,r$ over the horizon.}
\label{tab:static:a1cases}
\begin{tabular}{llll}
\toprule
Label $\Theta$ & Certification & Rule used (A1 switch) & Per-period / total regret \\
\midrule
$1$ (adequate) & $C{=}1$, correct & $f$-rule (match) & $0$ / $0$ \\
$1$ (adequate) & $C{=}0$, error & fallback (mismatch) & $r$ / $T\,r$ \\
$0$ (inadequate) & $C{=}1$, error & $f$-rule (mismatch) & $r$ / $T\,r$ \\
$0$ (inadequate) & $C{=}0$, correct & fallback (match) & $0$ / $0$ \\
\bottomrule
\end{tabular}
\end{table}

As Table~\ref{tab:static:a1cases} shows, the regret of the whole certification--switching scheme is governed entirely by the certification error event: in the two matched rows the per-period regret is $0$ (in the first row by $\Theta = 1$, i.e.\ $\Adeq_E(f) = 1$ for every environment carrying that label; in the fourth by the construction of the fallback under A1), while in the two mismatched rows it is exactly $r$, the symmetric clause of A1 asserting that the two mismatch directions carry the \emph{same} price $r$. Since certification is performed once and the policy is then fixed for the whole horizon, the total regret is $T \cdot r$ on the error event and $0$ on its complement, so the expected total regret of $C$ is
\begin{equation*}
T \cdot r \cdot P(\mathrm{err}),
\qquad
P(\mathrm{err}) = \tfrac12\, \bar{P}_0^t(C{=}1) + \tfrac12\, \bar{P}_1^t(C{=}0) = \mathrm{err}(C),
\end{equation*}
exactly the equal-label-prior Bayes error of Theorem~\ref{thm:tvbound}. By that theorem, $P(\mathrm{err}) \ge (1-\TV)/2$, so the expected total regret of $C$ is at least $\frac{T\,r}{2}(1-\TV)$. Attainability: Step~4 of the proof of Theorem~\ref{thm:tvbound} shows the likelihood-ratio test $A^* = \{\bar{p}_1 > \bar{p}_0\}$ attains Bayes error exactly $(1-\TV)/2$; coupling it with the A1 switch yields expected total regret exactly \eqref{eq:static:rint}. The underlying identity is the standard Le Cam two-point / Bayes error--TV relation \citep{tsybakov2009,dgl1996}; the only addition is the transcription, via A1, of error rate into regret units, with the label taken to be $\Adeq_E(f)$.

\textbf{(b).} A perfect \textsc{verify} fully reveals the adequacy label, so the certification error is zeroed; by A1 the matched rule is used under each label, the total regret is $0$, at the one-time cost $c$. The net value is the internal-only benchmark minus the cost,
\begin{equation*}
V_{\mathrm{verify}} = R^*_{\mathrm{int}} - c = \frac{T\,r}{2}\,(1 - \TV) - c,
\end{equation*}
which is \eqref{eq:static:vverify}, and $V_{\mathrm{verify}} > 0$ if and only if $c < c^*$ with $c^*$ as in \eqref{eq:static:cstar}. This is a direct instantiation of the value-of-information structure ``value $=$ with-information benchmark $-$ without-information benchmark $-$ cost,'' evaluated before the information is acquired \citep{howard1966,degroot1970}.

\textbf{(c).} Substitute the extremes of $\TV$ into \eqref{eq:static:cstar}: at $\TV = 0$, $c^* = Tr/2$, the maximum of the net-value expression since $1 - \TV \le 1$---this is the case in which internal certification cannot beat random guessing (Theorem~\ref{thm:tvbound}, Step~5), so external verification recovers the entire aliasing loss; as $\TV \to 1$, the optimal internal-only expected regret tends to $0$, so there is asymptotically nothing left for verification to recover.

\textbf{(d).} The joint observation is the pair $(T_t, Y)$, whose law under label $i$ is $\bar{J}_i$ of \eqref{eq:static:jointmix}. Applying to $(T_t, Y)$ the Markov kernel that discards $Y$---keeping only $T_t$---maps $\bar{J}_i$ to $\bar{P}_i^t$ for each $i \in \{0,1\}$; hence the experiment $\{\bar{J}_0, \bar{J}_1\}$ is at least as informative as $\{\bar{P}_0^t, \bar{P}_1^t\}$ in Blackwell's order, and by the comparison theorem of \citet{blackwell1951} a more informative experiment has pointwise no larger Bayes risk on every decision problem \citep{blackwell1951,lecam1964}. Specializing to the equal-label-prior binary test gives $B^*(\bar{J}) \le B^*(\bar{P}^t)$, i.e.\ the gain \eqref{eq:static:gain} is nonnegative. With noisy \textsc{verify} used, the agent performs the one-time optimal certification on the joint experiment (A1 unchanged), incurring expected total regret $T\,r\, B^*(\bar{J}) + c$; relative to the internal-only benchmark $T\,r\, B^*(\bar{P}^t)$, the net value is $V_{\mathrm{verify}} = T\,r\,[B^*(\bar{P}^t) - B^*(\bar{J})] - c$, strictly positive if and only if $c < T\,r\,[B^*(\bar{P}^t) - B^*(\bar{J})]$. Finally, if $Q(\cdot \mid E) = Q(\cdot)$ is the same in every environment, then $\bar{J}_i = \bar{P}_i^t \otimes Q$ for both labels; writing $\bar{p}_j$ for the density of $\bar{P}_j^t$ and $q$ for that of $Q$, the posterior odds of $\Theta = 1$ against $\Theta = 0$ given $(t,y)$ equal $\bar{p}_1(t) q(y) / [\bar{p}_0(t) q(y)] = \bar{p}_1(t)/\bar{p}_0(t)$ whenever $q(y) > 0$---the factor $q(y)$ cancels---so the posterior given $(T_t, Y)$ coincides with the posterior given $T_t$ alone, $B^*(\bar{J}) = B^*(\bar{P}^t)$, the gain is $0$, and $V_{\mathrm{verify}} = -c < 0$ at any $c > 0$: \textsc{verify} should not be performed. This identically-distributed \textsc{verify} is simultaneously the minimal counterexample to (b).

\textbf{(e).} The explicit constructions are deferred to Appendix~\ref{app:static}; we summarize each item and keep the self-contained case \textup{(ii)} in full.

\textup{(i)} \emph{Drop A1 symmetry (construction in Appendix~\ref{app:static:a1}).} There is a completely explicit two-point instance ($|\cE| = 2$, $f$ constant onto a single representation value) in which the fallback rule coincides with the $f$-rule's cell Bayes action under the adequate environment, so that mismatch direction costs $0$, not $r$: ``always fallback, never certify'' has total regret identically $0$ under both environments, yet $\TV(P_0^1, P_1^1) = \tfrac12 < 1$ and \eqref{eq:static:rint} prescribes the strictly positive value $\frac{Tr}{4}$. The formula no longer describes the optimum, and the threshold of (b) collapses with it; the counterexample bites exactly when the internal channel is not perfectly distinguishing. The full construction and its arithmetic are given in Appendix~\ref{app:static:a1}.

\textup{(ii)} \emph{Drop augmentation.} If $Q(\cdot \mid E_0) = Q(\cdot \mid E_1)$, the final paragraph of the proof of (d) gives gain $0$ and $V_{\mathrm{verify}} = -c < 0$ at any $c > 0$: the ``buy'' conclusion of (b) fails even though the closed-form threshold \eqref{eq:static:cstar} is well defined---(b) silently presupposes that \textsc{verify} augments the experiment, and an identically distributed signal violates exactly that presupposition.

\textup{(iii)} \emph{Drop equal label priors (construction in Appendix~\ref{app:static:prior}).} Under a general label prior $\pi$ the pointwise Bayes-error expansion gives the weighted threshold $c^* = T\,r \int \min\bigl(\pi_0\, d\bar{P}_0^t,\, \pi_1\, d\bar{P}_1^t\bigr)$, attained by the weighted likelihood-ratio test: still a closed form, but one in which the prior sits inside the minimum. The equal label prior is exactly the case in which it cancels, $\tfrac12 \int \min(\bar{p}_0, \bar{p}_1)\, d\mu = (1-\TV)/2$, the shape of (a) and (b). That the shape, not merely the value, is genuinely lost is witnessed by a cyclic three-environment instance on which the equal-prior shape prescribes $\tfrac14$ while the true optimum is $\tfrac13 \neq \tfrac14$; the general derivation and the complete cyclic computation are given in Appendix~\ref{app:static:prior}.
\end{proof}

\begin{remark*}[Why cost additivity is not among the assumptions]
A perfect \textsc{verify} exhausts the label's information in a single purchase: conditional on a fully revealing signal, the posterior over $\Theta$ degenerates to a point mass, so any repeat purchase of the same perfect channel has marginal information gain exactly zero, at any price. Consequently, under any nondecreasing cost structure $\gamma(k)$ for $k$ purchases---additive or not---the optimal number of perfect verifications is at most one, and the purchase decision compares the single-purchase cost $\gamma(1)$ against the internal-only benchmark of (a), which is exactly the threshold comparison of (b) with $c = \gamma(1)$. Cost additivity therefore plays no role in Theorem~\ref{thm:verify} and does not appear among the assumptions in (e); a cost structure becomes decision-relevant only in extended models where verification is noisy and repeatable, an open direction connected to the sequential model of Section~\ref{sec:model}.
\end{remark*}

\paragraph{Relation to classical VOI; scoped contribution; numerical check.}
Theorem~\ref{thm:verify} instantiates the with-minus-without information value of \citet{howard1966} and \citet{degroot1970}, differing in the decision object (purchasing verification of one's own representation adequacy), in sharpness (a $(1-\TV)$-weighted closed-form threshold), and in boundary treatment (the TV extremes and the counterexamples of (e) are part of the result). Its claim is limited to the combination ``labeling $+$ channel augmentation $+$ threshold characterization'': the component tools \citep{tsybakov2009,dgl1996,blackwell1951,lecam1964,howard1966,degroot1970} are standard, and (e) marks the validity boundary. All identities and thresholds of this chapter were machine-checkably verified on finite instances (algebraic identities to below $10^{-12}$ error, or by exhaustive pointwise enumeration; for Theorem~\ref{thm:verify}, $c^*$ and the switching decision were checked pointwise in three TV regimes, with indifference at $c = c^*$, and the cyclic-class computations of (e)(iii) and of the remark in Appendix~\ref{app:static:identity} were checked by exact rational arithmetic and exhaustive enumeration over all $2^3$ deterministic label rules).

\subsection{Instantiation: Decision-Theoretic Memory (ESD)}
\label{sec:static:esd}

We instantiate the static framework to the Event--State--Decision memory (ESD), in which maintaining a dialogue memory is state-conditioned decision making rather than retrieval. The instantiation involves no dialogue data and no empirical claims.

\paragraph{Setup.}
Fix an environment. A history has two components: the semantic observation $X \in \cX$ (the semantic unit) and a latent state $Z \in \cZ$, with joint law $p(x,z)$ and loss $\ell((x,z), a) \in [0,1]$; the agent's representation retains only $X$. Write $R^*(X)$ and $R^*(X,Z)$ for the Bayes risks of observing only $X$ and the full pair $(X,Z)$. For each $x \in \cX$, define the \emph{conditional state value} (the normalized regret of the cell $\{x\} \times \cZ$)
\begin{equation}
g(x) := \min_{a \in \cA} \sum_{z} p(z \mid x)\, \ell((x,z),a)
- \sum_{z} p(z \mid x) \min_{a \in \cA} \ell((x,z),a) \;\ge\; 0,
\label{eq:static:gdef}
\end{equation}
the \emph{conflict set} $C := \{x \in \cX : g(x) > 0\}$ (by property (ii) of Theorem~\ref{thm:aliasing}, $g(x) > 0$ if and only if the cell $\{x\} \times \cZ$ contains an action conflict), and the \emph{conflict mass} $\mu := P(X \in C)$.

\begin{proposition}[$\mu$--$g$ conflict decomposition]
\label{prop:mug}
In the setup above,
\begin{equation}
\Delta_{\mathrm{info}} := R^*(X) - R^*(X,Z)
= \sum_{x \in C} p(x)\, g(x)
= \mu \cdot \E\bigl[\, g(X) \,\big|\, X \in C \,\bigr].
\label{eq:static:mug}
\end{equation}
Boundary convention: if $\mu = 0$ ($C = \varnothing$) then $\Delta_{\mathrm{info}} = 0$ and the conditional-expectation term is treated as $0$.
\end{proposition}

\emph{Attribution.} Proposition~\ref{prop:mug} is a special case and instantiation of Theorem~\ref{thm:aliasing}; no novelty is claimed. Its nonnegativity also has a classical source: retaining only $X$ garbles $(X,Z)$ deterministically, so $\Delta_{\mathrm{info}} \ge 0$ follows from Blackwell risk domination \citep{blackwell1951}, and $g(x)$ matches the cell-restricted binary Bayes error of \citet[Theorem~2.1]{dgl1996}.

\begin{proof}[Proof of Proposition~\ref{prop:mug}]
Apply Theorem~\ref{thm:aliasing} with $H = (X,Z)$ and $U = X$: the history space is $\cX \times \cZ$ with law $p(x,z)$, and the representation map is the projection $(x,z) \mapsto x$, whose aliasing cells are exactly $\{x\} \times \cZ$, $x \in \cX$. The cell decomposition \eqref{eq:static:decomp} then reads
\begin{equation}
R^*(X) - R^*(X,Z)
= \sum_{x \in \cX} \Biggl[ \min_{a \in \cA} \sum_{z} p(x,z)\, \ell((x,z),a)
- \sum_{z} p(x,z) \min_{a \in \cA} \ell((x,z),a) \Biggr].
\label{eq:static:cellxz}
\end{equation}
For each $x$ with $p(x) > 0$, factoring $p(x,z) = p(x)\, p(z \mid x)$ shows the bracket equals $p(x)\, g(x)$ with $g(x)$ as in \eqref{eq:static:gdef}; for $x$ with $p(x) = 0$ the bracket vanishes and $p(x) g(x) = 0$ as well. Hence $\Delta_{\mathrm{info}} = \sum_{x \in \cX} p(x) g(x)$. By the definition of $C$, $g(x) = 0$ for $x \notin C$, so the sum restricts to $C$, giving the first equality of \eqref{eq:static:mug}. For the second, if $\mu > 0$ write $p(x) = \mu \cdot p(x \mid C)$ for $x \in C$ and factor out $\mu$: $\sum_{x \in C} p(x) g(x) = \mu \sum_{x \in C} p(x \mid C)\, g(x) = \mu \cdot \E[g(X) \mid X \in C]$. If $\mu = 0$ then every summand vanishes, $\Delta_{\mathrm{info}} = 0$, and the conditional-expectation term is treated as $0$ by convention.
\end{proof}

Equation \eqref{eq:static:mug} separates the \emph{extent} of inadequacy (the conflict mass $\mu$: how often the agent lands on a semantic unit whose cell contains an action conflict) from its \emph{severity} (the conditional expectation $\E[g(X) \mid X \in C]$: the normalized within-cell regret on those units); the certification problem is indifferent to this profile only because $\Adeq_E(f)$ is the binary indicator of $\Delta_{\mathrm{info}} = 0$.

\paragraph{The five-action decision memory; two types of value, kept separate.}
ESD instantiates the action set
\begin{equation}
\cA = \{\textsc{speak},\ \textsc{silence},\ \textsc{ask},\ \textsc{verify},\ \textsc{defer}\},
\label{eq:static:fiveactions}
\end{equation}
with readings tied to \eqref{eq:static:mug}. Off the conflict set ($g(x) = 0$), property (ii) of Theorem~\ref{thm:aliasing} guarantees only that the cell admits a \emph{common Bayes action}---some action, not necessarily \textsc{speak}; in the ESD instance, whatever the common Bayes action is, it is executed by default. Under conflict ($g(x) > 0$), \textsc{silence}/\textsc{ask}/\textsc{defer} postpone or transfer the decision. Two of the five actions are queries, and they answer different questions: \textbf{\textsc{ask}} is the \emph{state-clarification} action, a $V_{\mathrm{state}}$-type query about the current state $Z$; \textbf{\textsc{verify}} is the \emph{adequacy-certification} action, a $V_{\mathrm{cert}}$-type query about the label $\Theta = \Adeq_E(f)$, priced by Theorem~\ref{thm:verify}. The two value types are not conflated, and the connecting chain is exact: $\Delta_{\mathrm{info}}$ is a $V_{\mathrm{state}}$-type measure---the ex ante information value of the state $Z$ for semantic-memory decisions---and its cell-level source $\mu \cdot \E[g(X) \mid X \in C]$ prices the per-period aliasing regret $r$ that an inadequate representation inflicts (the per-period expected cost of the within-cell action conflicts); Theorem~\ref{thm:verify} then takes that $r$ as an input and prices certification through $c^* = (Tr/2)(1-\TV)$, a $V_{\mathrm{cert}}$-type quantity about the label, not the state. Proposition~\ref{prop:mug} is an identity under a known joint law, not a learning result.

\subsection{From One-Shot to Sequential Certification}
\label{sec:static:transition}

The static theory answers a deliberately narrow question: whether to buy one \textsc{verify}, once, at a closed-form threshold $c^*$. Two idealizations bound its reach. First, transcript exogeneity: the laws $\bar{P}_0^t, \bar{P}_1^t$ are fixed independently of the agent's conduct, so the model cannot express that waiting, exploring, or acting differently would change what the internal channel reveals. Second, the one-shot decision: $c^*$ compares two fixed schemes over a fixed horizon and says nothing about \emph{how} to buy a prescribed certification accuracy $\delta$ at minimum task cost---how long to rely on the internal channel as the transcript lengthens, when to stop waiting and purchase verification, and how these choices interact with the task loss accumulating meanwhile. Both idealizations are load-bearing in the proofs above, and both must be relaxed to address the sequential question. Section~\ref{sec:model} takes up exactly this relaxation: it introduces the policy-dependent sequential certification model and proves the matching lower and upper bounds of Theorems~\ref{thm:LB} and~\ref{thm:UB}, for which the present chapter supplies the label, the TV-based distinguishability measure, and the threshold logic that the sequential analysis refines. Adjacent tools for controlled observation and active sequential testing exist \citep{naghshvar2013,krishnamurthy2016}, but they assume exogenous estimable feedback or a known model; the interface to adequacy certification under loss-free transcripts is built in Section~\ref{sec:model}.

% ============================================================
% Section file: The Sequential Certification Model and Main Theorem
% Label: sec:model
% Shared-counter environments used here: theorem (thm:LB, thm:UB),
% corollary (cor:main) -- and nothing else, so that numbering stays
% Theorem 8, Theorem 9, Corollary 10 under the paper-wide convention.
% Requires: algorithm + algorithmic packages for the CTS float.
% ============================================================
\providecommand{\cE}{\mathcal{E}}
\providecommand{\cA}{\mathcal{A}}
\providecommand{\cY}{\mathcal{Y}}
\providecommand{\cF}{\mathcal{F}}
\providecommand{\E}{\mathbb{E}}
\providecommand{\Prob}{\mathbb{P}}
\providecommand{\ind}{\mathbf{1}}
\providecommand{\kl}{\mathrm{kl}}
\providecommand{\Alt}{\mathrm{Alt}}
\providecommand{\LLR}{\mathrm{LLR}}
\providecommand{\TV}{\mathrm{TV}}

\section{The Sequential Certification Model and Main Theorem}\label{sec:model}

This section gives the full definition of the sequential certification model, defines
the certification complexity constant that governs it, states the two main theorems and
their corollary, specifies the matching policy CTS, and records the non-triviality and
boundary statements that delimit the results. Proofs of the lower and upper bounds are
deferred to Appendices~\ref{sec:prooflb} and~\ref{sec:proofub} respectively; the boundary
statements B1--B3 are proved in Appendix~\ref{sec:prooflb}.

\subsection{Model M1: sequential adequacy certification under a fixed representation}
\label{sec:modelM1}

\paragraph{Model M1 (sequential certification).}\setmodelanchor{model:M1}{M1}
The model consists of the following ingredients.
\begin{enumerate}
\item[(M1.1)] \emph{Environments and labels.} A finite environment class
$\cE=\{E_1,\dots,E_K\}$ with $K\ge 2$ is fixed, together with a fixed representation
$f$. Both adequacy labels are assumed present; otherwise the known common
label can be returned at time zero with no certification loss. Each environment $E_i$ carries a binary \emph{adequacy label}
\begin{equation}\label{eq:label}
\Theta_i \;:=\; \ind\big\{\, R^{*}_{E_i}(f(H)) = R^{*}_{E_i}(H) \,\big\} \;\in\;\{0,1\},
\end{equation}
defined through the Bayes-risk framework of the static layer
(Section~\ref{sec:static}): $\Theta_i=1$ if acting optimally on the representation
achieves the Bayes risk of acting on the full history. The labels are known as a
function of the (known) environments; what is unknown to the agent is \emph{which}
environment $E_i$ it is in, and hence the realized label.
\item[(M1.2)] \emph{Actions, channels, losses.} The action set $\cA$ is finite. Each
action $a\in\cA$ in environment $E_i$ carries an observation channel $K_i^{a}$, a
probability distribution on a finite observation space $\cY$, and a task loss
$g_i(a)\ge 0$, finite. Audit costs are folded into the task loss: a \textsc{Verify}
action, when present, has $g_i(\textsc{Verify})$ equal to its audit price $c_{\mathrm{audit}}$.
We write $d(P\|Q)=\sum_{y\in\cY} P(y)\log(P(y)/Q(y))$ for the KL divergence between two
distributions on $\cY$, with the conventions $0\log 0 = 0$ and $d(P\|Q)=+\infty$ if $P$
is not absolutely continuous with respect to $Q$.
\item[(M1.3)] \emph{Interaction protocol.} A policy $\pi$ is a sequence of (possibly
randomized, via an independent seed) decision rules; at round $t$ it selects
$A_t$ as a measurable function of the history
$H_{t-1}=(A_1,Y_1,\dots,A_{t-1},Y_{t-1})$. In environment $E_i$, the observation then
follows the conditional channel law
$\Prob_i(Y_t\in\cdot\mid H_{t-1},A_t)=K_i^{A_t}(\cdot)$
(Assumption~(H1) below). Let $\cF_t=\sigma(A_1,Y_1,\dots,A_t,Y_t)$ be the
natural filtration, augmented by the independent policy seed when needed.
Only the specified feedback is observed; $g_i(A_t)$ is not extra feedback.
The task history in \eqref{eq:label} is the object compressed by $f$;
$H_t$ here is the certifier's retained testing transcript. Each pair $(E_i,\pi)$ induces a probability measure
$\Prob_i^{\pi}$ on interaction sequences, with expectation $\E_i^{\pi}$.
\item[(M1.4)] \emph{Stopping and decision.} The policy commits to a stopping time
$\tau$ with respect to $(\cF_t)$ and a terminal decision
$\hat\Theta\in\{0,1\}$ that is $\cF_{\tau}$-measurable.
\item[(M1.5)] \emph{$\delta$-correctness (fixed confidence), including admissibility.}
For a target confidence $\delta\in(0,\tfrac12)$, a policy $\pi$ is \emph{$\delta$-correct} if
\begin{equation}\label{eq:deltacorrect}
\Prob_i^{\pi}\big(\tau<\infty\big)=1 \ \ \text{for every } i\in\{1,\dots,K\}
\qquad\text{and}\qquad
\max_{i}\, \Prob_i^{\pi}\big(\hat\Theta \neq \Theta_i\big) \;\le\; \delta.
\end{equation}
The first requirement is an admissibility condition: certification must terminate under
every environment, not only under favorable ones. The lower bound uses
this same class, allowing $\E_i^\pi\tau=\infty$; Appendix~\ref{sec:prooflb}
handles that case by truncation.
\item[(M1.6)] \emph{Cost.} The performance measure is the expected cumulative task
loss until certification,
\begin{equation}\label{eq:cost}
R(\pi,E_i) \;:=\; \E_i^{\pi}\Big[\, \sum_{t=1}^{\tau} g_i(A_t) \,\Big],
\qquad
N_a(\tau) \;:=\; \sum_{t\le\tau} \ind\{A_t=a\},
\end{equation}
so that $R(\pi,E_i)=\sum_a g_i(a)\,\E_i^{\pi}[N_a(\tau)]$. We write
$R^{*}_{\delta}(E_i):=\inf\{R(\pi,E_i):\pi\ \delta\text{-correct}\}$.
\item[(M1.7)] \emph{Alternative sets.} For each environment $E_i$, the set of
label-disagreeing alternatives is
\begin{equation}\label{eq:alt}
\Alt(i) \;:=\; \{\, j\in\{1,\dots,K\} : \Theta_j \neq \Theta_i \,\}.
\end{equation}
The adequacy definition supplies the answer map. Same-label environments
are absent from these error constraints, as in general-answer testing.
Their loss profiles can still affect attainability of the individual
complexities; Section~\ref{sec:observable} gives the exact positive-cost
compatibility criterion when some such environments are observationally equal.
\end{enumerate}

\paragraph{Action semantics.}
The abstract action set instantiates the ESD repertoire described in
Section~\ref{sec:intro-motivation}: \textsc{Act} (execute the task policy induced by
$f$; with the specified feedback and task loss), \textsc{Ask} (query the
latent state, at a query fee), \textsc{Verify} (purchase an adequacy audit, at the
audit price), and \textsc{Defer} (wait, at low or zero loss and low or zero
information). The \textsc{Repair} action is excluded from Model~\hyperref[model:M1]{M1}
deliberately: it changes the representation, hence the label itself, and is treated in
Section~\ref{sec:switching}.

\paragraph{Assumptions.}
We collect the standing assumptions, each with a one-line statement of what it buys and
where it is used.

\noindent\textbf{(H1) Conditional independence of the channels.}\label{ass:H1}
At each round, $\Prob_i(Y_t\in\cdot\mid\cF_{t-1},A_t)=K_i^{A_t}(\cdot)$.
This causal channel condition does not assert independence conditional on the
entire adaptively chosen action sequence. This makes the log-likelihood ratio of any two environments a
sum of action-indexed increments and is used by both bounds.

\noindent\textbf{(H2) Feasibility.}\label{ass:H2}
For every $i$ and every $j\in\Alt(i)$ there exists an action $a$ with
$0 < d(K_i^{a}\|K_j^{a}) < +\infty$: every label-disagreeing alternative can be
statistically separated from $E_i$ by some action. The two failure modes are
asymmetric. If some $j\in\Alt(i)$ satisfies $d(K_i^{a}\|K_j^{a})=0$ for \emph{every}
$a$, then no amount of evidence separates $E_i$ from $E_j$ and $\delta$-certification
is impossible; this is the hopeless boundary B1 below, and the LP \eqref{eq:clp} is
exactly infeasible in this case. The opposite failure, $d(K_i^{a}\|K_j^{a})=+\infty$
for the available actions, is harmless---a single observation may then separate the two
environments perfectly. Theorem~\ref{thm:LB} in fact only requires $C_i<\infty$, and
hopelessness is characterized exactly by infeasibility of the LP.

\noindent\textbf{(H2$^+$) Pairwise distinguishability of all environments.}\label{ass:H2plus}
For every pair $i\neq j$ (including same-label pairs) there exists an action $a$ with
$0<d(K_i^{a}\|K_j^{a})$: the true environment itself is identifiable. This is stronger
than (H2), which constrains only label-disagreeing pairs, and it is needed only for the
upper bound, where the maximum-likelihood estimator of the environment must stabilize
exactly. The lower bound does not use it.
Section~\ref{sec:observable} replaces it, for positive losses, by a necessary
and sufficient compatibility condition on observational equivalence classes.

\noindent\textbf{(H3) Losses.}\label{ass:H3}
$0\le g_i(a)<\infty$ for all $i,a$. The degenerate case $g_i(a)=0$ is handled by
continuity of the LP (boundary B2 below).

\noindent\textbf{(H4) Full support.}\label{ass:H4}
$K_i^{a}(y)>0$ for every $i$, $a$, and $y\in\cY$. Consequently all log-likelihood-ratio
increments are uniformly bounded, which the concentration arguments of the upper bound
require. The lower bound does not use it.

\paragraph{Assumption dependence of the results.}
Theorem~\ref{thm:LB} is proved under (H1)--(H3) alone. Theorem~\ref{thm:UB} uses the
full set (H1), (H2), (H2$^+$), (H3), (H4). Corollary~\ref{cor:main} holds under the
union. We flag the dependence at each statement rather than absorbing it into a global
assumption block, because the two bounds genuinely differ in what they need.

\subsection{The certification complexity}\label{sec:complexity}

\paragraph{Definition (certification complexity).}
For each environment $E_i$, define the \emph{certification complexity} $C_i=C(E_i,f)$
by the covering linear program
\begin{equation}\label{eq:clp}
C_i \;:=\; \min\Big\{\, \sum_{a\in\cA} g_i(a)\, n_a \;:\;
n_a \ge 0 \ \ \forall a,\quad
\sum_{a\in\cA} n_a\, d\big(K_i^{a}\,\|\,K_j^{a}\big) \;\ge\; 1
\ \ \ \forall\, j\in\Alt(i) \,\Big\},
\end{equation}
and by the cost-ratio characteristic-time representation
\begin{equation}\label{eq:cmm}
C_i \;=\; \Bigg[\,\sup_{q\in\Delta(\cA)}\
\frac{\inf_{j\in\Alt(i)}\sum_{a\in\cA}q_a\,d\big(K_i^{a}\,\|\,K_j^{a}\big)}
{\sum_{a\in\cA}q_a g_i(a)}\,\Bigg]^{-1},
\end{equation}
where $\Delta(\cA)$ is the probability simplex over actions. The allocation $q^*$ in
this display is the one tracked by CTS. There is also a value-equivalent
reparameterization: for positive costs, set
\[
v_a=\frac{q_a g_i(a)}{\sum_b q_b g_i(b)},\qquad
\Psi_i(v):=\inf_{j\in\Alt(i)}\sum_a\frac{v_a}{g_i(a)}d(K_i^a\|K_j^a).
\]
Then $C_i=[\sup_{v\in\Delta(\cA)}\Psi_i(v)]^{-1}$, but $v$ and $q$ are different
coordinates. In general $\arg\max\Psi_i$ is not the cost-optimal allocation $q^*$;
the inverse map is $q_a\propto v_a/g_i(a)$. Zero-cost actions and infinite KL
terms use the perturbation/truncation conventions stated in Appendix~\ref{sec:prooflb}.
Conventions: if the LP
\eqref{eq:clp} is infeasible we set $C_i=+\infty$ (boundary B1); if its value is $0$
we set $C_i=0$ (boundary B2).

\paragraph{Why the two value representations agree (proof pointer).}
Feasibility under (H2) makes the LP bounded and feasible, so strong duality gives the
dual value $\max_\lambda \sum_{j\in\Alt(i)}\lambda_j$ subject to
$\sum_{j}\lambda_j\, d_a^{j}\le g_i(a)$ for all $a$, where we abbreviate
$d_a^{j}:=d(K_i^{a}\|K_j^{a})$. Writing $M$ for the bracketed sup-inf in
the $v$-reparameterized display above, Sion's minimax theorem---both $\Delta(\cA)$ and
$\Delta(\Alt(i))$ are compact and convex and the objective is bilinear and continuous---gives
\[
M\;=\;\inf_{\mu\in\Delta(\Alt(i))}\;\max_{a}\;\Big(\sum_j \mu_j d_a^{j}\Big)\Big/g_i(a),
\]
and a direct comparison of constraints shows the dual value equals $1/M$. The full
verification is recorded in Appendix~\ref{sec:prooflb}; the two value representations
have also been checked to agree on randomly generated finite instances by the
finite vertex-enumeration LP solver and grid-based evaluation. This is an equality
of optimal values, not an
assertion that the two displayed optimizers are the same allocation.

\paragraph{Two-point special case (bang-bang; boundary B3).}
If $|\Alt(i)|=1$, say $\Alt(i)=\{j\}$, the infimum in \eqref{eq:cmm} collapses and both
forms reduce to
\begin{equation}\label{eq:bangbang}
C_i \;=\; \min_{a\in\cA}\; \frac{g_i(a)}{d\big(K_i^{a}\,\|\,K_j^{a}\big)},
\end{equation}
so the optimal design concentrates on the single action with the best loss-per-nat
ratio. This has the binary controlled-experiment form of
\citet[Theorem~5.1]{nitinawarat2015controlled}. With two opposite-label
environments and a common per-action fee $g_i(a)=c(a)$, the model is the
corresponding known special case. We treat this consistency check as a boundary requirement (B3), not as a
contribution.

\paragraph{Comparison with cost-weighted characteristic times.}
The form \eqref{eq:cmm} has the shape of the cost-aware characteristic time of
\citet{kanarios2024cost} and the classic characteristic time of
\citet{garivier2016optimal}, and we inherited the shape from those templates. The adequacy-defined answer and hypothesis-dependent losses specify the
application, but the proof uses only the resulting finite table, as
Proposition~\ref{prop:table-reduction} makes explicit. Regret-priced
fixed-confidence identification is also established \citep{yang2026minimal}.
There is one coordinate distinction that must
not be hidden. Equation~\eqref{eq:cmm} uses a sampling allocation $q$ and divides its
information rate by its task-loss rate. The display immediately following it uses the
loss-normalized coordinate $v_a=q_a g_i(a)/(\sum_b q_b g_i(b))$ and the payoff
$\Psi_i(v)=\inf_j\sum_a(v_a/g_i(a))d(K_i^a\|K_j^a)$. The two displays have the same
optimal value, but $q$ and $v$ are different simplex coordinates; an optimizer in the
$v$ coordinate must be mapped back by $q_a\propto v_a/g_i(a)$. Treating $v$ itself as
the sampling allocation can produce a strictly larger realized loss ratio (the
$2\times2$ diagnostic in Appendix~\ref{sec:proofub} gives $25/6\approx4.17$ instead of
$C_i=30/11\approx2.73$). CTS therefore tracks the cost-ratio allocation
\eqref{eq:wratio}, equivalently the normalized LP solution.

\subsection{Main results}\label{sec:mainresults}

\begin{theorem}[Environment-wise lower bound]\label{thm:LB}
Suppose \textnormal{(H1)--(H3)} hold. Then for every $\delta\in(0,\tfrac12)$, every
$\delta$-correct policy $\pi$, and every environment
$E_i$,
\begin{equation}\label{eq:lb}
R(\pi,E_i) \;\ge\; C_i\; \kl\big(\delta,\,1-\delta\big)
\;\ge\; C_i\; \log\frac{1}{2.4\,\delta},
\end{equation}
where $\kl(p,q)=p\log(p/q)+(1-p)\log((1-p)/(1-q))$ is the Bernoulli KL divergence and
$C_i=C(E_i,f)$ is the certification complexity of Section~\ref{sec:complexity}. In
particular,
\begin{equation}\label{eq:lbliminf}
\liminf_{\delta\to 0}\;
\frac{R^{*}_{\delta}(E_i)}{\log(1/\delta)} \;\ge\; C_i .
\end{equation}
\end{theorem}

Only almost-sure termination in every environment is required, as in
\eqref{eq:deltacorrect}; $\E_i^\pi\tau$ may be infinite. The loss and the
inequality are interpreted in $[0,\infty]$. The proof first uses bounded
stopping times and then truncation and monotone convergence
(Appendix~\ref{subsec:lb:lpscale}).

The proof (Appendix~\ref{sec:prooflb}) combines the transportation inequality of
\citet{kaufmann2016complexity}---applied with actions in the role of arms and channel
outputs in the role of arm observations---with a change of measure against each
$j\in\Alt(i)$ on the event $\{\hat\Theta=\Theta_i\}$, followed by the scale
reparametrization $n_a=\E_i^{\pi}[N_a(\tau)]/L$, which turns the information
constraints into exactly the LP \eqref{eq:clp} with value $C_i\,\kl(\delta,1-\delta)$.
The second inequality in \eqref{eq:lb} is the standard Bernoulli-KL bound
$\kl(\delta,1-\delta)\ge\log(1/(2.4\delta))$ of \citet{kaufmann2016complexity}.
The technical route is standard and is credited as such; the content of the theorem is
that the certification currency $g$ and the label-defined alternatives $\Alt(i)$ pass
through the argument unchanged, so the classical constant is replaced by $C(E_i,f)$
with its decision-theoretic semantics.

\begin{theorem}[Matching upper bound: CTS]\label{thm:UB}
Suppose \textnormal{(H1)}, \textnormal{(H2)}, \textnormal{(H2$^+$)},
\textnormal{(H3)} and \textnormal{(H4)} hold, and let
$\pi_\delta^{\mathrm{CTS}}$ be the policy family explicitly selected in
\eqref{eq:cts-family} and Algorithm~\ref{alg:cts}, with stopping threshold
$\beta(t,\delta)=\log(2t(K-1)/\delta)$. Then:
\begin{enumerate}
\item[(a)] $\pi_\delta^{\mathrm{CTS}}$ is $\delta$-correct in the sense of \eqref{eq:deltacorrect};
moreover, the stopping and decision rules alone guarantee
$\max_i \Prob_i(\hat\Theta\neq\Theta_i,\ \tau<\infty)\le\delta$ under \emph{any} sampling rule
satisfying \textnormal{(H1)}. This last statement alone does not guarantee
termination for an arbitrary sampling rule;
\item[(b)] for every environment $E_i$,
\begin{equation}\label{eq:ub}
\limsup_{\delta\to 0}\; \frac{R(\pi_\delta^{\mathrm{CTS}},E_i)}{\log(1/\delta)} \;\le\; C_i .
\end{equation}
\end{enumerate}
\end{theorem}

The proof (Appendix~\ref{sec:proofub}) adapts the Track-and-Stop template of
\citet{garivier2016optimal} with the cost-weighted plug-in design of
Section~\ref{sec:cts}: part (a) is a mixture-martingale argument over the $K-1$
challengers and a geometric time grid; part (b) combines almost-sure stabilization of
the environment estimator (which is where (H2$^+$) and the forced exploration are
used), tracking of the plug-in allocation, and the LP optimality identity
$\big(\sum_a g_i(a)w^{*}_a\big)/\big(\min_j\sum_a w^{*}_a d_a^{j}\big)=C_i$ at the
cost-ratio optimizer $w^{*}$. If the known cost table contains a zero,
part (b) uses the vanishing-perturbation policy of
Theorem~\ref{thm:degenerate}, rather than assigning a sampling rule to
a zero denominator by continuity.

\begin{corollary}[Limit of the optimal certification cost]\label{cor:main}
Under the union of the assumptions of Theorems~\ref{thm:LB} and~\ref{thm:UB}, namely
\textnormal{(H1)}, \textnormal{(H2)}, \textnormal{(H2$^+$)}, \textnormal{(H3)},
\textnormal{(H4)}, the optimal certification cost satisfies, for every environment
$E_i$,
\begin{equation}\label{eq:main}
\lim_{\delta\to 0}\; \frac{R^{*}_{\delta}(E_i)}{\log(1/\delta)} \;=\; C(E_i,f),
\end{equation}
and the family $\pi_\delta^{\mathrm{CTS}}$ attains the limit. If $C(E_i,f)=0$
(boundary B2), the stronger pointwise conclusion is
$R^{*}_{\delta}(E_i)=0$ for each fixed $\delta\in(0,1/2)$
(Proposition~\ref{prop:zero-infimum}); this does not assert that CTS has
zero loss. The case $C(E_i,f)=+\infty$ is excluded
by \textnormal{(H2)} and corresponds to the hopeless boundary B1, where no
$\delta$-correct policy with finite expected loss exists for small $\delta$.
\end{corollary}

The first-order limit in Corollary~\ref{cor:main} follows directly from
\eqref{eq:lbliminf} and \eqref{eq:ub}. We emphasize the reading given in Section~\ref{sec:intro-preview}: the
price of certifying representation adequacy to confidence $\delta$ is
$C(E,f)\log(1/\delta)+o(\log(1/\delta))$, with $C(E,f)$ computable as a finite LP from
the model primitives---and the statement is exactly as strong as its assumptions, which
the boundaries B1--B3 below show to be individually sharp in the senses stated there.

\subsection{The CTS strategy}\label{sec:cts}

We now define the Certification Track-and-Stop family completely. Set
$L=\log(1/\delta)$ and $\eta_\delta=1/\log(\max\{e,L\})$, and use
\begin{equation}\label{eq:cts-family}
\pi_\delta^{\mathrm{CTS}}=
\begin{cases}
 \pi^*,& \min_{i,a}g_i(a)>0,\\
 \pi^*_{\eta_\delta},& \min_{i,a}g_i(a)=0.
\end{cases}
\end{equation}
The choice depends on the entire known table, not on the unknown true
environment. In the second branch every candidate's design costs are
$\widetilde g_i(a)=g_i(a)+\eta_\delta$; in the first they are
$\widetilde g_i(a)=g_i(a)$. Actual loss is always charged using $g_i$.
For $L>e$ the schedule is exactly $\eta_\delta=1/\log L$ used in the
appendix proof; the cap defines a positive value for all allowed $\delta$.
The selected policy
has four components: an environment estimator, a plug-in allocation rule, a tracking
rule with forced exploration, and a stopping rule with terminal decision. Throughout,
$d_a^{j}(\hat\imath):=d(K_{\hat\imath}^{a}\|K_j^{a})$ denotes the plug-in divergence
profile at the estimated environment.

\paragraph{Estimator (maximum likelihood / KL projection).}
At time $t$, CTS computes
\begin{equation}\label{eq:mle}
\hat\imath_t \;\in\; \arg\max_{i\in\{1,\dots,K\}}\;
\sum_{s\le t} \log K_i^{A_s}(Y_s),
\end{equation}
the environment whose induced channel law maximizes the likelihood of the realized
transcript (equivalently, the KL projection of the empirical action-conditional
observation histograms onto the model class); ties are broken by a fixed priority rule.
Under (H1), (H2$^+$) and the forced exploration below, $\hat\imath_t$ equals the true
environment eventually almost surely (Appendix~\ref{sec:proofub}); since $\cE$ is
finite, this is the only stability property the plug-in step needs.

\paragraph{Plug-in allocation (cost-ratio optimizer).}
Given $\hat\imath_t=\hat\imath$, CTS solves
\begin{equation}\label{eq:wratio}
w^{*}(\hat\imath) \;\in\; \arg\sup_{w\in\Delta(\cA)}\;\;
\inf_{j:\,\Theta_j\neq\Theta_{\hat\imath}}\;
\frac{\displaystyle\sum_{a\in\cA} w_a\, d\big(K_{\hat\imath}^{a}\,\|\,K_j^{a}\big)}
{\displaystyle\sum_{a\in\cA} \widetilde g_{\hat\imath}(a)\, w_a},
\end{equation}
where $w^*(\hat\imath)$ denotes the selected design, including its
$\delta$ dependence when costs are perturbed. Two equivalent readings
are useful: it is the normalized LP
solution $n^{*}/\sum_a n^{*}_a$ of \eqref{eq:clp} evaluated at $\hat\imath$
with costs $\widetilde g$; and its
value is the reciprocal of the plug-in loss constant
$\big(\sum_a\widetilde g\,w^{*}_a\big)/\big(\inf_j\sum_a w^{*}_a d_a^{j}\big)
=C_{\hat\imath}(\widetilde g)$.
As recorded in Section~\ref{sec:complexity}, the allocation in \eqref{eq:wratio} is
the cost-ratio coordinate $q$. The loss-normalized display above is value-equivalent
only after its optimizer is mapped back by $q_a\propto v_a/\widetilde g_i(a)$; using the $v$
coordinates directly as sampling frequencies is the source of the $2\times2$
diagnostic in Appendix~\ref{sec:proofub}. The denominator in the implemented
design is strictly positive in both branches. Continuity is used to
analyze the LP value as $\eta_\delta\to0$, not to define an algorithm.

\paragraph{Sampling rule (C-tracking with forced exploration).}
Writing $N_a(t)=\sum_{s\le t}\ind\{A_s=a\}$, CTS plays
\begin{equation}\label{eq:ctrack}
A_{t+1} \;\in\; \arg\max_{a\in\cA}\;
\Big[\, w^{*}_a\big(\hat\imath_t\big) - \frac{N_a(t)}{t} \,\Big],
\end{equation}
subject to the forced-exploration override: if after the choice some action would
violate
\begin{equation}\label{eq:forced}
N_a(t+1) \;\ge\; \sqrt{t+1+|\cA|^{2}} - 2|\cA| ,
\end{equation}
the most under-sampled such action is played instead. The tracking rule
\eqref{eq:ctrack} together with \eqref{eq:forced} is the template of
\citet[Lemma~7]{garivier2016optimal}: it guarantees $N_a(t)/t\to w^{*}_a$ whenever the
tracked allocation stabilizes, while the $\sqrt{t}$-rate exploration makes the
estimator \eqref{eq:mle} consistent and contributes only $o(\log(1/\delta))$ to the
loss (Appendix~\ref{sec:proofub}). Forced exploration can leave a
positive residual loss in B2; it does not characterize the pointwise infimum.

\paragraph{Stopping rule and terminal decision (GLR).}
For each ordered pair of environments $(i,j)$ define the log-likelihood ratio
\begin{equation}\label{eq:llr}
\LLR_{i,j}(t) \;:=\; \sum_{s\le t}
\log \frac{K_i^{A_s}(Y_s)}{K_j^{A_s}(Y_s)},
\end{equation}
which is well-defined and has uniformly bounded increments under (H4). CTS stops at
\begin{equation}\label{eq:glrstop}
\tau \;:=\; \inf\Big\{\, t \;:\;
\max_{i}\; \min_{j:\,\Theta_j\neq\Theta_i} \LLR_{i,j}(t) \;>\; \beta(t,\delta) \,\Big\},
\qquad
\beta(t,\delta) \;:=\; \log\frac{2t\,(K-1)}{\delta},
\end{equation}
and outputs
\begin{equation}\label{eq:decision}
\hat\Theta \;:=\; \Theta_{\hat\imath_\tau}.
\end{equation}
The threshold \eqref{eq:glrstop} is of the \citet[Theorem~10]{garivier2016optimal}
type: it absorbs a union bound over the $K-1$ challengers and a geometric time grid,
which is what makes the $\delta$-correctness of Theorem~\ref{thm:UB}(a) independent of
the sampling rule. The threshold is conservative relative to mixture thresholds
\citep{kaufmann2021mixture}; we keep it because the matching first-order constant is
already delivered and the tightness of second-order terms is left to the open problems
of Section~\ref{sec:switching}.

\begin{algorithm}[t]
\caption{Certification Track-and-Stop, $\pi_\delta^{\mathrm{CTS}}$}\label{alg:cts}
\begin{algorithmic}[1]
\REQUIRE environment class $\cE=\{E_1,\dots,E_K\}$ with labels
$(\Theta_i)$, channels $(K_i^{a})$, losses $(g_i(a))$; confidence $\delta\in(0,\tfrac12)$
\STATE select the branch in \eqref{eq:cts-family}: use $\widetilde g_i=g_i$ if all
entries are positive, and $\widetilde g_i=g_i+\eta_\delta$ otherwise
\STATE initialize $N_a(0)\gets 0$ for all $a$; play each action once (in an arbitrary
order) to initialize \eqref{eq:mle}
\FOR{$t=|\cA|,|\cA|+1,\dots$}
\STATE compute the estimator $\hat\imath_t$ by \eqref{eq:mle}
\STATE solve \eqref{eq:wratio} using the selected costs $\widetilde g$
\IF{$\max_i \min_{j:\Theta_j\neq\Theta_i}\LLR_{i,j}(t) >
\beta(t,\delta)$ with $\beta$ as in \eqref{eq:glrstop}}
\STATE set $\tau\gets t$; \textbf{return} $\hat\Theta=\Theta_{\hat\imath_\tau}$
\ELSE
\STATE choose $A_{t+1}$ by C-tracking \eqref{eq:ctrack}, overridden by forced
exploration \eqref{eq:forced} if violated; observe $Y_{t+1}\sim K^{A_{t+1}}$
\ENDIF
\ENDFOR
\end{algorithmic}
\end{algorithm}

\subsection{Non-triviality and boundary behavior}\label{sec:nontrivial}

We record a cost-sensitivity diagnostic (NT1), a model-scope example (NT2),
and the boundary instances B1--B3. These are not claimed as separations from
general cost-aware controlled testing. B1--B3 are proved in Appendix~\ref{sec:prooflb}; all numerical
claims below were checked on finite instances by finite vertex-enumeration LP
solves, exhaustive enumeration, or direct evaluation (agreement to floating-point
solver precision).

\paragraph{NT1: equal information, different losses, different optimal policies.}
There exist instances with two actions $a_1,a_2$ such that
$K_i^{a_1}=K_i^{a_2}$ for \emph{every} environment $E_i$---the two actions are
statistically indistinguishable under all hypotheses---yet $g_i(a_1)\neq g_i(a_2)$.
Since $C_i$ depends on the loss profile $g_i$ through \eqref{eq:clp}, the optimal
allocations differ across the two actions, and flipping the loss assignment flips which
action the LP selects. Any criterion built from observation kernels alone (a pure
Chernoff-type information rate) assigns the two actions identical value and therefore
cannot recover the optimal design. A concrete instance exhibiting the swap was
constructed and checked numerically. NT1 shows why an information-only design need not minimize a specified
cost. A cost-aware controlled test also responds to this change; the example
does not separate M1 from that literature.

\paragraph{NT2: policy-dependent regime switching is outside the static formula.}
Model~\hyperref[model:M1]{M1} assumes each action carries a fixed channel $K_i^{a}$.
The example of Section~\ref{sec:nt2example} changes the action regime after a
certification outcome: the subsequent fallback transcript and loss are different
from the pre-certification regime. That policy-dependent transition is outside M1,
so the fixed-kernel constant $C(E,f)$ is not a characterization of the enlarged
switching class. The example uses a perfect audit; if that audit is included in the
fixed-kernel menu, its opposite-label KL is infinite and the corresponding static
certification coefficient is already zero. Thus the example must not be presented
as a contradiction to a positive same-menu M1 coefficient. Removing the perfect
audit, or replacing it with a fixed noisy audit, defines a different information
experiment and requires a separate comparison. The valid conclusion is a scope
boundary: Theorems~\ref{thm:LB} and~\ref{thm:UB} apply to the fixed-kernel policy
class, while the regime-switching class is left open.

\paragraph{B1: hopeless instances ($C_i=+\infty$).}
If for some $j\in\Alt(i)$ every action satisfies $d(K_i^{a}\|K_j^{a})=0$, then the LP
\eqref{eq:clp} is infeasible, $C_i=+\infty$, and $\delta$-certification of the label
at $E_i$ is impossible for any $\delta<1/2$: environments $E_i$ and $E_j$ induce
identical transcript laws under every policy. Infeasibility of the LP is exactly the
hopeless case, in the same sense in which partial-monitoring games with no informative
feedback structure are hopeless \citep{bartok2014partial}: the analogy is structural
(identical observation laws under a label disagreement), and we record it as such
rather than as a reduction.

\paragraph{B2: free rides ($C_i=0$).}
Under the finite full-support assumptions, $C_i=0$ exactly when the
zero-loss actions collectively distinguish $i$ from every opposite-label
alternative. A single action $a^\circ$ with $g_i(a^\circ)=0$ and
$d(K_i^{a^\circ}\|K_j^{a^\circ})>0$ for every $j\in\Alt(i)$ is a
sufficient special case. Proposition~\ref{prop:zero-infimum} proves
$R_\delta^*(E_i)=0$ for every fixed $\delta$: a long finite free probe,
followed when necessary by a fresh finite fallback test, has target loss
arbitrarily close to zero while remaining correct and terminating in all
environments. Whether the infimum is attained is a different question,
answered by Proposition~\ref{prop:zero-attainment}. The common CTS family
has the weaker first-order guarantee $R=o(\log(1/\delta))$ here; positive
cost from its forced exploration is a property of that policy, not a
lower bound on the pointwise optimum. Free evidence also occurs in
regret models with zero-gap actions, so the phenomenon itself is not a
novelty claim.

\paragraph{B3: two-point degeneration ($|\cE|=2$).}
With two environments, $\Alt(i)$ is a singleton for each $i$ and both forms of $C_i$
reduce to the bang-bang ratio \eqref{eq:bangbang}. Theorems~\ref{thm:LB}
and~\ref{thm:UB} then recover the binary controlled-testing constants of
\citet[Theorem~5.1]{nitinawarat2015controlled} (when task loss is a common per-action
control cost), so the results correctly contain the known special case rather than
competing with it. This boundary also fixes the sense in which the composite
environment class matters: with $K\ge 3$ and shared labels, the max-min form
\eqref{eq:cmm} is genuinely a game, and single-action (bang-bang) allocations are
generically suboptimal.

% Major revision: separates the generic testing reduction from cost compatibility.
\section{Observable Classes and Compatibility of Optimal Designs}
\label{sec:observable}

The adequacy definition gives the terminal answer its decision-theoretic
meaning. It does not, by itself, change the controlled-testing proof. We first
make that reduction explicit, then identify what can fail when an unknown
environment affects task loss without affecting observable evidence.

\begin{proposition}[Finite-table reduction]\label{prop:table-reduction}
Fix the representation and the finite class in Model~\hyperref[model:M1]{M1}.
Its sequential decision problem is completely determined by
\[
  \bigl(\Theta_i,(K_i^a)_{a\in\cA},(g_i(a))_{a\in\cA}\bigr)_{i=1}^K.
\]
Replacing the representation problem by controlled testing with this same
binary answer map, channels, and hypothesis-dependent costs preserves every
policy's admissibility, error probabilities, and expected costs. In particular,
the LP, the lower bound, and the CTS proof use the representation only through
the supplied labels; the same arguments apply to any supplied binary labels.
\end{proposition}

\begin{proof}
Use the same randomized action rules in the two descriptions. Conditional on
the past and the selected action, the next observation has the same law
$K_i^a$ in each description. Induction therefore gives identical finite
transcript laws under each $i$, including the policy's independent seed.
Applying the same stopping and decision rules preserves their distribution,
and hence admissibility and error. The action counts have the same law and
are charged the same $g_i(a)$, so Tonelli's theorem gives identical expected
costs. In the lower-bound proof the answer appears only in
$\{\hat\Theta=\Theta_i\}$ and $\Alt(i)$; in the upper-bound proof it appears
only in the GLR comparison and design constraints. Neither proof uses a
further identity involving $R_E^*$ or $f$.
\end{proof}

This is a reduction to an abstract problem, not a claim that one particular
published theorem covers every assumption here. It does rule out attributing
new proof machinery to the Bayes-risk label. Model~M1 supplies a loss table;
it imposes no additional equation connecting that table to the Bayes-risk
gap. When a concrete task is used, that connection must be checked separately.
The construction below derives both the labels and the action losses from
one fixed loss function.

\subsection{Label identifiability is weaker than cost-optimality compatibility}
\label{sec:observable:compatibility}

Write $i\sim k$ when $K_i^a=K_k^a$ for every action $a$, and let $\mathcal B$
be the partition into these observational equivalence classes. Under (H2),
every class $B\in\mathcal B$ has one common label: an opposite-label pair in
the same class would make certification impossible. Nevertheless, members of
$B$ can have different loss profiles. Because losses are not additional
feedback in M1, a policy cannot learn which of these profiles is realized.

For $i\in B$, define the common information-feasible polyhedron
\[
  F_B=\left\{n\in\mathbb R_+^{\cA}:
       \sum_a n_a d(K_i^a\|K_j^a)\ge1
       \quad\text{for every }j\in\Alt(i)\right\}.
\]
This does not depend on the representative $i$. Put
\begin{equation}\label{eq:class-gap}
  J_B=\min_{n\in F_B}\sum_{i\in B}\sum_a g_i(a)n_a,
  \qquad \Gamma_B=J_B-\sum_{i\in B}C_i\ge0.
\end{equation}
The sum is a device for checking simultaneous attainability, not a new
Bayesian or worst-case definition of $R_\delta^*(E_i)$.

\begin{theorem}[Compatibility within observable classes]
\label{thm:observable-compatibility}
Assume (H1), (H2), (H4), both labels are present, and $0<g_i(a)<\infty$ for every
$i,a$. Do not assume (H2$^+$). Then:
\begin{enumerate}
\item For every $\delta$-correct policy and every $B\in\mathcal B$,
\begin{equation}\label{eq:class-lower}
  \sum_{i\in B}R(\pi,E_i)
  \ge J_B\,\kl(\delta,1-\delta).
\end{equation}
\item There exists a single family $(\pi_\delta)$ of $\delta$-correct policies satisfying
\[
  \lim_{\delta\to0}\frac{R(\pi_\delta,E_i)}{\log(1/\delta)}=C_i
  \quad\text{for every }i
\]
if and only if $\Gamma_B=0$ for every $B$. Equivalently, within every
observable class the covering LPs have a common minimizing count vector:
\begin{equation}\label{eq:common-optimum}
  \bigcap_{i\in B}\arg\min_{n\in F_B}\sum_a g_i(a)n_a
  \ne\varnothing.
\end{equation}
When this holds, CTS on the observable quotient with class cost
$\bar g_B(a)=\sum_{i\in B}g_i(a)$ attains all the constants simultaneously.
\end{enumerate}
\end{theorem}

\begin{proof}
All members of $B$ induce the same transcript law under any fixed policy,
including the law of the stopped action counts. If their expected losses
are finite, strict positivity of the costs implies finite expected stopping
time. Thus $x_a=\E_i[N_a(\tau)]$ is finite and independent of $i\in B$.
The transportation inequality against every opposite-label $j$ gives
$\sum_a x_a d(K_i^a\|K_j^a)\ge L_\delta$, where
$L_\delta=\kl(\delta,1-\delta)>0$. Hence $x/L_\delta\in F_B$.
Summing the costs and minimizing over $F_B$ proves
\eqref{eq:class-lower}. If a cost is infinite, that inequality is immediate.

Each objective on $F_B$ is at least $C_i$, so $J_B\ge\sum_{i\in B}C_i$.
The LP defining $J_B$ has a minimizer: it is feasible, and positive costs
make its bounded sublevel sets compact. Equality holds precisely when a
minimizer attains every $C_i$, proving the equivalence with
\eqref{eq:common-optimum}. If one policy family attains every $C_i$, divide
\eqref{eq:class-lower} by $\log(1/\delta)$ and take limits to obtain
$\sum_{i\in B}C_i\ge J_B$, proving necessity.

For sufficiency, replace each class by one hypothesis with its common
channels and label, and give it cost $\bar g_B$. Distinct quotient hypotheses
are pairwise distinguishable by definition; they satisfy the full-support,
positive-cost assumptions of Theorem~\ref{thm:UB}. Duplicate alternative
constraints do not change $F_B$, so the quotient's complexity is $J_B$.
Proposition~\ref{prop:table-reduction} allows application of the finite-label
CTS theorem without a new representation interpretation of $\bar g_B$.
Lifting that policy to the original class preserves correctness and gives
\[
  \limsup_{\delta\to0}
  \frac{\sum_{i\in B}R(\pi_\delta,E_i)}{\log(1/\delta)}
  \le J_B=\sum_{i\in B}C_i.
\]
The original environment-wise lower bound gives a liminf of at least $C_i$
for each summand. Since $B$ is finite, subtracting these lower bounds from
the sum proves the matching limsup for each individual environment.
\end{proof}

Theorem~\ref{thm:observable-compatibility} separates two requirements.
Opposite-label distinguishability makes the answer learnable. A common
optimal design within each observable class makes all the environment-wise
prices attainable by one policy. Assumption (H2$^+$) is a sufficient way to
ensure the second requirement, because it makes every class a singleton;
it is not necessary. Conversely, a positive $\Gamma_B$ prevents simultaneous
attainment even though the label is identifiable. This concerns the order
of quantifiers: separate pointwise infima over policies need not be realized
by the same policy family. The compatibility theorem is restricted to
positive finite losses. The boundary results below address pointwise zero
infima and exact zero-cost attainment separately; they do not extend its
full asymptotic criterion to arbitrary nonnegative losses.

\subsection{An adequacy instance with an exact factor-two obstruction}
\label{sec:observable:example}

Let $H\in\{0,1\}$ be uniform, let $f(H)$ be constant, and let the two task
actions be $a_1,a_2$. Fix $0<\eta\le1/4$. One loss function on hidden task
states $S=(i,H)$ is specified by the following rows, with $E_i$ fixing the
first coordinate of $S$:
\[
\begin{array}{c|cc}
  S & \ell(S,a_1)&\ell(S,a_2)\\\hline
  (1,0),(1,1)&\eta&3\eta\\
  (2,0),(2,1)&3\eta&\eta\\
  (3,0)&0&4\eta\\
  (3,1)&4\eta&0
\end{array}
\]
Consequently, the full-history and compressed Bayes risks are both $\eta$
in $E_1,E_2$, whereas they are $0$ and $2\eta$ in $E_3$. Thus
$(\Theta_1,\Theta_2,\Theta_3)=(1,1,0)$ and the expected task-loss table is
\[
  g_1=(\eta,3\eta),\qquad
  g_2=(3\eta,\eta),\qquad g_3=(2\eta,2\eta).
\]
For either selected action, supply independent feedback with law
$\mathrm{Bernoulli}(1/4)$ in $E_1,E_2$ and $\mathrm{Bernoulli}(3/4)$ in
$E_3$. The feedback is independent of the fresh task history; neither the
hidden state nor the realized loss is revealed to the certifier. This is
the observation protocol in M1. It derives the adequacy labels and the
charged losses from the same environment and loss function, rather than
assigning unrelated labels to a cost table.

Both directed Bernoulli divergences equal $d=\tfrac12\log3>0$. The only
non-singleton observable class is $B=\{1,2\}$, for which
\[
  F_B=\{n\ge0:n_1+n_2\ge1/d\},\quad
  C_1=C_2=\eta/d,\quad C_3=2\eta/d,\quad
  J_B=4\eta/d,\quad \Gamma_B=2\eta/d.
\]
The minimizing counts are $(1/d,0)$ for $E_1$ and $(0,1/d)$ for $E_2$.
Their intersection is empty, although certifying the shared label against
$E_3$ is straightforward. Equation~\eqref{eq:class-lower} gives, for every
$\delta$-correct policy,
\[
  \max_{i=1,2}R(\pi,E_i)\ge
  \frac{2\eta}{d}\,\kl(\delta,1-\delta).
\]
This factor two relative to $C_1=C_2$ is exact to first order. Alternate
$a_1,a_2$ and use the two-hypothesis GLR stopping rule on the observable
quotient. Since both actions have the same channels, the GLR stopping time has
the same law under any sampling rule, including unit-cost CTS. Its
bound therefore gives $\E_i\tau/\log(1/\delta)\to1/d$. Alternation makes each of
the first two task losses $2\eta\E_i\tau+O(\eta)$, attaining the displayed
bound asymptotically. Always using $a_1$, or always using $a_2$, instead
attains the pointwise constant for $E_1$, or for $E_2$, respectively;
both policies remain correct in every environment.

With a hypothesis-independent fee $g_i(a)=c(a)$, the objectives within an
observable class coincide, so $\Gamma_B=0$ automatically. In the usual bandit
feedback model, equal reward laws also force equal means and regret gaps.
Here the unobserved task losses can differ despite equal feedback laws.
That is the specific cost--observation separation used by this obstruction.
It is not a property exclusive to adequacy labels: the theorem also applies
to other finite-label tests with hidden, hypothesis-dependent losses.

\subsection{Free evidence: a zero infimum need not be a zero-cost policy}
\label{sec:free-evidence}

The pointwise optimization in Section~\ref{sec:model} permits a different globally
correct policy for each environment whose risk is being minimized. This
quantifier matters even at a fixed confidence level. Write
$Z_i=\{a:g_i(a)=0\}$ for the actions that are free in environment $i$.

\begin{proposition}[Zero pointwise infimum]\label{prop:zero-infimum}
Assume (H1)--(H4), both labels are present, and nonnegative finite costs.
For every $i$ and every fixed $\delta\in(0,1/2)$, the following are equivalent:
\begin{enumerate}
\item $C_i=0$;
\item for every $j\in\Alt(i)$ there exists $a\in Z_i$ with
$d(K_i^a\|K_j^a)>0$;
\item $R_\delta^*(E_i)=0$.
\end{enumerate}
In case (ii), for every $\epsilon>0$ there is a globally $\delta$-correct
policy with a deterministic finite bound on its stopping time and
$R(\pi,E_i)<\epsilon$. No assumption (H2$^+$) is needed.
\end{proposition}

\begin{proof}
The covering LP is feasible and attains its finite minimum. A zero
objective forces an optimizer to be supported on $Z_i$, proving (i)$\Rightarrow$(ii).
Conversely, sufficiently many copies of every action in $Z_i$ give a
feasible count vector of zero cost, proving (ii)$\Rightarrow$(i).
The lower bound of Theorem~\ref{thm:LB}, with
$\kl(\delta,1-\delta)>0$, proves (iii)$\Rightarrow$(i).

For (ii)$\Rightarrow$(iii), first construct a finite fallback test. One
block samples every action once. For opposite-label $k,l$, let
\[
 \rho_{kl}=\prod_{a\in\cA}\sum_y\sqrt{K_k^a(y)K_l^a(y)},
 \qquad \rho=\max_{\Theta_k\ne\Theta_l}\rho_{kl}<1.
\]
Full support gives $\rho>0$, and (H2) gives the strict upper bound.
After $m$ independent blocks, a maximum-likelihood environment estimate
has wrong label under $k$ with probability at most
$\sum_{l\in\Alt(k)}\rho_{kl}^m\le(K-1)\rho^m$: apply Markov's
inequality to the square root of each competing likelihood ratio.
Choose a finite integer $m=m_\delta$ making this at most $\delta/2$.
This fallback stops deterministically and its loss in any environment is
bounded by $B_\delta=m_\delta\max_k\sum_a g_k(a)<\infty$.

Fix a candidate $i$ before sampling. Take $N$ complete blocks of the
actions in $Z_i$, and accept label $\Theta_i$ only if
\[
 \LLR_{i,j}(N|Z_i|)\ge\beta_\delta:=\log(2/\delta)
 \quad\text{for every }j\in\Alt(i).
\]
Otherwise run the fallback on fresh samples. Under any opposite-label
$j$, false early acceptance has probability at most
$e^{-\beta_\delta}=\delta/2$, by the likelihood-ratio identity and
Markov's inequality for this deterministic probing schedule. Under a
same-label environment an early acceptance is correct. Conditional on
any probing history, (H1) makes the fresh fallback $\delta/2$-correct.
Thus every finite-$N$ policy is globally $\delta$-correct and stops by
$N|Z_i|+m_\delta|\cA|$ in every environment.

Under the target $i$, the mean log-likelihood ratio per free block is
$\sum_{a\in Z_i}d(K_i^a\|K_j^a)>0$ for each $j\in\Alt(i)$.
The strong law and finiteness of $\Alt(i)$ therefore imply that the
fallback probability $q_{i,N}$ tends to zero. The probing cost under $i$
is identically zero, so
\[
 0\le R(\pi_{i,N,\delta},E_i)\le B_\delta q_{i,N}\longrightarrow0
 \qquad(N\to\infty,\ \delta\text{ fixed}).
\]
Taking the infimum over these admissible policies proves (iii).
\end{proof}

The candidate $i$ is part of a policy's fixed specification, not knowledge
supplied during execution. Every policy just constructed is correct in
\emph{all} environments, while the pointwise infimum may choose different
candidates. The result neither asserts a single zero-cost policy for all
free-ride environments nor charges for the increasingly long free probe.

\begin{proposition}[When zero cost is attained]\label{prop:zero-attainment}
Under the assumptions of Proposition~\ref{prop:zero-infimum}, fix a
nonempty set $S\subseteq\{1,\ldots,K\}$ and put $Z_S=\bigcap_{i\in S}Z_i$.
For any $\delta\in(0,1/2)$, a single globally $\delta$-correct policy
has $R(\pi,E_i)=0$ for every $i\in S$ if and only if, for every pair
$k,l$ with opposite labels, some $a\in Z_S$ satisfies $K_k^a\ne K_l^a$.
\end{proposition}

\begin{proof}
If the condition holds, the finite-block maximum-likelihood test in the
preceding proof, using only $Z_S$ and enough blocks for error $\delta$,
is globally correct and has zero loss in every member of $S$.

Conversely, zero expected nonnegative loss under $i$ implies that no
positive-$g_i$ action is selected before stopping, with probability one.
Under full support, finite transcript laws under any two environments
are mutually absolutely continuous, including the policy's independent
random seed. For each fixed time $t$, the event
$\{t\le\tau,\ A_t\notin Z_i\}$ therefore has probability zero under
every environment. Taking the countable union over $t$ and the finite
union over $i\in S$ shows that the policy uses only $Z_S$ under every
environment, almost surely. If an opposite-label pair has identical
kernels on $Z_S$, its finite transcript laws, and hence its stopped
decision laws, coincide. Almost-sure termination and $\delta<1/2$ make
global correctness impossible for that pair.
\end{proof}

\paragraph{Nonattainment and simultaneous obstructions.}
For a concrete nonattainment example, take labels $(1,1,0)$, costs
$g_1=(0,\eta)$, $g_2=(\eta,0)$, $g_3=(\eta/2,\eta/2)$, and Bernoulli
feedback probabilities
\[
 \begin{array}{c|cc}
       &a_1&a_2\\\hline
 E_1&1/4&1/4\\
 E_2&3/4&1/4\\
 E_3&3/4&3/4
 \end{array}\qquad (0<\eta\le1).
\]
These labels and costs arise from the same task construction as above:
use constant loss rows $(0,\eta)$ in $E_1$ and $(\eta,0)$ in $E_2$,
and rows $(0,\eta),(\eta,0)$ for the two equiprobable histories in $E_3$.
In $E_1$, the free action $a_1$ separates the target from $E_3$, so
$R_\delta^*(E_1)=0$. But $a_1$ cannot separate $E_2$ from $E_3$, so
no globally correct policy attains zero loss in $E_1$.

If instead both actions have feedback probability $1/4$ in $E_1,E_2$
and $3/4$ in $E_3$, either free-ride environment admits its own
globally correct zero-cost policy. A common zero-cost policy is
impossible because $Z_1\cap Z_2=\varnothing$. More strongly, identical
transcript laws give
$R(\pi,E_1)+R(\pi,E_2)=\eta\E_1\tau
\ge(\eta/d)\kl(\delta,1-\delta)$, with $d=\tfrac12\log3$,
where infinite expected time makes the bound immediate. Thus even
simultaneously vanishing loss can fail. Sample budgets, time charges,
or common-policy criteria define further questions; the unconstrained
pointwise value in B2 is already determined by
Proposition~\ref{prop:zero-infimum}.

% ============================================================
% Section file: Switching Kernels, VERIFY/REPAIR, and Open Directions
% Label: sec:switching
% No numbered theorem-like environments are used in this file
% (open directions only; no new results are claimed here).
% ============================================================
\providecommand{\cE}{\mathcal{E}}
\providecommand{\cA}{\mathcal{A}}
\providecommand{\cY}{\mathcal{Y}}
\providecommand{\cF}{\mathcal{F}}
\providecommand{\E}{\mathbb{E}}
\providecommand{\Prob}{\mathbb{P}}
\providecommand{\ind}{\mathbf{1}}
\providecommand{\kl}{\mathrm{kl}}
\providecommand{\Alt}{\mathrm{Alt}}
\providecommand{\LLR}{\mathrm{LLR}}
\providecommand{\TV}{\mathrm{TV}}

\section{Switching Kernels, VERIFY/REPAIR, and Open Directions}\label{sec:switching}

The main results of Section~\ref{sec:model} are proved for a fixed-kernel model: every
action carries an a priori specified channel $K_i^{a}$ and loss $g_i(a)$, and the cost
account stops at the certification time $\tau$. This final section marks, as precisely
as we can, where that framework ends. We first work out the kernel-switching boundary
NT2 in detail (Section~\ref{sec:nt2example}), then discuss policy-dependent transcripts
in general and their neighbors in the control literature
(Section~\ref{sec:policydep}), then the \textsc{Repair} action and representation
dynamics (Section~\ref{sec:repair}). Section~\ref{sec:openproblems} collects the open
problems. Everything in this section is a boundary analysis or a research direction;
we claim no theorems here, and the honesty conventions of
Section~\ref{sec:intro-layers} apply with doubled force.

\subsection{A kernel-switching example: NT2 in detail}\label{sec:nt2example}

\paragraph{Example (kernel switching).}\label{ex:switching}
Consider two environments sharing the same representation $f$: $E_0$ is adequate
($\Theta_0=1$) and $E_1$ is inadequate ($\Theta_1=0$). Before any certification, the
available actions are \textsc{Act} and \textsc{Verify}. Acting executes the task policy
induced by $f$: under $E_0$ it costs $g_0(\textsc{Act})=0$ and generates observations
from $K_0^{\textsc{Act}}$; under $E_1$ the same action incurs the aliasing regret
$g_1(\textsc{Act})=r>0$ per round and generates observations from $K_1^{\textsc{Act}}$.
\textsc{Verify} purchases an external audit at price $c_{\mathrm{audit}}>0$; we take the
audit to be perfect---its channel is a point mass on the true label---and remark on the
noisy case below. Finally, and this is the ingredient that leaves Model
\hyperref[model:M1]{M1}: the environment admits an exogenously given \emph{fallback
rule} that is known to be optimal under inadequacy (zero loss under $E_1$), and once
the agent has certified $\hat\Theta=0$ it may switch to this rule permanently, after
which its observations follow a fallback channel $K^{\mathrm{fb}}$ and its task loss is
zero.

The switching policy $\pi^{\leftrightarrow}$ is: play \textsc{Verify} at $t=1$; if the
audit returns $\Theta=1$, act on $f$ thereafter; if it returns $\Theta=0$, switch to
the fallback rule thereafter. Its certification error is zero, hence $\delta$-correct
for every $\delta\ge 0$, and its total cost is
\[
R\big(\pi^{\leftrightarrow},E_0\big) \;=\; R\big(\pi^{\leftrightarrow},E_1\big)
\;=\; c_{\mathrm{audit}} \;=\; O(1),
\]
uniformly in $\delta$. In particular $R^{*}_{\delta}=o(\log(1/\delta))$ in the
strongest possible sense.

\paragraph{Cost structure: where the $O(1)$ comes from.}
The displayed $O(1)$ cost has two separate components. The perfect audit itself is
a singular-channel boundary case: as a fixed action it has infinite KL against the opposite label,
so the fixed-kernel certification coefficient is already zero. The second component
is the fallback regime, which changes the post-certification loss and transcript law
as a function of the policy history. That regime switch is outside M1. If the perfect
audit is removed, the information experiment changes and a positive fixed-kernel
coefficient may reappear under the assumptions of Theorem~\ref{thm:LB}; this is not
the same menu or the same statistical comparison. If the audit is replaced by a
fixed noisy channel, one purchase generally cannot meet arbitrarily small $\delta$,
and repeated or hybrid evidence accumulation is a separate problem.

\paragraph{What the static formula does and does not say here.}
Two separate points must be kept apart. First, the collapse of the \emph{certification}
cost is not, by itself, outside the static framework: a perfect audit has
$d(K_i^{\textsc{Verify}}\|K_j^{\textsc{Verify}})=+\infty$ for $j\in\Alt(i)$, so the LP
\eqref{eq:clp} already assigns $C_i=0$ to such instances. The explicit
one-audit policy costs $O(1)$, consistently giving zero first-order cost.
Neither Corollary~\ref{cor:main} nor Proposition~\ref{prop:zero-infimum}
is being applied to this singular channel: both require full support. Second, the genuinely non-static element is the \emph{switch}: after the
\textsc{Verify} outcome, the law of the agent's subsequent transcript depends on the
policy's own history---$K_0^{\textsc{Act}}$ or $K_1^{\textsc{Act}}$ before
certification, $K^{\mathrm{fb}}$ after a negative certification---and the objective
that the example prices (one audit fee, then a zero-loss regime under \emph{both}
environments) couples the certification purchase to the post-certification loss stream.
Model~\hyperref[model:M1]{M1} excludes both features by construction: (M1.2) fixes one
channel and one loss per action for the entire horizon, and \eqref{eq:cost} stops the
accounting at $\tau$. The change-of-measure argument behind Theorem~\ref{thm:LB}
requires the fixed-kernel premise exactly here: the transportation inequality compares
transcript laws generated by a single, policy-independent kernel family
\citep{kaufmann2016complexity}, and no per-environment constant computed from fixed
kernels can describe a class in which the kernel family is itself a function of the
policy. Here the audit menu and the post-certification fallback regime are both part
of a different model, so there is no contradiction between the example and
Theorem~\ref{thm:LB}. It is a scope statement: the static complexity $C(E,f)$
characterizes the fixed-kernel policy class, while the switching class requires its
own model and lower bound.

\paragraph{The noisy-audit variant.}
If the audit is imperfect, one purchase certifies the label only to the confidence the
audit channel affords, and repeated purchases or hybrid evidence accumulation become
necessary. A fixed noisy audit included in a fixed finite menu is
already an M1 action and is priced by its covering LP under the stated
assumptions. A purchase that changes the subsequent audit channel,
representation, or loss regime requires a different model, as in open
item (iii) below.

\subsection{Policy-dependent transcripts in general}\label{sec:policydep}

The example is an instance of a general phenomenon: whenever the certification outcome
is \emph{acted upon}---act on the representation while adequate, fall back when
inadequate---the certification process and the task process cease to be separable.
Three couplings arise simultaneously. (i) \emph{Information--loss coupling along the
path}: before certification, the f-rule is both the cheapest task action (under $E_0$)
and a source of label evidence; its value therefore depends on the posterior, not on
its static loss alone. (ii) \emph{Regime switching}: the action set effectively
available after certification is a function of the terminal decision, so the model
class for the post-certification phase is policy-dependent. (iii) \emph{Objective
coupling}: the natural performance measure is total---certification expenditure plus
subsequent task loss---so a verification is worth buying exactly when the regime it
unlocks dominates its price, which is the sequential analogue of the one-shot threshold
$c^{*}=(Tr/2)(1-\TV)$ of the static layer. In that layer, Theorem~6 of
Section~\ref{sec:static} prices a single certification-driven switch; the sequential
model of Section~\ref{sec:model} generalizes the acquisition side of that picture (the
one-shot model is the special case $n_{\textsc{Verify}}=1$ with $\delta$ fixed by the
audit channel's total variation), while the switching side remains open here.

A compact way to see why the static lower-bound technology does not extend by a routine
generalization is to follow the change-of-measure step of Appendix~\ref{sec:prooflb}
under a switching policy. The transportation inequality controls, for each alternative
$j$, the log-likelihood ratio accumulated along the realized action sequence; with
policy-dependent kernels, the likelihood ratio between $E_i$ and $E_j$ is no longer a
function of the action counts $N_a(\tau)$ alone, because the same action belongs to
different kernel families before and after the switch. The lower-bound variable is
therefore no longer an allocation $n_a$ over a fixed action set but an allocation over
\emph{regime-action pairs}, with regime transitions triggered by the certification
statistic itself; the covering LP \eqref{eq:clp} would have to be replaced by a dynamic
program over posterior-dependent regimes, and we do not know whether a finite-constant
characterization in the spirit of $C(E,f)$ survives. We record this as the technical
content of open item (i) rather than as a conjectured theorem.

The nearest classical frame for this phenomenon is dual control, where control actions
serve probing and regulation simultaneously \citep{feldbaum1960dual,feldbaum1963dual,
barshalom1974dual,tse1975generalized,witsenhausen1971separation,yame1987dual}; the
switching problem above can be read as dual control with a certification terminal and a
label-dependent regime. These references establish substantial overlap:
actions can affect both control and future information. We do not claim
an exhaustive theorem-level literature audit, or infer historical novelty
from not finding the same combination of terminology. POMDP bandits likewise
share actions between reward and observation \citep{krishnamurthy2009partially}, but
without a certification target, stopping rule, or label-defined alternative structure;
planning under belief-state sufficiency \citep{kaelbling1998,astrom1965} presumes the
representation question settled rather than pricing it. The lower-bound/matching-policy
pair for certification-driven switching therefore remains open, as stated in item (i).

\subsection{REPAIR and representation dynamics}\label{sec:repair}

The \textsc{Repair} action goes one step beyond switching: instead of replacing the
\emph{policy} on a certified negative, it modifies the \emph{representation} itself,
$f\mapsto f'$, attempting to enlarge the adequate region of the environment class. This
breaks the framework at a deeper level than kernel switching, for two distinct reasons.
First, the object of inference moves: the adequacy label becomes
$\Theta_t=\ind\{R^{*}_E(f_t(H))=R^{*}_E(H)\}$, a time-varying, policy-dependent
quantity, so the hypothesis being tested changes as a consequence of the agent's own
actions. Second, the alternative sets \eqref{eq:alt}, the channels, and the loss
profiles are all indexed by the representation (aliasing regret is a property of $f$),
so a \textsc{Repair} action rewrites the model class mid-stream. Certification with a
moving label is outside fixed-hypothesis sequential testing in the
\citet{chernoff1959sequential} line, outside the fixed-confidence pure-exploration
template of \citet{garivier2016optimal}---both presume a static hypothesis and a static
experiment family---and outside Model~\hyperref[model:M1]{M1} as defined. We know of no
standard template that absorbs a self-modifying hypothesis, and we therefore treat
representation dynamics strictly as a research direction, consistent with the
limitation recorded in the static layer, where the fallback and repair mechanisms are
explicitly exogenous inputs whose acquisition is not priced.

Two partial observations delimit the direction without resolving it. First, the
one-shot layer already contains the germ of the difficulty: the value of verification
in Theorem~6 of Section~\ref{sec:static} is priced against a \emph{fixed} fallback, and
the static layer is explicit that the acquisition of that fallback---its information
source and cost---is outside the model; \textsc{Repair} is precisely the attempt to
bring that acquisition inside the decision problem, and the analysis above indicates
that doing so changes the hypothesis space itself. Second, representation dynamics
interact with the certification label in both directions: a repair that succeeds
changes the set of adequate environments and hence the alternative sets \eqref{eq:alt}
of \emph{all} future certification problems, while a failed repair consumes the same
audit currency that certification uses. Any future theory of this interaction will
therefore have to price, jointly, evidence about a label and modifications of the
object the label describes; we state this as the content of open item (ii) and make no
further claim.

\subsection{Open problems}\label{sec:openproblems}

\begin{itemize}
\item[(i)] \emph{Coupling theorem for certification-driven switching.} Extend the
lower-bound/matching-policy pair of Theorems~\ref{thm:LB} and~\ref{thm:UB} to policy
classes whose post-certification kernels depend on the certification outcome
(f-rule$\,\leftrightarrow\,$fallback), with a total objective combining certification
expenditure and post-switch task loss. The example of Section~\ref{sec:nt2example}
illustrates a perfect audit and a post-certification regime switch;
it does not prove a separation from a positive same-menu M1 coefficient.
A characterization of the general dynamic scaling is missing.
\item[(ii)] \emph{\textsc{Repair} and representation dynamics.} Formulate certification
when the label $\Theta_t=\mathrm{Adeq}_E(f_t)$ is itself policy-dependent, including
the pricing of representation modifications against their certification value.
\item[(iii)] \emph{Common policies under time constraints; changing audit menus.}
The unconstrained B2 pointwise infimum is zero at each fixed $\delta$
(Proposition~\ref{prop:zero-infimum}), and exact zero-cost attainment is
characterized by Proposition~\ref{prop:zero-attainment}. Further questions
require an explicit additional constraint, such as a common-policy loss
tradeoff or a sample budget. Repeated noisy verification with a fixed
channel already belongs to M1; purchases that change future channels or
losses require a separate analysis.
\item[(iv)] \emph{Dual-control/POMDP extension.} Build and analyze a
model with a moving adequacy label and policy-dependent kernels, with
comparisons to the relevant control literature as its assumptions are
specified. The references here are a bounded comparison, not a completed
priority audit of this broad area.
\item[(v)] \emph{Non-asymptotic remainders.} Refine Corollary~\ref{cor:main} to
$R^{*}_{\delta}(E_i)=C_i\log(1/\delta)+r(\delta)$ with an explicit rate for
$r(\delta)=o(\log(1/\delta))$; already in pure exploration this is recognized as
delicate, and we expect no shortcut from the adequacy-label structure.
\end{itemize}

\subsection{Summary}\label{sec:switchingsummary}

The sequential layer of this paper prices the acquisition of adequacy evidence in a
fixed-kernel world: Theorem~\ref{thm:LB} shows that every $\delta$-correct policy pays
at least $C(E,f)\,\kl(\delta,1-\delta)$ in task loss, Theorem~\ref{thm:UB} shows the
CTS policy pays exactly that to first order, and Corollary~\ref{cor:main} identifies
the limit $R^{*}_{\delta}(E_i)/\log(1/\delta)\to C(E,f)$ with $C(E,f)$ computable as a
finite LP. The boundaries of the statement are as much a part of the result as the
limit: hopelessness is exactly LP infeasibility (B1), free rides are exactly zero LP
value (B2), and the two-point case recovers the classical controlled-testing constant
(B3). This section has argued that the framework's edge is sharp and locatable: the
moment certification is allowed to change the process being certified---through policy
switching, and a fortiori through representation repair---the fixed-kernel constant
does not by itself provide a theorem for the resulting dynamic objective.
Characterizing that regime is the natural continuation of the present work, and we have stated it as such,
without claiming it.

% ============================================================
% Section file: Numerical Illustrations on Finite Instances
% Label: sec:numerical
% Environments: none (text, tables, figures only).
% Requires: booktabs, graphicx, pgf/tikz (loaded by the assembler).
% Figures: fig_cts_convergence.pdf, fig_frontier.pdf, fig_nt1.pdf
%          + one self-contained TikZ schematic (fig:framework).
% ============================================================
\providecommand{\cE}{\mathcal{E}}
\providecommand{\cA}{\mathcal{A}}
\providecommand{\cY}{\mathcal{Y}}
\providecommand{\E}{\mathbb{E}}
\providecommand{\Prob}{\mathbb{P}}
\providecommand{\ind}{\mathbf{1}}
\providecommand{\kl}{\mathrm{kl}}
\providecommand{\Alt}{\mathrm{Alt}}
\providecommand{\LLR}{\mathrm{LLR}}

\section{Numerical Illustrations on Finite Instances}\label{sec:numerical}

This section collects numerical work on small, fully specified instances of
Model~\hyperref[model:M1]{M1}. Its role is deliberately limited, and we state the
limitation before any numbers: \emph{everything reported here is machine-checkable
verification on finite instances and a demonstration of the behavior predicted by
Theorems~\ref{thm:LB} and~\ref{thm:UB}; it is auxiliary evidence and is not a
substitute for the proofs.} The proofs stand on their own; the computations below
serve three narrower purposes. First, they make the abstract objects of
Section~\ref{sec:model} concrete: the certification complexity $C(E,f)$ is a finite
linear program, and we exhibit one instance in full---every channel, every loss, and
the numerically solved LP solutions---so that the reader can recompute every
constant we quote.
Second, they illustrate the predicted asymptotics: the simulated cost of the CTS
policy, normalized by $\log(1/\delta)$, approaches the LP constant from above as
$\delta$ decreases, at the pace that the conservative stopping threshold leads one
to expect. Third, they exercise the boundary statements B1--B3 and the
non-triviality statements NT1--NT2 on explicit numbers, so that the scope of the
assumptions is checked rather than merely asserted. All instance parameters are
given explicitly in the tables below, every linear program was solved by exhaustive
finite vertex enumeration from those tabulated parameters, and the simulation protocol is
specified completely in Section~\ref{sec:num-repro}. We report only what the
computations show; where the displayed behavior falls short of the limit theorems
(for instance, convergence that requires smaller $\delta$ than we simulate), we say
so.

\subsection{A four-environment instance, specified completely}\label{sec:num-instance}

Our running instance has $K=4$ environments sharing a fixed representation $f$,
with adequacy labels
\begin{equation}\label{eq:num-labels}
(\Theta_0,\Theta_1,\Theta_2,\Theta_3) \;=\; (1,1,0,0),
\end{equation}
so that $E_0,E_1$ are adequate, $E_2,E_3$ are inadequate, and the alternative sets
of Section~\ref{sec:modelM1} are $\Alt(0)=\Alt(1)=\{2,3\}$ and
$\Alt(2)=\Alt(3)=\{0,1\}$. The action set is
$\cA=\{\textsc{Act},\textsc{Ask},\textsc{Verify}\}$ and the observation space has
three elements, $\cY=\{y_1,y_2,y_3\}$. Table~\ref{tab:losses} gives the task
losses, including the label-coupled structure of \textsc{Act}: acting on the
representation is cheap ($0.05$ per round) when the representation is adequate and
expensive ($0.60$ per round, the per-round aliasing regret) when it is not, while
\textsc{Ask} costs a flat query fee $0.35$ and \textsc{Verify} costs the audit
price $c_v$, which is $1.0$ in the base instance and is varied in
Figure~\ref{fig:frontier}. Table~\ref{tab:channels} gives all twelve observation
channels. The tables are the instance: every number quoted in this section is
recomputed from them.

\begin{table}[t]
\centering
\caption{Task losses $g_i(a)$ and adequacy labels of the four-environment
instance. \textsc{Act} carries the label-coupled loss of
Section~\ref{sec:model}: cheap under adequacy, priced at the per-round aliasing
regret under inadequacy. The \textsc{Verify} loss is the audit price $c_v$,
equal to $1.0$ in the base instance and swept in
Figure~\ref{fig:frontier}.}\label{tab:losses}
\begin{tabular}{lcccc}
\toprule
 & $\Theta_i$ & $g_i(\textsc{Act})$ & $g_i(\textsc{Ask})$ & $g_i(\textsc{Verify})$ \\
\midrule
$E_0$ & $1$ & $0.05$ & $0.35$ & $c_v=1.0$ \\
$E_1$ & $1$ & $0.05$ & $0.35$ & $c_v=1.0$ \\
$E_2$ & $0$ & $0.60$ & $0.35$ & $c_v=1.0$ \\
$E_3$ & $0$ & $0.60$ & $0.35$ & $c_v=1.0$ \\
\bottomrule
\end{tabular}
\end{table}

\begin{table}[t]
\centering
\caption{Observation channels $K_i^{a}$ of the four-environment instance, as
probability rows on $\cY=\{y_1,y_2,y_3\}$. All entries are positive
(Assumption~(H4)); every ordered pair of environments is separated by at least
one action (Assumptions~(H2) and~(H2$^+$)), with \textsc{Verify} alone already
separating every pair. The \textsc{Act} rows of an adequate environment and of
its binding inadequate alternative are close but not equal, which is what makes
acting nearly uninformative about the label in this instance.}\label{tab:channels}
\begin{tabular}{llccc}
\toprule
Action & Environment & $K_i^{a}(y_1)$ & $K_i^{a}(y_2)$ & $K_i^{a}(y_3)$ \\
\midrule
\textsc{Act}
 & $E_0$ & $0.5000$ & $0.3000$ & $0.2000$ \\
 & $E_1$ & $0.5200$ & $0.2800$ & $0.2000$ \\
 & $E_2$ & $0.4500$ & $0.3000$ & $0.2500$ \\
 & $E_3$ & $0.4800$ & $0.3200$ & $0.2000$ \\
\midrule
\textsc{Ask}
 & $E_0$ & $0.7000$ & $0.2000$ & $0.1000$ \\
 & $E_1$ & $0.6800$ & $0.2200$ & $0.1000$ \\
 & $E_2$ & $0.3000$ & $0.4000$ & $0.3000$ \\
 & $E_3$ & $0.3200$ & $0.3800$ & $0.3000$ \\
\midrule
\textsc{Verify}
 & $E_0$ & $0.9000$ & $0.0700$ & $0.0300$ \\
 & $E_1$ & $0.8800$ & $0.0900$ & $0.0300$ \\
 & $E_2$ & $0.1000$ & $0.2000$ & $0.7000$ \\
 & $E_3$ & $0.1200$ & $0.1800$ & $0.7000$ \\
\bottomrule
\end{tabular}
\end{table}

Solving the covering LP \eqref{eq:clp} for each environment at $c_v=1.0$ gives
\begin{equation}\label{eq:num-cvals}
C(E_0,f)=0.6050,\qquad
C(E_1,f)=0.6264,\qquad
C(E_2,f)=0.4657,\qquad
C(E_3,f)=0.4783,
\end{equation}
each attained by a bang-bang LP solution that places all weight on
\textsc{Verify}: $n^{*}_{\textsc{Verify}}=C_i$ and
$n^{*}_{\textsc{Act}}=n^{*}_{\textsc{Ask}}=0$. Consequently the cost-ratio
optimizer \eqref{eq:wratio} that CTS tracks is the pure \textsc{Verify}
allocation, $w^{*}(i)=e_{\textsc{Verify}}$ for all four environments. The values
\eqref{eq:num-cvals} are finite LP outputs; recomputing the LP from the
four-decimal channel values of Table~\ref{tab:channels} reproduces them to three
decimals (solver agreement within $2.5\times 10^{-4}$), and a grid evaluation of
the max-min form \eqref{eq:cmm} agrees with the LP value to the same tolerance,
consistent with the duality of Section~\ref{sec:complexity}.

Table~\ref{tab:pernat} explains why the optimum is bang-bang on \textsc{Verify}
in this instance: it reports, for each environment and action, the divergence to
the \emph{binding} alternative (the alternative attaining
$\min_{j\in\Alt(i)}d(K_i^{\textsc{Verify}}\|K_j^{\textsc{Verify}})$) and the
associated loss-per-nat ratio $g_i(a)/d(K_i^{a}\|K_{j}^{a})$. \textsc{Verify}
dominates because it buys several times more divergence per round than
\textsc{Ask}, at less than three times the price; \textsc{Act}, despite being
nearly free under adequacy, is nearly uninformative about the label---its
channels under $E_0$ and its binding alternative $E_3$ differ by only
$0.001$~nats---so its loss-per-nat ratio is the worst of the three. This is the
regime of Section~\ref{sec:intro-motivation} in which certification must be
purchased rather than picked up as a by-product of acting; the free-ride regime
is exercised separately as boundary B2 below.

\begin{table}[t]
\centering
\caption{Divergences $d(K_i^{a}\|K_{j^{*}}^{a})$ to the binding alternative
$j^{*}(i)$ (in parentheses) and the resulting loss-per-nat ratios
$g_i(a)/d$, computed from Tables~\ref{tab:losses} and~\ref{tab:channels}.
\textsc{Verify} has the best ratio in every environment, which is why the LP
\eqref{eq:clp} puts all weight on it.}\label{tab:pernat}
\begin{tabular}{lcccccc}
\toprule
 & \multicolumn{2}{c}{\textsc{Act}} & \multicolumn{2}{c}{\textsc{Ask}}
 & \multicolumn{2}{c}{\textsc{Verify}} \\
\cmidrule(lr){2-3}\cmidrule(lr){4-5}\cmidrule(lr){6-7}
 & $d$ & $g/d$ & $d$ & $g/d$ & $d$ & $g/d$ \\
\midrule
$E_0$ ($j^{*}=3$) & $0.0010$ & $47.64$ & $0.3097$ & $1.130$ & $1.6528$ & $\mathbf{0.6050}$ \\
$E_1$ ($j^{*}=3$) & $0.0042$ & $11.81$ & $0.2825$ & $1.239$ & $1.5965$ & $\mathbf{0.6264}$ \\
$E_2$ ($j^{*}=1$) & $0.0114$ & $52.53$ & $0.3232$ & $1.083$ & $2.1471$ & $\mathbf{0.4657}$ \\
$E_3$ ($j^{*}=1$) & $0.0043$ & $139.23$ & $0.2961$ & $1.182$ & $2.0906$ & $\mathbf{0.4783}$ \\
\bottomrule
\end{tabular}
\end{table}

\subsection{Convergence of the CTS cost to the LP constant}
\label{sec:num-cts}

Figure~\ref{fig:cts-convergence} plots the simulated expected cost of the CTS
policy $\pi^{*}$ (Algorithm~\ref{alg:cts}) on the instance of
Tables~\ref{tab:losses}--\ref{tab:channels}, normalized by $\log(1/\delta)$,
against $\log(1/\delta)$ for the five confidence levels
$\delta=0.1$, $0.05$, $0.02$, $0.01$ and $0.005$. Each (environment, $\delta$)
combination was simulated with $1{,}000$ independent replications; the
sampling rule is C-tracking \eqref{eq:ctrack} of the plug-in cost-ratio
optimizer with the forced exploration \eqref{eq:forced}, and the stopping rule
is the GLR rule \eqref{eq:glrstop} with threshold
$\beta(t,\delta)=\log\big(2t(K-1)/\delta\big)$ at $K=4$. Horizontal lines mark
the LP constants \eqref{eq:num-cvals}, which by Corollary~\ref{cor:main} are the
limits the curves should approach; Table~\ref{tab:cts-trajectory} reports the
endpoint values underlying the plot.

\begin{figure}[t]
\centering
\includegraphics[width=0.78\textwidth]{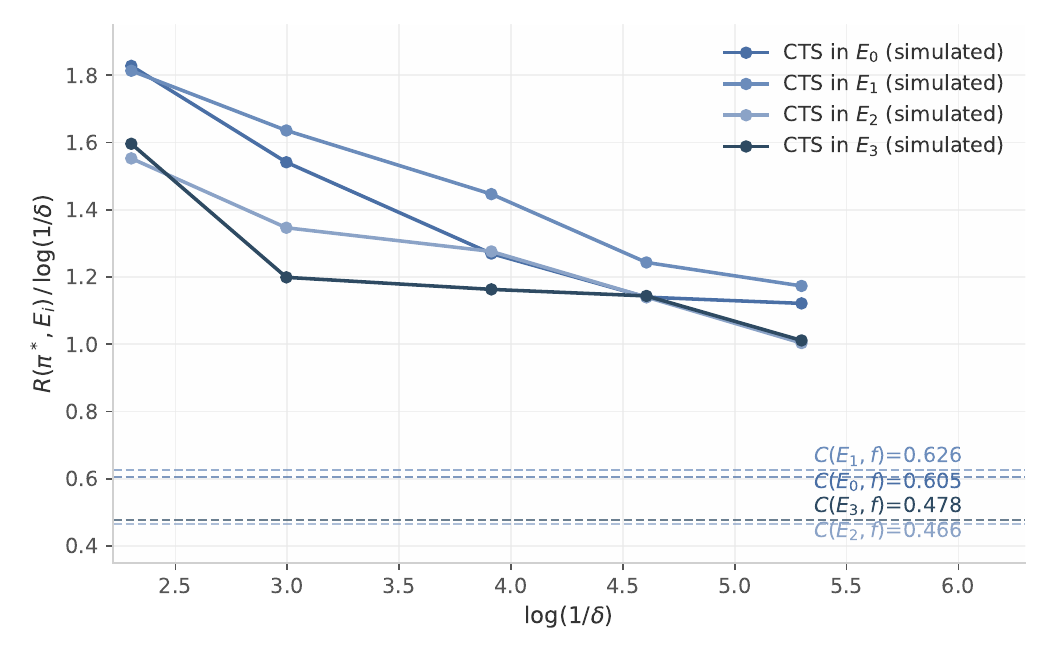}
\caption{Simulated certification cost of CTS on the four-environment instance,
normalized by $\log(1/\delta)$, as a function of $\log(1/\delta)$, with
$1{,}000$ independent replications per (environment, $\delta$) combination.
Horizontal lines mark the LP constants $C(E_i,f)$ of \eqref{eq:num-cvals}. All
four curves decrease monotonically toward their respective constants; at
$\delta=0.005$ ($\log(1/\delta)\approx 5.3$) they lie in the band
$1.0$--$1.2$, still above the limits $0.466$--$0.626$ predicted by
Corollary~\ref{cor:main}.}\label{fig:cts-convergence}
\end{figure}

\begin{table}[t]
\centering
\caption{Simulated values of $R(\pi^{*},E_i)/\log(1/\delta)$ underlying
Figure~\ref{fig:cts-convergence} at the two ends of the simulated range,
together with the LP constants \eqref{eq:num-cvals}. The third row is the
ratio of the $\delta=0.01$ entry to the corresponding constant: the simulated
cost at $\delta=0.01$ is still between $1.9$ and $2.5$ times the asymptotic
constant, and decreasing.}\label{tab:cts-trajectory}
\begin{tabular}{lcccc}
\toprule
 & $E_0$ & $E_1$ & $E_2$ & $E_3$ \\
\midrule
$R/\log(1/\delta)$ at $\delta=0.1$  & $1.665$ & $1.758$ & $1.610$ & $1.702$ \\
$R/\log(1/\delta)$ at $\delta=0.01$ & $1.173$ & $1.201$ & $1.120$ & $1.129$ \\
ratio to $C(E_i,f)$ at $\delta=0.01$ & $1.94$ & $1.92$ & $2.40$ & $2.36$ \\
\midrule
$C(E_i,f)$ (LP, \eqref{eq:num-cvals}) & $0.6050$ & $0.6264$ & $0.4657$ & $0.4783$ \\
\bottomrule
\end{tabular}
\end{table}

Three readings of Figure~\ref{fig:cts-convergence} and
Table~\ref{tab:cts-trajectory} are warranted, and we state them with the
conservatism they deserve. First, the qualitative prediction of
Corollary~\ref{cor:main} is visible: all four curves decrease monotonically and
are approaching their LP constants from above, with no sign of leveling off at
any other value. Second, the simulation is \emph{not} in the asymptotic regime:
at $\delta=0.01$ the normalized cost is between $1.9$ and $2.5$ times
$C(E_i,f)$. The bulk of this gap is explained by two standard and identified
mechanisms, not by a failure of the constant. The GLR threshold
$\beta(t,\delta)=\log(2t(K-1)/\delta)$ of \citet[Theorem~10]{garivier2016optimal}
absorbs a union bound over the $K-1$ challengers and a geometric time grid, and
is known to be conservative relative to mixture-based thresholds
\citep{kaufmann2021mixture}; and the stopping time inherits the usual
overshoot---the LLR process crosses the boundary with positive excess, so the
expected cost exceeds the first-order approximation by a quantity that decays
only relative to $\log(1/\delta)$. Both effects shrink as $\delta$ decreases,
which is exactly what the figure shows. We therefore report the honest
conclusion: asymptotic convergence to $C(E,f)$ requires smaller $\delta$ than
we simulate, and the trend displayed over the simulated range is consistent
with the theorem. Third, the certification guarantee holds with room to spare:
the empirical error frequency was at most $0.01$ in every cell, including those
with target $\delta=0.1$, reflecting the same threshold conservativeness; the
$\delta$-correctness of Theorem~\ref{thm:UB}(a) is satisfied in all runs.

\subsection{The complexity frontier in the audit price}
\label{sec:num-frontier}

Figure~\ref{fig:frontier} shows how the certification complexity of the same
instance moves with the audit price: we re-solve the LP \eqref{eq:clp} with
$g_i(\textsc{Verify})=c_v$ for $c_v\in[0.1,\,2.5]$, holding the channels and the
remaining losses of Tables~\ref{tab:losses}--\ref{tab:channels} fixed. Each
curve is the value of a parametric linear program, hence continuous,
nondecreasing, and concave piecewise linear in $c_v$; the figure displays
exactly this structure.

\begin{figure}[t]
\centering
\includegraphics[width=0.72\textwidth]{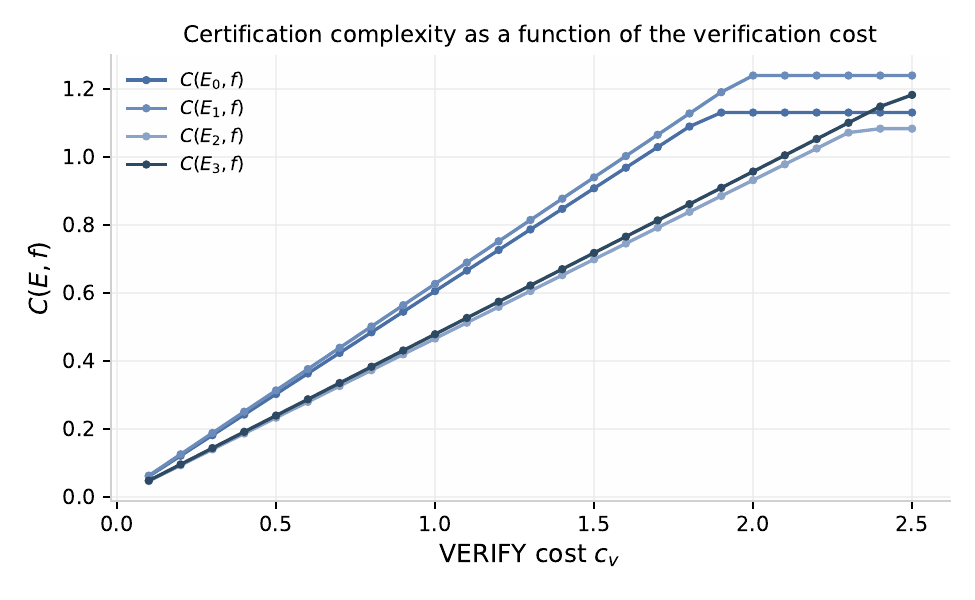}
\caption{Certification complexity $C(E_i,f)$ of the four-environment instance
as a function of the \textsc{Verify} price $c_v$, computed by re-solving the LP
\eqref{eq:clp} on a grid in $[0.1,2.5]$. Each curve is piecewise linear: a
pure-\textsc{Verify} segment through the origin up to an environment-dependent
kink, after which \textsc{Ask}-based plans take over and the curve flattens
toward the no-\textsc{Verify} LP value. The values at $c_v=1.0$ are those of
\eqref{eq:num-cvals}.}\label{fig:frontier}
\end{figure}

Two structural effects are visible. First, the initial segment of every curve is
the bang-bang line $c_v/d_i^{\textsc{Verify}}$, where
$d_i^{\textsc{Verify}}=\min_{j\in\Alt(i)}d(K_i^{\textsc{Verify}}\|K_j^{\textsc{Verify}})$:
while the audit is cheap enough, buying divergence through \textsc{Verify} is
optimal and the complexity is exactly linear in its price. The kink occurs where
the line meets the cost of cheaper plans; numerically the curves stay on the
pure-\textsc{Verify} line up to environment-dependent kink points in the range
$c_v\approx 1.9$--$2.5$ ($\approx 1.9$ for $E_0$, $\approx 2.0$ for $E_1$,
$\approx 2.3$ for $E_2$ and $\approx 2.5$ for $E_3$), beyond which
\textsc{Ask}-based plans take over and each curve flattens onto its
no-\textsc{Verify} LP value---$1.130,\ 1.239,\ 1.083,\ 1.182$ respectively,
the price of covering both alternative constraints with \textsc{Ask} alone.
The two adequate curves saturate first: at $c_v=2.0$ they already sit at
$1.130$ and $1.239$, while the inadequate curves are still on the
pure-\textsc{Verify} line there ($0.931$ and $0.957$) and keep rising past
$c_v\approx 2.2$--$2.3$. Second, the ordering of the
curves reflects the label-coupled loss structure rather than the channels
alone: the inadequate environments $E_2,E_3$ are \emph{cheaper} to certify than
the adequate ones throughout, not because their channels are more informative
in an absolute sense, but because the LP prices information in the currency of
the environment's own task loss, and the divergence profile of
\textsc{Verify} against the adequate alternatives is larger in that direction
($d$ values $2.15$ and $2.09$ against $1.65$ and $1.60$; see
Table~\ref{tab:pernat}). The frontier thus makes concrete the central feature
of the constant $C(E,f)$: it couples the statistical distinguishability of the
environments to the loss in which the certification bill is paid, and moving a
single price ($c_v$) re-optimizes the whole covering plan rather than rescaling
a fixed design.

\subsection{NT1 on numbers: equal kernels, different losses, different optima}
\label{sec:num-nt1}

Figure~\ref{fig:nt1} instantiates the non-triviality statement NT1 of
Section~\ref{sec:nontrivial}. The instance has two environments, $E_0$
adequate and $E_1$ inadequate, and two informative actions $a_1,a_2$ whose
observation channels are \emph{identical under both environments}: on a
two-point observation space $\cY=\{y_1,y_2\}$,
\begin{equation}\label{eq:nt1-kernels}
\begin{aligned}
K_0^{a_1}&=K_0^{a_2}=(0.8,\,0.2),\qquad
K_1^{a_1}=K_1^{a_2}=(0.3,\,0.7),\\
d\big(K_0^{a}\,\big\|\,K_1^{a}\big)&=0.5341,\qquad
d\big(K_1^{a}\,\big\|\,K_0^{a}\big)=0.5827\ \ \text{for } a\in\{a_1,a_2\}.
\end{aligned}
\end{equation}
The two actions are therefore statistically indistinguishable: any criterion
built from the observation kernels alone---in particular the Chernoff-type
most-informative-action rule of the classical sequential-design line
\citep{chernoff1959sequential,nitinawarat2015controlled}---assigns them the
same value. What differs is the loss: $a_1$ costs a fixed $0.5$ per round in
both environments, while $a_2$ is environment-dependent. Under loss profile~A,
$g_0(a_2)=0.01$ and $g_1(a_2)=1.0$; under profile~B these two values are
swapped, $g_0(a_2)=1.0$ and $g_1(a_2)=0.01$. Since $|\Alt(i)|=1$ the bang-bang
form \eqref{eq:bangbang} applies and the optimal certification action in environment
$i$ is the action with the smallest loss-to-directional-KL ratio. Under profile~A,
$E_0$ certifies through $a_2$ ($C_0=0.01/0.5341\approx0.019$) and $E_1$ through
$a_1$ ($C_1=0.5/0.5827\approx0.858$); under profile~B the choices interchange,
with $C_0=0.5/0.5341\approx0.936$ and $C_1=0.01/0.5827\approx0.017$.

\begin{figure}[t]
\centering
\includegraphics[width=0.72\textwidth]{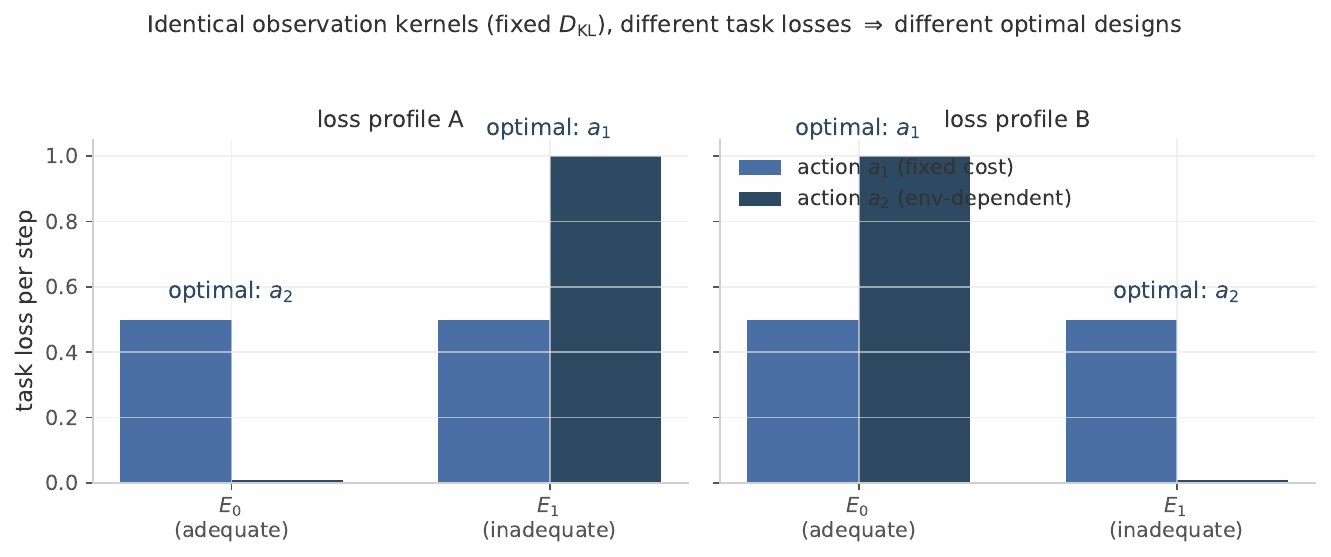}
\caption{The NT1 instance \eqref{eq:nt1-kernels}: per-round task losses of the
fixed-cost action $a_1$ and the environment-dependent action $a_2$ under the
two loss profiles. The observation kernels---and hence every kernel-only
information criterion---are identical across profiles and across the two
actions (the directional divergences are $0.5341$ and $0.5827$ nats), yet the LP-optimal certification action, marked in
each panel, swaps: profile~A selects $a_2$ at $E_0$ and $a_1$ at $E_1$, and
profile~B selects the reverse.}\label{fig:nt1}
\end{figure}

The point is quantitative, not cosmetic. A kernel-only criterion sees one
instance where there are two; in at least one of the two profiles it prescribes
an action whose loss-per-nat ratio exceeds the optimum by a factor of
$0.5/0.01=50$ in the affected environment---the same statistical evidence at
fifty times the asymptotic cost (and if the kernel-only rule happens to favor
$a_2$, profile~B inflates the factor to $1.0/0.01=100$). The example demonstrates sensitivity to the chosen cost objective.
It does not distinguish M1 from cost-aware controlled testing: that
literature also changes its design when the cost table changes. The
additional obstruction of Section~\ref{sec:observable} instead holds a
single model fixed and asks whether identical observations permit the
certifier to attain different environment-wise optimal designs.

\subsection{The sequential certification loop, schematically}
\label{sec:num-framework}

Figure~\ref{fig:framework} summarizes the model that all of the above numbers
come from: the interaction loop of Model~\hyperref[model:M1]{M1} with the four
components of the CTS policy marked at the points where they act. The diagram
is a rendering of definitions \eqref{eq:cost}--\eqref{eq:glrstop} rather than
new content, and is included to fix the data flow: the unknown environment
index $i$ selects a row of the channel and loss tables; the policy, driven by
the history, picks the action; the environment returns an observation through
the selected channel and charges the selected loss; the transcript feeds the
environment estimator, the plug-in allocation, and the GLR statistic; and on
stopping, the certified label is read off the estimated environment.

\begin{figure}[t]
\centering
\begin{tikzpicture}[
  font=\small,
  >=stealth,
  box/.style={draw=black!60, rounded corners=2pt, align=center,
              inner sep=5pt, minimum height=9mm},
  env/.style={box, fill=black!8},
  pol/.style={box, fill=blue!12},
  chn/.style={box, fill=orange!15},
  trn/.style={box, fill=black!5},
  stp/.style={box, fill=green!12},
  dec/.style={box, fill=blue!8},
  cst/.style={box, fill=red!8},
  lab/.style={font=\footnotesize\itshape, text=black!70}
]
% nodes
\node[env, minimum width=27mm] (env) at (0,0)
  {Environment $E_i$\\ \footnotesize index $i$ unknown, $\Theta_i$ realized};
\node[pol, minimum width=30mm] (pol) at (5.6,0)
  {Policy $\pi$\\ \footnotesize $A_t$ from $H_{t-1}$; CTS: MLE $\hat\imath_t$,\\ \footnotesize $w^{*}(\hat\imath_t)$, C-tracking};
\node[chn, minimum width=30mm] (chn) at (5.6,-2.8)
  {Action output triple\\ \footnotesize channel $K_i^{A_t}$, loss $g_i(A_t)$,\\ \footnotesize label information};
\node[trn, minimum width=27mm] (trn) at (0,-2.8)
  {Transcript $H_t$\\ \footnotesize $(A_1,Y_1,\dots,A_t,Y_t)$};
\node[stp, minimum width=27mm] (stp) at (0,-5.6)
  {GLR stopping\\ \footnotesize $\max_i\min_{j\in\Alt(i)}\LLR_{i,j}(t)>\beta(t,\delta)$};
\node[dec, minimum width=27mm] (dec) at (5.6,-5.6)
  {Certification output\\ \footnotesize $\hat\Theta=\Theta_{\hat\imath_\tau}\in\{0,1\}$};
\node[cst, minimum width=27mm] (cst) at (10.2,-2.8)
  {Cost accumulator\\ \footnotesize $R=\E_i[\sum_{t\le\tau}g_i(A_t)]$};
% arrows
\draw[->, thick, black!60] (env) -- (pol)
  node[lab, midway, above, align=center]{channel/loss\\ rows};
\draw[->, thick, black!60] (pol) -- (chn)
  node[lab, midway, right]{$A_t$};
\draw[->, thick, black!60] (chn) -- (trn)
  node[lab, midway, above]{$Y_t\sim K_i^{A_t}$};
\draw[->, thick, black!60] (trn) -- (stp)
  node[lab, midway, right, align=left]{LLR\\ increments};
\draw[->, thick, black!60] (stp) -- (dec)
  node[lab, midway, above]{$\tau$};
\draw[->, thick, black!60] (chn) -- (cst)
  node[lab, midway, above]{$g_i(A_t)$};
\draw[->, thick, black!60, dashed] (trn.north) -- (pol.south -| trn.north)
  node[lab, midway, left, align=right]{history-\\ driven};
\end{tikzpicture}
\caption{The sequential certification loop of Model~\hyperref[model:M1]{M1}
with the CTS components in place. The unknown environment fixes the channel
and loss rows (Tables~\ref{tab:channels} and~\ref{tab:losses}); the policy
selects actions from the history; each action returns an observation through
its channel and charges its loss, which accumulates into the objective
\eqref{eq:cost}; the transcript drives the estimator \eqref{eq:mle}, the
plug-in allocation \eqref{eq:wratio}, and the GLR statistic; stopping
\eqref{eq:glrstop} triggers the terminal decision \eqref{eq:decision}. Solid
arrows are the data flow; the dashed arrow is the information feedback that
makes the policy history-dependent.}\label{fig:framework}
\end{figure}
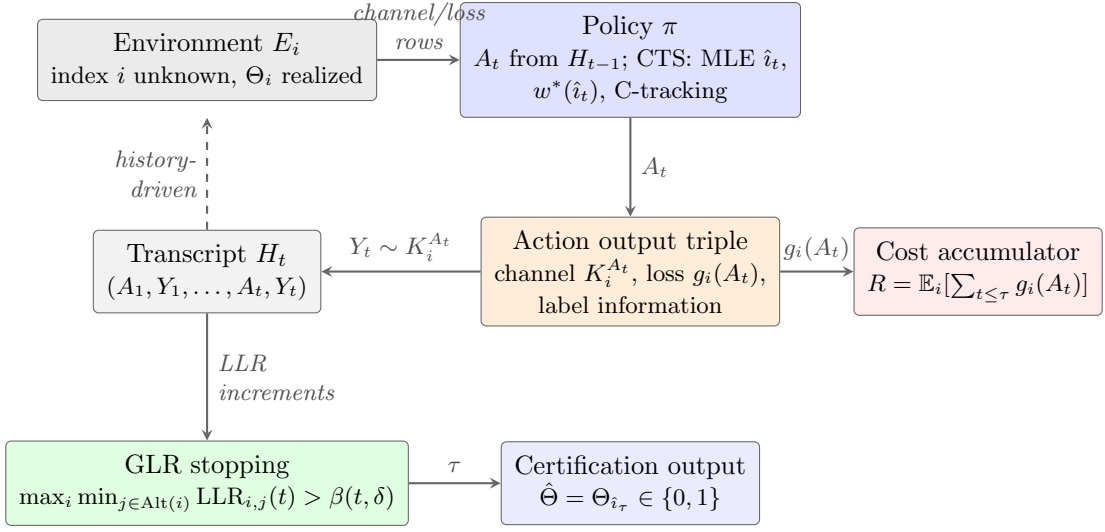

The schematic makes one modeling choice visible that the equations keep
implicit: the loss is charged \emph{per round, in parallel with} the
observation, and the certification constraint couples only to the transcript,
never directly to the loss stream. It also localizes precisely where the
boundary phenomena of the next subsection enter: B1 removes the divergence
between two rows of the channel box, so no transcript can separate them; B2
sets a loss entry to zero, so the accumulator stops charging for informative
rounds; B3 collapses the environment box to two rows; and NT2, which the model
excludes by fiat, would allow the certification output to rewire the channel
box itself after a \textsc{Verify} outcome, making the channel family a
function of the policy's own history (Section~\ref{sec:switching}).

\subsection{Boundary behavior in numbers: B1, B2, B3, and NT2}
\label{sec:num-boundary}

We close the numerical tour by evaluating the four boundary statements of
Sections~\ref{sec:nontrivial} and~\ref{sec:switching} on explicit instances
derived from Tables~\ref{tab:losses}--\ref{tab:channels}. Each boundary is
reported in the same format: the modification of the base instance, the
numerical outcome, and the consistency check against the theorems.

\paragraph{B1: hopeless instances, LP infeasibility.}
Modify the base instance by replacing every channel of $E_3$ with the
corresponding channel of $E_0$: $K_3^{a}\gets K_0^{a}$ for
$a\in\{\textsc{Act},\textsc{Ask},\textsc{Verify}\}$. Then
$3\in\Alt(0)$ and $d(K_0^{a}\|K_3^{a})=0$ for every action, so the covering LP
\eqref{eq:clp} for $E_0$ has an unsatisfiable constraint row. Solved from the
tabulated parameters, the LP returns \emph{infeasible} (\texttt{linprog}
status~$2$), i.e.\ $C(E_0,f)=+\infty$ under the convention of
Section~\ref{sec:complexity}. This matches Theorem~\ref{thm:LB} and the B1
statement of Section~\ref{sec:nontrivial} exactly: $E_0$ and $E_3$ induce
identical transcript laws under every policy, no $\delta$-correct policy with
$\delta<1/2$ and finite expected loss exists at $E_0$, and the failure is
certified by the same object that computes the constant---the LP
itself---rather than by a separate diagnostic. The situation is the
hopelessness of partial-monitoring games without informative feedback
\citep{bartok2014partial}, transplanted to the certification currency.

\paragraph{B2: free rides, LP value zero.}
Keep the channels of the base instance and change one loss:
$g_0(\textsc{Act})\gets 0$. Since both divergences
$d(K_0^{\textsc{Act}}\|K_2^{\textsc{Act}})=0.0081$ and
$d(K_0^{\textsc{Act}}\|K_3^{\textsc{Act}})=0.0010$ are strictly positive, the
zero-loss action covers every alternative constraint, and the LP value drops to
exactly $0$ (solved value $0.0$, status optimal).
Proposition~\ref{prop:zero-infimum} strengthens the first-order conclusion
to $R_\delta^*(E_0)=0$ at each fixed $\delta$. A finite free probe and a
fresh fallback approach this infimum with globally admissible policies.
CTS can still incur positive loss by exploring \textsc{Ask} and
\textsc{Verify}; that policy-specific residual is not the optimal value.
The separate finite checks in \texttt{check\_free\_evidence.py} evaluate
fallback probabilities and both zero-cost attainment examples of
Section~\ref{sec:free-evidence}.

\paragraph{B3: two-point degeneration, consistency with the binary constant.}
Restrict the environment class to $\{E_0,E_2\}$, so that $\Alt(0)=\{2\}$ and the
bang-bang form \eqref{eq:bangbang} applies. The three loss-per-nat ratios,
computed from the tabulated channels, are
\begin{equation}\label{eq:b3-ratios}
\begin{aligned}
\frac{g_0(\textsc{Act})}{d(K_0^{\textsc{Act}}\|K_2^{\textsc{Act}})}
&=\frac{0.05}{0.0081}\approx 6.21,\\
\frac{g_0(\textsc{Ask})}{d(K_0^{\textsc{Ask}}\|K_2^{\textsc{Ask}})}
&=\frac{0.35}{0.3446}\approx 1.016,\\
\frac{g_0(\textsc{Verify})}{d(K_0^{\textsc{Verify}}\|K_2^{\textsc{Verify}})}
&=\frac{1.0}{1.8095}\approx 0.5526,
\end{aligned}
\end{equation}
so on the restricted class
$C(E_0,f)=\min_a g_0(a)/d(K_0^{a}\|K_2^{a})=0.5526$, attained by
\textsc{Verify}. In the reverse direction the ratios are
$0.6/0.0084\approx 71.65$ (\textsc{Act}), $0.35/0.3527\approx 0.993$
(\textsc{Ask}) and $1.0/2.1952\approx 0.4555$ (\textsc{Verify}), so
$C(E_2,f)=0.4555$. Both binary constants lie below their four-environment
counterparts \eqref{eq:num-cvals}: in the full instance the binding
alternatives are $E_3$ for $E_0$ ($d=1.6528$) and $E_1$ for $E_2$
($d=2.1471$), and restricting the class to $\{E_0,E_2\}$ deletes exactly those
binding constraints from the covering LP. This is the
binary specialization of the controlled sequential testing constant of
\citet[Theorem~5.1]{nitinawarat2015controlled}, with the task loss in the role
of the control cost: with a single alternative there is no covering problem to
balance and the cheapest nat wins. The numerical agreement is a boundary
requirement, not a contribution---the theorems correctly contain the known
binary case, and the genuinely combinatorial content of \eqref{eq:clp} appears
only when $|\Alt(i)|\ge 2$.

\paragraph{NT2: a changed audit menu and a post-certification switch.}
Instantiate the switching example of Section~\ref{sec:nt2example} with explicit
channels: two environments, $E_0$ adequate and $E_1$ inadequate; before any
certification the only task action is \textsc{Act}, with
$g_0(\textsc{Act})=0$, $g_1(\textsc{Act})=r=0.6$, and channels
\begin{equation}\label{eq:nt2-kernels}
\begin{aligned}
K_0^{\textsc{Act}}&=(0.7583,\,0.1249,\,0.1168),\qquad
K_1^{\textsc{Act}}=(0.1168,\,0.1249,\,0.7583),\\
d\big(K_1^{\textsc{Act}}\,\big\|\,K_0^{\textsc{Act}}\big)&=1.2000.
\end{aligned}
\end{equation}
Viewed as a fixed-kernel instance (the audit removed, as in
Section~\ref{sec:nt2example}), it falls under B3: $C(E_1,f)=r/1.2=0.5$ per nat,
and Theorem~\ref{thm:LB} gives
$R(\pi,E_1)\ge 0.5\,\kl(\delta,1-\delta)\ge 0.5\log\big(1/(2.4\delta)\big)$ for
every $\delta$-correct fixed-kernel policy---numerically $\ge 1.86$ at
$\delta=0.01$ and $\ge 4.17$ at $\delta=10^{-4}$, growing without bound as
$\delta\to 0$. The switching policy of Section~\ref{sec:nt2example}---one
perfect \textsc{Verify} purchase at price $c_{\mathrm{audit}}=1.0$, then act on
$f$ if certified adequate and switch to the zero-loss fallback otherwise---has
certification error zero and total cost exactly $1.0$ at \emph{every} $\delta$.
These figures compare different information menus: the positive
coefficient $0.5$ belongs to the \textsc{Act}-only instance. Adding
perfect \textsc{Verify} gives infinite opposite-label KL and a static
coefficient of zero in both environments, so this calculation is not a
same-menu separation from the M1 lower bound. The additional change to
the post-certification loss regime illustrates the dynamic objective
outside M1. In the \textsc{Act}-only full-support instance, the free
$\textsc{Act}$ in $E_0$ also gives $R_\delta^*(E_0)=0$ by
Proposition~\ref{prop:zero-infimum}.

\subsection{Reproducibility}
\label{sec:num-repro}

All instance parameters are given in
Tables~\ref{tab:losses}--\ref{tab:channels} and in the displayed equations
\eqref{eq:nt1-kernels} and \eqref{eq:nt2-kernels} of this section; no other
inputs are used. Every linear program reported here---the constants
\eqref{eq:num-cvals}, the frontier of Figure~\ref{fig:frontier}, the per-nat
ratios of Table~\ref{tab:pernat}, and the boundary computations B1--B3---was
recomputed by the repository's finite covering-LP implementation, which
enumerates candidate vertices directly from the tabulated four-decimal
parameters and requires no external numerical optimizer. The quoted values
agree with an independent \texttt{scipy.optimize.linprog} (HiGHS dual simplex)
cross-check to within $2.5\times10^{-4}$, and the LP and max-min forms
\eqref{eq:clp}--\eqref{eq:cmm} were also cross-checked by grid evaluation of the
simplex and on randomly generated five-environment three-action instances,
with agreement within the grid tolerance (relative error below $1\%$). The CTS
curves of Figure~\ref{fig:cts-convergence} were
produced by Monte Carlo simulation of Algorithm~\ref{alg:cts}: C-tracking
\eqref{eq:ctrack} of the plug-in cost-ratio optimizer \eqref{eq:wratio} with
forced exploration \eqref{eq:forced}, GLR stopping \eqref{eq:glrstop} at
$\beta(t,\delta)=\log(2t(K-1)/\delta)$, terminal decision \eqref{eq:decision};
$1{,}000$ independent replications per (environment, $\delta$) cell with
independent random streams seeded per cell, and empirical error frequencies
recorded alongside costs (all $\le 0.01$, as reported in
Section~\ref{sec:num-cts}). All identities, LP values, and thresholds quoted in
this section were verified by finite vertex-enumeration LP solves, exhaustive enumeration, or
direct evaluation from the tabulated parameters, with agreement to solver
precision; the simulation is the only Monte Carlo ingredient, and it is
reported as such.

% Related work: strongest conceptual and theorem-level neighbours first.
\section{Related Work}\label{sec:related}

The comparison separates the origin of the adequacy concept, statistical
assessment of representations, and the sequential optimization problem.
Several ingredients and combinations already occur in prior work.
Table~\ref{tab:audit} records the nearest objectives; it is not a proof that
their intersection is new. Proposition~\ref{prop:table-reduction} explains
why changing the interpretation of a label does not change the testing proof.

\begin{table}[tbp]
\centering\small
\caption{Closest comparisons. The last column states a scope difference,
not a claim that a generalization or reduction is impossible.}
\label{tab:audit}
\begin{tabularx}{\textwidth}{@{}>{\raggedright\arraybackslash}p{0.22\textwidth}*{3}{>{\raggedright\arraybackslash}X}@{}}
\toprule
Work & Terminal target & Objective / guarantee & Comparison with M1\\
\midrule
\citet{jiang2015abstraction} & Abstraction selection & Finite-sample
policy-loss control & Data-based selection, not M1's stopped loss objective\\
\citet{nitinawarat2015controlled} & Hypothesis & Non-uniform control
cost; matching sequential bounds & Specified control fee $c(u)$\\
\citet{degenne2019pure} & General / multiple correct answers &
Fixed-confidence sample complexity & General answer structure already present\\
\citet{degenne2019bridging} & Best arm & Low regret with
$\delta$-PAC identification & Experimentation regret already present\\
\citet{yang2026minimal} & Best arm & Minimum stopped regret;
matching first-order bounds & Reward feedback determines gaps\\
\citet{kanarios2024cost} & Best arm & Heterogeneous sampling cost &
Sampling-cost model, reward-based answer\\
M1 & Fixed representation's adequacy & Stopped task loss &
Finite supplied model; loss need not be observed\\
\bottomrule
\end{tabularx}
\end{table}

\subsection{Risk-Based Sufficiency and Representation Adequacy}
\label{sec:related:blackwell}

Blackwell's comparison of experiments and Le Cam's sufficiency theory
provide the decision-theoretic background
\citep{blackwell1951,blackwell1953,lecam1964}.
\citet{takeuchi1975adequacy} characterize prediction sufficiency (adequacy)
in terms of risk functions, with distinctions depending on the permitted
loss functions. This is an antecedent of risk-based adequacy, not a theorem
about M1's particular fixed loss and sequential protocol.
\citet{sevetlidis2026bayes} define Bayes-sufficient representations for a
specified distribution and loss through preservation of a Bayes-optimal
action rule. That static criterion is closely related to
$R_E^*(f(H))=R_E^*(H)$. We do not claim priority for that criterion.

Our grouping identity (Theorem~\ref{thm:aliasing}) is ordinary Bayes-risk
bookkeeping on representation cells \citep{dgl1996}. The internal-testing
bound (Theorem~\ref{thm:tvbound}) is the Bayes-error--total-variation
identity and Le Cam two-point method \citep{lecam1986,tsybakov2009}; the
mutual-information statement uses Jensen--Shannon divergence
\citep{lin1991}. The verification threshold (Theorem~\ref{thm:verify}) is
a value-of-information calculation under its stated fallback convention
\citep{howard1966,degroot1970}. These results set up the certification
question and its observational obstruction; their underlying principles
are established tools.

\subsection{Perceptual Aliasing, Abstraction Tests, and Model Criticism}
\label{sec:related:abstraction}

\citet{chrisman1992aliasing} study perceptual aliasing in which identical
percepts require different responses. \citet{mccallum1994state} construct
memory to disambiguate hidden states and discuss earlier utility-based
state splitting. \citet{jong2005abstraction} discover state variables that
can be ignored while retaining optimal behavior. Thus neither detecting
missing task distinctions nor refining an internal state is a new goal.

There are direct statistical precedents.
\citet{jiang2015abstraction} use hypothesis tests to select among candidate
state abstractions, with a finite-sample guarantee under assumptions on
those candidates. \citet{shi2020markov} test the Markov property of
sequential decision data. \citet{lin2023transfer} test whether transferred
features retain the information needed for target prediction.
\citet{albrecht2015criticising} use frequentist hypothesis tests to criticize
an agent's behavioral model of another agent. These works preclude claiming
that statistical assessment of one's representation or model is itself new.
Their targets and data assumptions differ: Markov structure, predictive
sufficiency, and behavioral-model correctness need not equal optimal-risk
preservation for a particular task loss. M1 fixes that latter property as
its terminal answer and prices adaptive evidence acquisition for it.

\subsection{Sequential Design and Controlled Testing}
\label{sec:related:chernoff}

\citet{chernoff1959sequential} establish the sequential experimental-design
template behind the max-min information rate. Active testing and controlled
sensing develop this framework further
\citep{naghshvar2013,nitinawarat2013controlled}.
\citet{whittle1965sequential} already study sequential design with
observation costs and costs of transitions between experiment types.
Cost-sensitive and switching-cost experimental design therefore have
classical precedents; neither is introduced by M1.
\citet{nitinawarat2015controlled} treat controlled Markovian observations
and a non-uniform control cost $c(u)$, proving matching asymptotic
performance. Their cost-weighted allocation is a direct mathematical
predecessor of \eqref{eq:cmm}. The more recent preprint
\citet{vershinin2025costs} also adapts Chernoff sampling to positive
heterogeneous action costs, emphasizing the ratio of information and cost
rates. Its stated model uses constant action fees and an average-error
criterion. The allocation-coordinate distinction is therefore not claimed
as a new principle here. A binary case with a common per-action fee
is already covered by this tradition; B3 records the reduction.

M1 permits $g_i(a)$ to depend on the hypothesis, with the interpretation
of task loss. That difference matters when indistinguishable hypotheses
have different optimal cost profiles, as
Theorem~\ref{thm:observable-compatibility} shows. However, describing a
cost as a task consequence instead of a fee creates no new algebra by
itself. The finite-table reduction remains valid, and we do not claim that
controlled testing intrinsically requires a separate task action or can
never accommodate such losses.

\subsection{Pure Exploration and Identification with Minimal Regret}
\label{sec:related:ba}

\citet{kaufmann2016complexity} provide the adaptive change-of-measure
inequality used in Theorem~\ref{thm:LB}.
\citet{garivier2016optimal} develop Track-and-Stop for asymptotically
optimal fixed-confidence best-arm identification.
\citet{degenne2019pure} treat general answer structures and multiple
correct answers. M1 has one correct binary label per environment;
several environments sharing it is already permitted by general-answer
formulations. \citet{kanarios2024cost} address heterogeneous sampling
costs and cost-aware allocation. These works supply the proof and design
templates, rather than merely distant analogies.

The closest comparison in objective is identification with low regret.
\citet{degenne2019bridging} study regret-minimizing algorithms that also
provide $\delta$-PAC identification and a decision-time guarantee.
\citet{yang2026minimal} minimize cumulative regret until fixed-confidence
best-arm identification in one-parameter exponential families. Their
Theorem~3 gives an information lower bound and Theorem~7 a matching
asymptotic algorithm; their Theorem~5 also distinguishes regret optimality
from sample efficiency. Thus a stopped-regret objective with matching
$\log(1/\delta)$ bounds is established prior work.

The remaining differences are the answer map, the finite supplied model
class, and feedback: in M1, $g_i(a)$ need not be recoverable from the
observed channel. In reward-observing bandits, equal arm reward laws imply
equal means and regret gaps. M1 allows observationally identical
environments with different loss profiles, leading to the compatibility
obstruction of Section~\ref{sec:observable}. This is a difference in
available information, not a claim that regret cannot depend on the
unknown environment. Raw task loss also must not be identified with
bandit regret without specifying the comparator.

Structured-bandit lower bounds already use information-covering programs
with regret weights
\citep{graves1997asymptotically,combes2015combinatorial,combes2017minimal}.
Consequently the LP shape, multiple alternative constraints, and a change
of the optimal action when costs change are not independent novelty claims.
The same caution applies to zero-loss actions: bandit regret already has
zero-gap actions, so B2 is an explicitly analyzed boundary, not an
unprecedented phenomenon.

\subsection{Indistinguishable Models and Common Optimal Decisions}
\label{sec:related:compatibility}

The obstruction behind Theorem~\ref{thm:observable-compatibility} has
explicit antecedents. In partial monitoring, indistinguishable feedback
distributions with conflicting optimal actions give impossibility
results; see, for example, \citet[Theorem~5.2]{chaudhuri2016ranking}.
\citet[Appendix~B.2, Definition~2]{parisi2024monitored} define solvability
of a monitored MDP through a nonempty intersection of optimal-policy
sets over its indistinguishability class. Their model explicitly allows
unobserved rewards. We therefore do not claim that hidden costs,
observational equivalence, or the common-optimum principle are new.

Our theorem specializes this principle to fixed-confidence finite-label
certification: the common objects are KL-covering count vectors, the
obstruction has the numerical joint bound $J_B\kl(\delta,1-\delta)$,
and a quotient CTS construction proves sufficiency for all first-order
constants simultaneously. The equality between the sum LP minimum and
the sum of individual minima is elementary optimization. The cited
MDP criterion concerns discounted-return optimal policies, and the cited
ranking theorem concerns horizon regret; neither displayed result is
the fixed-confidence LP statement proved here. This scoped comparison
supports the stated specialization, not a claim of historical priority.

\subsection{Information-Directed Sampling and Partially Observed Control}
\label{sec:related:ids}\label{sec:related:pomdp}

Information-directed sampling couples immediate regret and information
gain \citep{russo2016information,russo2018learning}; satisficing variants
relax exact optimal-action learning \citep{russo2022satisficing}. Its usual
Bayesian regret objective differs from a fixed-confidence stopping target.
That distinction is specific to this comparison and does not apply to the
regret-aware identification papers above.

Belief-state sufficiency and POMDP planning establish conditions under which
an information state supports optimal control
\citep{astrom1965,kaelbling1998,krishnamurthy2016}.
Partially observed bandits share reward-producing and informative actions
\citep{krishnamurthy2009partially}. State abstraction and predictive state
representations study value preservation and alternative state descriptions
\citep{givan2003,li2006,littman2001}. These are conceptual foundations of
the representation question. They should not be described collectively as
known-model analyses that never assess representations from data; the
statistical abstraction works above do precisely that.

\subsection{Dual Control and the Boundary to Switching and Repair}
\label{sec:related:dual}

Dual control explicitly recognizes that actions influence both task
performance and future information
\citep{feldbaum1963dual,barshalom1974dual,tse1975generalized}.
Separation can fail \citep{witsenhausen1971separation}, and the information
value of probing has been studied in adaptive control
\citep{yame1987dual}. The simultaneous use of actions for control and
learning is therefore not a novelty of M1.

M1 fixes a representation, a label for each environment, and a channel for
each action. Section~\ref{sec:switching} concerns changes to those premises.
The perfect-audit example is a scope illustration: including that audit
in the original menu already gives infinite opposite-label KL and a zero
first-order coefficient. It cannot show a positive-coefficient separation
against the same action menu. The present manuscript does not prove a
moving-label or representation-repair theorem.

\subsection{Other Boundaries and Related Decompositions}
\label{sec:related:pm}\label{sec:related:bennett}

Partial monitoring formalizes losses and feedback as separate objects
\citep{rustichini1999,cesabianchi2006,bartok2014partial}. Its observational
obstructions are related to B1, but its standard objective is regret over
a horizon. We do not infer that feedback must reveal the loss or that
robust control excludes observational ambiguity. Concept drift
\citep{widmer1996} and robust dynamic programming \citep{iyengar2005}
address different forms of uncertainty with different performance criteria.
The concurrent decomposition in \citet{bennett2026forgetting} also connects
aliasing cells to unavoidable decision regret. Our use of a cell-wise
Bayes-risk identity is credited as a standard ingredient, independently of
that comparison.

\subsection{Scope of the Contribution}\label{sec:related:positioning}

The claim is a specified certification formulation and its analyzed
boundaries. The terminal answer is the Bayes-risk adequacy of a fixed
representation; evidence acquisition incurs the task losses of the same
unknown environment; the objective is uniform fixed-confidence correctness
with environment-wise expected loss. The finite sequential analysis
instantiates established testing machinery for this formulation. It is not
evidence that every ingredient, or the general intersection of answer maps
and regret-priced testing, lacked a precedent.

The additional result in Section~\ref{sec:observable} concerns a concrete
assumption question: when can full environment identifiability be weakened
to label identifiability while retaining simultaneous first-order
optimality? For positive losses the answer is compatibility of the optimal
designs inside each observable class. We provide the criterion, its proof,
and a task-derived example. Its mathematical statement is more general
than adequacy, and no claim of exhaustive historical priority is made.

\section{Conclusion}\label{sec:conclusion}

This paper formulates fixed-confidence certification of a fixed
representation's Bayes-risk adequacy, with experimentation charged in
task-loss units. Its finite, known-class sequential analysis yields the
covering-LP constant $C(E_i,f)$ and matching first-order loss under the
stated identifiability and support assumptions. The static risk identity,
two-point bound, value-of-information calculation, and sequential testing
techniques are established ingredients. The adequacy interpretation does
not introduce a new general testing theorem: the finite-table reduction
states exactly what the sequential proof uses.

The observable-class result distinguishes learning the answer from
attaining all of its environment-wise optimal prices. Under positive
losses, one policy can attain those prices simultaneously exactly when
the covering designs within each observational equivalence class have
a common minimizer. A task-derived example has an identifiable adequacy
label but an exact factor-two obstruction to simultaneous attainment.
This explains the role of the stronger environment-identifiability
assumption and gives a precise condition under which it can be relaxed.

The model supplies candidate channels, labels, and loss tables; it does
not infer an unrestricted representation property from arbitrary data.
For nonnegative losses, a zero LP value gives a zero pointwise infimum
at every fixed confidence level; finite free probing and fresh fallback
establish global admissibility. A separate common-zero-action criterion
determines exact zero-cost attainment. These boundary results do not
extend the full positive-cost compatibility theorem to all zero-loss
profiles or to sample-constrained objectives.
Certification-driven switching and representation repair require further
analysis of changing channels and labels. The fixed-kernel result and the
perfect-audit boundary example do not establish those dynamic guarantees.

% ---- Technical appendices (jmlr2e titles each as "Appendix X.") ----
\appendix

% sec_lowerbound.tex -- JMLR full version, proof chapter:
%   "Proof of the Lower Bound (Theorem \ref{thm:LB})"
% Section content file to be \input into the jmlr2e main file. No preamble.
%
% Required from the shared preamble (do NOT redefine here):
%   amsthm with a shared counter: theorem, proposition, lemma, corollary,
%   definition, assumption, remark, example (consecutive numbering);
%   natbib (\citep/\citet); amsmath.
%
% Cross-references defined elsewhere: model:M1 (model tag), thm:UB
% (Section \ref{sec:proofub}), cor:main (Section \ref{sec:proofub}).
%
% NUMBERING NOTE FOR THE ASSEMBLER: under the assembly order
%   [static chapter (ends at Theorem 7), this file, sec_upperbound.tex]
% the statement block below is the first numbered environment of this
% file, so it renders as Theorem 8 automatically. For the shared-counter
% plan "Theorem 8 / Theorem 9 / Corollary 10, internal lemmas from 11" to
% hold literally, the delimited statement blocks of this file and of
% sec_upperbound.tex must precede all other numbered environments, i.e.
% be hoisted into the model/results section that comes before the proof
% sections; the blocks are delimited by %%% BEGIN/END STATEMENT BLOCK.

\section{Proof of the Lower Bound (Theorem \ref{thm:LB})}\label{sec:prooflb}

\paragraph{Guide to this section.}
This section gives a self-contained proof of Theorem~\ref{thm:LB} for
model~\eqref{model:M1}, including the transportation lemma, the
information constraint, the LP scaling argument, the equivalence of the
two forms of the certification complexity $C_i$, the $\liminf$ statement,
the three boundary examples B1/B2/B3, and the non-triviality note NT1.
The logical dependencies are as follows.
Subsection~\ref{subsec:lb:model} formalizes the model and restates the
hypotheses. Subsection~\ref{subsec:lb:transport} proves the
transportation Lemma~\ref{lem:transport}, built up through the
finite-horizon density (Lemma~\ref{lem:lrdensity}), the likelihood-ratio
martingale (Lemma~\ref{lem:lrmg}), the stopped density
(Lemma~\ref{lem:stopped}), a Wald-type identity (Lemma~\ref{lem:wald}),
and Bernoulli contraction (Lemma~\ref{lem:contraction}).
Subsection~\ref{subsec:lb:info} turns $\delta$-correctness into the
information constraint (Lemma~\ref{lem:infoconstraint}, via the
monotonicity Lemma~\ref{lem:klmono}) and proves the estimate
$\mathrm{kl}(\delta,1-\delta)\ge\log(1/(2.4\delta))$
(Lemma~\ref{lem:kcgest}); combined with Lemma~\ref{lem:transport} this
yields the per-alternative bound of Corollary~\ref{cor:peralt}.
Subsection~\ref{subsec:lb:lpscale} carries out the LP scaling and proves
Theorem~\ref{thm:LB}. Subsection~\ref{subsec:lb:equiv} proves the
equivalence of the LP and max--min forms of $C_i$
(Theorem~\ref{thm:equiv}), the statement used by the upper-bound design
of Appendix~\ref{sec:proofub}. Subsection~\ref{subsec:lb:liminf} derives
the $\liminf$ form (Corollary~\ref{cor:liminf}),
Subsection~\ref{subsec:lb:boundary} treats the boundary examples
B1/B2/B3 (Propositions~\ref{prop:B1},\ \ref{prop:B2},\ \ref{prop:B3}), and
Subsection~\ref{subsec:lb:nt1} proves the non-triviality note NT1
(Proposition~\ref{prop:NT1}). Subsection~\ref{subsec:lb:checklist}
collects the completeness self-check and the residual statements.

\paragraph{Provenance and claims.}
All lower-bound \emph{techniques} used here are standard tools from the
sequential identification literature: the transportation
(change-of-measure) inequality is that of \citet[Lemma~1]{kaufmann2016complexity}
applied under the isomorphism \emph{action} $\leftrightarrow$
\emph{arm}, \emph{channel output} $\leftrightarrow$ \emph{arm reward};
the inequality $\mathrm{kl}(\delta,1-\delta)\ge \log(1/(2.4\delta))$ is
the standard estimate of \citet{kaufmann2016complexity}; the two-point
degeneration B3 is compared against \citet[Theorem~5.1]{nitinawarat2015controlled};
the hopeless boundary B1 corresponds to the hopeless class of partial
monitoring \citep{bartok2014partial}. The contribution of the present
work is \emph{not} the lower-bound machinery; it is the model
itself---the alternative set $\mathrm{Alt}(i)$ defined through adequacy
labels and the coupling to task loss $g$---and the decision-theoretic
semantics of the constant $C(E,f)$.

\paragraph{Conventions.}
All logarithms are natural. For probability vectors $p,q$ on a finite
space, $d(p\|q):=\sum_y p(y)\log\frac{p(y)}{q(y)}$ with $0\log 0 = 0$,
$0\log(0/0)=0$, and $d(p\|q)=+\infty$ if $p\not\ll q$. For Bernoulli
laws, $\mathrm{kl}(p,q):=p\log\frac{p}{q}+(1-p)\log\frac{1-p}{1-q}$ for
$(p,q)\in(0,1)^2$, extended to $[0,1]^2$ by continuity where possible
and by $\mathrm{kl}(p,0)=+\infty$ for $p>0$,
$\mathrm{kl}(1,q)=\log(1/q)$, $\mathrm{kl}(0,q)=\log(1/(1-q))$,
$\mathrm{kl}(0,0)=\mathrm{kl}(1,1)=0$. Products with scalars follow
$0\cdot(+\infty)=0$ and $c\cdot(+\infty)=+\infty$ for $c>0$.

\subsection{Model \eqref{model:M1}: Complete Formalization}\label{subsec:lb:model}

\paragraph{Environments and labels.}
The environment class is finite: $\mathcal{E}=\{E_1,\dots,E_K\}$,
$K\ge 2$. The representation $f$ is fixed. For each $i$, the
\emph{adequacy label}
$\Theta_i := \mathbf{1}\{R^*_{E_i}(f(H)) = R^*_{E_i}(H)\}\in\{0,1\}$ is
a \textbf{deterministic constant} determined by the Bayes-risk framework
of the companion paper; the only randomness in the model comes from the
interaction transcript. We write $\Theta(i)$ and $\Theta_i$
interchangeably. The alternative set is
\[
\mathrm{Alt}(i) := \{\, j\in\{1,\dots,K\} : \Theta_j \neq \Theta_i \,\}.
\]
If $\mathrm{Alt}(i)=\varnothing$ (all environments share the label of
$i$), the certification question is trivial; we adopt the convention
$C_i=0$ in that case (the LP below has no constraints; the max--min form
has $\inf$ over an empty set equal to $+\infty$, whose reciprocal is
$0$). Henceforth $\mathrm{Alt}(i)\neq\varnothing$.

\paragraph{Actions, channels, losses.}
The action set $\mathcal{A}$ is finite. Under environment $E_i$, action
$a$ produces an observation through the channel $K_i^a$, a probability
distribution on the \textbf{finite} observation space $\mathcal{Y}$,
with mass function $k_i^a(y):=K_i^a(\{y\})$, and incurs a task loss
$g_i(a)\in[0,\infty)$ (audit costs are absorbed into $g$; e.g.\ a
\textsc{verify} action carries $c_{\mathrm{audit}}$ inside $g$).

\paragraph{Canonical probability space.}
Randomized policies are allowed through an external randomization
sequence. Let
\[
\Omega := \mathcal{Y}^{\mathbb{N}}\times[0,1]^{\mathbb{N}},\qquad \omega=\big((y_1,y_2,\dots),(u_1,u_2,\dots)\big),
\]
equipped with the product $\sigma$-field $\mathbb{F}$ generated by
cylinders ($\mathcal{Y}$ finite carries the discrete $\sigma$-field;
$[0,1]$ the Borel $\sigma$-field). Coordinate maps are
$Y_t(\omega)=y_t$, $U_t(\omega)=u_t$.

\paragraph{Policies.}
A policy is a sequence $\pi=(\pi_t)_{t\ge 1}$ of measurable maps
$\pi_t:(\mathcal{A}\times\mathcal{Y})^{t-1}\times[0,1]\to\mathcal{A}$.
The action process is defined recursively:
\[
A_1(\omega):=\pi_1(u_1),\qquad A_t(\omega):=\pi_t\big(A_1,Y_1,\dots,A_{t-1},Y_{t-1},u_t\big),\quad t\ge 2.
\]
Since all spaces involved are standard Borel (finite or $[0,1]$),
measurability is preserved by the recursion.

\paragraph{Environment measures.}
For each $i$, $P_i$ is the unique probability on $(\Omega,\mathbb{F})$
such that: (i) $(U_t)_{t\ge1}$ are i.i.d.\ $\mathrm{Unif}[0,1]$; (ii)
conditionally on $(A_1,Y_1,\dots,A_{t-1},Y_{t-1},U_t)$, the observation
$Y_t$ is independent of the past and of $U_{t+1},U_{t+2},\dots$ and
satisfies $P_i(Y_t=y\mid \cdot)=k_i^{A_t}(y)$. Existence and uniqueness
follow from the Ionescu--Tulcea theorem: the transition kernel from
histories to $(U_t,Y_t)$ is the Markov kernel
$\mathrm{Unif}[0,1](du)\otimes K_i^{\pi_t(\text{history},u)}(dy)$; all
kernels are regular because every space is standard Borel (finite spaces
and $[0,1]$). Point (ii) is precisely hypothesis \textbf{(H1)}
(conditional independence of channel outputs given the action history).
Because the $U_t$'s have the same law under every $P_i$, the
policy-induced randomization does not enter likelihood ratios between
environments (Lemma~\ref{lem:lrdensity}).

\paragraph{Filtration.}
Set $\mathcal{F}_0:=\{\varnothing,\Omega\}$ and
\[
\mathcal{F}_t := \sigma\big(U_1,\dots,U_{t+1},\,A_1,\dots,A_{t+1},\,Y_1,\dots,Y_t\big),\qquad t\ge 1.
\]
Then $A_{t+1}$ is $\mathcal{F}_t$-measurable, $Y_t$ is
$\mathcal{F}_t$-measurable, and $A_s,Y_s,U_s$ are
$\mathcal{F}_t$-measurable for $s\le t$. This is the natural filtration
of the interaction (with the usual harmless enlargement that reveals the
next randomization seed; it does not affect any argument below).

\paragraph{Stopping and terminal decision.}
A \emph{strategy} is a triple $(\pi,\tau,\hat\Theta)$ where $\tau$ is a
stopping time with respect to $(\mathcal{F}_t)$, and
$\hat\Theta:\Omega\to\{0,1\}$ is $\mathcal{F}_\tau$-measurable, with
\[
\mathcal{F}_\tau := \{\, F\in\mathbb{F} : F\cap\{\tau\le t\}\in\mathcal{F}_t\ \ \forall t\ge 0\,\},
\]
the standard stopped $\sigma$-field (it is a $\sigma$-field: closed
under complement and countable unions, checked directly from the
definition). A strategy is \textbf{admissible} if $\tau<\infty$
$P_i$-a.s.\ for \emph{every} $i=1,\dots,K$. (Without admissibility,
$\hat\Theta$ is undefined on $\{\tau=\infty\}$ and error probabilities
are not well-defined; admissibility is therefore part of the
correctness requirement, not an extra restriction. It is also what
excludes the vacuous ``never stop, never err'' strategy. The B2 construction in Proposition~\ref{prop:zero-infimum} also obeys
this global admissibility requirement, using a finite probe followed
when needed by a finite fallback.)

\paragraph{Counts and risk.}
For $a\in\mathcal{A}$,
$N_a(\tau):=\sum_{t=1}^{\tau}\mathbf{1}\{A_t=a\}=\sum_{t\ge1}\mathbf{1}\{t\le\tau,\,A_t=a\}\in[0,\infty]$
(the sum is well-defined since
$\{t\le\tau\}=\{\tau\le t-1\}^c\in\mathcal{F}_{t-1}$ and $A_t$ are
measurable). The expected certification cost (risk) of strategy
$(\pi,\tau,\hat\Theta)$ in $E_i$ is
\[
R(\pi,E_i) := E_i\Big[\sum_{t=1}^{\tau} g_i(A_t)\Big] = \sum_{a\in\mathcal{A}} g_i(a)\,E_i[N_a(\tau)],
\]
where the exchange of sum and expectation is Tonelli's theorem for
nonnegative summands (valid also when the value is $+\infty$).

\paragraph{$\delta$-correctness (fixed-confidence PAC).}
A strategy is \textbf{$\delta$-correct} if it is admissible and
\[
\max_{1\le i\le K}\; P_i\big(\hat\Theta \neq \Theta_i\big)\;\le\;\delta.
\]
Recall $\Theta_i$ is deterministic, so
$\{\hat\Theta\neq\Theta_i\}\in\mathcal{F}_\tau$ is an ordinary event.

\paragraph{Hypotheses (restated).}
The lower bound uses hypotheses (H1)--(H3) of the main text, restated
here for convenience in the numbered environments below.
\begin{assumption}\label{asm:H1}
Channel conditional independence, as built into the construction of
$P_i$ above. \textup{(}This is hypothesis \textup{(H1)} of the main
text.\textup{)}
\end{assumption}
\begin{assumption}\label{asm:H2}
Feasibility: for every $i$ and every $j\in\mathrm{Alt}(i)$ there exists
$a\in\mathcal{A}$ with $0<d(K_i^a\|K_j^a)<+\infty$.
\textup{(}This is hypothesis \textup{(H2)} of the main text. Only the
B1-type violation---some $j\in\mathrm{Alt}(i)$ with $d_a^j=0$ for
\emph{all} $a$---implies hopelessness; see
Subsection~\ref{subsec:lb:B1}. A violation in the $d=+\infty$ direction
is harmless, since infinite divergence only eases discrimination.
Theorem~\ref{thm:LB} in fact needs only $C_i<\infty$; the exact
characterization of hopelessness is the infeasibility of
\eqref{eq:lp} (Subsection~\ref{subsec:lb:equiv}, infeasibility
step).\textup{)}
\end{assumption}
\begin{assumption}\label{asm:H3}
$g_i(a)\in[0,\infty)$ for all $i,a$. \textup{(}This is hypothesis
\textup{(H3)} of the main text. The degenerate case $g_i(a)=0$ is
treated by LP limits: Part~2 of Subsection~\ref{subsec:lb:equiv} and
Subsection~\ref{subsec:lb:B2}.\textup{)}
\end{assumption}

\paragraph{KL notation.}
Throughout, for $j\in\mathrm{Alt}(i)$ we abbreviate
\[
d_a^j \;:=\; d\big(K_i^a\,\big\|\,K_j^a\big)\;\in\;[0,+\infty],
\]
the environment $i$ being fixed by context.

\subsection{The Transportation Lemma}\label{subsec:lb:transport}

\paragraph{Source.}
This is Lemma~1 of \citet{kaufmann2016complexity} under the isomorphism
\emph{action} $\leftrightarrow$ \emph{arm}, \emph{channel output}
$\leftrightarrow$ \emph{arm reward sample}; the proof below is a
complete, self-contained reproduction adapted to
model~\eqref{model:M1}. The condition checklist showing that the
isomorphism is legitimate appears in Subsubsection~\ref{subsubsec:lb:iso}.

\begin{lemma}[Transportation]\label{lem:transport}
Fix $i\neq j$. Let $(\pi,\tau,\hat\Theta)$ be admissible and let
$\mathcal{F}\in\mathcal{F}_\tau$. Then
\[
\sum_{a\in\mathcal{A}} E_i[N_a(\tau)]\cdot d_a^j \;\ge\; \mathrm{kl}\big(P_i(\mathcal{F}),\,P_j(\mathcal{F})\big),
\]
with the conventions of the preamble (in particular the left side is
$+\infty$ if any summand is, and the inequality is then trivial).
\end{lemma}

The proof occupies
Subsubsections~\ref{subsubsec:lb:reduction}--\ref{subsubsec:lb:dataproc}.
Fix the pair $(i,j)$ throughout.

\subsubsection{Reduction to the Absolutely Continuous Case}\label{subsubsec:lb:reduction}

Let $\mathcal{A}^0:=\{a: d_a^j=+\infty\}$. If $E_i[N_a(\tau)]>0$ for some
$a\in\mathcal{A}^0$, the left side of Lemma~\ref{lem:transport} is
$+\infty$ and there is nothing to prove. Otherwise $E_i[N_a(\tau)]=0$,
hence $P_i(N_a(\tau)=0)=1$, for every $a\in\mathcal{A}^0$; i.e.,
$P_i$-almost surely only actions in $\mathcal{A}_{ij}:=\{a:
d_a^j<\infty\}$ are played before $\tau$. For $a\in\mathcal{A}_{ij}$,
$d_a^j<\infty$ implies $k_i^a\ll k_j^a$, so the per-step log-likelihood
ratio
\[
Z_t := \log\frac{k_i^{A_t}(Y_t)}{k_j^{A_t}(Y_t)}\quad\text{is well-defined and bounded on the }P_i\text{-full event }\{A_t\in\mathcal{A}_{ij}\}:
\]
indeed $|Z_t|\le
B_{ij}:=\max_{a\in\mathcal{A}_{ij}}\max_{y:\,k_i^a(y)>0}\big|\log(k_i^a(y)/k_j^a(y))\big|<\infty$,
because $\mathcal{Y}$ and $\mathcal{A}_{ij}$ are finite and
$k_j^a(y)>0$ whenever $k_i^a(y)>0$. We work on the $P_i$-full event
that no action of $\mathcal{A}^0$ is played up to time $\tau$; all
statements below are understood modulo this event.

\subsubsection{Finite-Horizon Likelihood Ratio}\label{subsubsec:lb:finitelr}

\begin{lemma}\label{lem:lrdensity}
For every $t\ge 0$, on the event $E_t:=\{A_s\in\mathcal{A}_{ij}\ \
\forall s\le t\}$ that no action of $\mathcal{A}^0$ is played in the
first $t$ steps, $P_i$ is absolutely continuous with respect to $P_j$
on $\mathcal{F}_t$: precisely,
$P_i(F\cap E_t)=E_j[\mathbf{1}_{F\cap E_t}\,L_t]$ for every
$F\in\mathcal{F}_t$, with density
\[
L_t \;=\; \prod_{s=1}^{t}\frac{k_i^{A_s}(Y_s)}{k_j^{A_s}(Y_s)},\qquad L_0:=1.
\]
\textup{(}Whenever $E_i[N_a(\tau)]=0$ for all $a\in\mathcal{A}^0$,
Subsubsection~\ref{subsubsec:lb:reduction} gives $P_i(E_t)=1$ for every
$t$, so this is a restriction only modulo a $P_i$-null event;
equivalently, one may without loss of generality replace the strategy
by one that never plays $\mathcal{A}^0$ up to $\tau$.\textup{)}
\end{lemma}

\begin{proof}
By induction on $t$. The case $t=0$ is trivial
($\mathcal{F}_0=\{\varnothing,\Omega\}$, $E_0=\Omega$, $L_0=1$). Assume
the claim at $t-1$. The induction step below is applied with $G$
replaced by $G\cap E_t$: note
$E_t=E_{t-1}\cap\{A_t\in\mathcal{A}_{ij}\}$, where both factors are
$\mathcal{F}_{t-1}$-measurable (Subsection~\ref{subsec:lb:model}: $A_t$
is $\mathcal{F}_{t-1}$-measurable), and on
$\{A_t\in\mathcal{A}_{ij}\}$ one has $k_i^{A_t}\ll k_j^{A_t}$, so the
absolute-continuity computation is legitimate. The cylinder rectangles
\[
F = G\cap\{Y_t\in C_Y\}\cap\{U_{t+1}\in C_U\},\qquad G\in\sigma\big((A_s,Y_s)_{s\le t-1},U_1,\dots,U_t\big),\ C_Y\subseteq\mathcal{Y},\ C_U\ \text{Borel},
\]
form a $\pi$-system generating $\mathcal{F}_t$. For such $F$,
\[
\begin{aligned}
P_i(F)
&= E_i\Big[\mathbf{1}_G\,\mathbf{1}\{U_{t+1}\in C_U\}\,K_i^{A_t}(C_Y)\Big] \qquad(\text{\ref{asm:H1}; def.\ of }P_i)\\
&= P(U_{t+1}\in C_U)\cdot E_i\big[\mathbf{1}_G\,K_i^{A_t}(C_Y)\big] \qquad(U_{t+1}\ \perp\ \mathcal{F}_{t-1};\ \text{same $U$-law under }P_i,P_j)\\
&= P(U_{t+1}\in C_U)\cdot E_j\big[L_{t-1}\,\mathbf{1}_G\,K_i^{A_t}(C_Y)\big] \qquad(\text{induction hypothesis})\\
&= E_j\Big[L_{t-1}\,\mathbf{1}_G\,\mathbf{1}\{U_{t+1}\in C_U\}\,K_i^{A_t}(C_Y)\Big],
\end{aligned}
\]
where the induction hypothesis applies because $\mathbf{1}_G
K_i^{A_t}(C_Y)$ is $\mathcal{F}_{t-1}$-measurable
($G\in\sigma((A_s,Y_s)_{s\le t-1},U_1,\dots,U_t)\subseteq\mathcal{F}_{t-1}$
and $A_t$ is $\mathcal{F}_{t-1}$-measurable), and the last step
re-absorbs the independent factor $U_{t+1}$ under $P_j$. Next, for each
$a\in\mathcal{A}_{ij}$ and $C_Y\subseteq\mathcal{Y}$,
\[
K_i^a(C_Y)=\sum_{y\in C_Y} k_i^a(y)=\sum_{y\in C_Y} k_j^a(y)\,\frac{k_i^a(y)}{k_j^a(y)} = E_j\Big[\frac{k_i^{a}(Y)}{k_j^{a}(Y)}\mathbf{1}\{Y\in C_Y\}\Big],\quad Y\sim K_j^a,
\]
since terms with $k_j^a(y)=0$ satisfy $k_i^a(y)=0$ by absolute
continuity. Using that $A_t$ is measurable with respect to
$\sigma((A_s,Y_s)_{s\le t-1},U_t)$ and \ref{asm:H1} under $P_j$,
\[
P_i(F)=E_j\Big[L_{t-1}\,\mathbf{1}_G\,\mathbf{1}\{U_{t+1}\in C_U\}\,\frac{k_i^{A_t}(Y_t)}{k_j^{A_t}(Y_t)}\,\mathbf{1}\{Y_t\in C_Y\}\Big]=E_j\big[\mathbf{1}_F\,L_t\big].
\]
Both sides are measures in $F$ agreeing on a generating $\pi$-system
that contains $\Omega$ (take $G=\Omega$, $C_Y=\mathcal{Y}$,
$C_U=[0,1]$), hence they agree on all of $\mathcal{F}_t$ by the
$\pi$--$\lambda$ (uniqueness of measure) theorem.
\end{proof}

\begin{lemma}[Likelihood-ratio martingale]\label{lem:lrmg}
$(L_t)_{t\ge0}$ is a nonnegative $(\mathcal{F}_t)$-martingale under
$P_j$, with $E_j[L_t]=1$.
\end{lemma}

\begin{proof}
$L_t\ge0$ is $\mathcal{F}_t$-measurable and bounded by
$e^{B_{ij}t}<\infty$ on the relevant event
(Subsubsection~\ref{subsubsec:lb:reduction}), hence integrable. Since
$L_t=L_{t-1}\cdot e^{Z_t}$ and $L_{t-1}$ is
$\mathcal{F}_{t-1}$-measurable,
\[
\begin{aligned}
E_j[L_t\mid\mathcal{F}_{t-1}]&=L_{t-1}\,E_j[e^{Z_t}\mid\mathcal{F}_{t-1}],\\
E_j[e^{Z_t}\mid\mathcal{F}_{t-1}]
&=\sum_{a}\mathbf{1}\{A_t=a\}\sum_{y:\,k_j^a(y)>0}k_j^a(y)\frac{k_i^a(y)}{k_j^a(y)}
=\sum_a\mathbf{1}\{A_t=a\}\cdot 1=1,
\end{aligned}
\]
using \ref{asm:H1}, the $\mathcal{F}_{t-1}$-measurability of $A_t$, and
$k_i^a\ll k_j^a$ (so $\sum_{y:k_j^a>0}k_i^a(y)=1$). Taking $F=\Omega$
in Lemma~\ref{lem:lrdensity} (restricted to the $P_i$-full event $E_t$
of Subsubsection~\ref{subsubsec:lb:reduction}) gives
$E_j[L_t]=P_i(E_t)=1$.
\end{proof}

\subsubsection{Stopped Density}\label{subsubsec:lb:stopped}

\begin{lemma}\label{lem:stopped}
Let $\sigma$ be a stopping time bounded by a constant $n$. Then
$P_i\ll P_j$ on $\mathcal{F}_\sigma$ with density
$L_\sigma:=\sum_{t=0}^{n}\mathbf{1}\{\sigma=t\}L_t$.
\end{lemma}

\begin{proof}
For $F\in\mathcal{F}_\sigma$ and each $t\le n$,
$F\cap\{\sigma=t\}\in\mathcal{F}_t$ (definition of
$\mathcal{F}_\sigma$: $F\cap\{\sigma\le t\}\in\mathcal{F}_t$, and
$\{\sigma=t\}=\{\sigma\le t\}\setminus\{\sigma\le t-1\}$). Hence
\[
P_i(F)=\sum_{t=0}^{n}P_i\big(F\cap\{\sigma=t\}\big)=\sum_{t=0}^{n}E_j\big[\mathbf{1}_{F\cap\{\sigma=t\}}L_t\big]=E_j\big[\mathbf{1}_F L_\sigma\big],
\]
where the second equality is Lemma~\ref{lem:lrdensity} applied on
$\mathcal{F}_t$.
\end{proof}

\subsubsection{A Wald-Type Identity for the Information Accumulated by Bounded Stopping Times}\label{subsubsec:lb:wald}

\begin{lemma}\label{lem:wald}
For every $n\ge1$,
\[
E_i\big[\log L_{\tau\wedge n}\big]\;=\;\sum_{a\in\mathcal{A}} d_a^j\cdot E_i\big[N_a(\tau\wedge n)\big].
\]
\end{lemma}

\begin{proof}
Write $\log L_{\tau\wedge n}=\sum_{t=1}^{n}\mathbf{1}\{t\le\tau\wedge
n\}\,Z_t$, a finite sum of bounded variables ($|Z_t|\le B_{ij}$,
Subsubsection~\ref{subsubsec:lb:reduction}), so all expectations below
are finite and the interchange of sum and expectation is legitimate.
The event
$\{t\le\tau\wedge n\}=\{\tau\wedge n\le t-1\}^c=\{\tau\le t-1\}^c$
lies in $\mathcal{F}_{t-1}$ (since $\tau$ is a stopping time). By
\ref{asm:H1} and the $\mathcal{F}_{t-1}$-measurability of $A_t$,
\[
E_i[Z_t\mid\mathcal{F}_{t-1}]=\sum_{a}\mathbf{1}\{A_t=a\}\,E_{Y\sim K_i^a}\Big[\log\frac{k_i^a(Y)}{k_j^a(Y)}\Big]=\sum_{a}\mathbf{1}\{A_t=a\}\,d_a^j,
\]
the inner expectation being exactly $d_a^j$ (finite for
$a\in\mathcal{A}_{ij}$; on the $P_i$-null complement the convention of
Subsubsection~\ref{subsubsec:lb:reduction} applies). Therefore
\[
E_i\big[\mathbf{1}\{t\le\tau\wedge n\}Z_t\big]=E_i\big[\mathbf{1}\{t\le\tau\wedge n\}\,E_i[Z_t\mid\mathcal{F}_{t-1}]\big]=\sum_a d_a^j\,E_i\big[\mathbf{1}\{t\le\tau\wedge n,\,A_t=a\}\big].
\]
Summing over $t=1,\dots,n$ and noting
$\sum_{t=1}^n\mathbf{1}\{t\le\tau\wedge n,A_t=a\}=N_a(\tau\wedge n)$
term by term gives the claim.
\end{proof}

\subsubsection{Data Processing on $\mathcal{F}_{\tau\wedge n}$, and Passage to the Limit}\label{subsubsec:lb:dataproc}

\begin{lemma}[Bernoulli contraction]\label{lem:contraction}
Let $\sigma$ be bounded by $n$ and let $\mathcal{E}\in\mathcal{F}_\sigma$.
Then
\[
E_i[\log L_\sigma]\;\ge\;\mathrm{kl}\big(P_i(\mathcal{E}),\,P_j(\mathcal{E})\big).
\]
\end{lemma}

\begin{proof}
Let $\psi(x)=x\log x$ on $(0,\infty)$, $\psi(0):=0$; $\psi$ is convex
($\psi''(x)=1/x>0$). By Lemma~\ref{lem:stopped},
$E_i[g]=E_j[L_\sigma\,g]$ for every $\mathcal{F}_\sigma$-measurable $g$
with $E_i|g|<\infty$; taking $g=\log L_\sigma$ gives $E_i[\log
L_\sigma]=E_j[\psi(L_\sigma)]$ (finite by Lemma~\ref{lem:wald}). For
any $E\in\mathcal{F}_\sigma$ with $P_j(E)>0$, Jensen's inequality for
the regular conditional law $P_j(\cdot\mid E)$ gives
\[
\begin{aligned}
E_i[\mathbf{1}_E\log L_\sigma]
&=E_j[\mathbf{1}_E\,\psi(L_\sigma)]
=P_j(E)\,E_j\big[\psi(L_\sigma)\,\big|\,E\big]
\ge P_j(E)\,\psi\big(E_j[L_\sigma\mid E]\big)\\
&=P_j(E)\,\psi\Big(\frac{P_i(E)}{P_j(E)}\Big)
=P_i(E)\log\frac{P_i(E)}{P_j(E)},
\end{aligned}
\]
where $E_j[L_\sigma\mid E]=P_i(E)/P_j(E)$ follows from
Lemma~\ref{lem:stopped} ($P_i(E)=E_j[\mathbf{1}_E L_\sigma]$). If
$P_j(E)=0$ then $P_i(E)=0$ (absolute continuity,
Lemma~\ref{lem:stopped}) and both sides of the last inequality are $0$
by convention. Apply to $E=\mathcal{E}$ and $E=\mathcal{E}^c$ and add:
\[
E_i[\log L_\sigma]\ge P_i(\mathcal{E})\log\frac{P_i(\mathcal{E})}{P_j(\mathcal{E})}+P_i(\mathcal{E}^c)\log\frac{P_i(\mathcal{E}^c)}{P_j(\mathcal{E}^c)}=\mathrm{kl}\big(P_i(\mathcal{E}),P_j(\mathcal{E})\big).\qedhere
\]
\end{proof}

\begin{proof}[Proof of Lemma~\ref{lem:transport}]
For each $n\ge1$, $\sigma_n:=\tau\wedge n$ is bounded, and the
\emph{event} $\mathcal{E}_n:=\mathcal{F}\cap\{\tau\le n\}$ lies in
$\mathcal{F}_{\tau\wedge n}$: indeed, for $t<n$,
$\mathcal{E}_n\cap\{\tau\wedge n\le
t\}=\mathcal{F}\cap\{\tau\le t\}\in\mathcal{F}_t$ because
$\mathcal{F}\in\mathcal{F}_\tau$; for $t\ge n$, $\{\tau\wedge n\le
t\}=\Omega$ and $\mathcal{E}_n\in\mathcal{F}_n\subseteq\mathcal{F}_t$
(since $\mathcal{F}\in\mathcal{F}_\tau$ implies
$\mathcal{F}\cap\{\tau\le n\}\in\mathcal{F}_n$, the $\sigma$-field at
time $n$). Combining Lemmas~\ref{lem:wald} and~\ref{lem:contraction}
with $\sigma=\sigma_n$, $\mathcal{E}=\mathcal{E}_n$:
\begin{equation}
\sum_a d_a^j\,E_i[N_a(\tau\wedge n)]\;\ge\;\mathrm{kl}\big(P_i(\mathcal{E}_n),\,P_j(\mathcal{E}_n)\big)\qquad\forall n.\label{eq:lb:trunc}
\end{equation}
\emph{Left side.} Since $\tau<\infty$ $P_i$-a.s.\ (admissibility),
$N_a(\tau\wedge n)\uparrow N_a(\tau)$ $P_i$-a.s., so by monotone
convergence $E_i[N_a(\tau\wedge n)]\uparrow E_i[N_a(\tau)]\in[0,\infty]$,
and the left side of \eqref{eq:lb:trunc} converges to $\sum_a d_a^j
E_i[N_a(\tau)]$ (monotone limit of a finite sum; conventions of the
preamble). If this limit is $+\infty$, Lemma~\ref{lem:transport} holds
trivially.
\emph{Right side.} Since $\tau<\infty$ $P_i$-a.s.\ and $P_j$-a.s.\
(admissibility under \emph{every} environment, in particular $j$),
$\mathcal{E}_n\uparrow\mathcal{F}$, so by continuity of probability
measures from below,
\[
p_n:=P_i(\mathcal{E}_n)\to p:=P_i(\mathcal{F}),\qquad q_n:=P_j(\mathcal{E}_n)\to q:=P_j(\mathcal{F}).
\]
The function $\mathrm{kl}:[0,1]^2\to[0,\infty]$ (with our conventions)
is jointly lower semicontinuous: on $(0,1)^2$ it is continuous (sum of
continuous terms); at boundary points with finite value ($p=0$,
$q\in(0,1)$, etc.) it is extended by continuity; at $(p,0)$ with
$p>0$, any sequence $(p_n,q_n)\to(p,0)$ satisfies
$\mathrm{kl}(p_n,q_n)\ge p_n\log(p_n/q_n)+(1-p_n)\log(1-p_n)\to+\infty=\mathrm{kl}(p,0)$
(the first term dominates: $p_n\log p_n+(-p_n\log q_n)\to+\infty$ since
$-p_n\log q_n\to+\infty$, and the second term converges); the remaining
boundary points are analogous. Hence
$\liminf_n\mathrm{kl}(p_n,q_n)\ge\mathrm{kl}(p,q)$. Taking $\liminf$ in
\eqref{eq:lb:trunc}:
\[
\sum_a d_a^j\,E_i[N_a(\tau)]\;\ge\;\liminf_{n\to\infty}\mathrm{kl}(p_n,q_n)\;\ge\;\mathrm{kl}\big(P_i(\mathcal{F}),P_j(\mathcal{F})\big).\qedhere
\]
\end{proof}

\subsubsection{Isomorphism with \citet[Lemma~1]{kaufmann2016complexity}: Condition Checklist}\label{subsubsec:lb:iso}

Lemma~1 of \citet{kaufmann2016complexity} states, for a bandit model
with arms $a$ whose reward laws $\nu_a,\nu_a'$ satisfy
$\mathrm{KL}(\nu_a,\nu_a')<\infty$, and any almost surely finite
stopping time $\tau$: $\sum_a
E_\nu[N_a(\tau)]\,\mathrm{KL}(\nu_a,\nu_a')\ge
\mathrm{kl}(P_\nu(\mathcal{F}),P_{\nu'}(\mathcal{F}))$ for
$\mathcal{F}\in\mathcal{F}_\tau$. The correspondence is exact:

\begin{center}
\begin{tabular}{@{}p{0.44\textwidth}p{0.52\textwidth}@{}}
\hline
KCG16 bandit & model \eqref{model:M1} \\
\hline
arm $a$ & action $a\in\mathcal{A}$ \\
reward law $\nu_a$ (env.\ $\nu$) & channel $K_i^a$ \\
rewards of arm $a$ i.i.d., independent across arms given the pull sequence & \ref{asm:H1}: $Y_t\sim K_i^{A_t}$ conditionally on the history, independent of the past given $A_t$ \\
sampling rule selects $A_t$ from $H_{t-1}$ & policy $\pi_t(H_{t-1},U_t)$ (external randomization $U_t$ is environment-independent and cancels in likelihood ratios, Lemma~\ref{lem:lrdensity}) \\
$\mathrm{KL}(\nu_a,\nu_a')<\infty$ & $d_a^j<\infty$ on the $P_i$-a.s.\ played action set; otherwise the inequality is trivial (Subsubsection~\ref{subsubsec:lb:reduction}) \\
a.s.\ finite stopping time $\tau$ w.r.t.\ the interaction filtration & admissible $\tau$ (a.s.\ finite under \emph{all} $P_i$) \\
\hline
\end{tabular}
\end{center}

One technical strengthening relative to the statement in
\citet{kaufmann2016complexity}: our truncation argument
(Subsubsection~\ref{subsubsec:lb:dataproc}) requires only $P_i$-,
$P_j$-a.s.\ finiteness of $\tau$---no integrability assumption on
$\tau$ is needed---because the limit is taken with monotone convergence
on the left of \eqref{eq:lb:trunc} and lower semicontinuity of
$\mathrm{kl}$ on the right. Finite observation space $\mathcal{Y}$
gives bounded per-step log-ratios
(Subsubsection~\ref{subsubsec:lb:reduction}), which makes the Wald
identity (Lemma~\ref{lem:wald}) hold without extra moment conditions.

\subsection{From $\delta$-Correctness to the Information Constraint}\label{subsec:lb:info}

\begin{lemma}[Monotonicity of $\mathrm{kl}$]\label{lem:klmono}
\textup{(a)} For fixed $q\in(0,1)$, $p\mapsto\mathrm{kl}(p,q)$ is
strictly increasing on $[q,1]$. \textup{(b)} For fixed $p\in(0,1)$,
$q\mapsto\mathrm{kl}(p,q)$ is strictly decreasing on $(0,p]$.
\end{lemma}

\begin{proof}
Differentiate $\mathrm{kl}(p,q)=p\log\frac pq+(1-p)\log\frac{1-p}{1-q}$:
\[
\frac{\partial\,\mathrm{kl}}{\partial p}=\log\frac pq-\log\frac{1-p}{1-q}=\log\frac{p(1-q)}{q(1-p)},\qquad
\frac{\partial\,\mathrm{kl}}{\partial q}=-\frac pq+\frac{1-p}{1-q}=\frac{q-p}{q(1-q)}.
\]
For (a): if $p>q$ then $p(1-q)-q(1-p)=p-q>0$, so the log is positive
and $\partial\mathrm{kl}/\partial p>0$; strict increase follows, and
extends to the endpoint $p=q$ (where $\mathrm{kl}=0$) by continuity.
For (b): if $q<p$ then
$\partial\mathrm{kl}/\partial q=(q-p)/(q(1-q))<0$.
\end{proof}

\begin{lemma}[Information constraint]\label{lem:infoconstraint}
Let $\delta\in(0,\tfrac12)$, let $(\pi,\tau,\hat\Theta)$ be
$\delta$-correct, fix $i$, and set
$\mathcal{F}:=\{\hat\Theta=\Theta_i\}$. Then
$\mathcal{F}\in\mathcal{F}_\tau$ and, for every $j\in\mathrm{Alt}(i)$,
\[
\mathrm{kl}\big(P_i(\mathcal{F}),P_j(\mathcal{F})\big)\;\ge\;\mathrm{kl}(1-\delta,\delta)\;=\;\mathrm{kl}(\delta,1-\delta)\;>\;0.
\]
\end{lemma}

\begin{proof}
$\mathcal{F}\in\mathcal{F}_\tau$ because $\hat\Theta$ is
$\mathcal{F}_\tau$-measurable by definition of a strategy. By
$\delta$-correctness, $P_i(\mathcal{F})=1-P_i(\hat\Theta\neq\Theta_i)\ge
1-\delta$. For $j\in\mathrm{Alt}(i)$, $\Theta_j\neq\Theta_i$, so
$\{\hat\Theta=\Theta_i\}\subseteq\{\hat\Theta\neq\Theta_j\}$ and
$P_j(\mathcal{F})\le P_j(\hat\Theta\neq\Theta_j)\le\delta$. Since
$\delta<\tfrac12$ we have $P_i(\mathcal{F})\ge 1-\delta>\delta\ge
P_j(\mathcal{F})$; if $P_j(\mathcal{F})=0$ then
$\mathrm{kl}(P_i(\mathcal{F}),0)=+\infty$ and the claim is trivial;
otherwise two applications of Lemma~\ref{lem:klmono} (first (a) with
$p\downarrow$ to $1-\delta$, then (b) with $q\uparrow$ to $\delta$)
give
\[
\mathrm{kl}\big(P_i(\mathcal{F}),P_j(\mathcal{F})\big)\ge\mathrm{kl}\big(1-\delta,P_j(\mathcal{F})\big)\ge\mathrm{kl}(1-\delta,\delta).
\]
Symmetry:
$\mathrm{kl}(1-\delta,\delta)=(1-\delta)\log\frac{1-\delta}{\delta}+\delta\log\frac{\delta}{1-\delta}=\mathrm{kl}(\delta,1-\delta)$
by swapping the two summands. Positivity: $\delta\neq\tfrac12$, and
$\mathrm{kl}(p,q)=0$ iff $p=q$ (Gibbs' inequality:
$\mathrm{kl}(p,q)=d(\mathrm{Ber}(p)\|\mathrm{Ber}(q))\ge0$ with equality
iff the laws agree).
\end{proof}

\begin{remark}[The role of $\delta<\tfrac12$]\label{rem:deltahalf}
The constraint set $\{p\ge1-\delta,\ q\le\delta\}$ intersects the
diagonal $\{p=q\}$ when $\delta>\tfrac12$ (e.g.\ $p=q=\tfrac12$ is
admissible for $\delta=\tfrac34$ and gives $\mathrm{kl}=0$), so no
positive information lower bound follows from $\delta$-correctness
alone for $\delta>\tfrac12$. The fixed-confidence regime is
$\delta\in(0,\tfrac12)$, and Theorem~\ref{thm:LB} is stated in that
regime; this restriction is used only through
Lemma~\ref{lem:klmono}. The boundary $\delta=\tfrac12$ gives
$\mathrm{kl}=0$ and a vacuous bound.
\end{remark}

\begin{lemma}[KCG16 estimate]\label{lem:kcgest}
For every $\delta\in(0,1)$,
\[
\mathrm{kl}(\delta,1-\delta)\;\ge\;\log\frac{1}{2.4\,\delta}.
\]
\end{lemma}

\begin{proof}
If $\delta\ge 1/2.4$ the right side is $\le 0$ while $\mathrm{kl}\ge0$,
so assume $\delta\in(0,1/2.4)$. We in fact prove the inequality on all
of $(0,1)$. Define, for $\delta\in(0,1)$,
\[
h(\delta):=\mathrm{kl}(\delta,1-\delta)+\log(2.4\delta).
\]
Using
$\mathrm{kl}(\delta,1-\delta)=\delta\log\frac{\delta}{1-\delta}+(1-\delta)\log\frac{1-\delta}{\delta}$,
add $\log\delta=\delta\log\delta+(1-\delta)\log\delta$:
\[
h(\delta)=\underbrace{\delta\log\frac{\delta^2}{1-\delta}+(1-\delta)\log(1-\delta)}_{=:\,\phi(\delta)}+\log 2.4.
\]
\emph{Critical points.} With $u:=\delta/(1-\delta)\in(0,\infty)$
(strictly increasing in $\delta$),
\[
\phi'(\delta)=\log\frac{\delta^2}{1-\delta}+\delta\cdot\frac{d}{d\delta}\log\frac{\delta^2}{1-\delta}-\log(1-\delta)-1
=2\log\frac{\delta}{1-\delta}+\Big(1+\frac{\delta}{1-\delta}\Big)=2\log u+1+u,
\]
where $\frac{d}{d\delta}\log\frac{\delta^2}{1-\delta}=\frac2\delta+\frac1{1-\delta}$
and $\delta\big(\frac2\delta+\frac1{1-\delta}\big)=2+\frac{\delta}{1-\delta}$.
The function $\delta\mapsto 2\log u+1+u$ is strictly increasing (both
$u$ and $\log u$ are strictly increasing in $\delta$), tends to
$-\infty$ as $\delta\to0$ and to $+\infty$ as $\delta\to1$; hence
$\phi$, and therefore $h$, has a \emph{unique} critical point
$\delta^*\in(0,1)$, which is the global minimizer (note
$h(0^+)=\log 2.4>0$ and $h(1^-)=+\infty$). Let
$u^*:=\delta^*/(1-\delta^*)$ be the unique root of
\[
\psi(u):=u+2\log u+1=0,\qquad u>0,
\]
unique because $\psi'(u)=1+2/u>0$, $\psi(0^+)=-\infty$,
$\psi(+\infty)=+\infty$. At the critical point $2\log u^*=-(1+u^*)$, so
with $\delta^*=u^*/(1+u^*)$, $1-\delta^*=1/(1+u^*)$:
\[
\begin{aligned}
\phi(\delta^*)
&=\frac{2u^*\log u^*-(1+u^*)\log(1+u^*)}{1+u^*}
=\frac{-u^*(1+u^*)-(1+u^*)\log(1+u^*)}{1+u^*}\\
&=-u^*-\log(1+u^*),
\end{aligned}
\]
where the first equality rewrites
$\phi(\delta)=\big[2u\log u-(1+u)\log(1+u)\big]/(1+u)$ using
$\delta^2/(1-\delta)=u^2/(1+u)$ and $\delta=u/(1+u)$. Therefore
\[
\min_{\delta\in(0,1)} h(\delta)=\log 2.4-u^*-\log(1+u^*)=\log\frac{2.4}{(1+u^*)\,e^{u^*}}.
\]
\emph{Locating $u^*$.} $\psi$ is strictly increasing, and:
\begin{itemize}
\item $\psi(0.47)=1.47+2\log 0.47<0$ because $\log 0.47<-0.735$: indeed
$\log 0.47<-0.735\iff e^{0.735}<\frac1{0.47}=2.12766$, and truncating
the exponential series with a geometric tail bound,
\[
e^{0.735}\le\sum_{k=0}^{14}\frac{0.735^k}{k!}+\frac{0.735^{15}/15!}{1-0.735/16}=2.08548\ldots<2.12766;
\]
\item $\psi(0.48)=1.48+2\log 0.48>0$ because $\log
0.48>-0.74\iff e^{0.74}>\frac1{0.48}=2.08334$, and the (termwise
positive) partial Taylor sum already gives
$e^{0.74}\ge\sum_{k=0}^{12}\frac{0.74^k}{k!}=2.09594\ldots>2.08334$.
\end{itemize}
Hence $u^*\in(0.47,\,0.48)$. Finally,
\[
\begin{aligned}
(1+u^*)\,e^{u^*}\;
&\le\;1.48\,e^{0.48}\;\le\;1.48\Big(\sum_{k=0}^{8}\frac{0.48^k}{k!}+\frac{0.48^{9}/9!}{1-0.48/10}\Big)\\
&=1.48\times1.616075\ldots=2.39179\ldots<2.4,
\end{aligned}
\]
so $\min h>0$, i.e.\ $h(\delta)>0$ for all $\delta\in(0,1)$, which
rearranges to the claim. (The slack is small but strictly positive:
numerically $\min h\approx 0.00733$ at $\delta^*\approx 0.3233$; the
Taylor bounds above constitute a rigorous proof.)
\end{proof}

\begin{corollary}[Per-alternative information bound]\label{cor:peralt}
Under the hypotheses of Lemma~\ref{lem:infoconstraint}, for every
$j\in\mathrm{Alt}(i)$,
\[
\sum_{a\in\mathcal{A}} E_i[N_a(\tau)]\cdot d_a^j\;\ge\;\mathrm{kl}(\delta,1-\delta)\;\ge\;\log\frac{1}{2.4\delta}.
\]
\end{corollary}

\begin{proof}
Lemma~\ref{lem:transport} with $\mathcal{F}=\{\hat\Theta=\Theta_i\}$,
then Lemma~\ref{lem:infoconstraint}, then Lemma~\ref{lem:kcgest}.
\end{proof}

\subsection{LP Scaling: Proof of Theorem \ref{thm:LB}}\label{subsec:lb:lpscale}

\begin{definition}[Certification complexity, LP (covering) form]\label{def:complexity}
The \emph{certification complexity} of environment $i$ is
\begin{equation}
C_i \;=\; \inf\Big\{\sum_{a} g_i(a)\,n_a :\ n_a\ge0\ \forall a,\quad \sum_a n_a\,d_a^j\ \ge\ 1\ \ \forall j\in\mathrm{Alt}(i)\Big\},\tag{LP}\label{eq:lp}
\end{equation}
with the convention $C_i=+\infty$ if \eqref{eq:lp} is infeasible and
$C_i=0$ if $\mathrm{Alt}(i)=\varnothing$.
\end{definition}

Fix $i$. Under \ref{asm:H2}, \eqref{eq:lp} is feasible: for each
$j\in\mathrm{Alt}(i)$ pick $a(j)$ with $0<d_{a(j)}^j<\infty$ and set
$n:=t\sum_{j}e_{a(j)}$ with $t:=\max_j 1/d_{a(j)}^j$; then $\sum_a n_a
d_a^j\ge t\,d_{a(j)}^j\ge1$ for every $j$, and the objective is finite.
Hence \ref{asm:H2} implies $C_i<\infty$; also $C_i\ge0$ always.

\begin{proof}[Proof of Theorem~\ref{thm:LB}]
The second inequality is Lemma~\ref{lem:kcgest} combined with the first
(and $C_i\ge0$). We prove the first. If $C_i=0$ there is nothing to
prove ($R\ge0$); if $\mathrm{Alt}(i)=\varnothing$, $C_i=0$ by
convention. So assume $C_i>0$ and $\mathrm{Alt}(i)\neq\varnothing$.
Write $n_a:=E_i[N_a(\tau)]\in[0,\infty]$ and
$L:=\mathrm{kl}(\delta,1-\delta)>0$. By Corollary~\ref{cor:peralt},
\begin{equation}
\sum_a n_a\,d_a^j\;\ge\;L\qquad\forall j\in\mathrm{Alt}(i).\label{eq:lb:scale41}
\end{equation}

\textbf{Case A: $E_i[\tau]<\infty$.} Then $n_a\le E_i[\tau]<\infty$ for
every $a$. Define $m_a:=n_a/L\in[0,\infty)$. By \eqref{eq:lb:scale41},
$\sum_a m_a d_a^j\ge 1$ for all $j$, i.e.\ $m$ is feasible for
\eqref{eq:lp}; hence $\sum_a g_i(a)m_a\ge C_i$. Multiplying by $L$ and
using the Tonelli identity of Subsection~\ref{subsec:lb:model},
\begin{equation}
R(\pi,E_i)=\sum_a g_i(a)\,n_a=L\sum_a g_i(a)m_a\ge L\,C_i.\label{eq:lb:caseA}
\end{equation}

\textbf{Case B: $E_i[\tau]=\infty$.} Two sub-cases.

\emph{(a) $\underline g:=\min_a g_i(a)>0$.} Then
\[
R(\pi,E_i)=\sum_a g_i(a)n_a\ge \underline g\sum_a n_a=\underline g\,E_i[\tau]=+\infty\ge L\,C_i,
\]
since $C_i<\infty$ under \ref{asm:H2}. The bound holds trivially.

\emph{(b) Some $g_i(a)=0$.} Apply Lemma~\ref{lem:transport} to the
bounded stopping time $\tau\wedge m$ and the event
$\mathcal{E}_m:=\mathcal{F}\cap\{\tau\le m\}\in\mathcal{F}_{\tau\wedge
m}$, where $\mathcal{F}=\{\hat\Theta=\Theta_i\}$ (membership proved in
Subsubsection~\ref{subsubsec:lb:dataproc}). Exactly as in the proof of
Lemma~\ref{lem:transport},
\[
\sum_a E_i[N_a(\tau\wedge m)]\,d_a^j\;\ge\;\mathrm{kl}\big(p_m,q_m\big),\qquad p_m:=P_i(\mathcal{E}_m),\ q_m:=P_j(\mathcal{E}_m),
\]
for every $j\in\mathrm{Alt}(i)$ and every $m\ge1$. Since
$\mathcal{E}_m\uparrow\mathcal{F}$ (admissibility under $P_i$ and
$P_j$), $p_m\to P_i(\mathcal{F})\ge1-\delta$ and $q_m\to
P_j(\mathcal{F})\le\delta$ (continuity from below; the inequalities are
Lemma~\ref{lem:infoconstraint}'s). Fix
$\varepsilon\in(0,\tfrac12-\delta)$. Taking a maximum over the finite
set $\mathrm{Alt}(i)$, there exists one $m_0$ such that for
all $m\ge m_0$: $p_m\ge 1-\delta-\varepsilon$ and
$q_m\le\delta+\varepsilon$, with $1-\delta-\varepsilon>\delta+\varepsilon$;
hence by Lemma~\ref{lem:klmono} (and the boundary convention as in
Lemma~\ref{lem:infoconstraint}),
\[
\mathrm{kl}(p_m,q_m)\ge \mathrm{kl}(1-\delta-\varepsilon,\ \delta+\varepsilon)=:L_\varepsilon>0\qquad\forall m\ge m_0.
\]
Now $E_i[N_a(\tau\wedge m)]\le m<\infty$, so
$m_a^{(m)}:=E_i[N_a(\tau\wedge m)]/L_\varepsilon$ is feasible for
\eqref{eq:lp}, giving
\[
\sum_a g_i(a)\,E_i[N_a(\tau\wedge m)]\;\ge\;L_\varepsilon\,C_i\qquad\forall m\ge m_0.
\]
Let $m\to\infty$: $E_i[N_a(\tau\wedge m)]\uparrow n_a$ by monotone
convergence, the (finite) sum passes to the limit with values in
$[0,\infty]$, and we obtain $R(\pi,E_i)\ge L_\varepsilon C_i$. Finally
let $\varepsilon\downarrow0$:
$(1-\delta-\varepsilon,\delta+\varepsilon)\to(1-\delta,\delta)\in(0,1)^2$
is an interior point where $\mathrm{kl}$ is continuous, so
$L_\varepsilon\to\mathrm{kl}(1-\delta,\delta)=L$, and $R(\pi,E_i)\ge
L\,C_i$.
\end{proof}

\begin{remark}\label{rem:etau}
Case B establishes Theorem~\ref{thm:LB} for the full admissible class,
without an expected-time assumption: if $E_i[\tau]=\infty$ the bound is either trivial (sub-case
(a)) or recovered by truncation and monotone convergence (sub-case
(b)). Concerning the role of \ref{asm:H2}: the quantifier ``for every
$\delta$-correct strategy'' is vacuously correct only when \ref{asm:H2}
fails in the B1 way---i.e.\ some $j\in\mathrm{Alt}(i)$ has $d_a^j=0$
for \emph{every} $a$ (Subsection~\ref{subsec:lb:B1}: no $\delta$-correct
admissible strategy exists for $\delta<\tfrac12$, consistently with
\eqref{eq:lp} being infeasible and $C_i=+\infty$). Failure of
\ref{asm:H2} in the $d=+\infty$ direction ($d_a^j=+\infty$ for some
$a$, while some other $a'$ has $d_{a'}^j>0$) is harmless---infinite
divergence only makes $i$ and $j$ easier to distinguish---and the proof
of Theorem~\ref{thm:LB} in fact requires only $C_i<\infty$; the exact
characterization of hopelessness is the infeasibility of \eqref{eq:lp}
(Subsection~\ref{subsec:lb:equiv}, infeasibility step), not the failure
of \ref{asm:H2} per se.
\end{remark}

\subsection{Equivalence of the Two Forms of $C_i$}\label{subsec:lb:equiv}

Fix $i$ with $\mathrm{Alt}(i)=\{j_1,\dots,j_m\}\neq\varnothing$; write
$g_a:=g_i(a)$, and $d_a^j$ as before. Define
\begin{equation}
M \;:=\; \sup_{w\in\Delta(\mathcal{A})}\;\inf_{j\in\mathrm{Alt}(i)}\;\sum_{a\in\mathcal{A}}\frac{w_a}{g_a}\,d_a^j,\tag{MM}\label{eq:mm}
\end{equation}
where $\Delta(\mathcal{A}),\Delta(\mathrm{Alt})$ are probability
simplices, and the ratio $w_a/g_a$ is defined by the \textbf{limit
convention}
\[
\frac{w_a}{g_a}:=\lim_{\varepsilon\downarrow0}\frac{w_a}{g_a+\varepsilon}\in[0,+\infty]\qquad\Big(=\;+\infty\ \text{if }w_a>0,\,g_a=0;\quad =0\ \text{if }w_a=0\Big),
\]
and likewise $d_a^j\in[0,+\infty]$ terms are read as
$\lim_{T\to\infty}\min(d_a^j,T)$. The max--min form of the complexity
is $C_i^{\mathrm{MM}}:=[M]^{-1}$, with $1/0=+\infty$ and
$1/(+\infty)=0$.

\begin{theorem}[Equivalence of the two forms of $C_i$]\label{thm:equiv}
Under \ref{asm:H2} and \ref{asm:H3}, $C_i^{LP}=C_i^{MM}$, i.e.\
\[
C_i\;=\;\Big[\sup_{w\in\Delta(\mathcal{A})}\inf_{j\in\mathrm{Alt}(i)}\sum_a\frac{w_a}{g_a}\,d_a^j\Big]^{-1}.
\]
Moreover, whenever all $d_a^j<\infty$ (and $g_a\ge0$) the LP value is
attained; the $\sup$ in \eqref{eq:mm} is attained in the regular case
below. With $d_a^j=+\infty$ allowed, the LP value need not be
attained---e.g.\ $d^j=(+\infty,1)$, $g=(1,1)$: every $\varepsilon\,e_1$
is feasible with objective $\varepsilon\downarrow0$, so $C_i=0$ but no
feasible $n$ achieves it---yet the identity $C_i=1/M$ under the limit
conventions, and the equivalence $C_i=+\infty$ iff \eqref{eq:lp} is
infeasible iff $M=0$, remain valid.
\end{theorem}

The proof proceeds in three parts: Part~1 treats the \textbf{regular
case} $g_a>0$ and $d_a^j<\infty$ for all $a,j$; Part~2 extends to
$g_a\ge0$ by perturbation; Part~3 extends to $d_a^j\in[0,\infty]$ by
truncation and settles the infeasibility convention.

\subsubsection{Part 1: Regular Case ($g_a>0$, $d_a^j<\infty$ for all $a,j$)}

\paragraph{Step 1.1 (primal LP: attainment, finiteness, positivity).}
Feasibility holds as shown above (without \ref{asm:H2}'s strictness,
since now all $d_a^j$ are finite, feasibility needs for each $j$ some
$a$ with $d_a^j>0$, which \ref{asm:H2} supplies). The objective
$n\mapsto\sum_a g_a n_a$ is continuous; the feasible set
$\{n\ge0:\sum_a n_a d_a^j\ge1\ \forall j\}$ is closed and nonempty;
since $g_a>0$ for all $a$, the sublevel set $\{n\ge0:\sum_a g_a n_a\le
c\}$ is bounded, hence its intersection with the feasible set is
compact for $c$ larger than the objective value at any feasible point;
the infimum $C$ is therefore attained at some $n^*$, with $C<\infty$.
Moreover $C>0$: if $C=0$ then $\sum_a g_a n_a^*=0$ forces $n^*=0$ (all
$g_a>0$), which violates every constraint $\sum_a n_a^* d_a^j\ge1$.

\paragraph{Step 1.2 (reparametrization of the LP as a max--min).}
For $w\in\Delta(\mathcal{A})$ define
\[
M(w):=\min_{j}\ \sum_a\frac{w_a}{g_a}\,d_a^j\ \in\ [0,\infty).
\]
There is a bijection between feasible $n$ with $t:=\sum_a g_a n_a>0$
and pairs $(w,t)\in\Delta(\mathcal{A})\times(0,\infty)$ with $M(w)\ge
1/t$: set $w_a:=g_a n_a/t$ (then $w\ge0$, $\sum_a w_a=1$) and compute
\[
\sum_a n_a d_a^j\ \ge\ 1\iff t\sum_a\frac{w_a}{g_a}d_a^j\ \ge\ 1\iff \sum_a\frac{w_a}{g_a}d_a^j\ge\frac1t;
\]
minimizing over $j$ turns the family of constraints into the single
inequality $M(w)\ge 1/t$. Since every feasible $n\neq0$ has $t>0$ and
$n=0$ is infeasible,
\begin{equation}
C=\inf\big\{\,t:\ \exists\, w\in\Delta(\mathcal{A}),\ M(w)\ge 1/t\,\big\}=\inf_{w:\,M(w)>0}\frac{1}{M(w)}=\frac{1}{\sup_{w}M(w)}.\label{eq:lb:reparam}
\end{equation}
The $\sup$ is attained: $M(w)=\min_j\sum_a(w_a/g_a)d_a^j$ is the
minimum of finitely many linear (hence continuous) functions of $w$ on
the compact simplex, hence continuous, so $w^*\in\arg\max M(w)$ exists;
and $M^*:=M(w^*)>0$, because with
$w^\circ:=\frac1m\sum_{j}e_{a(j)}$ (uniform over the distinguishing
actions supplied by \ref{asm:H2}),
\[
M(w^\circ)=\min_j\sum_a\frac{w^\circ_a}{g_a}d_a^j\ \ge\ \min_j\ \frac{1}{m\,g_{a(j)}}\,d_{a(j)}^j\ >\ 0.
\]
Hence $C=1/M^*\in(0,\infty)$.

\paragraph{Step 1.3 (Sion minimax: condition checklist).}
Consider the two-player zero-sum game on
$\Delta(\mathcal{A})\times\Delta(\mathrm{Alt})$ with payoff
\[
F(w,\mu):=\sum_{a}\sum_{j}w_a\,\mu_j\,\frac{d_a^j}{g_a}.
\]
Checklist for Sion's minimax theorem (Sion 1958): (i)
$\Delta(\mathcal{A})$ and $\Delta(\mathrm{Alt})$ are \textbf{compact
convex} subsets of Euclidean spaces (finite-dimensional simplices);
(ii) $F$ is \textbf{jointly continuous} (bilinear with finite
coefficients $d_a^j/g_a$); (iii) $w\mapsto F(w,\mu)$ is
\textbf{convex} for every fixed $\mu$ (it is linear); (iv) $\mu\mapsto
F(w,\mu)$ is \textbf{concave} for every fixed $w$ (it is linear). All
conditions hold, so
\[
\sup_{w}\inf_{\mu}F(w,\mu)=\inf_{\mu}\sup_{w}F(w,\mu).
\]
Both extremizations over a simplex of a linear function are attained at
vertices: $\inf_\mu F(w,\mu)=\min_j F(w,e_j)=\min_j\sum_a(w_a/g_a)d_a^j=M(w)$
(the minimum of a linear function over the simplex equals the minimum
over its vertices $e_j$: $F(w,\mu)=\sum_j\mu_j F(w,e_j)\ge\min_j
F(w,e_j)$ with equality at the minimizing vertex); similarly $\sup_w
F(w,\mu)=\max_a\sum_j\mu_j d_a^j/g_a$. Therefore
\begin{equation}
M^*=\sup_w M(w)=\inf_{\mu\in\Delta(\mathrm{Alt})}\ \max_{a}\ \frac{\sum_j\mu_j d_a^j}{g_a}.\label{eq:lb:sion}
\end{equation}

\paragraph{Step 1.4 (explicit dual LP and equality of values).}
The Lagrangian/LP dual of \eqref{eq:lp} is
\begin{equation}
D^*:=\sup\Big\{\sum_{j}\lambda_j:\ \lambda_j\ge0\ \forall j,\quad \sum_j\lambda_j d_a^j\le g_a\ \forall a\Big\}.\tag{D}\label{eq:lpdual}
\end{equation}
\emph{Weak duality} (one line, no assumptions): for primal-feasible $n$
and dual-feasible $\lambda$,
\[
\sum_j\lambda_j\ \le\ \sum_j\lambda_j\sum_a n_a d_a^j=\sum_a n_a\sum_j\lambda_j d_a^j\ \le\ \sum_a n_a g_a,
\]
the first inequality using $\sum_a n_a d_a^j\ge1$ and $\lambda_j\ge0$,
the second using $\sum_j\lambda_j d_a^j\le g_a$ and $n_a\ge0$; hence
$D^*\le C$. \emph{Strong duality} is obtained constructively through
Steps~1.2--1.3: reparametrize dual-feasible $\lambda\neq0$ as
$\lambda=s\mu$ with $s:=\sum_j\lambda_j>0$ and
$\mu:=\lambda/s\in\Delta(\mathrm{Alt})$; the dual constraints read
\[
s\sum_j\mu_j d_a^j\le g_a\ \ \forall a\iff s\,\max_a\frac{\sum_j\mu_j d_a^j}{g_a}\le1\iff s\le\Big[\max_a\frac{\sum_j\mu_j d_a^j}{g_a}\Big]^{-1}
\]
(the max is positive for every $\mu$: for $\mu=e_j$ it is $\max_a
d_a^j/g_a>0$ by \ref{asm:H2}, and a general $\mu$ puts mass $\mu_j>0$
on some $j$, giving $\sum_j\mu_j d_a^j\ge\mu_j d_a^j$, positive at
$a=a(j)$). Hence
\begin{equation}
D^*=\sup_{\mu\in\Delta(\mathrm{Alt})}\Big[\max_a\frac{\sum_j\mu_j d_a^j}{g_a}\Big]^{-1}=\Big[\inf_{\mu}\max_a\frac{\sum_j\mu_j d_a^j}{g_a}\Big]^{-1}\overset{\eqref{eq:lb:sion}}{=}\frac1{M^*}\overset{\eqref{eq:lb:reparam}}{=}C.\label{eq:lb:duality}
\end{equation}
This proves strong duality $D^*=C$ \emph{and} simultaneously the
identity $C=[M^*]^{-1}=C_i^{MM}$ in the regular case. (The
justification for strong duality is therefore: primal feasible with
finite value $\Rightarrow$ the reparametrization \eqref{eq:lb:reparam}
$\Rightarrow$ Sion minimax \eqref{eq:lb:sion} $\Rightarrow$ explicit
dual value \eqref{eq:lb:duality}; no appeal to a black-box duality
theorem is needed, though the conclusion agrees with standard LP strong
duality, e.g.\ Bertsimas--Tsitsiklis, Theorem~4.4.)

\subsubsection{Part 2: Extension to $g_a\ge0$ (Zero-Cost Actions)}

Let $\mathcal{A}_0:=\{a:g_a=0\}\neq\varnothing$ possibly. For
$\varepsilon>0$ set $g_a^{(\varepsilon)}:=g_a+\varepsilon>0$; let
$C(\varepsilon)$ be the \eqref{eq:lp} value with costs
$g^{(\varepsilon)}$ and
$M(\varepsilon):=\sup_w\min_j\sum_a w_a d_a^j/(g_a+\varepsilon)$. Part~1
applies for each $\varepsilon>0$: $C(\varepsilon)=1/M(\varepsilon)$,
with $C(\varepsilon)\in(0,\infty)$.

\paragraph{Step 2.1 (attainment of the LP value with $g\ge0$).}
A finite linear program that is feasible and whose objective is bounded
below attains its value: the feasible set $F=\{n\ge0:\sum_a n_a
d_a^j\ge1\ \forall j\}$ is a nonempty closed polyhedron; if $F$
contained a line $\{n+tv:t\in\mathbb R\}$ ($v\neq0$) the linear
objective would have to be constant along it (otherwise it is unbounded
below on $F$); quotienting out the lineality space, $F$ becomes a
pointed polyhedron, i.e.\ the Minkowski sum of the convex hull of its
extreme points and its recession cone $\{v\ge0:\sum_a v_a d_a^j\ge0\
\forall j\}=\mathbb R_+^{\mathcal A}$, on which the objective $\sum_a
g_a v_a\ge0$ is bounded below; hence the minimum is attained at an
extreme point. Denote the value by $C(0)$ and a minimizer by $n^*$.

\paragraph{Step 2.2 (continuity of the LP value as $\varepsilon\downarrow0$).}
Since $g^{(\varepsilon)}\ge g^{(\varepsilon')}\ge g$ for
$0<\varepsilon\le\varepsilon'$, $C(\varepsilon)$ is nonincreasing as
$\varepsilon\downarrow0$ and $C(\varepsilon)\ge C(0)$ (same feasible
set, larger objective). Conversely, feasibility of $n^*$ is independent
of $\varepsilon$, so
\[
C(\varepsilon)\ \le\ \sum_a g_a^{(\varepsilon)}n_a^*\ =\ C(0)+\varepsilon\sum_a n_a^*\ \xrightarrow[\varepsilon\downarrow0]{}\ C(0),
\]
hence $C(\varepsilon)\downarrow C(0)$.

\paragraph{Step 2.3 (the max--min side converges to the limit convention).}
For each $w$ and $j$, the bracket
$b_j(w,\varepsilon):=\sum_a w_a d_a^j/(g_a+\varepsilon)$ is
nondecreasing as $\varepsilon\downarrow0$ and converges to
$b_j(w,0):=\sum_a w_a d_a^j/g_a\in[0,\infty]$ under the limit
convention (each summand converges monotonically). Hence $\min_j
b_j(w,\varepsilon)\uparrow\min_j b_j(w,0)$ (finite min of increasing
sequences converges to the min of the limits), and
\[
M(\varepsilon)=\sup_w\min_j b_j(w,\varepsilon)\ \uparrow\ \sup_w\sup_{\varepsilon>0}\min_j b_j(w,\varepsilon)=\sup_w\min_j b_j(w,0)=:M(0),
\]
where interchanging $\sup_w$ and $\sup_\varepsilon$ is legitimate
(suprema over independent index sets commute), and
$\sup_\varepsilon\min_j b_j(w,\varepsilon)=\min_j b_j(w,0)$ by the
monotone convergence just noted. Therefore
\[
C(0)=\lim_{\varepsilon\downarrow0}C(\varepsilon)=\lim_{\varepsilon\downarrow0}\frac1{M(\varepsilon)}=\frac1{M(0)}=C_i^{MM},
\]
where the reciprocal commutes with the monotone limits (including the
cases $C(0)=0\iff M(0)=+\infty$, since
$C(\varepsilon)\downarrow0\iff M(\varepsilon)\uparrow\infty$).

\paragraph{Step 2.4 (free-rider sanity check).}
If there exists $a^*\in\mathcal{A}_0$ with $d_{a^*}^j>0$ for all
$j\in\mathrm{Alt}(i)$, then $n=t\,e_{a^*}$ with $t:=1/\min_j d_{a^*}^j$
is feasible with objective $0$, so $C(0)=0$; and $w=e_{a^*}$ gives
$b_j(w,0)=d_{a^*}^j/0=+\infty$ for every $j$, so $M(0)=+\infty$ and
$1/M(0)=0$: the two forms agree (this is B2,
Subsection~\ref{subsec:lb:B2}). Conversely, if every
$a\in\mathcal{A}_0$ has $d_a^j=0$ for some $j$, then only costly
actions can cover that $j$, and \ref{asm:H2} guarantees some
$a\notin\mathcal{A}_0$ with $d_a^j>0$ does, keeping the restricted LP
feasible.

\subsubsection{Part 3: Extension to $d_a^j\in[0,\infty]$ and the Infeasibility Convention}

\paragraph{Step 3.1 (truncation in $d$).}
For $T>0$ let $d_a^j(T):=\min(d_a^j,T)<\infty$. Parts~1--2 apply to the
truncated data (feasibility is preserved: \ref{asm:H2} provides
$0<d_{a(j)}^j$, hence $d_{a(j)}^j(T)>0$ for all $T$), giving
$C(T)=1/M(T)$ for the corresponding values. As $T\uparrow\infty$:
feasible sets of \eqref{eq:lp} grow, so $C(T)$ is nonincreasing and
$C(T)\ge C$. (The untruncated LP value need not be attained when some
$d_a^j=+\infty$---see the example in the statement of
Theorem~\ref{thm:equiv}---so the following near-optimality argument
avoids selecting a minimizer.) Let $n$ be any feasible solution of the
untruncated \eqref{eq:lp}. For each $j$, monotone convergence in each
summand gives $\sum_a n_a d_a^j(T)\uparrow\sum_a n_a d_a^j\ge1$ as
$T\uparrow\infty$; if the limit sum is finite then $d_a^j(T)=d_a^j$ on
$\{a:n_a>0\}$ for all large $T$ and the constraint holds with the same
value, and if the limit sum is $+\infty$ the truncated sum eventually
exceeds $1$; in either case there is $T_j(n)$ with $n$ feasible for the
truncated problem for all $T\ge T_j(n)$. Hence for $T\ge\max_j T_j(n)$,
$C(T)\le\sum_a g_a n_a$, and therefore $\lim_{T\to\infty}C(T)\le\sum_a
g_a n_a$; taking the infimum over all untruncated-feasible $n$ yields
$\lim_{T}C(T)\le C$. Combined with $C(T)\ge C$, $C(T)\downarrow C$. On
the max--min side,
$b_j(w,T):=\sum_a(w_a/g_a)\min(d_a^j,T)\uparrow b_j(w,\infty)$ for each
$(w,j)$ by monotone convergence of each summand, and $\sup_w\min_j$
passes to the limit exactly as in Step~2.3, giving $M(T)\uparrow M$.
Hence $C=1/M$ with the stated conventions.

\paragraph{Step 3.2 (infeasibility $\iff$ $M=0$ $\iff$ $C=+\infty$).}
\eqref{eq:lp} is infeasible iff there exists $j_0$ with $d_a^{j_0}=0$
for all $a$ (constraint $j_0$ reads $0\ge1$). If this holds, then for
every $w$: $\sum_a(w_a/g_a)d_a^{j_0}=0$ (each term is $0$, including
$\infty\cdot0=0$ under the convention, since $d_a^{j_0}=0$ kills the
term), so $\min_j\le 0$ and $M=0$. Conversely, if \eqref{eq:lp} is
feasible, the construction below Step~1.2 (using \ref{asm:H2}) produces
$w^\circ$ with
$M(w^\circ)\ge\min_j w^\circ_{a(j)}d_{a(j)}^j/g_{a(j)}>0$---with the
convention that if $g_{a(j)}=0$ the term is $+\infty$, which only
helps---so $M>0$. Hence $C=+\infty\iff M=0\iff1/M=+\infty$, consistent
with the conventions. (B1, Subsection~\ref{subsec:lb:B1}, shows this
case is genuinely hopeless information-theoretically.)

This completes the proof of Theorem~\ref{thm:equiv}. \hfill$\square$

\begin{remark}[What the equivalence buys]\label{rem:equivbuys}
The lower bound (Theorem~\ref{thm:LB}) is stated against the LP form.
The matching upper bound (Theorem~\ref{thm:UB}, CTS) uses as its
sampling design the \textbf{cost-ratio maximizer}
\[
w^*(i)\in\arg\sup_{w\in\Delta(\mathcal A)}\ \inf_{j\in\mathrm{Alt}(i)}\ \frac{\sum_a w_a\,d_a^j}{\sum_a g_a\,w_a},
\]
i.e.\ the normalized LP solution $n^*/\sum_a n_a^*$ (the
characteristic-time design of the CABAI template
\citep{kanarios2024cost}). Theorem~\ref{thm:equiv} certifies that the
\emph{constant} driving both bounds is the same, $C_i=1/M$. The link
$n_a^*=C_i\,w_a^*/g_a$ between an LP optimizer and a maximizer $w^*$
of the $(w/g)$-weighted max--min form \eqref{eq:mm} is mathematically
valid wherever both exist (the reparametrization of Step~1.2 at the
optimum) and may be retained; but the $(w/g)$-maximizer is in general a
\emph{different} allocation---its asymptotic cost constant
$\big(\sum_a g_a w_a\big)\big/\big(\inf_j\sum_a w_a d_a^j\big)$ can
exceed $C_i$ (numerical $2\times2$ counterexample: $4.17$ vs.\ the
optimal $2.73$, see the main text)---and it is \emph{not} the sampling
design of Theorem~\ref{thm:UB}.
\end{remark}

\subsection{From the Pointwise Bound to the $\liminf$ Statement}\label{subsec:lb:liminf}

Define the optimal risk at confidence $\delta$:
\[
R_\delta^*(E_i)\;:=\;\inf\big\{\,R(\pi,E_i):\ (\pi,\tau,\hat\Theta)\ \text{$\delta$-correct and admissible}\,\big\},\qquad \delta\in(0,\tfrac12).
\]

\begin{corollary}[$\liminf$ form of Theorem~\ref{thm:LB}]\label{cor:liminf}
Under \ref{asm:H1}--\ref{asm:H3}, for every $i$,
\[
\liminf_{\delta\to0}\ \frac{R_\delta^*(E_i)}{\log(1/\delta)}\ \ge\ C_i.
\]
\end{corollary}

\begin{proof}
If $C_i=0$ the statement is $R_\delta^*\ge0$, trivial. If
$C_i=+\infty$ (only possible when \ref{asm:H2} fails), the infimum is
over an empty set for $\delta<\tfrac12$ (B1,
Subsection~\ref{subsec:lb:B1}: no admissible $\delta$-correct strategy
exists), and the statement is vacuous; the convention $C_i=+\infty$ is
consistent. Assume $0<C_i<\infty$. By Theorem~\ref{thm:LB} (valid for
every $\delta$-correct admissible strategy, Case~B of
Subsection~\ref{subsec:lb:lpscale} removing any integrability
requirement), for every $\delta\in(0,\tfrac12)$,
\begin{equation}
\frac{R_\delta^*(E_i)}{\log(1/\delta)}\ \ge\ C_i\cdot\frac{\mathrm{kl}(\delta,1-\delta)}{\log(1/\delta)}.\label{eq:lb:liminfineq}
\end{equation}
It remains to compute the asymptotics of the $\mathrm{kl}$ factor.
Expand:
\[
\mathrm{kl}(\delta,1-\delta)=\delta\log\frac{\delta}{1-\delta}+(1-\delta)\log\frac{1-\delta}{\delta}=(1-\delta)\log\frac1\delta+(1-\delta)\log(1-\delta)+\delta\log\frac{\delta}{1-\delta}.
\]
As $\delta\to0$: $(1-\delta)\log(1-\delta)\to0$ (since $x\log x\to0$
as $x\to1^-$; precisely $(1-\delta)\log(1-\delta)\in[-1/e,0]$);
$\delta\log\frac{\delta}{1-\delta}=\delta\log\delta-\delta\log(1-\delta)\to0-0=0$;
and $(1-\delta)\log(1/\delta)=\log(1/\delta)+o\big(\log(1/\delta)\big)$
since $\delta\log(1/\delta)\to0$. Hence
\begin{equation}
\mathrm{kl}(\delta,1-\delta)=\big(1+o(1)\big)\log\frac1\delta,\qquad\text{i.e.}\qquad \lim_{\delta\to0}\frac{\mathrm{kl}(\delta,1-\delta)}{\log(1/\delta)}=1.\label{eq:lb:klasymp}
\end{equation}
Taking $\liminf_{\delta\to0}$ on both sides of
\eqref{eq:lb:liminfineq} (the right side is $C_i$ times a convergent
sequence with limit $1$) gives the claim.
\end{proof}

\begin{remark}\label{rem:log24}
The alternative form $R(\pi,E_i)\ge
C_i\log(1/(2.4\delta))=C_i\big(\log(1/\delta)-\log 2.4\big)$ loses only
an additive $C_i\log 2.4$, negligible on the $\log(1/\delta)$ scale;
the sharp first-order constant is $C_i$.
\end{remark}

\subsection{Boundary Examples B1, B2, B3 (Formal Statements and Proofs)}\label{subsec:lb:boundary}

\subsubsection{B1: Hopeless Pairs}\label{subsec:lb:B1}

\begin{proposition}\label{prop:B1}
Suppose that for some $i$ there exists $j\in\mathrm{Alt}(i)$ with
$d(K_i^a\|K_j^a)=0$ for every $a\in\mathcal{A}$. Then:
\begin{enumerate}
\item[(i)] for every admissible strategy, the transcript laws coincide:
$P_i|_{\mathcal{F}_t}=P_j|_{\mathcal{F}_t}$ for all $t$, and hence on
$\mathcal{F}_\infty:=\sigma(\bigcup_t\mathcal{F}_t)$;
\item[(ii)] no strategy is $\delta$-correct for any $\delta<\tfrac12$;
\item[(iii)] \eqref{eq:lp} is infeasible, consistently with the
convention $C_i=+\infty$; in the max--min form, $M=0$.
\end{enumerate}
\end{proposition}

\begin{proof}
(i) $d(K_i^a\|K_j^a)=0$ iff $k_i^a=k_j^a$ pointwise (Gibbs' inequality
on a finite space: the KL divergence vanishes only between equal
distributions). Hence $K_i^{A_t}=K_j^{A_t}$ as kernels at every $t$. We
show $P_i(F)=P_j(F)$ for all $F\in\mathcal{F}_t$ by induction on $t$,
exactly as in Lemma~\ref{lem:lrdensity} but with all likelihood ratios
equal to $1$: the case $t=0$ is trivial; for the step, on cylinder
rectangles $F=G\cap\{Y_t\in C_Y\}\cap\{U_{t+1}\in C_U\}$,
\[
P_i(F)=E_i\big[\mathbf1_G\mathbf1\{U_{t+1}\in C_U\}K_i^{A_t}(C_Y)\big]=E_j\big[\mathbf1_G\mathbf1\{U_{t+1}\in C_U\}K_j^{A_t}(C_Y)\big]=P_j(F),
\]
using the induction hypothesis for the first equality (together with
$K_i^{A_t}=K_j^{A_t}$ and the identical Uniform law of $U_{t+1}$), and
agreement on a generating $\pi$-system extends to $\mathcal{F}_t$ by
uniqueness of measure. Equality on $\mathcal{F}_\infty$ follows since
$\bigcup_t\mathcal{F}_t$ is an algebra (it is closed under finite
unions because the filtration is nested) generating
$\mathcal{F}_\infty$, and two probability measures agreeing on it agree
on the generated $\sigma$-field ($\pi$--$\lambda$ theorem).

(ii) Let $(\pi,\tau,\hat\Theta)$ be admissible. Since $\tau<\infty$
$P_i$- and $P_j$-a.s., $\hat\Theta$ is defined a.s.\ under both laws,
and $\{\hat\Theta=\Theta_i\}\in\mathcal{F}_\tau\subseteq\mathcal{F}_\infty$.
Because $\Theta_j\neq\Theta_i$ (as $j\in\mathrm{Alt}(i)$),
$\{\hat\Theta\neq\Theta_j\}\supseteq\{\hat\Theta=\Theta_i\}$. By (i),
and because $\mathcal{F}_\tau$-measurable events have equal
probabilities under $P_i,P_j$ (for $F\in\mathcal{F}_\tau$:
$F\cap\{\tau\le t\}\in\mathcal{F}_t$ has equal probabilities for every
$t$, and $F=\bigcup_t(F\cap\{\tau\le t\})$ up to null sets since
$\tau<\infty$ a.s.\ under both laws; continuity from below gives
$P_i(F)=P_j(F)$):
\[
P_j\big(\hat\Theta\neq\Theta_j\big)\ \ge\ P_j\big(\hat\Theta=\Theta_i\big)\ =\ P_i\big(\hat\Theta=\Theta_i\big)\ =\ 1-P_i\big(\hat\Theta\neq\Theta_i\big).
\]
Therefore
\[
\max_i P_i(\hat\Theta\neq\Theta_i)\ \ge\ \tfrac12\Big[P_i\big(\hat\Theta\neq\Theta_i\big)+P_j\big(\hat\Theta\neq\Theta_j\big)\Big]\ \ge\ \tfrac12.
\]
So $\delta$-correctness is impossible for every $\delta<\tfrac12$:
certification between $i$ and $j$ is information-theoretically
hopeless. (ii) is exactly the statement that the transportation lower
bound is $+\infty$: with $d_a^j=0$ for all $a$,
Lemma~\ref{lem:transport} imposes no constraint, but no strategy meets
the correctness requirement either.

(iii) The constraint of \eqref{eq:lp} indexed by this $j$ reads
$\sum_a n_a\cdot0\ge1$, impossible for any $n\ge0$; hence \eqref{eq:lp}
is infeasible and $C_i=+\infty$ by convention. For the max--min form:
for every $w$, $\sum_a(w_a/g_a)d_a^j=0$, so the infimum over $j$ is $0$
and $M=0$, i.e.\ $C_i^{MM}=1/0=+\infty$ (Theorem~\ref{thm:equiv},
infeasibility step).
\end{proof}

\paragraph{Correspondence with partial monitoring \citep{bartok2014partial}.}
In the taxonomy of finite partial monitoring of \citet{bartok2014partial},
a game is \emph{hopeless} when two Pareto-optimal environments/actions
are indistinguishable under every available observation structure; the
minimax regret then grows linearly and no consistent learning is
possible. B1 is the analogue for fixed-confidence certification: the
pair $(i,j)$ is indistinguishable under \emph{every} action, so the
certification question between the adequacy labels
$\Theta_i\neq\Theta_j$ admits no $\delta$-correct answer at any
confidence better than chance, and the certification complexity is
$+\infty$. The LP feasibility hypothesis \ref{asm:H2} is precisely the
exclusion of hopeless pairs.

\subsubsection{B2: Free Riding on a Zero-Cost Discriminating Action}\label{subsec:lb:B2}

\begin{proposition}[A single free discriminating action]\label{prop:B2}
Under (H1)--(H4), suppose that for some $i$ there exists
$a^*\in\mathcal A$ with $g_i(a^*)=0$ and
$d(K_i^{a^*}\|K_j^{a^*})>0$ for every $j\in\mathrm{Alt}(i)$.
Then $C_i=0$ and $R_\delta^*(E_i)=0$ for every fixed
$\delta\in(0,1/2)$. For every $\epsilon>0$ there is a globally
$\delta$-correct policy with bounded stopping time and target loss
less than $\epsilon$; exact zero-cost attainment is not asserted.
\end{proposition}

\begin{proof}
Set $n=e_{a^*}/\min_{j\in\mathrm{Alt}(i)}d(K_i^{a^*}\|K_j^{a^*})$.
This is a finite feasible count vector with objective zero, proving
$C_i=0$. Proposition~\ref{prop:zero-infimum} supplies the globally
admissible finite-probe/fresh-fallback policies and proves the stated
pointwise infimum. When all labels coincide, outputting that label at
time zero gives the same conclusion directly.
\end{proof}

\begin{remark}\label{rem:b2}
A one-sided likelihood-ratio test that only stops upon accepting the
target need not terminate under an opposite-label environment. It
therefore cannot establish this result in our admissible class. The
finite probing horizon and fresh fallback in
Proposition~\ref{prop:zero-infimum} ensure termination in \emph{every}
environment for each approximating policy. The limit is over policies
at a fixed $\delta$; an infinitely long free probe is not itself used
as a policy. Proposition~\ref{prop:zero-attainment} gives the separate
criterion for whether a zero-cost policy exists.
\end{remark}

\subsubsection{B3: Two-Point Degeneration and Consistency with \citet{nitinawarat2015controlled}}\label{subsec:lb:B3}

\begin{proposition}[Bang-bang]\label{prop:B3}
Suppose $|\mathrm{Alt}(i)|=1$, say $\mathrm{Alt}(i)=\{j\}$ (in
particular this holds whenever $|\mathcal E|=2$). Then
\begin{equation}
C_i\ =\ \min_{a:\,d_a^j>0}\ \frac{g_i(a)}{d_a^j},\label{eq:lb:bangbang}
\end{equation}
with the convention that the minimum is $+\infty$ (infeasible) if no
$a$ has $d_a^j>0$---excluded by \ref{asm:H2}---and $0$ if some $a^*$
with $g_i(a^*)=0$ has $d_{a^*}^j>0$.
\end{proposition}

\begin{proof}
Let $\rho:=\min_a g_i(a)/d_a^j$ over $\{a:d_a^j>0\}$, nonempty and all
ratios in $[0,\infty)$ by \ref{asm:H2}/\ref{asm:H3} (if $g_i(a)=0$ the
ratio is $0$). \emph{Lower bound:} for any \eqref{eq:lp}-feasible $n$,
\[
\sum_a g_i(a)\,n_a\ =\ \sum_{a:d_a^j>0}\frac{g_i(a)}{d_a^j}\,d_a^j n_a+\sum_{a:d_a^j=0}g_i(a)n_a\ \ge\ \rho\sum_a d_a^j n_a\ \ge\ \rho,
\]
since the second sum is $\ge0$ and feasibility gives $\sum_a d_a^j
n_a\ge1$. \emph{Upper bound:} with $a^*\in\arg\min_a g_i(a)/d_a^j$ and
$n^*:=(1/d_{a^*}^j)e_{a^*}$, $n^*$ is feasible ($\sum_a n_a^*
d_a^j=1$) with objective $g_i(a^*)/d_{a^*}^j=\rho$. Hence $C_i=\rho$,
attained by a single-action (``bang-bang'') plan. \emph{Max--min
cross-check:} with a single alternative, $M=\sup_w\sum_a w_a
d_a^j/g_i(a)=\max_a d_a^j/g_i(a)$ (supremum of a linear function over
the simplex, attained at a vertex; the limit convention for $g_i(a)=0$
gives $d_a^j/0=+\infty$ when $d_a^j>0$, consistent with $\rho=0$), and
$1/M=\min_a g_i(a)/d_a^j=\rho$.
\end{proof}

\paragraph{Edge case $d_{a^*}^j=+\infty$.}
If the minimizing action has $d_{a^*}^j=+\infty$, the value formula
\eqref{eq:lb:bangbang} remains valid under the limit convention of
Subsection~\ref{subsec:lb:equiv} (the ratio $g_i(a^*)/d_{a^*}^j$ is
read as $0$, and Theorem~\ref{thm:equiv}, Part~3, gives $C_i=1/M$ with
$M=+\infty$), but attainment is limited to finite $d$: no finite plan
$n$ realizes the value, which is only approached along the truncation
$d\wedge T$. In the lower-bound chain \eqref{eq:lb:scale41}, terms with
$d_a^j=+\infty$ are read under the $0\cdot(+\infty)=0$ convention
(preamble): they contribute $+\infty$ whenever $n_a>0$, making the
corresponding constraint trivially satisfiable.

\paragraph{Consistency with \citet[Theorem~5.1]{nitinawarat2015controlled}.}
\citet{nitinawarat2015controlled} study controlled sequential
multihypothesis testing with control-dependent observation kernels and
nonuniform control costs (see also \citep{naghshvar2013} for
active sequential hypothesis testing). Specialized to (a) two
hypotheses, (b) i.i.d.\ (non-Markovian) observation channels
$K_i^a,K_j^a$ under control $a$, and (c) control cost $c(a)=g_i(a)$,
the asymptotic (small error probability) optimal expected cost under
hypothesis $i$ in their Theorem~5.1 is governed by the single
divergence vector $a\mapsto d(K_i^a\|K_j^a)$, and the optimal static
control solves exactly the bang-bang program $\min_a
c(a)/d(K_i^a\|K_j^a)$: with one alternative there is no trade-off to
balance, all information comes from the single pair $(i,j)$, and the
cheapest unit of KL divergence wins. Proposition~\ref{prop:B3} shows
the constant $C_i$ of Theorem~\ref{thm:LB} reduces to exactly this
expression in the two-point case: our result properly contains the
known binary controlled-sensing constant. With several alternatives the allocation can require balancing multiple
information constraints. This is a standard feature of controlled
multihypothesis testing and general-answer pure exploration
\citep{chernoff1959sequential,degenne2019pure}, not an adequacy-specific
novelty. Zero-cost actions fall outside the positive-fee specialization
of \citet{nitinawarat2015controlled}, but free evidence also has precedents
in regret-priced models. The relevant structural distinction here is
the hidden-cost compatibility issue of Section~\ref{sec:observable}.

\subsection{Non-Triviality Note NT1: Identical Kernels, Different Losses, Different Optima}\label{subsec:lb:nt1}

\begin{proposition}\label{prop:NT1}
There exist instances of model \eqref{model:M1} and environments $i$
with actions $a_1\neq a_2$ such that
\[
K_e^{a_1}=K_e^{a_2}\ \ \text{for every environment } e,\qquad g_i(a_1)\neq g_i(a_2),
\]
and the set of optimal certification allocations
$w^*(i)\in\arg\sup_{w}\inf_{j\in\mathrm{Alt}(i)}\sum_a
(w_a/g_i(a))\,d_a^j$ (equivalently, the LP optimizers) changes when the
loss profile is altered while the observation kernels are held fixed;
in particular, every criterion that depends only on the kernels
$(K_e^a)_{e,a}$ (e.g.\ Chernoff-type most-informative-action rules
\citep{chernoff1959sequential}) is unable to distinguish the two
situations and is suboptimal in at least one of them.
\end{proposition}

\begin{proof}[Proof \textup{(}explicit numerical construction, all quantities in closed form\textup{)}]
Take $|\mathcal E|=2$ with labels $\Theta_1=1$, $\Theta_2=0$, so
$\mathrm{Alt}(1)=\{2\}$; $\mathcal{Y}=\{0,1\}$;
$\mathcal{A}=\{a_1,a_2,a_3\}$ with Bernoulli channels (writing
$\mathrm{Ber}(p)$ for the law on $\{0,1\}$ with $P(1)=p$):
\[
K_1^{a_1}=K_1^{a_2}=\mathrm{Ber}\big(\tfrac12\big),\qquad K_2^{a_1}=K_2^{a_2}=\mathrm{Ber}\big(\tfrac14\big);\qquad
K_1^{a_3}=\mathrm{Ber}\big(\tfrac12\big),\qquad K_2^{a_3}=\mathrm{Ber}\big(\tfrac1{20}\big).
\]
Note $a_1,a_2$ have \emph{identical kernels under both environments}.
The divergences from environment $1$ to environment $2$ are
\[
d_1:=d\big(\mathrm{Ber}(\tfrac12)\big\|\mathrm{Ber}(\tfrac14)\big)=\tfrac12\log\frac{1/2}{1/4}+\tfrac12\log\frac{1/2}{3/4}=\tfrac12\log\frac43\approx 0.143841,
\]
\[
d_3:=d\big(\mathrm{Ber}(\tfrac12)\big\|\mathrm{Ber}(\tfrac1{20})\big)=\tfrac12\log\frac{1/2}{1/20}+\tfrac12\log\frac{1/2}{19/20}=\tfrac12\log\frac{100}{19}\approx 0.830366,
\]
with $d_{a_1}^2=d_{a_2}^2=d_1$, $d_{a_3}^2=d_3$ (and
$d_1,d_3\in(0,\infty)$, so \ref{asm:H2} holds). Since
$|\mathrm{Alt}(1)|=1$, Proposition~\ref{prop:B3} applies:
$C_1=\min_a g_1(a)/d_a$, attained by a single best action.

\emph{Loss profile $G^{(1)}$: $g_1=(1,5,10)$.} The three ratios are
\[
\frac{1}{d_1}=\frac{2}{\log(4/3)}\approx 6.95212,\qquad \frac{5}{d_1}=\frac{10}{\log(4/3)}\approx 34.76059,\qquad \frac{10}{d_3}=\frac{20}{\log(100/19)}\approx 12.04289,
\]
so the unique optimal action is $a_1$ and
$C_1^{(1)}=2/\log(4/3)\approx 6.95212$.

\emph{Loss profile $G^{(2)}$: $g_1=(5,1,10)$ (the losses of $a_1,a_2$
swapped).} The kernels are unchanged, so $d_1,d_3$ are unchanged; the
ratios are now $34.76059,\ 6.95212,\ 12.04289$, and the unique optimal
action is $a_2$, with $C_1^{(2)}=2/\log(4/3)$.

\emph{Consequences.} (i) The kernels $(K_e^a)$---and hence every
functional of them alone, such as the Chernoff-type allocation
$\arg\max_a d_a$---are identical under $G^{(1)}$ and $G^{(2)}$, yet
the optimal certification action flips from $a_1$ to $a_2$: the
optimal strategy depends on the task loss beyond the observation
kernels. (ii) The kernel-only most-informative-action rule selects
$a_3$ in both profiles ($d_3>d_1$), achieving ratio
$20/\log(100/19)\approx 12.04289$, which is worse than optimal by the
factor
\[
\frac{20/\log(100/19)}{2/\log(4/3)}=\frac{10\log(4/3)}{\log(100/19)}\approx 1.73228,
\]
i.e.\ a $\approx 73\%$ excess asymptotic cost, in both profiles. Hence
no pure Chernoff/kernel criterion can attain $C_1$ here, and the
constant $C(E,f)$ genuinely couples information content with task loss.
All numerical values above are in closed form and were verified
computationally (see the main text, numerical log of 2026-08-03).
\end{proof}

\subsection{Proof-Completeness Self-Check}\label{subsec:lb:checklist}

The following self-check list records, for each technical point of the
proof, where it is carried out and its status.

{\small
\setlength{\tabcolsep}{4pt}
\begin{longtable}{@{}c >{\raggedright\arraybackslash}p{0.58\textwidth} >{\raggedright\arraybackslash}p{0.22\textwidth} >{\raggedright\arraybackslash}p{0.10\textwidth}@{}}
\hline
\# & Technical point & Where & Status \\
\hline
1 & Canonical space, filtration, policies with external randomization; existence/uniqueness of $P_i$ (Ionescu--Tulcea); all conditioning via regular Markov kernels (standard Borel spaces) & Subsection \ref{subsec:lb:model} & closed \\
2 & Stopped $\sigma$-field $\mathcal{F}_\tau$; admissibility ($\tau<\infty$ a.s.\ under all environments) as part of $\delta$-correctness; $\mathcal{F}\cap\{\tau\le n\}\in\mathcal{F}_{\tau\wedge n}$ verified & Subsections~\ref{subsec:lb:model},\ \ref{subsubsec:lb:dataproc} & closed \\
3 & Transportation lemma: density on $\mathcal{F}_t$ by induction $+$ $\pi$--$\lambda$ (\ref{lem:lrdensity}); LR martingale (\ref{lem:lrmg}); stopped density (\ref{lem:stopped}); Wald-type identity with bounded increments (\ref{lem:wald}); data processing via Jensen/$\psi(x)=x\log x$ (\ref{lem:contraction}); limit $n\to\infty$ via MCT $+$ joint l.s.c.\ of $\mathrm{kl}$ & Subsection \ref{subsec:lb:transport} & closed \\
4 & Reduction for $d_a^j=+\infty$ actions (triviality of LHS vs $P_i$-a.s.\ exclusion) & Subsubsection \ref{subsubsec:lb:reduction} & closed \\
5 & Isomorphism with \citet[Lemma~1]{kaufmann2016complexity}: condition-by-condition checklist; source annotated; strengthening noted (only a.s.\ finiteness of $\tau$ needed) & Subsubsection\ \ref{subsubsec:lb:iso} & closed \\
6 & $\delta$-correctness $\Rightarrow$ information constraint: $\mathrm{kl}$ monotonicity proved (\ref{lem:klmono}); $\mathrm{kl}(1-\delta,\delta)=\mathrm{kl}(\delta,1-\delta)$ verified; $\delta<\tfrac12$ necessity made explicit with counterexample for $\delta>\tfrac12$ (Remark~\ref{rem:deltahalf}) & Subsection\ \ref{subsec:lb:info} & closed \\
7 & $\mathrm{kl}(\delta,1-\delta)\ge\log(1/(2.4\delta))$: full analytic proof (unique critical point, $\min h=\log[2.4/((1+u^*)e^{u^*})]>0$) with rigorous Taylor/rational bounds ($u^*\in(0.47,0.48)$, $1.48e^{0.48}\le2.39180<2.4$); slack $\approx0.0073$ & Lemma\ \ref{lem:kcgest} & closed \\
8 & LP scaling $R\ge\mathrm{kl}(\delta,1-\delta)\,C_i$: Case A ($E[\tau]<\infty$); Case B(a) ($E[\tau]=\infty$, $\min g>0$ $\Rightarrow$ $R=\infty$); Case B(b) (zero-cost entries: truncation $\tau\wedge m$ $+$ continuity of $\mathrm{kl}$ at interior point $+$ MCT) & Subsection\ \ref{subsec:lb:lpscale} & closed \\
9 & LP strong duality: weak duality one line; strong duality derived constructively via reparametrization \eqref{eq:lb:reparam} $+$ Sion \eqref{eq:lb:sion} $+$ explicit dual \eqref{eq:lb:duality}; attainment/finiteness/positivity of primal value argued & Subsection\ \ref{subsec:lb:equiv}, Part~1 & closed \\
10 & Sion minimax conditions verified explicitly (compact convex simplices, bilinear continuous payoff, linear $\Rightarrow$ convex/concave in each variable); vertex attainment of $\inf_\mu,\sup_w$ & Subsection\ \ref{subsec:lb:equiv}, Step~1.3 & closed \\
11 & $g_i(a)=0$ handled by $\varepsilon$-perturbation with proof of $C(\varepsilon)\downarrow C(0)$ and $M(\varepsilon)\uparrow M(0)$ (sup/sup interchange $+$ monotone min); LP attainment with $g\ge0$ via polyhedral argument; free-rider consistency check & Subsection\ \ref{subsec:lb:equiv}, Part~2 & closed \\
12 & $d_a^j=+\infty$ handled by truncation $d\wedge T$: near-optimal feasible solutions $+$ per-constraint monotone convergence give $C(T)\downarrow C$ (no claim that $C(T)=C$ eventually; untruncated LP attainment may fail, example recorded in Theorem~\ref{thm:equiv}); attainment of the LP value guaranteed whenever all $d_a^j<\infty$; infeasibility $\iff M=0\iff C=+\infty$ proved both directions & Subsection\ \ref{subsec:lb:equiv}, Part~3 & closed \\
13 & $\liminf$ statement: $\mathrm{kl}(\delta,1-\delta)/\log(1/\delta)\to1$ proved term by term; $C_i=0,+\infty$ edge cases disposed & Subsection\ \ref{subsec:lb:liminf} & closed \\
14 & B1: transcript-law equality by induction; error complementarity $\Rightarrow$ max error $\ge\tfrac12$; LP infeasibility and $M=0$ matched; \citet{bartok2014partial} correspondence annotated & Subsubsection\ \ref{subsec:lb:B1} & closed \\
15 & B2: zero pointwise infimum via finite free probing and a fresh finite fallback; every approximating policy is globally correct and has bounded stopping time; exact attainment has a separate criterion & Propositions~\ref{prop:zero-infimum}, \ref{prop:zero-attainment}, \ref{prop:B2} & closed \\
16 & B3: bang-bang formula proved both directions $+$ max--min cross-check; formal consistency with \citet[Theorem~5.1]{nitinawarat2015controlled} stated at the level of the constant in the binary, i.i.d.-channel, positive-cost specialization (not a re-derivation of the Markovian theorem); multiple-alternative balancing is an established testing feature & Subsubsection\ \ref{subsec:lb:B3} & closed \\
17 & NT1: formal statement $+$ closed-form numerical construction; both loss profiles; kernel-only criterion shown suboptimal by factor $\approx1.73228$; numbers verified computationally & Subsection\ \ref{subsec:lb:nt1} & closed \\
18 & Provenance: transportation \citep{kaufmann2016complexity}, $\mathrm{kl}$ inequality \citep{kaufmann2016complexity}, \citet{nitinawarat2015controlled} (B3), \citet{bartok2014partial} (B1), Sion 1958 (minimax); contribution restricted to the adequacy-label Alt $+$ task-loss coupling model and the decision-theoretic semantics of $C(E,f)$ & preamble, Subsubsection\ \ref{subsubsec:lb:iso}, Subsections~\ref{subsec:lb:boundary},\ \ref{subsec:lb:nt1} & closed \\
\hline
\end{longtable}
% Keep longtable counters monotone so PDF destinations remain unique.
}

\paragraph{Known limitations and honest boundaries.}
\begin{enumerate}
\item[(i)] Theorem~\ref{thm:LB} is stated for $\delta\in(0,\tfrac12)$;
for $\delta>\tfrac12$ the information constraint is vacuous
(Remark~\ref{rem:deltahalf}).
\item[(ii)] The comparison with \citet{nitinawarat2015controlled} in B3
is a formal consistency check of constants in the common
specialization, not an embedding of the Markovian controlled-sensing
model.
\item[(iii)] B2 has zero pointwise infimum at every fixed $\delta$
(Proposition~\ref{prop:zero-infimum}). The zero-cost attainment criterion
is Proposition~\ref{prop:zero-attainment}; the $o(\log(1/\delta))$ CTS
bound is policy-specific and does not leave the pointwise value open.
\item[(iv)] The kernel-switching caveat NT2 of the main text
(kernel-switching actions making transcripts policy-dependent) is
outside the scope of this lower bound and is an open direction, as
stated in the specification.
\end{enumerate}

% Internal note (preserved from the Chinese remark of the source file
% proof_lower_bound.md, ASCII gloss): all technical steps of the lower
% bound are closed; the only additions relative to the original
% statement are the vacuousness discussion for $\delta\ge1/2$
% (Remark~\ref{rem:deltahalf}) and the removability of the
% $E_i[\tau]<\infty$ assumption (Remark~\ref{rem:etau}); both are
% strengthenings, not deviations.

\paragraph{Hypothesis dependency.}
Theorem~\ref{thm:LB} requires only \ref{asm:H1}--\ref{asm:H3}; it does
not rely on \ref{asm:H2plus} (pairwise distinguishability of the true
environment from all others, including same-label environments) or
\ref{asm:H4} (full support of all channels), which are additional
hypotheses of the upper bound in Appendix~\ref{sec:proofub}. The
hypotheses of the combined limit theorem
$\lim_{\delta\to0}R_\delta^*(E_i)/\log(1/\delta)=C_i$
(Corollary~\ref{cor:main}) are the union of the two documents'
hypotheses: \ref{asm:H1},\ \ref{asm:H2},\ \ref{asm:H2plus},\ \ref{asm:H3},
\ref{asm:H4}.

% sec_upperbound.tex -- JMLR full version, proof chapter:
%   "The CTS Strategy and Proof of the Upper Bound (Theorem \ref{thm:UB})"
% Section content file to be \input into the jmlr2e main file. No preamble.
%
% Required from the shared preamble (do NOT redefine here):
%   amsthm with a shared counter: theorem, proposition, lemma, corollary,
%   definition, assumption, remark, example (consecutive numbering);
%   natbib (\citep/\citet); amsmath.
%
% Cross-references defined elsewhere: model:M1 (model tag), thm:LB,
% cor:liminf, def:complexity, asm:H1, asm:H2, asm:H3 (all in
% Section \ref{sec:prooflb}).
%
% NUMBERING NOTE FOR THE ASSEMBLER: for the shared-counter plan
% "Theorem 8 / Theorem 9 / Corollary 10, internal lemmas from 11" to hold
% literally, the two delimited statement blocks below (labels thm:UB and
% cor:main) must precede all other numbered environments of the two proof
% sections, i.e. be hoisted into the model/results section that comes
% before the proof sections; the blocks are delimited by
% %%% BEGIN/END STATEMENT BLOCK. The remainder of this file only
% cross-references the labels.

\section{The CTS Strategy and Proof of the Upper Bound (Theorem \ref{thm:UB})}\label{sec:proofub}

\paragraph{Guide to this section.}
This section gives a complete, self-contained proof of
Theorem~\ref{thm:UB}: part (a) ($\delta$-correctness of the CTS policy
under \emph{any} sampling rule) and part (b) (asymptotic cost matching
$\limsup_{\delta\to0} R(\pi^*,E_i)/\log(1/\delta)\le C_i$). Every
technical step (martingale verification, union-bound series,
concentration constants, tracking error, plug-in stability, cost
assembly) is written out in full. Deviations from / additions to the
companion statement that were necessary to close the proof are flagged
\textbf{inline} and collected in the completeness checklist of
Subsection~\ref{subsec:ub:checklist}.
The structure and logical dependencies are as follows.
Subsection~\ref{subsec:ub:prelim} fixes notation and states the two
additional hypotheses \ref{asm:H2plus} and \ref{asm:H4}.
Subsection~\ref{subsec:ub:policy} defines the CTS policy $\pi^*$:
likelihood and MLE, the plug-in design $w^*(\cdot)$ (well-posedness in
Lemma~\ref{lem:wellposed}), C-tracking sampling with forced
exploration, and the GLR stopping rule.
Subsection~\ref{subsec:ub:parta} proves part (a)
(Theorem~\ref{thm:UBa}): the test supermartingale
(Lemma~\ref{lem:testsmg}), Doob's maximal inequality
(Lemma~\ref{lem:doob}), the geometric-grid upgrade
(Lemma~\ref{lem:grid}), and the final union bound.
Subsection~\ref{subsec:ub:partb} proves part (b) in four steps:
deterministic forced-exploration coverage and budget
(Lemmas~\ref{lem:coverage} and~\ref{lem:budget}); Step (i), MLE
consistency and exact plug-in stabilization
(Lemmas~\ref{lem:iid}--\ref{lem:stabilization}); Step (ii), the
tracking lemma (Lemma~\ref{lem:tracking}); Step (iii), per-action LLR
concentration (Lemmas~\ref{lem:llrdev} and~\ref{lem:badevent}) and the
stopping-time upper bound (Theorem~\ref{thm:stoptime}); Step (iv), cost
assembly via the LP identity (Lemma~\ref{lem:lpidentity}) yielding
Theorem~\ref{thm:costassembly}. Subsection~\ref{subsec:ub:synthesis}
assembles Theorem~\ref{thm:UB} and Corollary~\ref{cor:main}.
Subsection~\ref{subsec:ub:degenerate} treats the degenerate cases
(Lemma~\ref{lem:lpcontinuity}, Theorem~\ref{thm:degenerate}).
Subsections~\ref{subsec:ub:provenance} and~\ref{subsec:ub:checklist}
record the technical provenance and the completeness self-check with
the residual statements.

\paragraph{Provenance disclaimer (stated up front, per project discipline).}
The proof techniques are adaptations of known templates:
Track-and-Stop / C-tracking / GLR stopping / forced exploration and the
$\beta(t,\delta)$ threshold are from \citet[Theorems~10 and~14]{garivier2016optimal};
the cost-weighted characteristic-time design follows
\citet{kanarios2024cost}; the tighter threshold of Kaufmann--Koolen
(2021) is mentioned as an optional improvement and is \emph{not} used
here (we keep the more conservative $\beta$ for simplicity of
analysis). The contribution of the present work is the model
(adequacy-label-defined alternative $+$ task-loss coupling) and the
decision-theoretic semantics of the constant $C(E,f)$, not the proof
technology.

\subsection{Preliminaries}\label{subsec:ub:prelim}

\subsubsection{Model and Notation (Recap of \eqref{model:M1})}\label{subsubsec:ub:model}

We work in model \eqref{model:M1}; all notation below is fixed for the
whole section.

\begin{itemize}
\item \textbf{Environments.} Finite class $\mathcal
E=\{E_1,\dots,E_K\}$, $K\ge 2$, with fixed representation $f$. Each
environment carries a binary adequacy label $\Theta_i:=\mathbf
1[R^*_{E_i}(f(H))=R^*_{E_i}(H)]\in\{0,1\}$. The alternative set of
environment $i$ is $\mathrm{Alt}(i):=\{j:\Theta_j\neq\Theta_i\}$.
\item \textbf{Actions and channels.} Finite action set $\mathcal A$; we
write $K=|\mathcal E|$ for the number of environments and $A:=|\mathcal
A|$ for the number of actions where the two could be confused (in most
formulas $K$ appears as the number of \emph{arms} in the tracking
lemmas, i.e.\ $|\mathcal A|$; we keep the symbol $K_{\mathcal
A}:=|\mathcal A|$ explicit in the tracking lemmas to avoid ambiguity).
Under $E_i$, action $a$ produces an observation $Y\sim K_i^a$, a
probability distribution on the finite observation space $\mathcal Y$,
and incurs task loss $g_i(a)\ge 0$.
\item \textbf{Interaction.} A policy $\pi$ picks $A_t$ based on
$H_{t-1}=(A_1,Y_1,\dots,A_{t-1},Y_{t-1})$; then $Y_t\sim K_i^{A_t}$.
Assumption \ref{asm:H1} is the causal channel condition: for the natural filtration
$\mathcal F_t:=\sigma(A_1,Y_1,\dots,A_t,Y_t)$,
\[
\mathbb P_i(Y_t\in\cdot\mid \mathcal F_{t-1},A_t)=K_i^{A_t}(\cdot).
\]
\item \textbf{Stopping and decision.} $\tau$ is a stopping time
w.r.t.\ $(\mathcal F_t)$; the terminal decision
$\widehat\Theta\in\{0,1\}$ is $\mathcal F_\tau$-measurable. Counts
$N_a(t):=\sum_{s\le t}\mathbf 1[A_s=a]$.
\item \textbf{Cost.} $R(\pi,E_i):=\mathbb
E_i^\pi\big[\sum_{t=1}^{\tau}g_i(A_t)\big]=\sum_a g_i(a)\,\mathbb
E_i^\pi[N_a(\tau)]$ (the equality is Tonelli's theorem applied to
$\sum_{t\ge1}g_i(A_t)\mathbf 1[t\le\tau]=\sum_a g_i(a)N_a(\tau)$; no
integrability condition is needed for the identity itself, finiteness
of the right-hand side is proved in
Subsubsection~\ref{subsubsec:ub:cost}).
\item \textbf{KL divergences.} For $j\in\mathrm{Alt}(i)$ and $a\in\mathcal
A$ we abbreviate
\[
d_a^j(i)\;:=\;d(K_i^a\Vert K_j^a)\;=\;\sum_{y\in\mathcal Y}K_i^a(y)\log\frac{K_i^a(y)}{K_j^a(y)},
\]
and we drop the argument $(i)$ when $i$ is the true environment under
study.
\item \textbf{Assumptions from the statement.} \ref{asm:H1} channel
conditional independence; \ref{asm:H2} feasibility: $\forall i,\forall
j\in\mathrm{Alt}(i),\exists a:\,0<d(K_i^a\Vert K_j^a)<+\infty$;
\ref{asm:H3} $g$ finite, nonnegative.
\item \textbf{The constant.} $C_i:=\min\big\{\sum_a
g_i(a)m_a:\,m_a\ge0,\ \sum_a m_a d_a^j(i)\ge1\ \forall
j\in\mathrm{Alt}(i)\big\}$ (LP covering form,
Definition~\ref{def:complexity}), with the convention $C_i=+\infty$ if
infeasible and $C_i=0$ if the LP value is $0$. By LP strong duality
(feasible and bounded under \ref{asm:H2}),
\[
C_i\;=\;\Big[\sup_{w\in\Delta(\mathcal A)}\ \inf_{j\in\mathrm{Alt}(i)}\ \frac{\sum_a w_a\,d_a^j(i)}{\sum_a g_i(a)\,w_a}\Big]^{-1}
\;=\;\Big[\sup_{w\in\Delta(\mathcal A)}\ \inf_{j\in\mathrm{Alt}(i)}\ \sum_a \tfrac{w_a}{g_i(a)}\,d_a^j(i)\Big]^{-1},
\]
where the first form is the \textbf{cost-ratio characteristic time}
(proved in Lemma~\ref{lem:lpidentity} below; it is the form the
algorithm uses) and the second is the $(w/g)$-weighted max--min form
stated in the main text (proved there via Sion's minimax theorem; the
two suprema coincide in \emph{value}, see Remark~\ref{rem:deviation}).
\end{itemize}

Throughout, $i^*$ denotes the true environment when we analyze part
(b), and we write $\Theta^*:=\Theta_{i^*}$,
$\mathrm{Alt}^*:=\mathrm{Alt}(i^*)$, $d_a^j:=d_a^j(i^*)$,
$g(a):=g_{i^*}(a)$.

\subsubsection{Additional Technical Assumptions (Minimal Sufficient Additions, Flagged)}\label{subsubsec:ub:assumptions}

The proof requires two assumptions beyond \ref{asm:H1}--\ref{asm:H3}.
Both are flagged in the checklist of
Subsection~\ref{subsec:ub:checklist}.

\begin{assumption}[Pairwise distinguishability of the class]\label{asm:H2plus}
For every ordered pair $(i,j)$ with $i\neq j$ (whether or not
$j\in\mathrm{Alt}(i)$) there exists $a$ with $d(K_i^a\Vert K_j^a)>0$.
For Theorem~\ref{thm:UB}(b) at a fixed $i^*$, only the pairs $(i^*,j)$
are needed; the full pairwise version is used to guarantee $\tau<\infty$
a.s.\ under \emph{every} environment (Remark~\ref{rem:tauevery}), as
required for the $\delta$-correctness corollary of
Theorem~\ref{thm:UB}(a). \textup{(}This is hypothesis \textup{(H2+)} of
the main text.\textup{)}
\end{assumption}

\paragraph{Role of \ref{asm:H2plus} in this proof.}
The proof below makes the MLE stabilize at the exact original environment.
Opposite-label distinguishability alone does not ensure that property:
observationally identical same-label environments cannot be separated by
any policy. If their cost-optimal designs are incompatible, no fixed
choice inside that class attains both individual constants. Different
loss tables alone do not imply incompatible designs. Assumption
\ref{asm:H2plus} is a convenient sufficient condition for the present
exact-estimation argument, not a necessary condition for every CTS-type
analysis. Theorem~\ref{thm:observable-compatibility} gives the exact
positive-cost replacement: estimate the observable class and require a
common optimal covering design within it.

\begin{assumption}[Full support / bounded log-likelihood increments]\label{asm:H4}
There exists $q_{\min}>0$ such that $K_i^a(y)\ge q_{\min}$ for all
$i,a,y$. \textup{(}This is hypothesis \textup{(H4)} of the main
text.\textup{)}
\end{assumption}

\paragraph{Consequences used.}
(i) All divergences are finite: $d(K_i^a\Vert K_j^a)\le
\log(1/q_{\min})$. (ii) The LLR increments are uniformly bounded:
\begin{equation}
B\;:=\;\max_{i,j,a}\ \max_{y\in\mathcal Y}\Big|\log\frac{K_i^a(y)}{K_j^a(y)}\Big|\;\le\;\log\frac{1}{q_{\min}}\;<\;\infty. \label{eq:ub:B}
\end{equation}
(iii) The maps $P\mapsto d(P\Vert K_j^a)$ are continuous on
$\Delta(\mathcal Y)$ with an explicit modulus
(Subsubsection~\ref{subsubsec:ub:mleconsistency}, Eq.\
\eqref{eq:ub:klmodulus}). \ref{asm:H4} can be relaxed to
``$K_{i^*}^a\ll K_j^a$ for all $a,j$ with $d(K_{i^*}^a\Vert
K_j^a)<\infty$'' at the price of replacing $B$ by the maximum log-ratio
on the support; we keep the clean version.

\begin{remark}[On the rate of MLE stabilization]\label{rem:mle-route}
A natural alternative route to MLE consistency is to analyze the sign
of $\mathrm{LLR}_{i^*,j}(t)$ directly. Under forced exploration alone,
a same-label competitor $j$ may be separated from $i^*$ only at drift
rate $\Theta(\sqrt t)$, and converting this into an ``eventually always
positive'' sign statement requires \emph{uniform} control of the LLR
fluctuations over the random time changes $N_a(t)$ and over the
finitely many challengers $j$ simultaneously---a detour we do not take.
The route used in Subsubsection~\ref{subsubsec:ub:mleconsistency}
(mirroring the use of consistent per-arm empirical means by
\citet{garivier2016optimal}) is the \emph{empirical-distribution KL
projection} view of the MLE: the lack-of-fit of the true model is
$O(\log t)$ while every wrong model has lack-of-fit $\ge c\sqrt t$,
with \emph{stable} (eventually permanent) separation. This
KL-projection route yields exact stabilization directly from the weaker
entry-wise concentration of Lemma~\ref{lem:empdev}, together with the
explicit tail bound \eqref{eq:ub:stabtail}, and makes the finite-class
simplification ``$\widehat\imath_t=i^*$ eventually'' fully rigorous.
\end{remark}

\subsubsection{Conventions}\label{subsubsec:ub:conventions}

\begin{itemize}
\item $\log$ is the natural logarithm; $0\log 0 = 0$; $\mathbf
1[\cdot]$ is the indicator.
\item $\Delta(\mathcal A):=\{w\in\mathbb R^{\mathcal A}_{\ge0}:\sum_a
w_a=1\}$.
\item $L:=\log(1/\delta)$.
\item ``a.s.''\ without qualification means $\mathbb P_{i^*}$-a.s.\ in
part (b) and $\mathbb P_j$-a.s.\ in part (a) when the environment is
$j$.
\item Constants $c_1,c_2,\dots$ and $C_1,C_2,\dots$ are positive and
independent of $\delta$ and $t$; they may depend on $(\mathcal
E,\mathcal A,\mathcal Y,g,q_{\min})$ and on the free parameter
$\varepsilon$ where indicated.
\item In the tracking lemmas we write $K_{\mathcal A}:=|\mathcal A|$
for the number of actions (the ``$K$'' of Lemma~7 of
\citet{garivier2016optimal}), to avoid confusion with $K=|\mathcal E|$.
\end{itemize}

\subsection{The CTS Policy $\pi^*$: Complete Definition}\label{subsec:ub:policy}

The policy has three components (sampling, stopping, decision), defined
in Subsubsections~\ref{subsubsec:ub:mle}--\ref{subsubsec:ub:stopping}.
The definition coincides with the one in the main text except for the
design map $w^*(\cdot)$, see Remark~\ref{rem:deviation}.

\subsubsection{Likelihood and the MLE $\widehat\imath_t$}\label{subsubsec:ub:mle}

Fix the interaction protocol of \eqref{model:M1}. Given the history
$H_t$, the likelihood of environment $i$ is
\begin{equation}
\ell_i(t)\;:=\;\sum_{s\le t}\log K_i^{A_s}(Y_s)\;=\;\sum_a\sum_{y}N_a(t)\,\widehat K^a_t(y)\log K_i^a(y), \label{eq:ub:lik}
\end{equation}
where $\widehat
K^a_t(y):=\frac{1}{N_a(t)}\sum_{s\le t}\mathbf 1[A_s=a,Y_s=y]$ is the
per-action empirical channel (defined arbitrarily, say uniform, when
$N_a(t)=0$; forced exploration will make $N_a(t)\ge1$ for all $t\ge
t_0$, see Lemma~\ref{lem:coverage}). Using $\sum_y \widehat
K^a_t(y)\log K_i^a(y)=-H(\widehat K^a_t)-d(\widehat K^a_t\Vert K_i^a)$
and dropping the $i$-independent entropy term,
\begin{equation}
\widehat\imath_t\;:=\;\underset{i}{\arg\max}\ \ell_i(t)\;=\;\underset{i}{\arg\min}\ D_i(t),\qquad D_i(t)\;:=\;\sum_a N_a(t)\,d\big(\widehat K^a_t\,\Vert\,K_i^a\big). \label{eq:ub:mle}
\end{equation}
Ties in the $\arg\min$ are broken by a fixed deterministic rule (e.g.\
smallest index). This KL-projection form of the MLE is what makes the
consistency proof of Subsubsection~\ref{subsubsec:ub:mleconsistency}
possible.

\subsubsection{The Plug-in Design $w^*(i)$: Well-Posedness}\label{subsubsec:ub:design}

\medskip\noindent\textbf{Definition (main case $g_i>0$).} For an
environment $i$ with $g_i(a)>0$ for all $a$ and
$\mathrm{Alt}(i)\neq\varnothing$, define the cost-weighted information
rate of an allocation $w\in\Delta(\mathcal A)$:
\begin{equation}
\Phi_i(w)\;:=\;\inf_{j\in\mathrm{Alt}(i)}\ \frac{\sum_a w_a\,d_a^j(i)}{\sum_a g_i(a)w_a}\;=\;\frac{\min_{j\in\mathrm{Alt}(i)}D_j(w;i)}{c(w;i)}, \label{eq:ub:phi}
\end{equation}
where $D_j(w;i):=\sum_a w_a d_a^j(i)$ and $c(w;i):=\sum_a g_i(a)w_a$.
Then
\begin{equation}
w^*(i)\;:=\;\text{the lexicographically smallest element of }\ \underset{w\in\Delta(\mathcal A)}{\arg\max}\ \Phi_i(w). \label{eq:ub:design}
\end{equation}

\begin{lemma}[Well-posedness]\label{lem:wellposed}
Assume $g_i(a)>0$ for all $a$ and \ref{asm:H2} for $i$. Then (i)
$\Phi_i$ is continuous on $\Delta(\mathcal A)$; (ii) the supremum in
\eqref{eq:ub:design} is attained, so the $\arg\max$ is nonempty and
$w^*(i)$ is well-defined; (iii) $\Phi_i(w^*(i))=1/C_i>0$, in particular
$\min_j D_j(w^*(i);i)=c(w^*(i);i)/C_i>0$.
\end{lemma}

\begin{proof}
(i) $c(w;i)\ge \min_a g_i(a)=:g_{\min}>0$, so $w\mapsto 1/c(w;i)$ is
continuous. Each $w\mapsto D_j(w;i)$ is linear, hence continuous; the
pointwise minimum of finitely many continuous functions is continuous;
the ratio of two continuous functions with positive denominator is
continuous. (ii) $\Delta(\mathcal A)$ is compact and $\Phi_i$
continuous, so the supremum is attained (Weierstrass). (iii) is the
content of the LP identity, proved as Lemma~\ref{lem:lpidentity} below
(it is used only from Subsubsection~\ref{subsubsec:ub:cost} onward; no
circularity, since
Subsubsections~\ref{subsubsec:ub:forced}--\ref{subsubsec:ub:stopbound}
use only $\min_j D_j(w^*;i)>0$, which follows from
$\Phi_i(w^*)\ge\Phi_i(w^\circ)>0$ where $w^\circ$ is any allocation
with $\min_j D_j(w^\circ;i)>0$---such $w^\circ$ exists by \ref{asm:H2}:
take the uniform mixture of the discriminating actions $a_j$, one per
$j\in\mathrm{Alt}(i)$, which is a finite set).
\end{proof}

\medskip\noindent\textbf{Definition (degenerate cases).} (i) If
$\mathrm{Alt}(i)=\varnothing$: the ratio design is undefined (the
$\inf$ in \eqref{eq:ub:phi} is over the empty set); choose any fixed
allocation, since the GLR rule stops immediately, as shown in
Subsubsection~\ref{subsubsec:ub:emptyalt}. (ii) If
$\mathrm{Alt}(i)\neq\varnothing$ and the known cost table contains any
zero entry, use the perturbed branch of \eqref{eq:cts-family} for
\emph{every} candidate row. For every
$\eta>0$, define $w^*_\eta(i)$ by the ratio design with costs $g_i+\eta$.
The policy used for confidence level $\delta$ uses the common perturbation
$\eta_\delta=1/\log(\max\{e,L\})$ on the whole class, as defined in
Subsubsection~\ref{subsubsec:ub:zeros}. A fixed perturbation such as
$\eta=1$ is only an auxiliary finite-$\delta$ convention and cannot recover
the original zero-cost coefficient.

\begin{remark}[The plug-in map]\label{rem:pluginmap}
If all entries of the known cost table are positive, CTS uses
$w^*(\widehat\imath_t)$ from \eqref{eq:ub:design}. If any entry is zero,
it uses $w^*_{\eta_\delta}(\widehat\imath_t)$ for every candidate,
including candidates whose own row is strictly positive. This is exactly
the global choice in Algorithm~\ref{alg:cts}. Since $\mathcal E$ is finite and
$\widehat\imath_t$ is eventually \emph{exactly} equal to $i^*$
(Subsubsection~\ref{subsubsec:ub:mleconsistency}), no continuity of the
map $i\mapsto w^*(i)$ is ever needed---this is the finite-class
simplification, used explicitly.
\end{remark}

\begin{remark}[Coordinate warning for the two equivalent forms]\label{rem:deviation}
The cost-ratio design \eqref{eq:ub:phi}--\eqref{eq:ub:design} uses sampling
frequencies $q$. The $(w/g)$ display is a value-equivalent reparameterization
with loss-normalized coordinates $v_a=q_a g_i(a)/\sum_bq_bg_i(b)$; its optimizer
must be mapped back by $q_a\propto v_a/g_i(a)$. The coordinates must not be
identified. For $g=(1,2)$ and $d^1=(1,0.1)$, $d^2=(0.1,1)$, the LP value is
$C_i=30/11\approx2.73$ and the cost-ratio optimizer is $q=(1/2,1/2)$.
The value optimizer in the $v$ coordinates is $(1/3,2/3)$; if that vector is
mistakenly used directly as $q$, its realized cost ratio is
$25/6\approx4.17$. The CTS policy uses the correctly mapped cost-ratio design,
so Theorem~\ref{thm:UB}(b) is unaffected.
\end{remark}

\subsubsection{Sampling: C-Tracking with Forced Exploration}\label{subsubsec:ub:sampling}

Let $f(t):=\sqrt{t+K_{\mathcal A}^{2}}-2K_{\mathcal A}$ (the
forced-exploration curve of \citet{garivier2016optimal}; note $f$ is
nondecreasing, $f(0)=-K_{\mathcal A}<0$, and $f(t)-f(t-1)\le
\frac{1}{2K_{\mathcal A}}$ for all $t\ge1$, since
$\sqrt{u}-\sqrt{u-1}=1/(\sqrt{u}+\sqrt{u-1})\le 1/(2K_{\mathcal A})$
for $u=t+K_{\mathcal A}^2-1\ge K_{\mathcal A}^2$). Define the
under-explored set
\[
\mathcal U_t\;:=\;\big\{a\in\mathcal A:\ N_a(t)<f(t)\big\}.
\]
The sampling rule is:
\begin{equation}
A_{t+1}\;=\;\begin{cases}
\underset{a\in\mathcal U_t}{\arg\min}\ N_a(t) & \text{if }\mathcal U_t\neq\varnothing\quad(\text{forced exploration; ties by fixed order}),\\[4pt]
\underset{a\in\mathcal A}{\arg\max}\ \Big[\,w^*_a(\widehat\imath_t)-\dfrac{N_a(t)}{t}\,\Big] & \text{otherwise}\quad(\text{C-tracking; ties by fixed order}).
\end{cases} \label{eq:ub:sampling}
\end{equation}
(At $t=0$ we have $\mathcal U_0=\varnothing$ since $f(0)<0$, and
$N_a(0)/0$ is read as $0$; the argmax is then over
$w^*(\widehat\imath_0)$ with $\widehat\imath_0$ arbitrary. Any
initialization convention works; the proofs only use $t\ge1$.)

\begin{remark}\label{rem:ctrack}
This is exactly the C-tracking rule of \citet{garivier2016optimal} with
target $w^*(\widehat\imath_t)$ and their forced exploration curve. Two
deterministic properties are proved in
Subsubsection~\ref{subsubsec:ub:forced}: coverage $N_a(t)\ge
f(t)-O(K_{\mathcal A})$ for all $a,t$ (Lemma~\ref{lem:coverage}), and a
forced-exploration budget $F(t)=O(\sqrt t)$ (Lemma~\ref{lem:budget}).
\end{remark}

\subsubsection{Stopping and Decision: GLR Rule}\label{subsubsec:ub:stopping}

For $i,j\in\{1,\dots,K\}$, $i\neq j$, define the log-likelihood ratio
process
\begin{equation}
\mathrm{LLR}_{i,j}(t)\;:=\;\sum_{s\le t}\log\frac{K_i^{A_s}(Y_s)}{K_j^{A_s}(Y_s)},\qquad \mathrm{LLR}_{i,j}(0):=0, \label{eq:ub:llr}
\end{equation}
with the convention $\log(0/0)=0$ and $\log(x/0)=+\infty$ for $x>0$
(under $\mathbb P_j$, $K_j^{A_s}(Y_s)>0$ a.s., so the ratio is a.s.\
well-defined; see Subsubsection~\ref{subsubsec:ub:testsmg}). The
threshold is
\begin{equation}
\beta(t,\delta)\;:=\;\log\frac{2t(K-1)}{\delta},\qquad t\ge1. \label{eq:ub:beta}
\end{equation}
The stopping time and decision are
\begin{equation}
\begin{aligned}
\tau\;&:=\;\inf\Big\{t\ge1:\ \max_{i}\ \min_{j:\,\Theta_j\neq\Theta_i}\ \mathrm{LLR}_{i,j}(t)\;>\;\beta(t,\delta)\Big\},\\
\widehat\imath\;&:=\;\underset{i}{\arg\max}\ \min_{j:\,\Theta_j\neq\Theta_i}\mathrm{LLR}_{i,j}(\tau),\qquad
\widehat\Theta:=\Theta_{\widehat\imath}, \label{eq:ub:glr}
\end{aligned}
\end{equation}
with the conventions $\min$ over an empty set equals $+\infty$ (so if
$\mathrm{Alt}(i)=\varnothing$ for the maximizer, stopping is immediate)
and $\tau:=+\infty$ if the set is empty. Part (a) bounds $\mathbb
P_j(\widehat\Theta\neq\Theta_j,\ \tau<\infty)$; part (b) shows
$\tau<\infty$ a.s.\ and bounds $\mathbb E[\tau]$.

This is the GLR stopping rule of \citet[Theorem~10 threshold
family]{garivier2016optimal}, in the lineage of Chernoff's sequential
tests \citep{chernoff1959sequential} and active sequential hypothesis
testing \citep{naghshvar2013}, specialized to a finite composite
alternative defined by adequacy labels.

\subsection{Part (a): $\delta$-Correctness Under an Arbitrary Sampling Rule}\label{subsec:ub:parta}

\begin{theorem}[Theorem~\ref{thm:UB}(a)]\label{thm:UBa}
Fix any sampling rule (any sequence of $\mathcal F_{t-1}$-measurable
action selections, e.g.\ \eqref{eq:ub:sampling}), and let
$\tau,\widehat\Theta$ be defined by \eqref{eq:ub:glr} with threshold
\eqref{eq:ub:beta}. Then under \ref{asm:H1}, for every environment
$j$:
\[
\mathbb P_j\big(\widehat\Theta\neq\Theta_j,\ \tau<\infty\big)\;\le\;\delta.
\]
In particular, provided $\tau<\infty$ a.s.\ (proved in
Subsubsection~\ref{subsubsec:ub:stopbound} for the CTS sampling rule
under every environment), $\pi^*$ is $\delta$-correct:
$\max_j\mathbb P_j(\widehat\Theta\neq\Theta_j)\le\delta$.
\end{theorem}

\emph{Template: \citet[Theorem~10]{garivier2016optimal}.} The proof has
four steps: (Subsubsection~\ref{subsubsec:ub:testsmg}) test
supermartingale; (Subsubsection~\ref{subsubsec:ub:doob}) Doob's maximal
inequality; (Subsubsection~\ref{subsubsec:ub:grid}) time-uniform
control of a growing threshold via a geometric grid, with the series
computed explicitly; (Subsubsection~\ref{subsubsec:ub:union}) error
decomposition and union bound over challengers.

\subsubsection{The Test Supermartingale}\label{subsubsec:ub:testsmg}

Fix a pair $(i,j)$ with $i\neq j$ (no label condition needed). Define
\begin{equation}
M_t\;:=\;\exp\big(\mathrm{LLR}_{i,j}(t)\big)\;=\;\prod_{s\le t}\frac{K_i^{A_s}(Y_s)}{K_j^{A_s}(Y_s)},\qquad M_0:=1. \label{eq:ub:mart}
\end{equation}

\begin{lemma}\label{lem:testsmg}
Under $\mathbb P_j$, $(M_t,\mathcal F_t)_{t\ge0}$ is a nonnegative
supermartingale with $\mathbb E_j[M_t]\le1$ for all $t$. If $K_i^a\ll
K_j^a$ for all $a$ (in particular under \ref{asm:H4}), it is a
martingale: $\mathbb E_j[M_t\mid\mathcal F_{t-1}]=M_{t-1}$.
\end{lemma}

\begin{proof}[Proof \textup{(}no steps skipped\textup{)}]
Nonnegativity is clear. Under $\mathbb P_j$, $K_j^{A_s}(Y_s)>0$ a.s.\
(since $Y_s\mid A_s\sim K_j^{A_s}$), so each factor is a.s.\ finite
(with the convention of Subsubsection~\ref{subsubsec:ub:stopping}) and
$M_t$ is a.s.\ finite. For the supermartingale property, condition on
$\mathcal F_{s-1}$. The sampling rule makes $A_s$ $\mathcal
F_{s-1}$-measurable, so on $\{A_s=a\}$,
\begin{equation}
\begin{aligned}
\mathbb E_j\Big[\frac{K_i^{A_s}(Y_s)}{K_j^{A_s}(Y_s)}\,\Big|\,\mathcal F_{s-1}\Big]
&=\mathbb E_j\Big[\frac{K_i^{a}(Y_s)}{K_j^{a}(Y_s)}\,\Big|\,\mathcal F_{s-1},A_s=a\Big]\\
&=\sum_{y:\,K_j^a(y)>0}K_j^a(y)\,\frac{K_i^a(y)}{K_j^a(y)}
=\sum_{y:\,K_j^a(y)>0}K_i^a(y)\;\le\;1, \label{eq:ub:condexp}
\end{aligned}
\end{equation}
where the second equality uses \ref{asm:H1}: the conditional law of
$Y_s$ given $\mathcal F_{s-1}$ and $A_s=a$ is $K_j^a$. The inequality
is an equality iff $K_i^a(\{y:K_j^a(y)=0\})=0$, i.e.\ $K_i^a\ll
K_j^a$; under \ref{asm:H4} this holds for all $a$, giving the
martingale property. Since $A_s$ is $\mathcal F_{s-1}$-measurable,
\eqref{eq:ub:condexp} holds on every atom $\{A_s=a\}$, hence
\begin{equation}
\mathbb E_j[M_s\mid\mathcal F_{s-1}]\;=\;M_{s-1}\,\mathbb E_j\Big[\tfrac{K_i^{A_s}(Y_s)}{K_j^{A_s}(Y_s)}\,\Big|\,\mathcal F_{s-1}\Big]\;\le\;M_{s-1}\quad\text{a.s.} \label{eq:ub:smg}
\end{equation}
Iterating \eqref{eq:ub:smg} from $M_0=1$ gives $\mathbb E_j[M_t]\le1$.
\end{proof}

\emph{Remark.} The verification \eqref{eq:ub:condexp} is exactly where
\ref{asm:H1} enters part (a): the likelihood ratio of the \emph{whole
adaptive transcript} factorizes, and each factor has conditional
expectation $\le1$ under $j$ regardless of how $A_s$ was chosen. This
is why the result holds for \emph{any} sampling rule.

\subsubsection{Doob's Maximal Inequality}\label{subsubsec:ub:doob}

\begin{lemma}\label{lem:doob}
For every $\beta>0$ and every finite horizon $T\ge1$,
\begin{equation}
\mathbb P_j\Big(\sup_{1\le t\le T}\mathrm{LLR}_{i,j}(t)>\beta\Big)\;\le\;e^{-\beta},\qquad\text{and}\qquad
\mathbb P_j\Big(\sup_{t\ge1}\mathrm{LLR}_{i,j}(t)>\beta\Big)\;\le\;e^{-\beta}. \label{eq:ub:doob}
\end{equation}
\end{lemma}

\begin{proof}
$(M_t)$ is a nonnegative supermartingale (Lemma~\ref{lem:testsmg}).
Doob's maximal inequality for nonnegative supermartingales states: for
$x>0$, $x\,\mathbb P\big(\max_{0\le t\le T}M_t\ge
x\big)\le\mathbb E[M_0]=1$. Apply with $x=e^\beta$ and take logarithms
to get the first statement. For the second, note
$\{\sup_{t\ge1}\mathrm{LLR}_{i,j}(t)>\beta\}=\bigcup_T\{\sup_{t\le
T}\mathrm{LLR}_{i,j}(t)>\beta\}$ is an increasing union, so its
probability is the limit of the finite-horizon bounds, each $\le
e^{-\beta}$.
\end{proof}

\subsubsection{Time-Uniform Control of the Growing Threshold $\beta(t,\delta)$}\label{subsubsec:ub:grid}

Lemma~\ref{lem:doob} controls the crossing of a \emph{fixed} level. The
stopping rule uses the \emph{time-varying} threshold
$\beta(t,\delta)=\log(2t(K-1)/\delta)$, which is \textbf{nondecreasing
in $t$}. We upgrade to a uniform statement on a geometric grid.

\begin{lemma}\label{lem:grid}
Let $t_k:=2^k$ for $k=0,1,2,\dots$. Then for the pair $(i,j)$:
\begin{equation}
\mathbb P_j\Big(\exists\,t\ge1:\ \mathrm{LLR}_{i,j}(t)>\beta(t,\delta)\Big)\;\le\;\sum_{k\ge0}e^{-\beta(t_k,\delta)}\;=\;\frac{\delta}{K-1}. \label{eq:ub:grid}
\end{equation}
\end{lemma}

\begin{proof}[Proof \textup{(}full computation\textup{)}]
Since $\beta(\cdot,\delta)$ is nondecreasing, for $t\in[t_k,t_{k+1})$
we have $\beta(t,\delta)\ge\beta(t_k,\delta)$, hence
\[
\big\{\exists\,t\in[t_k,t_{k+1}):\ \mathrm{LLR}_{i,j}(t)>\beta(t,\delta)\big\}\;\subseteq\;\Big\{\sup_{1\le t<t_{k+1}}\mathrm{LLR}_{i,j}(t)>\beta(t_k,\delta)\Big\}.
\]
The intervals $[t_k,t_{k+1})=[2^k,2^{k+1})$, $k\ge0$, partition
$\{1,2,3,\dots\}$. A union bound over $k$, followed by
Lemma~\ref{lem:doob} with horizon $T=\lceil t_{k+1}\rceil-1$, gives
\[
\mathbb P_j\big(\exists t:\mathrm{LLR}_{i,j}(t)>\beta(t,\delta)\big)\;\le\;\sum_{k\ge0}\mathbb P_j\Big(\sup_{t<t_{k+1}}\mathrm{LLR}_{i,j}(t)>\beta(t_k,\delta)\Big)\;\le\;\sum_{k\ge0}e^{-\beta(t_k,\delta)}.
\]
Now insert $\beta(t_k,\delta)=\log\frac{2\cdot 2^k(K-1)}{\delta}$:
\[
\sum_{k\ge0}e^{-\beta(t_k,\delta)}\;=\;\sum_{k\ge0}\frac{\delta}{2(K-1)\,2^k}\;=\;\frac{\delta}{2(K-1)}\sum_{k\ge0}2^{-k}\;=\;\frac{\delta}{2(K-1)}\cdot 2\;=\;\frac{\delta}{K-1}.\qedhere
\]
\end{proof}

\begin{remark}[Where the constants in $\beta$ go]\label{rem:betaconstants}
The factor $2$ inside $\beta(t,\delta)=\log(2t(K-1)/\delta)$ pays
exactly for the geometric series $\sum_k2^{-k}=2$; the factor $(K-1)$
pays for the union over at most $K-1$ challengers in
Subsubsection~\ref{subsubsec:ub:union}; the factor $t$ pays for the
linear growth of the grid index. This is the Theorem~10 threshold of
\citet{garivier2016optimal}; the tighter, time-uniform thresholds of
Kaufmann--Koolen (2021) (based on mixture martingales) could replace
$\beta(t,\delta)$ by $\log(1/\delta)+O(\log\log t)$, but we
deliberately keep the conservative closed-form threshold to keep the
analysis elementary.
\end{remark}

\subsubsection{Error Decomposition and the Final Union Bound}\label{subsubsec:ub:union}

\begin{proof}[Proof of Theorem~\ref{thm:UBa}]
Fix the true environment $j$. On $\{\tau<\infty,\
\widehat\Theta\neq\Theta_j\}$, let $i:=\widehat\imath$; then
$\Theta_i\neq\Theta_j$, i.e.\ $j\in\{j':\Theta_{j'}\neq\Theta_i\}$,
and by the definition of the stopping rule and of $\widehat\imath$ as
the GLR maximizer,
\[
\min_{j':\,\Theta_{j'}\neq\Theta_i}\mathrm{LLR}_{i,j'}(\tau)\;>\;\beta(\tau,\delta)\quad\Longrightarrow\quad \mathrm{LLR}_{i,j}(\tau)>\beta(\tau,\delta).
\]
Therefore
\[
\big\{\tau<\infty,\ \widehat\Theta\neq\Theta_j\big\}\;\subseteq\;\bigcup_{i:\,\Theta_i\neq\Theta_j}\ \big\{\exists\,t\ge1:\ \mathrm{LLR}_{i,j}(t)>\beta(t,\delta)\big\}.
\]
Taking $\mathbb P_j$ and applying a union bound over $i$ (there are at
most $K-1$ indices with $\Theta_i\neq\Theta_j$), then
Lemma~\ref{lem:grid} to each pair $(i,j)$:
\[
\mathbb P_j\big(\tau<\infty,\ \widehat\Theta\neq\Theta_j\big)\;\le\;\sum_{i:\,\Theta_i\neq\Theta_j}\frac{\delta}{K-1}\;\le\;\delta.\qedhere
\]
\end{proof}

\begin{remark}\label{rem:anyrule}
The proof used nothing about the sampling rule except $\mathcal
F_{t-1}$-measurability of $A_t$ (needed for \eqref{eq:ub:condexp}); in
particular it does not use \ref{asm:H2}, \ref{asm:H3}, \ref{asm:H4}, or
the convergence of $N_a(t)/t$. This is the sense in which
Theorem~\ref{thm:UB}(a) holds ``for any sampling rule''.
\end{remark}

\subsection{Part (b): Asymptotic Cost Matching}\label{subsec:ub:partb}

\paragraph{Standing setup for Subsection \ref{subsec:ub:partb}.}
The true environment is $i^*$; \ref{asm:H1}--\ref{asm:H3},
\ref{asm:H2plus}, \ref{asm:H4} hold; $g(a):=g_{i^*}(a)>0$ for all $a$
(degenerate case in Subsection~\ref{subsec:ub:degenerate});
$\mathrm{Alt}^*\neq\varnothing$ (otherwise $C_{i^*}=0$, trivial, see
Subsubsection~\ref{subsubsec:ub:emptyalt}). We write
$w^*:=w^*(i^*)$, $d_a^j:=d(K_{i^*}^a\Vert K_j^a)$,
$D_j(w):=\sum_a w_a d_a^j$, $c(w):=\sum_a g(a)w_a$, and
\begin{equation}
S^*\;:=\;\min_{j\in\mathrm{Alt}^*}D_j(w^*)\;>\;0,\qquad s^*\;:=\;\frac1{S^*}, \label{eq:ub:setup}
\end{equation}
where $S^*>0$ by Lemma~\ref{lem:wellposed}(iii). The goal is
$\limsup_{\delta\to0}R(\pi^*,E_{i^*})/L\le C_{i^*}$, achieved through
four steps: (i) MLE consistency and plug-in stability
(Subsubsection~\ref{subsubsec:ub:mleconsistency}); (ii) tracking
convergence (Subsubsection~\ref{subsubsec:ub:tracking}); (iii)
stopping-time upper bound
(Subsubsections~\ref{subsubsec:ub:llrconc}--\ref{subsubsec:ub:stopbound});
(iv) cost assembly (Subsubsection~\ref{subsubsec:ub:cost}).
Subsubsection~\ref{subsubsec:ub:forced} collects the deterministic
forced-exploration lemmas used everywhere.

\subsubsection{Forced Exploration: Coverage and Budget Lemmas}\label{subsubsec:ub:forced}

Recall $f(t)=\sqrt{t+K_{\mathcal A}^2}-2K_{\mathcal A}$ and the rule
\eqref{eq:ub:sampling}. Both lemmas are \textbf{deterministic}: they
hold for every realization, regardless of the targets
$w^*(\widehat\imath_t)$.

\begin{lemma}[Coverage]\label{lem:coverage}
For all $t\ge0$ and all $a\in\mathcal A$:
\begin{equation}
N_a(t)\;\ge\;f(t)-K_{\mathcal A}\;\ge\;\sqrt t\;-\;3K_{\mathcal A}. \label{eq:ub:coverage}
\end{equation}
In particular $N_a(t)\to\infty$ at rate $\sqrt t$ uniformly over $a$.
\end{lemma}

\begin{proof}
The second inequality uses
$f(t)=\sqrt{t+K_{\mathcal A}^2}-2K_{\mathcal A}\ge\sqrt t-2K_{\mathcal
A}$. For the first, define the \textbf{deficit} of arm $a$ at time $t$
by $\mathrm{df}_a(t):=f(t)-N_a(t)$ and consider the potential
\[
S(t)\;:=\;\sum_{a\in\mathcal A}\big[\mathrm{df}_a(t)-1\big]^+\;\ge\;0.
\]
We show $S(t)\le K_{\mathcal A}-1$ for all $t$ by induction on $t$; the
claim then follows from $\mathrm{df}_a(t)-1\le[\mathrm{df}_a(t)-1]^+\le
S(t)\le K_{\mathcal A}-1$. Base: $\mathrm{df}_a(0)=f(0)=-K_{\mathcal
A}<0$, so $S(0)=0$. Inductive step $t-1\to t$: write
$f':=f(t)-f(t-1)=\frac{1}{\sqrt{t+K_{\mathcal
A}^2}+\sqrt{t-1+K_{\mathcal A}^2}}\le\frac{1}{2K_{\mathcal
A}}\le\frac12$, and note the key structural fact: \emph{when forced
exploration is active ($\mathcal U_{t-1}\neq\varnothing$), the rule
plays a violator of maximum deficit}, because $\arg\min_{c\in\mathcal
U_{t-1}}N_c(t-1)=\arg\max_{c\in\mathcal
U_{t-1}}\mathrm{df}_c(t-1)$ (the same $f(t-1)$ is subtracted from every
count). In one step an unplayed arm's deficit grows by exactly $f'$ and
the played arm's deficit changes by $-1+f'$. Distinguish three cases.

\textbf{Case 1: $\mathcal U_{t-1}=\varnothing$ (tracking step).} Every
arm is a non-violator: $\mathrm{df}_c(t-1)\le0$. The played arm's new
deficit is $\le-1+f'<0$; every unplayed arm's new deficit is $\le
f'\le\frac12<1$. Hence $S(t)=0$.

\textbf{Case 2: $\mathcal U_{t-1}\neq\varnothing$ and
$D^*:=\max_{c\in\mathcal U_{t-1}}\mathrm{df}_c(t-1)<1$.} Every arm's
deficit at $t-1$ is $<1$ (violators by assumption, non-violators since
their deficit is $\le0$). Hence every arm's new deficit is $<1+f'$,
i.e.\ its new excess over $1$ is at most $f'$; and the played arm (a
violator) has new deficit $<1-1+f'=f'<1$, excess $0$. Therefore
\[
S(t)\;\le\;(K_{\mathcal A}-1)\,f'\;\le\;\frac{K_{\mathcal A}-1}{2K_{\mathcal A}}\;\le\;K_{\mathcal A}-1.
\]

\textbf{Case 3: $\mathcal U_{t-1}\neq\varnothing$ and the played arm
$b:=A_t$ has $e_b:=\mathrm{df}_b(t-1)-1\ge0$ (i.e.\ $D^*\ge1$).}
Unplayed violators $c$ satisfy $\mathrm{df}_c(t-1)\le D^*=1+e_b$, so
their excess over $1$ grows by at most $f'$ and their new excess is at
most $e_b+f'$; non-violators keep excess $0$ (their new deficit is $\le
f'\le1$).

\begin{itemize}
\item \emph{Sub-case 3a: $e_b\ge 1-f'$.} The played arm's new excess
is exactly $e_b-(1-f')\ge0$: it drops by $1-f'$. Each unplayed
violator's excess grows by at most $f'$, and there are at most
$K_{\mathcal A}-1$ of them, so
\[
S(t)\;\le\;S(t-1)-(1-f')+(K_{\mathcal A}-1)f'\;=\;S(t-1)-1+K_{\mathcal A}f'\;\le\;S(t-1)-\frac12\;\le\;K_{\mathcal A}-1,
\]
using $S(t-1)\le K_{\mathcal A}-1$ (induction) and
$f'\le\frac{1}{2K_{\mathcal A}}$.
\item \emph{Sub-case 3b: $0\le e_b<1-f'$.} The played arm's new excess
is $0$, and every unplayed violator's new excess is at most
$e_b+f'<1$, so
\[
S(t)\;\le\;(K_{\mathcal A}-1)\,(e_b+f')\;<\;K_{\mathcal A}-1.
\]
\end{itemize}

In all cases $S(t)\le K_{\mathcal A}-1$, completing the induction.
\end{proof}

\begin{remark}\label{rem:potential}
The proof is tie-break-agnostic (any maximizer of the deficit works)
and uses nothing about the targets $w^*(\widehat\imath_t)$: coverage
holds uniformly over the pre-stabilization phase as well. The potential
argument is a self-contained version of the Lemma~7 analysis of
\citet{garivier2016optimal}; the curve $\sqrt{t+K_{\mathcal
A}^2}-2K_{\mathcal A}$ is chosen precisely so that the one-step deficit
growth $f'\le\frac{1}{2K_{\mathcal A}}$ is dominated by the unit
decrease from playing the maximum-deficit violator. (Numerical
spot-check: with adversarially switching targets the observed maximum
deficit is far below $K_{\mathcal A}$.)
\end{remark}

\begin{lemma}[Forced-exploration budget]\label{lem:budget}
Let $F(t):=\sum_{s\le t}\mathbf 1[\mathcal U_{s-1}\neq\varnothing]$ be
the number of forced-exploration rounds up to $t$. Then
\begin{equation}
F(t)\;\le\;K_{\mathcal A}\big(f(t)+1\big)\;\le\;K_{\mathcal A}\big(\sqrt t+K_{\mathcal A}\big). \label{eq:ub:budget}
\end{equation}
\end{lemma}

\begin{proof}
A forced play of arm $a$ at time $s$ requires $a\in\mathcal U_{s-1}$,
i.e.\ $N_a(s-1)<f(s-1)\le f(t)$, and increases $N_a$ by one. Let
$s_r\le t$ be the last forced-play time of arm $a$ up to $t$: then the
number of forced plays of $a$ up to $t$ is at most
$N_a(s_r)=N_a(s_r-1)+1$, with $N_a(s_r-1)<f(s_r-1)\le f(t)$. Since
$N_a$ is integer-valued, nondecreasing and starts at $0$, this number
is at most $\lfloor f(t)\rfloor+1\le f(t)+1$. Summing over the
$K_{\mathcal A}$ arms and using $\sqrt{t+K_{\mathcal A}^2}-2K_{\mathcal
A}+1\le\sqrt t+K_{\mathcal A}$ (since $\sqrt{t+K_{\mathcal
A}^2}\le\sqrt t+K_{\mathcal A}$) gives \eqref{eq:ub:budget}.
\end{proof}

\subsubsection{Step (i): MLE Consistency and Exact Plug-in Stabilization}\label{subsubsec:ub:mleconsistency}

This subsubsection proves that $\widehat\imath_t=i^*$ for all
sufficiently large $t$, a.s., with a quantitative tail bound on the
stabilization time. The argument uses only \ref{asm:H1}, \ref{asm:H4},
\ref{asm:H2plus} and the coverage Lemma~\ref{lem:coverage} ---
\textbf{not} the tracking rule --- so there is no circularity with
Subsubsection~\ref{subsubsec:ub:tracking}.

\begin{lemma}[Per-action observations are i.i.d.]\label{lem:iid}
Under $\mathbb P_{i^*}$, for each action $a$, the observations
$(Y_s)_{s:\,A_s=a}$ are i.i.d.\ $\sim K_{i^*}^a$; more precisely, for
every $n$ and every measurable set $B$,
\[
\mathbb P_{i^*}\big(Y_{s_1}\in B_1,\dots,Y_{s_n}\in B_n\big)=\prod_{r=1}^n K_{i^*}^a(B_r),
\]
where $s_1<\dots<s_n$ are the first $n$ times action $a$ is played
(stopping times w.r.t.\ the enlarged filtration including $A_s$).
\end{lemma}

\begin{proof}
By \ref{asm:H1}, $\mathbb P_{i^*}(Y_s\in
B\mid\mathcal F_{s-1},A_s=a)=K_{i^*}^a(B)$. The claim follows by
induction on $n$: condition on the history up to $s_n-1$; $s_n$ and
$A_{s_n}=a$ are determined by that history, so $\mathbb
P(Y_{s_n}\in B_n\mid\mathcal F_{s_n-1},s_n,A_{s_n})=K_{i^*}^a(B_n)$,
which is constant, hence independent of
$(Y_{s_1},\dots,Y_{s_{n-1}})$.
\end{proof}

\begin{lemma}[Uniform empirical deviation bound]\label{lem:empdev}
Let $h(n):=\sqrt{5\log n/n}$ for $n\ge2$. For $m\ge3$, define the event
\[
\mathcal G^{\mathrm{emp}}_m\;:=\;\Big\{\forall a\in\mathcal A,\ \forall y\in\mathcal Y,\ \forall n\ge m:\ \big|\widehat K^a_{(n)}(y)-K_{i^*}^a(y)\big|\le h(n)\Big\},
\]
where $\widehat K^a_{(n)}$ is the empirical distribution of the first
$n$ observations of action $a$ (Lemma~\ref{lem:iid}). Then
\begin{equation}
\mathbb P_{i^*}\big((\mathcal G^{\mathrm{emp}}_m)^c\big)\;\le\;\frac{2K_{\mathcal A}|\mathcal Y|}{9}\,(m-1)^{-9}. \label{eq:ub:emptail}
\end{equation}
\end{lemma}

\begin{proof}[Proof \textup{(}full computation\textup{)}]
For fixed $(a,y)$, the indicators $\mathbf 1[Y_{s_r}=y]$ are i.i.d.\
Bernoulli with mean $K_{i^*}^a(y)$ (Lemma~\ref{lem:iid}), and $\widehat
K^a_{(n)}(y)$ is their empirical mean. Hoeffding's inequality for
i.i.d.\ $[0,1]$-bounded variables gives, for each fixed $n$,
\[
\mathbb P\big(|\widehat K^a_{(n)}(y)-K_{i^*}^a(y)|>h(n)\big)\;\le\;2\exp\big(-2n\,h(n)^2\big)\;=\;2\exp\big(-10\log n\big)\;=\;2n^{-10}.
\]
Union bound over $n\ge m$ (with $\sum_{n\ge
m}n^{-10}\le\int_{m-1}^{\infty}x^{-10}\mathrm
dx=\frac{1}{9}(m-1)^{-9}$) and then over the $K_{\mathcal A}|\mathcal
Y|$ pairs $(a,y)$ gives \eqref{eq:ub:emptail}.
\end{proof}

\begin{lemma}[KL modulus of continuity under \ref{asm:H4}]\label{lem:klmodulus}
For any $P,P',Q\in\Delta(\mathcal Y)$ with $Q\ge q_{\min}$:
\begin{equation}
\big|d(P\Vert Q)-d(P'\Vert Q)\big|\;\le\;\omega\big(\|P-P'\|_1\big),\qquad \omega(u):=u\log\frac{1}{q_{\min}}+\frac{u}{2}\log(|\mathcal Y|-1)+h_2\Big(\frac{u}{2}\Big), \label{eq:ub:klmodulus}
\end{equation}
where $h_2(p):=-p\log p-(1-p)\log(1-p)$; in particular
$\omega(u)\to0$ as $u\to0$. Moreover
\begin{equation}
d(P\Vert Q)\;\le\;\sum_y\frac{(P(y)-Q(y))^2}{Q(y)}\;\le\;\frac{|\mathcal Y|}{q_{\min}}\,\max_y\,(P(y)-Q(y))^2. \label{eq:ub:chi2}
\end{equation}
\end{lemma}

\begin{proof}
Since $d(P\Vert Q)=\sum_y P(y)\log P(y)-\sum_y P(y)\log
Q(y)=-H(P)-\sum_y P(y)\log Q(y)$, we have
$d(P\Vert Q)-d(P'\Vert Q)=\big[H(P')-H(P)\big]+\sum_y\big(P'(y)-P(y)\big)\log
Q(y)$. The second term is bounded by
$\|P-P'\|_1\max_y|\log Q(y)|\le u\log(1/q_{\min})$. For the entropy
term, the Audenaert sharp bound gives $|H(P)-H(P')|\le T\log(|\mathcal
Y|-1)+h_2(T)$ with $T:=\frac12\|P-P'\|_1$. This yields
\eqref{eq:ub:klmodulus}. For \eqref{eq:ub:chi2}, use $\log x\le x-1$:
$d(P\Vert Q)=\sum_y P(y)\log\frac{P(y)}{Q(y)}\le\sum_y
P(y)\frac{P(y)-Q(y)}{Q(y)}=\sum_y\frac{(P(y)-Q(y))^2}{Q(y)}$ (writing
$P=(P-Q)+Q$ and simplifying), then bound $1/Q\le1/q_{\min}$ and the sum
over $y$ by $|\mathcal Y|\max_y(\cdot)^2$.
\end{proof}

\begin{lemma}[Exact stabilization of the MLE]\label{lem:stabilization}
Under \ref{asm:H1}, \ref{asm:H4}, \ref{asm:H2plus} and the sampling
rule \eqref{eq:ub:sampling}, there exist (deterministic) $t_3$ and
$c_4>0$ such that
\begin{equation}
\mathbb P_{i^*}\big(\exists\,s\ge t:\ \widehat\imath_s\neq i^*\big)\;\le\;c_4\,t^{-9/2}\qquad\forall\,t\ge t_3. \label{eq:ub:stabtail}
\end{equation}
Consequently $T_0:=\sup\{t:\widehat\imath_t\neq i^*\}$ is a.s.\ finite
with $\mathbb P(T_0>t)\le c_4t^{-9/2}$, $\mathbb E[T_0]<\infty$ (indeed
$\mathbb E[T_0^p]<\infty$ for $p<7/2$), and
\begin{equation}
\widehat\imath_t=i^*\ \ \forall t>T_0\quad\text{a.s.}\qquad\Longrightarrow\qquad w^*(\widehat\imath_t)=w^*(i^*)\ \ \forall t>T_0\quad\text{a.s.} \label{eq:ub:stabilization}
\end{equation}
\end{lemma}

\begin{proof}
Set $m(t):=\lceil\sqrt t\rceil-3K_{\mathcal A}-2$; by
Lemma~\ref{lem:coverage} (Eq.\ \eqref{eq:ub:coverage}),
$N_a(s)\ge\sqrt s-3K_{\mathcal A}\ge m(s)\ge m(t)$ for all $a$ and all
$s\ge t$ (for $t$ large enough that $m(t)\ge3$). Work on $\mathcal
G^{\mathrm{emp}}_{m(t)}$. By \eqref{eq:ub:mle},
$\widehat\imath_s=\arg\min_i D_i(s)$ with $D_i(s)=\sum_a
N_a(s)d(\widehat K^a_s\Vert K_i^a)$.

\emph{True model.} By \eqref{eq:ub:chi2} and the definition of
$\mathcal G^{\mathrm{emp}}_{m(t)}$ (note $\widehat K^a_s=\widehat
K^a_{(N_a(s))}$ with $N_a(s)\ge m(t)$), for every $s\ge t$:
\begin{equation}
\begin{aligned}
D_{i^*}(s)\;
&\le\;\sum_a N_a(s)\cdot\frac{|\mathcal Y|}{q_{\min}}\,h(N_a(s))^2
\;\le\;\sum_a N_a(s)\cdot\frac{|\mathcal Y|}{q_{\min}}\cdot\frac{5\log N_a(s)}{N_a(s)}\\
&\le\;\frac{5K_{\mathcal A}|\mathcal Y|}{q_{\min}}\,\log s\;=:\;C_5\log s. \label{eq:ub:truedvi}
\end{aligned}
\end{equation}

\emph{Wrong models.} Fix $j\neq i^*$. By \ref{asm:H2plus} choose $a_j$
with $d_{a_j}:=d(K_{i^*}^{a_j}\Vert K_j^{a_j})>0$. On $\mathcal
G^{\mathrm{emp}}_{m(t)}$, for $s\ge t$:
\begin{equation}
d\big(\widehat K^{a_j}_s\,\Vert\,K_j^{a_j}\big)\;\ge\;d_{a_j}-\omega\big(\|\widehat K^{a_j}_s-K_{i^*}^{a_j}\|_1\big)\;\ge\;d_{a_j}-\omega\big(|\mathcal Y|\,h(m(s))\big)\;\ge\;\frac{d_{a_j}}{2}, \label{eq:ub:wronglb}
\end{equation}
where the second inequality uses $\|\cdot\|_1\le|\mathcal
Y|\max_y|\cdot|$ and monotonicity of $\omega$ for small arguments, and
the last holds for all $s\ge t\ge t_3$ with $t_3$ chosen
(deterministically) so that $\omega(|\mathcal
Y|h(m(t_3)))\le\frac12\min_{j\neq i^*}d_{a_j}$ --- possible since
$h(m(t))\to0$ and $\omega(u)\to0$. Since all terms in $D_j(s)$ are
nonnegative (KL $\ge0$),
\begin{equation}
D_j(s)\;\ge\;N_{a_j}(s)\cdot\frac{d_{a_j}}{2}\;\ge\;\frac{m(s)\,d_{a_j}}{2}\;\ge\;\frac{m(t)\,d_{\min}^{+}}{2}\;\ge\;c_d\,\big(\sqrt t-3K_{\mathcal A}-2\big), \label{eq:ub:wronggrowth}
\end{equation}
with $d_{\min}^{+}:=\min_{j\neq i^*}d_{a_j}>0$. Since $C_5\log
s=o(\sqrt s)$ while \eqref{eq:ub:wronggrowth} grows like $\sqrt s$,
there is a deterministic $t_4\ge t_3$ such that for all $s\ge t\ge
t_4$: $D_{i^*}(s)<\min_{j\neq i^*}D_j(s)$, hence
$\widehat\imath_s=i^*$ (the tie-break does not trigger since the
inequality is strict). Therefore
\[
\mathbb P_{i^*}\big(\exists s\ge t:\widehat\imath_s\neq i^*\big)\;\le\;\mathbb P\big((\mathcal G^{\mathrm{emp}}_{m(t)})^c\big)\;\le\;\frac{2K_{\mathcal A}|\mathcal Y|}{9}(m(t)-1)^{-9}\;\le\;c_4\,t^{-9/2},
\]
using $m(t)-1\ge\sqrt t/3$ for $t$ large. $\mathbb
E[T_0]=\sum_t\mathbb P(T_0>t)<\infty$ and the moment statement follow
from the polynomial tail. Finally \eqref{eq:ub:stabilization}:
$w^*(\cdot)$ is a fixed deterministic function (lexicographic
selection, Subsubsection~\ref{subsubsec:ub:design}), so
$\widehat\imath_t=i^*$ implies $w^*(\widehat\imath_t)=w^*(i^*)$
exactly.
\end{proof}

\begin{remark}[The finite-class simplification, made explicit]\label{rem:finiteclass}
Because $\mathcal E$ is finite, MLE consistency in the discrete
topology means \emph{exact} stabilization, which turns the plug-in map
$w^*(\widehat\imath_t)$ into the constant $w^*(i^*)$ eventually. No
continuity or differentiability of $w^*(\cdot)$ in any parameter is
needed---in contrast to the continuous-parameter setting of
\citet{garivier2016optimal}, where one must prove regularity of
$\theta\mapsto w^*(\theta)$. The price is \ref{asm:H2plus}: exact
identification must be \emph{possible} against every competitor,
including same-label ones; see
Subsubsection~\ref{subsubsec:ub:assumptions} for why this is minimal,
and Remark~\ref{rem:mle-route} for why the proof goes through the
KL-projection form \eqref{eq:ub:mle} rather than through LLR signs.
\end{remark}

\subsubsection{Step (ii): The Tracking Lemma}\label{subsubsec:ub:tracking}

By Lemma~\ref{lem:stabilization}, the tracking target of rule
\eqref{eq:ub:sampling} is frozen at $w^*:=w^*(i^*)$ for all $t>T_0$.
The following lemma is deterministic given $T_0$; it is the content of
Lemmas~8--9 of \citet{garivier2016optimal} adapted to our notation,
proved in full.

\begin{lemma}[C-tracking with frozen target]\label{lem:tracking}
Define the (random) clock shift $B_a(u):=(T_0+u)\,w^*_a-N_a(T_0+u)$ for
$u\ge0$, where $T_0$ is the stabilization time of
Lemma~\ref{lem:stabilization}. Then $\sum_a B_a(u)=0$ for all $u$, and
for every $t\ge T_0$ and every $a$:
\begin{equation}
\big|N_a(t)-t\,w^*_a\big|\;\le\;\underbrace{(K_{\mathcal A}-1)\,\big(T_0+1\big)+(K_{\mathcal A}-1)\,K_{\mathcal A}\big(\sqrt t+K_{\mathcal A}\big)}_{=:\ \xi(t)}. \label{eq:ub:trackerr}
\end{equation}
Consequently (deterministically) $N_a(t)/t\to w^*_a$ for all $a$ on
$\{T_0<\infty\}$, hence $\mathbb P_{i^*}$-a.s.; and
\begin{equation}
\mathbb E\big[\,\big|N_a(\tau)-\tau w^*_a\big|\,\big]\;\le\;(K_{\mathcal A}-1)\big(\mathbb E[T_0]+1\big)+(K_{\mathcal A}-1)K_{\mathcal A}\big(\sqrt{\mathbb E[\tau]}+K_{\mathcal A}\big)\;+\;\mathbb E[T_0]. \label{eq:ub:trackexp}
\end{equation}
\end{lemma}

\begin{proof}
\emph{Bookkeeping.} For $t=T_0+u$ write $B_a(u):=tw^*_a-N_a(t)$; then
$\sum_aB_a(u)=t-t=0$. If arm $b$ is played at time $t+1$ (forced or
tracking), then
\begin{equation}
B_b(u+1)=B_b(u)-(1-w^*_b),\qquad B_c(u+1)=B_c(u)+w^*_c\ \ (c\neq b). \label{eq:ub:bookkeep}
\end{equation}

\emph{Minimum bound.} On a \textbf{tracking step} (non-forced), the
played arm maximizes $w^*_a-N_a(t)/t$, equivalently maximizes $B_a(u)$
(multiply by $t$); since $\sum_aB_a(u)=0$, the maximizer satisfies
$B_b(u)\ge0$, hence after the step $B_b(u+1)\ge-1+w^*_b\ge-1$;
unplayed arms only increase ($w^*_c\ge0$). On a \textbf{forced step},
the played arm's $B$ \emph{decreases} by $1-w^*_b\le1$. Hence, with
$F(u):=$ number of forced steps in $(T_0,T_0+u]$ and
$M_0:=\max\big\{1,\ -\min_a\big(T_0w^*_a-N_a(T_0)\big)\big\}\le T_0+1$
(using $N_a(T_0)\le T_0$),
\begin{equation}
\min_a B_a(u)\;\ge\;-\big(M_0+F(u)\big)\;\ge\;-\big(T_0+1+K_{\mathcal A}(\sqrt t+K_{\mathcal A})\big), \label{eq:ub:minbound}
\end{equation}
where $F(u)\le F(t)\le K_{\mathcal A}(\sqrt t+K_{\mathcal A})$ by
Lemma~\ref{lem:budget}.

\emph{Maximum bound.} Since $\sum_aB_a(u)=0$, at most $K_{\mathcal
A}-1$ terms are negative, so for each $a$:
$B_a(u)=-\sum_{c\neq a}B_c(u)\le\sum_{c\neq
a}[B_c(u)]^-\le(K_{\mathcal A}-1)(M_0+F(u))$. Combined with
\eqref{eq:ub:minbound},
$|B_a(u)|\le(K_{\mathcal A}-1)(M_0+F(u))\le(K_{\mathcal
A}-1)(T_0+1)+(K_{\mathcal A}-1)K_{\mathcal A}(\sqrt t+K_{\mathcal A})$,
which is \eqref{eq:ub:trackerr}.

\emph{Convergence.} On $\{T_0<\infty\}$, $\xi(t)/t\to0$ as
$t\to\infty$ ($T_0$ is a finite constant for the realization, $\sqrt
t/t\to0$), so $N_a(t)/t\to w^*_a$; by Lemma~\ref{lem:stabilization},
$T_0<\infty$ a.s. For \eqref{eq:ub:trackexp}: on $\{\tau\le T_0\}$ use
$|N_a(\tau)-\tau w^*_a|\le\tau\le T_0$; on $\{\tau>T_0\}$ use
\eqref{eq:ub:trackerr} at $t=\tau$; then take expectations and apply
Jensen $\mathbb E[\sqrt\tau]\le\sqrt{\mathbb E[\tau]}$ (valid since
$\mathbb E[\tau]<\infty$, proved in
Subsubsection~\ref{subsubsec:ub:stopbound}).
\end{proof}

\begin{remark}\label{rem:trackingui}
The proof needs no continuity of the target map: the target is
\emph{exactly} constant after $T_0$ (finite-class simplification). Note
the error is $O(\sqrt t)$ deterministically given $T_0$, and
$O(\sqrt{\mathbb E\tau})$ in expectation---this is what replaces a
uniform-integrability argument in the cost assembly
(Subsubsection~\ref{subsubsec:ub:cost}).
\end{remark}

\subsubsection{Step (iii), Part 1: Per-Action LLR Concentration (Good Events with Explicit Constants)}\label{subsubsec:ub:llrconc}

For $j\in\mathrm{Alt}^*$ and $a\in\mathcal A$, let
$S^a_j(n):=\sum_{r=1}^{n}X^a_{j,r}$, where
$X^a_{j,r}:=\log\frac{K_{i^*}^a(Y^a_r)}{K_j^a(Y^a_r)}$ and $(Y^a_r)$
are the i.i.d.\ observations of action $a$ (Lemma~\ref{lem:iid}). Then
$\mathbb E[X^a_{j,r}]=d_a^j$ and $|X^a_{j,r}|\le B$ a.s.\ by
\ref{asm:H4}, Eq.\ \eqref{eq:ub:B}. Note
\begin{equation}
\mathrm{LLR}_{i^*,j}(t)\;=\;\sum_a S^a_j\big(N_a(t)\big). \label{eq:ub:llrdecomp}
\end{equation}

\begin{lemma}[Uniform deviation bound, explicit constants]\label{lem:llrdev}
For every $\varepsilon>0$, $x>0$, and each pair $(a,j)$:
\begin{equation}
\mathbb P_{i^*}\Big(\exists\,n\ge1:\ \big|S^a_j(n)-n\,d_a^j\big|>\varepsilon n+x\Big)\;\le\;\underbrace{\frac{2}{1-e^{-\varepsilon^2/(2B^2)}}}_{=:\ c_1(\varepsilon)}\;\exp\Big(-\frac{\varepsilon}{B^2}\,x\Big). \label{eq:ub:devbound}
\end{equation}
\end{lemma}

\begin{proof}[Proof \textup{(}full computation\textup{)}]
Hoeffding's inequality applied to the i.i.d.\ $[-B,B]$-bounded
increments: for each fixed $n$, $u>0$,
\[
\mathbb P\big(|S^a_j(n)-nd_a^j|>u\big)\;\le\;2\exp\Big(-\frac{u^2}{2B^2n}\Big).
\]
Set $u=\varepsilon n+x$ and use $(\varepsilon
n+x)^2=\varepsilon^2n^2+2\varepsilon nx+x^2\ge\varepsilon^2n^2+2\varepsilon
nx$:
\[
\frac{(\varepsilon n+x)^2}{2B^2n}\;\ge\;\frac{\varepsilon^2n}{2B^2}+\frac{\varepsilon x}{B^2}\quad\Longrightarrow\quad
\mathbb P\big(|S^a_j(n)-nd_a^j|>\varepsilon n+x\big)\;\le\;2\,e^{-\varepsilon x/B^2}\,e^{-\varepsilon^2n/(2B^2)}.
\]
Union bound over $n\ge1$: $\sum_{n\ge1}e^{-\varepsilon^2n/(2B^2)}=\frac{e^{-\varepsilon^2/(2B^2)}}{1-e^{-\varepsilon^2/(2B^2)}}\le\frac{1}{1-e^{-\varepsilon^2/(2B^2)}}$,
which gives \eqref{eq:ub:devbound}.
\end{proof}

\begin{definition}[Good event]\label{def:goodevent}
For $\varepsilon>0$ and $t\ge1$, define
\begin{equation}
\mathcal E_t(\varepsilon)\;:=\;\Big\{\forall a\in\mathcal A,\ \forall j\in\mathrm{Alt}^*:\ \big|S^a_j(N_a(t))-N_a(t)\,d_a^j\big|\;\le\;\varepsilon N_a(t)+\frac{\varepsilon t}{K_{\mathcal A}}\Big\}. \label{eq:ub:goodevent}
\end{equation}
\end{definition}

\begin{lemma}[Bad-event probability]\label{lem:badevent}
With $c_3(\varepsilon):=K_{\mathcal A}(K-1)\,c_1(\varepsilon)$ and
$c_4:=1/(K_{\mathcal A}B^2)$,
\begin{equation}
\mathbb P_{i^*}\big(\mathcal E_t(\varepsilon)^c\big)\;\le\;c_3(\varepsilon)\,e^{-c_4\varepsilon^2 t}. \label{eq:ub:badevent}
\end{equation}
Moreover on $\mathcal E_t(\varepsilon)\cap\{t\ge T_0\}$, for every
$j\in\mathrm{Alt}^*$:
\begin{equation}
\mathrm{LLR}_{i^*,j}(t)\;\ge\;t\,D_j(w^*)\;-\;2\varepsilon t\;-\;K_{\mathcal A}\,\bar d\;\xi(t),\qquad \bar d:=\max_{a,j}d_a^j\le\log\frac1{q_{\min}}, \label{eq:ub:llrlb}
\end{equation}
and consequently
\begin{equation}
\min_{j\in\mathrm{Alt}^*}\mathrm{LLR}_{i^*,j}(t)\;\ge\;t\big(S^*-2\varepsilon\big)\;-\;C_6\sqrt t\qquad\text{on }\mathcal E_t(\varepsilon)\cap\{T_0\le\sqrt t\}\cap\{t\ge\bar t\}, \label{eq:ub:llrlb2}
\end{equation}
where $C_6$ and $\bar t$ are deterministic constants (depending on
$(\mathcal E,\mathcal A,\mathcal Y,g,q_{\min})$ only).
\end{lemma}

\begin{proof}
\eqref{eq:ub:badevent}: apply Lemma~\ref{lem:llrdev} with $x=\varepsilon
t/K_{\mathcal A}$ to each of the $K_{\mathcal A}|\mathrm{Alt}^*|\le
K_{\mathcal A}(K-1)$ pairs $(a,j)$ and take a union bound; the
exponential factor per pair is $\exp(-\varepsilon
x/B^2)=\exp\big(-\varepsilon^2 t/(K_{\mathcal A}B^2)\big)$, whence the
stated value $c_4=1/(K_{\mathcal A}B^2)$; note the bound is uniform in
$n$, so it applies to the random count $N_a(t)$ without any
conditioning. \eqref{eq:ub:llrlb}: on $\mathcal E_t(\varepsilon)$,
summing the deviation bound over $a$:
\[
\mathrm{LLR}_{i^*,j}(t)=\sum_aS^a_j(N_a(t))\;\ge\;\sum_aN_a(t)d_a^j-\varepsilon\sum_aN_a(t)-\varepsilon t\;=\;\sum_aN_a(t)d_a^j-2\varepsilon t,
\]
and by Lemma~\ref{lem:tracking} (for $t\ge T_0$), $N_a(t)\ge
tw^*_a-\xi(t)$, so $\sum_aN_a(t)d_a^j\ge tD_j(w^*)-\xi(t)\sum_ad_a^j\ge
tD_j(w^*)-K_{\mathcal A}\bar d\,\xi(t)$. \eqref{eq:ub:llrlb2}: on
$\{T_0\le\sqrt t\}$,
$\xi(t)=(K_{\mathcal A}-1)(T_0+1)+(K_{\mathcal A}-1)K_{\mathcal
A}(\sqrt t+K_{\mathcal A})\le C_6'\sqrt t$ for $t\ge\bar t$ with
deterministic $C_6',\bar t$; take $C_6:=K_{\mathcal A}\bar d\,C_6'$.
\end{proof}

\subsubsection{Step (iii), Part 2: Stopping-Time Upper Bound}\label{subsubsec:ub:stopbound}

\begin{theorem}\label{thm:stoptime}
Under the standing setup, with $s^*=1/S^*$ (Eq.\ \eqref{eq:ub:setup})
and $L=\log(1/\delta)$:
\begin{equation}
\mathbb E_{i^*}[\tau]\;\le\;s^*\,L\,\big(1+o(1)\big)\qquad(\delta\to0), \label{eq:ub:stopbound}
\end{equation}
where the $o(1)$ depends on $(\mathcal E,\mathcal A,\mathcal
Y,g,q_{\min})$ but not on the sampling history. Moreover $\tau<\infty$
a.s.\ and $\mathbb E[\tau^p]<\infty$ for every $p<\frac54$ (in fact
$\mathbb P(\tau>t)=O(t^{-9/4})$ beyond the $O(L)$ scale, plus an
exponentially small term). More precisely, for fixed $\delta$ the tail
bound is the sum of the exponential term
$c_3(\varepsilon_\delta)e^{-c_4\varepsilon_\delta^2 t}$ --- which by
itself would permit moments of every order --- and the polynomial term
$O(t^{-9/4})$, which alone permits all $p<\frac94$ (the stated
$p<\frac54$ is conservative; higher prescribed moments are available by
raising the exponent in Lemma~\ref{lem:empdev}, cf.\
Remark~\ref{rem:noui}).
\end{theorem}

\begin{proof}
Set $\varepsilon=\varepsilon_\delta:=L^{-1/8}$ (any schedule with
$\varepsilon_\delta\to0$ and $\varepsilon_\delta^{-4}=o(L)$ works).
\emph{Step 1: the deterministic crossing time.} On the event $\mathcal
E_t(\varepsilon)\cap\{T_0\le\sqrt t\}\cap\{t\ge\bar t\}$, by
\eqref{eq:ub:llrlb2} the stopping statistic satisfies
\[
\max_i\min_{j:\Theta_j\neq\Theta_i}\mathrm{LLR}_{i,j}(t)\;\ge\;\min_{j\in\mathrm{Alt}^*}\mathrm{LLR}_{i^*,j}(t)\;\ge\;t\big(S^*-2\varepsilon\big)-C_6\sqrt t.
\]
Hence $\tau\le t$ as soon as $t(S^*-2\varepsilon)-C_6\sqrt
t>\beta(t,\delta)=L+\log(2t(K-1))$. Define
\[
t_1(\delta,\varepsilon)\;:=\;\min\big\{t\ge\bar t:\ t\big(S^*-2\varepsilon\big)-C_6\sqrt t\;\ge\;L+\log\big(2t(K-1)\big)\big\}.
\]
The constraint function
$\psi(t):=t(S^*-2\varepsilon)-C_6\sqrt t-\beta(t,\delta)$ is
nondecreasing on $t\ge\bar t'$ for a deterministic $\bar t'$: its
continuous derivative $(S^*-2\varepsilon)-\frac{C_6}{2\sqrt t}-\frac1t$
is positive for all sufficiently large $t$ since $S^*-2\varepsilon>0$.
Enlarging the deterministic $\bar t$ to $\max\{\bar t,\bar t'\}$ if
needed, the crossing inequality therefore holds for \emph{every} $t\ge
t_1(\delta,\varepsilon)$, which is exactly what the contrapositive of
Step 2 below uses.
\emph{Inversion computation.} Any $t$ satisfying the constraint obeys
$t\ge L/S^*$; for $t\le 2s^*L$ we have
$\log(2t(K-1))\le\log(4s^*L(K-1))=:c_{\log}$ and $C_6\sqrt t\le
C_6\sqrt{2s^*L}$. Hence the constraint is implied by
$t(S^*-2\varepsilon)\ge L+c_{\log}+C_6\sqrt{2s^*L}$, i.e.\
\[
t\;\ge\;\frac{L+c_{\log}+C_6\sqrt{2s^*L}}{S^*-2\varepsilon}\;=\;\frac{s^*L\big(1+\frac{c_{\log}}{L}+\frac{C_6\sqrt{2s^*}}{\sqrt L}\big)}{1-2\varepsilon/S^*}\;=:\;\bar t_1(\delta,\varepsilon),
\]
provided $\bar t_1\le 2s^*L$ (true for $L$ large, since the numerator
is $L(1+o(1))$ and $S^*-2\varepsilon\to S^*$). So
$t_1(\delta,\varepsilon_\delta)\le\bar
t_1(\delta,\varepsilon_\delta)=s^*L\big(1+O(\varepsilon_\delta)+O\big(\tfrac{\log
L}{L}\big)\big)=s^*L(1+o(1))$.

\emph{Step 2: tail sum.} $\mathbb E[\tau]=\sum_{t\ge0}\mathbb
P(\tau>t)\le t_1+\sum_{t\ge t_1}\mathbb P(\tau>t)$. For $t\ge
t_1(\ge\bar t)$, the contrapositive of Step 1 gives
\[
\{\tau>t\}\;\subseteq\;\mathcal E_t(\varepsilon)^c\;\cup\;\{T_0>\sqrt t\},
\]
so by Lemma~\ref{lem:badevent} and Lemma~\ref{lem:stabilization},
\begin{equation}
\sum_{t\ge t_1}\mathbb P(\tau>t)\;\le\;\sum_{t\ge1}c_3(\varepsilon)e^{-c_4\varepsilon^2t}+\sum_{t\ge1}c_4'\,(\sqrt t)^{-9/2}\;\le\;\frac{c_3(\varepsilon)}{c_4\varepsilon^2}+c_4'\sum_{t\ge1}t^{-9/4}.\label{eq:ub:tailsum}
\end{equation}
The second sum is a finite constant. For the first:
$c_3(\varepsilon)=K_{\mathcal
A}(K-1)\cdot\frac{2}{1-e^{-\varepsilon^2/(2B^2)}}=O(\varepsilon^{-2})$
as $\varepsilon\to0$ (since $1-e^{-u}\sim u$), so
$\frac{c_3(\varepsilon)}{c_4\varepsilon^2}=O(\varepsilon^{-4})=O(L^{1/2})=o(L)$
by the choice $\varepsilon_\delta=L^{-1/8}$. Hence $\mathbb
E[\tau]\le t_1+o(L)=s^*L(1+o(1))+o(L)=s^*L(1+o(1))$.
\end{proof}

\begin{remark}[$\tau<\infty$ under every environment]\label{rem:tauevery}
Fix any environment $i'$ with $\mathrm{Alt}(i')\neq\varnothing$ as the
truth. The machinery of
Subsubsections~\ref{subsubsec:ub:forced}--\ref{subsubsec:ub:stopbound}
is environment-generic: forced-exploration coverage and budget
\eqref{eq:ub:coverage} are sampling-rule properties; MLE stabilization
(Subsubsection~\ref{subsubsec:ub:mleconsistency}) uses only
\ref{asm:H1}, \ref{asm:H4}, \ref{asm:H2plus} for the pairs $(i',j)$
--- all assumed; the tracking lemma
(Subsubsection~\ref{subsubsec:ub:tracking}) holds for any frozen
target; the good events (Subsubsection~\ref{subsubsec:ub:llrconc})
involve only the channels. After stabilization, the target is $w^*(i')$ in
the unperturbed branch and $w^*_{\eta_\delta}(i')$ in the perturbed branch.
By \ref{asm:H2} and Lemma~\ref{lem:wellposed}(iii), the corresponding
information rate is strictly positive for every fixed $\delta$, because
$\eta_\delta>0$. The crossing-time analysis of Theorem~\ref{thm:stoptime} then
applies verbatim and yields $\mathbb P_{i'}(\tau<\infty)=1$ (indeed
$\mathbb E_{i'}[\tau]\le s^*(i')L(1+o(1))$). If
$\mathrm{Alt}(i')=\varnothing$, then $\tau=1$ by the empty-min
convention of \eqref{eq:ub:glr}. Hence $\mathbb P_j(\tau<\infty)=1$
for every $j$, which closes the ``$\delta$-correct'' statement of
Theorem~\ref{thm:UBa}. This mirrors the Proposition~13 structure of
\citet{garivier2016optimal} (good event $+$ forced exploration $+$ tail
sum), with the tail computed explicitly in \eqref{eq:ub:tailsum}.
\end{remark}

\subsubsection{Step (iv): Cost Assembly and the LP Identity}\label{subsubsec:ub:cost}

\begin{lemma}[LP identity for the ratio design]\label{lem:lpidentity}
Let $g>0$ and \ref{asm:H2} hold for $i^*$. Then
\begin{equation}
\min_{w\in\Delta(\mathcal A)}\ \frac{c(w)}{\min_{j\in\mathrm{Alt}^*}D_j(w)}\;=\;C_{i^*}, \label{eq:ub:lpidentity}
\end{equation}
the minimum is attained at $w^*$ (Definition \eqref{eq:ub:design}), and
therefore
\begin{equation}
\frac{c(w^*)}{S^*}\;=\;\frac{\sum_a g(a)w^*_a}{\min_{j\in\mathrm{Alt}^*}\sum_a w^*_a d_a^j}\;=\;C_{i^*}. \label{eq:ub:ratioidentity}
\end{equation}
Moreover
$C_{i^*}=\big[\sup_w\inf_j\sum_a\frac{w_a}{g(a)}d_a^j\big]^{-1}$ (the
$(w/g)$-weighted max--min form of the main text), and
$C_{i^*}=\max\big\{\sum_j\lambda_j:\ \sum_j\lambda_jd_a^j\le g(a)\
\forall a,\ \lambda\ge0\big\}$ (LP dual form).
\end{lemma}

\begin{proof}
\emph{Step 1 ($\inf\ge C_{i^*}$).} Let $w\in\Delta(\mathcal A)$ with
$\rho:=\min_jD_j(w)>0$. Set $m_a:=w_a/\rho$. Then $\sum_a m_a
d_a^j=D_j(w)/\rho\ge1$ for all $j$, so $m$ is feasible for the LP
defining $C_{i^*}$, and $\sum_a g(a)m_a=c(w)/\rho\ge C_{i^*}$. (If
$\rho=0$, the ratio is $+\infty$.) Taking the infimum over $w$ gives
$\inf_w\frac{c(w)}{\min_jD_j(w)}\ge C_{i^*}$.

\emph{Step 2 ($\min\le C_{i^*}$, attained).} The LP is feasible
(\ref{asm:H2}) and its objective $m\mapsto\sum_ag(a)m_a$ is coercive on
the feasible set ($g>0$: $\|m\|\to\infty\Rightarrow\sum g m\to\infty$),
and the feasible set is closed; hence an LP optimizer $m^*$ exists with
$\sum_ag(a)m^*_a=C_{i^*}$. Since the constraints scale,
$\min_j\sum_am^*_ad_a^j=1$ (if it were $>1$, scaling $m^*$ down
strictly reduces the objective; if any constraint failed, infeasible).
Set $w^\circ_a:=m^*_a/\sum_bm^*_b$ (note $m^*\neq0$). Then
$D_j(w^\circ)=\frac{\sum_am^*_ad_a^j}{\sum_bm^*_b}\ge\frac{1}{\sum_bm^*_b}$
with equality at the minimizing $j$, and
$c(w^\circ)=\frac{C_{i^*}}{\sum_bm^*_b}$, so
$\frac{c(w^\circ)}{\min_jD_j(w^\circ)}=C_{i^*}$. Hence the infimum in
\eqref{eq:ub:lpidentity} equals $C_{i^*}$ and is attained at $w^\circ$
(and $w^*$ is \emph{a} maximizer of
$\Phi_{i^*}=\frac{\min_jD_j(\cdot)}{c(\cdot)}$, i.e.\ a minimizer of
the ratio, by Definition \eqref{eq:ub:design}; the lexicographic
selection only picks among optimizers).

\emph{Step 3 (other forms).} LP strong duality for the covering LP
(feasible, bounded) gives
$C_{i^*}=\max\{\sum_j\lambda_j:\sum_j\lambda_jd_a^j\le g(a)\,\forall
a\}$. For the $(w/g)$-weighted form, apply Sion's minimax theorem to
$(w,\mu)\mapsto\sum_a\frac{w_a}{g(a)}\sum_j\mu_jd_a^j$ on
$\Delta(\mathcal A)\times\Delta(\mathrm{Alt}^*)$ (both compact convex,
payoff bilinear continuous):
$\sup_w\inf_j\sum_a\frac{w_a}{g(a)}d_a^j=\inf_{\mu}\max_a\frac{\bar
D_a(\mu)}{g(a)}$ with $\bar D_a(\mu):=\sum_j\mu_jd_a^j$. The dual LP
value equals $\sup_{\mu}\min_a\frac{g(a)}{\bar
D_a(\mu)}=\big[\inf_\mu\max_a\frac{\bar D_a(\mu)}{g(a)}\big]^{-1}$
(restrict the dual constraint $\sum_j\lambda_jd_a^j\le g(a)$ to
$\lambda=\Lambda\mu$ and optimize $\Lambda$). Hence
$\sup_w\inf_j\sum_a\frac{w_a}{g(a)}d_a^j=\frac{1}{C_{i^*}}$: the two
max--min forms coincide in \emph{value} (their maximizers differ in
general --- Remark~\ref{rem:deviation}).
\end{proof}

\begin{theorem}[Cost assembly; Theorem~\ref{thm:UB}(b)]\label{thm:costassembly}
Under the standing setup of Subsection~\ref{subsec:ub:partb},
\begin{equation}
\limsup_{\delta\to0}\ \frac{R(\pi^*,E_{i^*})}{\log(1/\delta)}\;\le\;C_{i^*}. \label{eq:ub:costbound}
\end{equation}
\end{theorem}

\begin{proof}
By Tonelli's theorem (Subsubsection~\ref{subsubsec:ub:model}),
$R(\pi^*,E_{i^*})=\sum_ag(a)\,\mathbb E[N_a(\tau)]$; all terms are
finite since $\mathbb E[N_a(\tau)]\le\mathbb E[\tau]<\infty$
(Theorem~\ref{thm:stoptime}) and $g$ is finite. For each $a$,
\eqref{eq:ub:trackexp} of Lemma~\ref{lem:tracking} gives
\[
\mathbb E[N_a(\tau)]\;\le\;w^*_a\,\mathbb E[\tau]\;+\;(K_{\mathcal A}-1)\big(\mathbb E[T_0]+1\big)+(K_{\mathcal A}-1)K_{\mathcal A}\big(\sqrt{\mathbb E[\tau]}+K_{\mathcal A}\big)\;+\;\mathbb E[T_0].
\]
Multiply by $g(a)$ and sum over $a$:
\[
R(\pi^*,E_{i^*})\;\le\;c(w^*)\,\mathbb E[\tau]\;+\;K_{\mathcal A}\,g_{\max}\Big[(K_{\mathcal A}-1)\big(\mathbb E[T_0]+1\big)+(K_{\mathcal A}-1)K_{\mathcal A}\big(\sqrt{\mathbb E[\tau]}+K_{\mathcal A}\big)+\mathbb E[T_0]\Big].
\]
The bracket is $O(1)+O(\sqrt{\mathbb E[\tau]})=o(L)$: $\mathbb
E[T_0]<\infty$ is independent of $\delta$
(Lemma~\ref{lem:stabilization}; the sampling rule does not depend on
$\delta$), and $\mathbb E[\tau]\le s^*L(1+o(1))$
(Theorem~\ref{thm:stoptime}) gives $\sqrt{\mathbb E[\tau]}=O(\sqrt L)$.
Inserting the stopping-time bound:
\[
R(\pi^*,E_{i^*})\;\le\;c(w^*)\,s^*\,L\,\big(1+o(1)\big)+o(L)\;=\;\frac{c(w^*)}{S^*}\,L\,\big(1+o(1)\big)\;=\;C_{i^*}\,L\,\big(1+o(1)\big),
\]
where the last equality is Lemma~\ref{lem:lpidentity}, Eq.\
\eqref{eq:ub:ratioidentity}.
\end{proof}

\begin{remark}[Why no uniform-integrability detour is needed]\label{rem:noui}
The assembly above controls $\mathbb E|N_a(\tau)-\tau w^*_a|$ directly
by the \emph{pointwise} tracking error \eqref{eq:ub:trackerr} plus
Jensen's inequality $\mathbb E\sqrt\tau\le\sqrt{\mathbb E\tau}$; the
only integrability inputs are $\mathbb E[T_0]<\infty$ (polynomial tail
\eqref{eq:ub:stabtail}) and $\mathbb E[\tau]<\infty$
(Theorem~\ref{thm:stoptime}). Higher moments of $\tau$ are also
available if wanted: the tail bounds of Theorem~\ref{thm:stoptime} give
$\mathbb E[\tau^p]<\infty$ for $p<5/4$ as stated (in fact $p<9/4$ from
the polynomial summand alone, the exponential summand at fixed $\delta$
permitting all orders), and raising the exponent $q=5$ in
$h(n)=\sqrt{q\log n/n}$ of Lemma~\ref{lem:empdev} yields any prescribed
finite moment at the price of larger constants. Since the Jensen route
suffices for the cost assembly, we do not need that strength.
Proposition~13 of \citet{garivier2016optimal} uses the same
good-event/tail-sum structure.
\end{remark}

\subsection{Synthesis}\label{subsec:ub:synthesis}

Theorem~\ref{thm:UB} is the conjunction of the results above: part (a)
is Theorem~\ref{thm:UBa} together with the a.s.-stopping statement of
Remark~\ref{rem:tauevery}; part (b) is
Theorem~\ref{thm:costassembly}, with the degenerate cases treated in
Subsection~\ref{subsec:ub:degenerate}. Corollary~\ref{cor:main} then
combines Theorem~\ref{thm:UB}(b) with Theorem~\ref{thm:LB} of
Appendix~\ref{sec:prooflb}. \hfill$\square$

\subsection{Degenerate Cases}\label{subsec:ub:degenerate}

\subsubsection{Vanishing Costs: $g_{i^*}$ Has Zeros ($C_{i^*}=0$ Included)}\label{subsubsec:ub:zeros}

If $g(a)=0$ for some actions (\ref{asm:H3} still holds), the ratio
design \eqref{eq:ub:design} may then be ill-defined ($c(w)=0$
denominators). The following perturbation device recovers
Theorem~\ref{thm:UB}(b) in full. Its proof also applies when the true
row is positive but another candidate row contains a zero, so that
Algorithm~\ref{alg:cts} selects the perturbed branch globally.

\medskip\noindent\textbf{Definition (perturbed CTS).} For $\eta>0$ let
$g^\eta(a):=g(a)+\eta$, and let $\pi^*_\eta$ be the CTS policy of
Subsection~\ref{subsec:ub:policy} with the design map computed from the
perturbed costs: $w^*_\eta(i):=\arg\max_w\min_j
D_j(w;i)/c_\eta(w;i)$ (well-defined by Lemma~\ref{lem:wellposed} since
$g^\eta>0$), everything else unchanged.

\begin{lemma}[LP value continuity]\label{lem:lpcontinuity}
Let $C(\eta)$ be the LP value with costs $g^\eta:=g+\eta$. Then for
every $\eta>0$,
\begin{equation}
C_{i^*}\;\le\;C(\eta)\;\le\;C_{i^*}+\eta\,\|m^*\|_1, \label{eq:ub:lpcont}
\end{equation}
where $m^*$ is an optimizer of the unperturbed LP.
\end{lemma}

\begin{proof}
The covering LP $\min\{\sum_ag(a)m_a:m\ge0,\ \sum_am_ad_a^j\ge1\
\forall j\}$ is feasible by (\ref{asm:H2}) and bounded below by $0$; by
the fundamental theorem of linear programming a feasible LP with finite
value attains its optimum, so $m^*$ exists with
$\sum_ag(a)m^*_a=C_{i^*}$. (i) $m^*$ is feasible for the perturbed LP
(same constraints), so
$C(\eta)\le\sum_ag^\eta(a)m^*_a=C_{i^*}+\eta\|m^*\|_1$. (ii) $g^\eta\ge
g$ pointwise, so the perturbed objective dominates the unperturbed one
on the common feasible set, giving $C(\eta)\ge C_{i^*}$.
\end{proof}

\begin{theorem}[Degenerate-cost upper bound]\label{thm:degenerate}
Let \ref{asm:H1}--\ref{asm:H3}, \ref{asm:H2plus}, \ref{asm:H4} hold and
$\mathrm{Alt}^*\neq\varnothing$. Run $\pi^*_{\eta_\delta}$ with
$\eta_\delta:=1/\log(\max\{e,L\})$ (for $L>e$, with the good-event parameter
$\varepsilon_\delta:=1/(\log L)^2$). Then
\begin{equation}
\limsup_{\delta\to0}\frac{R(\pi^*_{\eta_\delta},E_{i^*})}{\log(1/\delta)}\;\le\;C_{i^*}\qquad(\text{possibly }C_{i^*}=0). \label{eq:ub:degenbound}
\end{equation}
\end{theorem}

\begin{proof}
All of Subsections~\ref{subsec:ub:parta}--\ref{subsubsec:ub:stopbound}
apply to $\pi^*_\eta$ verbatim with $g$ replaced by $g^\eta$ in the
\emph{design} (the channels, LLR, concentration, MLE stabilization,
forced exploration, and the tracking lemma never used $g$; only the
design $w^*_\eta$ and the cost assembly did). Write
$S^*_\eta:=\min_jD_j(w^*_\eta)$, $s^*_\eta:=1/S^*_\eta$. By
Lemma~\ref{lem:lpidentity} with costs $g^\eta$:
$c_\eta(w^*_\eta)=C(\eta)\,S^*_\eta$. Hence the true cost rate
satisfies
\begin{equation}
c(w^*_\eta)\;=\;c_\eta(w^*_\eta)-\eta\;=\;C(\eta)\,S^*_\eta-\eta. \label{eq:ub:degenrate}
\end{equation}
Also $C(\eta)=c_\eta(w^*_\eta)/S^*_\eta\ge\eta/S^*_\eta$, i.e.\
$s^*_\eta\le C(\eta)/\eta\le(C_{i^*}+\eta\|m^*\|_1)/\eta$ --- and
$S^*_\eta\ge\eta/C(\eta)\ge\eta/(C_{i^*}+\eta\|m^*\|_1)$, which for
$\eta\le1$ is $\ge\eta/(C_{i^*}+\|m^*\|_1)=:\underline S\,\eta$ with a
constant $\underline S>0$.

\emph{Stopping time.} Theorem~\ref{thm:stoptime} applied with
$\varepsilon_\delta=(\log L)^{-2}$: the inversion factor is
$\big(1-2\varepsilon_\delta/S^*_{\eta_\delta}\big)^{-1}=1+O\big(\varepsilon_\delta/(\underline
S\eta_\delta)\big)=1+O\big(\tfrac{(\log L)^{-2}}{(\log
L)^{-1}}\big)=1+O\big((\log L)^{-1}\big)$, and the tail sum
\eqref{eq:ub:tailsum} is
$O(\varepsilon_\delta^{-4})=O\big((\log L)^8\big)=o(L)$. Hence
\begin{equation}
\mathbb E[\tau]\;\le\;s^*_{\eta_\delta}\,L\,\big(1+O\big((\log L)^{-1}\big)\big)\;+\;o(L). \label{eq:ub:stoptimedegen}
\end{equation}
(Note $s^*_{\eta_\delta}\le C(\eta_\delta)/\eta_\delta=O(\log L)$, so
the additive $o(L)$ terms with $\eta$-dependent constants---all
polynomial in $1/\varepsilon_\delta$ and $1/S^*_{\eta_\delta}$, i.e.\
polylogarithmic in $L$---remain $o(L)$.)

\emph{Cost.} As in Theorem~\ref{thm:costassembly}, $R=\sum_ag(a)\mathbb
E[N_a(\tau)]\le c(w^*_{\eta_\delta})\,\mathbb E[\tau]+o(L)$, where the
$o(L)$ tracking-error term is unchanged (it depends on $K_{\mathcal
A},\mathbb E[T_0],\sqrt{\mathbb E\tau}$ only, and $\sqrt{\mathbb
E\tau}=O(\sqrt{s^*_{\eta_\delta}L})=O(\sqrt{L\log L})=o(L)$). Insert
\eqref{eq:ub:degenrate} and \eqref{eq:ub:stoptimedegen}:
\[
R\;\le\;\big(C(\eta_\delta)S^*_{\eta_\delta}-\eta_\delta\big)\,s^*_{\eta_\delta}\,L\,\big(1+o(1)\big)+o(L)\;\le\;C(\eta_\delta)\,L\,(1+o(1))+o(L),
\]
using $C(\eta_\delta)S^*_{\eta_\delta}s^*_{\eta_\delta}=C(\eta_\delta)$
and $\eta_\delta s^*_{\eta_\delta}\ge0$. By
Lemma~\ref{lem:lpcontinuity}, $C(\eta_\delta)\to C_{i^*}$ (in
particular $C(\eta_\delta)\le\eta_\delta\|m^*\|_1\to0$ when
$C_{i^*}=0$).
\end{proof}

\begin{remark}[The B2 corner]\label{rem:b2corner}
When $C_{i^*}=0$, \eqref{eq:ub:degenbound} gives
$R(\pi^*_{\eta_\delta},E_{i^*})=o(\log(1/\delta))$ for the common CTS
family. The pointwise infimum is stronger:
$R_\delta^*(E_{i^*})=0$ at every fixed $\delta$, by
Proposition~\ref{prop:zero-infimum}. These statements concern different
quantifiers and different objects; CTS need not have zero loss.
\end{remark}

\subsubsection{Empty Alternative: $\mathrm{Alt}(i^*)=\varnothing$}\label{subsubsec:ub:emptyalt}

Then every environment has the same label $\Theta^*$, and the constant
decision $\widehat\Theta=\Theta^*$ is $0$-correct. In the GLR rule the
inner $\min$ is over the empty set ($=+\infty$ by convention), so
$\tau=1$ a.s., $R\le\max_ag(a)=O(1)$, and the LP defining $C_{i^*}$
has no constraints, so $C_{i^*}=0$: both Theorem~\ref{thm:UB}(a) and
Theorem~\ref{thm:UB}(b) hold trivially.

\subsection{Technical Provenance}\label{subsec:ub:provenance}
% (Source-md heading note: technical provenance annotation.)

{\small
\setlength{\tabcolsep}{4pt}
\begin{longtable}{@{}>{\raggedright\arraybackslash}p{0.40\textwidth}>{\raggedright\arraybackslash}p{0.20\textwidth}>{\raggedright\arraybackslash}p{0.30\textwidth}@{}}
\hline
Component & Where here & Source template \\
\hline
C-tracking sampling rule, forced-exploration curve $\sqrt{t+K^2}-2K$, coverage/budget/tracking lemmas & Sec.~\ref{subsubsec:ub:sampling}, Lemmas~\ref{lem:coverage}--\ref{lem:budget}, \ref{lem:tracking} & \citet[Lemmas~7--9]{garivier2016optimal} (proofs redone self-contained: potential argument for coverage; positive/negative-part bookkeeping for tracking) \\
GLR stopping rule, threshold $\beta(t,\delta)=\log(2t(K-1)/\delta)$, supermartingale $+$ Doob $+$ geometric grid correctness argument & Sec.~\ref{subsubsec:ub:stopping}, Sec.~\ref{subsec:ub:parta} & \citet[Theorem~10]{garivier2016optimal} (same threshold family; our grid computation is spelled out in Lemma~\ref{lem:grid}) \\
Good-event $+$ tail-sum stopping-time analysis & Secs.~\ref{subsubsec:ub:llrconc}--\ref{subsubsec:ub:stopbound} & \citet[Theorem~14 / Proposition~13 template]{garivier2016optimal} \\
Cost-weighted characteristic time $\sup_w\inf_j D_j(w)/c(w)$ as the optimal design & Sec.~\ref{subsubsec:ub:design}, Lemma~\ref{lem:lpidentity} & \citet{kanarios2024cost} (CABAI), cost-weighted T\&S variant \\
Transportation lemma, $\mathrm{kl}(\delta,1-\delta)\ge\log(1/(2.4\delta))$ (lower bound only, cited) & Sec.~\ref{subsec:ub:synthesis} & \citet[Lemma~1]{kaufmann2016complexity} \\
Tighter time-uniform thresholds (mixture martingales) & not used; Remark~\ref{rem:betaconstants} & Kaufmann--Koolen 2021 (optional improvement; we keep the conservative closed-form $\beta$ for simpler analysis) \\
Audenaert entropy bound (continuity of $H$) & Lemma~\ref{lem:klmodulus} & standard (Audenaert 2007) \\
\hline
\end{longtable}
% Keep longtable counters monotone so PDF destinations remain unique.
}

\paragraph{Statement of contribution.}
The proof technology above is an adaptation of known templates (the
Track-and-Stop machinery of \citet{garivier2016optimal}; the
cost-weighting of \citet{kanarios2024cost}). The contribution of the
present work is the \emph{model}---the alternative set defined through
adequacy labels $\Theta_i$ and the coupling of observation channels to
task losses $g_i$---and the decision-theoretic semantics of the
constant $C(E,f)$, not new concentration or tracking technology.

\subsection{Proof-Completeness Self-Check Checklist}\label{subsec:ub:checklist}
% (Source-md heading note: proof-completeness self-check list.)

Legend: \textbf{[done]} = proved in full in this section;
\textbf{[cite]} = standard result cited with reference;
\textbf{[added]} = assumption/convention added beyond the companion
statement, explicitly flagged; \textbf{[fixed]} = a correction to the
companion statement, explicitly documented.

{\small
\setlength{\tabcolsep}{4pt}
\begin{longtable}{@{}c >{\raggedright\arraybackslash}p{0.54\textwidth} >{\raggedright\arraybackslash}p{0.08\textwidth} >{\raggedright\arraybackslash}p{0.28\textwidth}@{}}
\hline
\# & Technical point & Status & Location / remark \\
\hline
1 & Likelihood factorization under adaptive sampling; MLE = KL-projection form & [done] & Subsubsection~\ref{subsubsec:ub:mle}, Eqs.\ \eqref{eq:ub:lik}--\eqref{eq:ub:mle}; uses \ref{asm:H1} \\
2 & $w^*(i)$ well-posed: $\Phi_i$ continuous on compact $\Delta$, sup attained, tie-break fixed & [done] & Lemma~\ref{lem:wellposed} \\
3 & \textbf{Design-map correction}: cost-ratio optimizer $\arg\sup\inf_j D_j/c$ ($=$ normalized LP solution) instead of the $(w/g)$-weighted argmax of the companion statement & [fixed] & Remark~\ref{rem:deviation}, with explicit $2\times2$ counterexample showing the $(w/g)$-weighted argmax achieves cost ratio $25/6\approx4.17>C_i=30/11\approx2.73$; values of the two max--min problems agree (Lemma~\ref{lem:lpidentity}, Step~3) \\
4 & Forced-exploration coverage $N_a(t)\ge f(t)-K_{\mathcal A}$ & [done] & Lemma~\ref{lem:coverage}, self-contained potential proof (analogue of \citet[Lemma~7]{garivier2016optimal}) \\
5 & Forced-exploration budget $F(t)=O(\sqrt t)$ & [done] & Lemma~\ref{lem:budget} \\
6 & $\exp(\mathrm{LLR})$ nonnegative (super)martingale under $j$, any sampling rule & [done] & Lemma~\ref{lem:testsmg}, conditional-expectation computation \eqref{eq:ub:condexp} written out; \ref{asm:H1} used \\
7 & Doob maximal inequality $\Rightarrow$ fixed-level crossing bound & [done] & Lemma~\ref{lem:doob} \\
8 & Geometric time grid $+$ series $\sum_k e^{-\beta(2^k,\delta)}=\delta/(K-1)$ & [done] & Lemma~\ref{lem:grid}, full computation; constant accounting in Remark~\ref{rem:betaconstants} \\
9 & Error event decomposition $+$ union bound over $\le K-1$ challengers & [done] & Subsubsection~\ref{subsubsec:ub:union} \\
10 & Per-action i.i.d.\ structure of observations & [done] & Lemma~\ref{lem:iid} \\
11 & Uniform empirical-deviation bound with explicit exponent ($h(n)=\sqrt{5\log n/n}$, tail $O(m^{-9})$) & [done] & Lemma~\ref{lem:empdev}, full computation \\
12 & MLE exact stabilization $\widehat\imath_t=i^*$ eventually a.s., tail $\mathbb P(T_0>t)\le c_4t^{-9/2}$, $\mathbb E[T_0]<\infty$ & [done] & Lemma~\ref{lem:stabilization} via KL projection ($D_{i^*}=O(\log t)$ vs $D_j\ge c\sqrt t$); Remark~\ref{rem:mle-route} explains why the KL-projection route is used instead of the LLR-sign route \\
13 & \textbf{Pairwise distinguishability incl.\ same-label environments} & [added] & Assumption~\ref{asm:H2plus}, Subsubsection~\ref{subsubsec:ub:assumptions}; minimal-sufficiency justification given there \\
14 & \textbf{Full support / bounded LLR increments} ($|X|\le B$, KL modulus $\omega$, $\chi^2$ bound \eqref{eq:ub:chi2}) & [added] & Assumption~\ref{asm:H4}, Subsubsection~\ref{subsubsec:ub:assumptions}; used in Lemmas~\ref{lem:klmodulus}, \ref{lem:llrdev} \\
15 & Plug-in stability $w^*(\widehat\imath_t)=w^*(i^*)$ eventually---no continuity of $w^*(\cdot)$ needed (finite class) & [done] & \eqref{eq:ub:stabilization}, Remark~\ref{rem:finiteclass} \\
16 & Tracking convergence $N_a(t)/t\to w^*_a$ a.s.\ with pointwise error $\le(K_{\mathcal A}-1)(T_0+1)+O(\sqrt t)$ & [done] & Lemma~\ref{lem:tracking} (analogue of \citet[Lemmas~8--9]{garivier2016optimal}, proved in full) \\
17 & Per-action LLR concentration, uniform over random counts, explicit constants $c_1(\varepsilon),c_4$ & [done] & Lemma~\ref{lem:llrdev} (Hoeffding $+$ dyadic-free union over $n$, AM--GM split of $(\varepsilon n+x)^2$) \\
18 & Good-event probability $\mathbb P(\mathcal E_t^c)\le c_3(\varepsilon)e^{-c_4\varepsilon^2t}$ with $c_4:=1/(K_{\mathcal A}B^2)$ (exponent $\varepsilon^2t/(K_{\mathcal A}B^2)$ from $x=\varepsilon t/K_{\mathcal A}$ in Lemma~\ref{lem:llrdev}) and LLR lower bound \eqref{eq:ub:llrlb}--\eqref{eq:ub:llrlb2} & [done] & Lemma~\ref{lem:badevent} \\
19 & Stopping-time bound $\mathbb E[\tau]\le s^*L(1+o(1))$: inversion of $t(S^*-2\varepsilon)-C_6\sqrt t\ge\beta(t,\delta)$, tail sums $\sum c_3e^{-c_4\varepsilon^2t}=O(\varepsilon^{-4})$, $\sum t^{-9/4}<\infty$, schedule $\varepsilon_\delta=L^{-1/8}$ & [done] & Theorem~\ref{thm:stoptime}, all series computed \\
20 & $\tau<\infty$ a.s.\ under every environment & [done] & Remark~\ref{rem:tauevery}; uses full pairwise \ref{asm:H2plus} (item~13) \\
21 & Cost identity $R=\sum g\,\mathbb E[N_a(\tau)]$ (Tonelli; no Wald needed) & [done] & Subsubsection~\ref{subsubsec:ub:model}, Theorem~\ref{thm:costassembly} \\
22 & $\mathbb E|N_a(\tau)-\tau w^*_a|=o(L)$ without heavy UI: pointwise tracking error $+$ Jensen $+$ $\mathbb E[T_0]<\infty$ & [done] & \eqref{eq:ub:trackexp}, Remark~\ref{rem:noui}; $\mathbb E[\tau^2]<\infty$ available by raising the exponent in Lemma~\ref{lem:empdev} but not needed \\
23 & LP identity $c(w^*)/\min_jD_j(w^*)=C_{i^*}$; equivalence of LP / cost-ratio / $(w/g)$-weighted / dual forms (values) & [done] & Lemma~\ref{lem:lpidentity} (three steps, including Sion $+$ LP duality for the value equivalence) \\
24 & Degenerate $g$ with zeros / $C_{i^*}=0$ (B2): perturbed design $\pi^*_{\eta_\delta}$, $\eta_\delta=1/\log(\max\{e,L\})$; $\varepsilon_\delta=(\log L)^{-2}$ for $L>e$; LP value continuity & [done] & Subsubsection~\ref{subsubsec:ub:zeros}, Lemma~\ref{lem:lpcontinuity}, Theorem~\ref{thm:degenerate} \\
25 & $\mathrm{Alt}(i^*)=\varnothing$ corner & [done] & Subsubsection~\ref{subsubsec:ub:emptyalt} \\
26 & Synthesis with Theorem~\ref{thm:LB} $\Rightarrow$ $\lim R_\delta^*/\log(1/\delta)=C_i$ & [done] & Subsection~\ref{subsec:ub:synthesis}; Theorem~\ref{thm:LB} cited from the companion lower-bound appendix (Appendix~\ref{sec:prooflb}), not re-proved \\
\hline
\end{longtable}
% Keep longtable counters monotone so PDF destinations remain unique.
}

\paragraph{Explicitly declared residual items (nothing hidden).}
\begin{enumerate}
\item \textbf{\ref{asm:H2plus} and \ref{asm:H4} are additions} to the
\ref{asm:H1}--\ref{asm:H3} of the companion statement. \ref{asm:H2plus}
ensures exact original-environment estimation in this proof; it can be
replaced, for positive losses, by the compatibility condition of
Theorem~\ref{thm:observable-compatibility}; \ref{asm:H4} is a
convenience that can be relaxed to per-pair absolute continuity with
support-restricted constants
(Subsubsection~\ref{subsubsec:ub:assumptions}).
\item \textbf{Design-map deviation (item~3)} is a correction to the
companion statement's plug-in formula; the theorem statements (a), (b)
and the constant $C_i$ are unaffected.
\item In the degenerate case
(Subsection~\ref{subsec:ub:degenerate}), the \emph{policy} is modified
to the perturbed design $\pi^*_{\eta_\delta}$ with $\eta_\delta\to0$;
this is a convention addition (\citet{garivier2016optimal} handle
degeneracies similarly via limiting designs).
\item The tail exponent $-9/2$ in \eqref{eq:ub:stabtail} (via $q=5$ in
$h(n)$) limits the moments of $T_0$ to $p<7/2$ and the available
$\tau$-moments accordingly; any finite moment can be obtained by
increasing $q$. Only $\mathbb E[T_0]$ is used.
\item Theorem~\ref{thm:LB} is used as stated in the companion
lower-bound appendix (Appendix~\ref{sec:prooflb}) and is not reproved
here.
\item \textbf{Policy-specific B2 residuals.} The perturbed CTS family
has $o(\log(1/\delta))$ loss, with positive residuals possible from
forced exploration. This is not the scale of the pointwise optimum,
which is exactly zero by Proposition~\ref{prop:zero-infimum}. Improving
one common policy under an explicit time budget is a distinct problem.
\end{enumerate}

% End of sec_upperbound.tex (end of source file proof_upper_bound.md).

% sec_static_app.tex -- JMLR full version, Appendix C
% "Self-Certification of Representation Adequacy"
% Detailed constructions deferred from the proof of Theorem~\ref{thm:verify}(e)
% (Section~\ref{sec:static:verify}). Content moved verbatim from
% sec_static.tex; no mathematical content added or removed.
%
% Notation macros (guarded; identical to the mathematical specification):
\providecommand{\Adeq}{\mathrm{Adeq}}
\providecommand{\TV}{\mathrm{TV}}
\providecommand{\JS}{\mathrm{JS}}
\providecommand{\KL}{\mathrm{KL}}
\providecommand{\cA}{\mathcal{A}}
\providecommand{\cE}{\mathcal{E}}
\providecommand{\cF}{\mathcal{F}}
\providecommand{\cT}{\mathcal{T}}
\providecommand{\cV}{\mathcal{V}}
\providecommand{\cX}{\mathcal{X}}
\providecommand{\cZ}{\mathcal{Z}}
\providecommand{\cY}{\mathcal{Y}}
\providecommand{\E}{\mathbb{E}}
\providecommand{\ind}{\mathbf{1}}

\section{Detailed Constructions for Theorem~\ref{thm:verify}(e)}
\label{app:static}

This appendix collects the explicit constructions deferred from the proof of
Theorem~\ref{thm:verify}(e) in Section~\ref{sec:static:verify}: the two-point
instance showing that the A1 symmetry is indispensable (item~(e)(i),
Appendix~\ref{app:static:a1}), the general label-prior form and the cyclic
three-environment instance showing that the equal label prior is indispensable
to the closed-form \emph{shape} (item~(e)(iii), Appendix~\ref{app:static:prior}),
and the companion remark separating environment identification from label
identification (Appendix~\ref{app:static:identity}). All computations use the
notation of Section~\ref{sec:static}.

\subsection{Item (e)(i): dropping A1 symmetry, an explicit two-point instance}
\label{app:static:a1}

We give a completely explicit instance in the two-point special case $|\cE| = 2$: $H = \{h_0, h_1\}$; $f$ constant onto a single representation value, so the unique cell is all of $H$; $\cA = \{0,1\}$; $\ell(h,a) = \ind[a \neq h]$ in both environments; $\cE = \{E_0, E_1\}$ with $p_{E_0}(h_1) = 1$ and $p_{E_1}(h_0) = p_{E_1}(h_1) = \tfrac12$. Under $E_0$ the cell Bayes action is $a = 1$ with cell risk $0$, so $\Adeq_{E_0}(f) = 1$; under $E_1$ the within-cell posterior is $p = \tfrac12$ with cell mass $m = 1$, so by Proposition~\ref{prop:binary} the $f$-rule's per-period aliasing regret is $m \min(p, 1-p) = \tfrac12 =: r$ and $\Adeq_{E_1}(f) = 0$. Take the fallback to be the full-history Bayes rule $\delta_{\mathrm{fb}}(h) = h$: matched under $E_1$ with regret $0$, as A1 requires. But under $E_0$, $h = h_1$ almost surely, so the fallback plays $a = 1$ almost surely---exactly the $f$-rule's cell Bayes action---and the mismatch direction ``wrongly using the fallback under $E_0$'' costs $0$, not $r$: the symmetric clause fails, and ``always fallback, never certify'' has total regret identically $0$ under both environments. Yet with the one-step transcript equal to the observed history, $P_0^1$ is the point mass on $h_1$ and $P_1^1 = (\tfrac12, \tfrac12)$, so $\TV(P_0^1, P_1^1) = \tfrac12 < 1$ and \eqref{eq:static:rint} prescribes the strictly positive value $\frac{Tr}{4}$ for the optimum: the formula no longer describes the optimum, and the threshold of (b) collapses with it, there being nothing to buy when certification itself is unnecessary. (Consistently, \eqref{eq:static:asym} here evaluates to $0$, since $\min(r_0\, d\bar{P}_1^t, r_1\, d\bar{P}_0^t) = 0$ when $r_0 = 0$.) The counterexample bites exactly when the internal channel is not perfectly distinguishing.

\subsection{Item (e)(iii): general label priors and the cyclic instance}
\label{app:static:prior}

Under a general label prior $\pi$, the Bayes error of a certifier $C$ with certification event $A = \{C = 1\}$ expands pointwise as
\begin{equation}
\begin{aligned}
\mathrm{err}_\pi(C)
&= \pi_0\, \bar{P}_0^t(A) + \pi_1\, \bar{P}_1^t(A^c)
= \int \Bigl[ \pi_0\, \bar{p}_0\, \ind\{x \in A\} + \pi_1\, \bar{p}_1\, \ind\{x \notin A\} \Bigr] d\mu \\
&\;\ge\; \int \min\bigl(\pi_0\, \bar{p}_0,\, \pi_1\, \bar{p}_1\bigr) d\mu,
\end{aligned}
\label{eq:static:unequalprior}
\end{equation}
because at each $x$ the integrand equals one of the two weighted densities, each at least their minimum; equality holds for the weighted likelihood-ratio test $A^* = \{\pi_1 \bar{p}_1 > \pi_0 \bar{p}_0\}$, which selects the minimum everywhere (the pointwise argument is identical to Step~4 of the proof of Theorem~\ref{thm:tvbound}). Under A1 the expected total regret is $T\,r\, \mathrm{err}_\pi(C)$, so the optimal internal-only benchmark and the perfect-\textsc{verify} threshold are
\begin{equation}
R^*_{\mathrm{int}} = c^* = T\,r \int \min\bigl(\pi_0\, d\bar{P}_0^t,\, \pi_1\, d\bar{P}_1^t\bigr),
\label{eq:static:unequalthresh}
\end{equation}
still a closed form, but one in which the prior sits inside the minimum. The equal label prior is exactly the case in which it cancels: substituting $\pi_0 = \pi_1 = \tfrac12$ and the Scheff\'e identity gives $\tfrac12 \int \min(\bar{p}_0, \bar{p}_1)\, d\mu = (1-\TV)/2$, the shape of (a) and (b). That the shape, not merely the value, is genuinely lost is shown by the cyclic three-environment instance on the transcript space $\{0,1,2\}$ with
\begin{equation}
p_{E_0} = \bigl(\tfrac12, \tfrac12, 0\bigr),\quad
p_{E_1} = \bigl(0, \tfrac12, \tfrac12\bigr),\quad
p_{E_2} = \bigl(\tfrac12, 0, \tfrac12\bigr),
\label{eq:static:cyclic}
\end{equation}
equal environment prior $q = (\tfrac13, \tfrac13, \tfrac13)$, and labels $\Theta = (0,0,1)$, so $\pi = (\tfrac23, \tfrac13)$. The mixtures are $\bar{P}_0 = (\tfrac14, \tfrac12, \tfrac14)$ and $\bar{P}_1 = (\tfrac12, 0, \tfrac12)$, with $\TV(\bar{P}_0, \bar{P}_1) = \tfrac12$, so the equal-prior shape would prescribe $(1-\TV)/2 = \tfrac14$; but the weighted pointwise minimum is $\bigl(\min(\tfrac23 \cdot \tfrac14, \tfrac13 \cdot \tfrac12),\, \min(\tfrac23 \cdot \tfrac12, \tfrac13 \cdot 0),\, \min(\tfrac23 \cdot \tfrac14, \tfrac13 \cdot \tfrac12)\bigr) = (\tfrac16, 0, \tfrac16)$, integrating to $\tfrac13 \neq \tfrac14$. The weighted likelihood-ratio test attains $\tfrac13$ (it declares $\Theta = 0$ at $x = 1$ and breaks the ties at $x \in \{0,2\}$ arbitrarily). Hence no expression of the shape $(1-\TV)/2$ describes the optimum under this non-equal label prior, while \eqref{eq:static:unequalthresh} does: the equal label prior is indispensable to the \emph{shape} of the closed forms in (a) and (b).

\subsection{Environment identity versus label identity}
\label{app:static:identity}

\begin{remark*}[Environment identity versus label identity]
Identifying \emph{which} environment obtains and identifying the adequacy \emph{label} are different testing problems already at $|\cE| = 3$. On the cyclic class \eqref{eq:static:cyclic} with equal environment prior, the multi-class Bayes error is $1 - \tfrac13 \int \max\{p_0, p_1, p_2\}\, d\mu$ (at each $x$ the indicator sum collapses to the single term $p_{\delta(x)}(x) \le \max_j p_j(x)$, with equality for any pointwise-maximum rule); the pointwise maximum is $\tfrac12$ at all three transcript points, so environment identification has Bayes error $\tfrac12$, while every pairwise total variation equals $\tfrac12$ and any two-point formula predicts $\tfrac14$. The label test instead aggregates the class into the two mixtures $\bar{P}_0^t, \bar{P}_1^t$ and retains the binary closed form at every finite $|\cE|$: on the same class with prior $q = (\tfrac14, \tfrac14, \tfrac12)$ inducing equal label priors, Theorem~\ref{thm:tvbound} prescribes $(1-\TV)/2 = \tfrac14$, attained by the likelihood-ratio rule. The two problems return genuinely different answers ($\tfrac12$ versus $\tfrac14$) on the same class, which is why the framework fixes its test object: adequacy label, not environment identity. Unlike Fano-type multi-class identification, where the error depends on the full joint structure of $K$ hypotheses \citep{fano1961}, label aggregation restores the two-density situation in which a single scalar---the mixture TV---is exactly informative.
\end{remark*}

\vskip 0.2in
% Slightly tighter bibliography spacing (natbib parameter; does not
% alter the jmlr2e layout) so the reference list ends on the last
% full page.
\setlength{\bibsep}{1pt plus 1pt minus 1pt}
% jmlr2e already sets \bibliographystyle{plainnat}; do not repeat it
% (a second \bibstyle in the .aux is a BibTeX error).
\bibliography{jmlr}

\end{document}